\documentclass{report}
\usepackage{setspace}
\usepackage[utf8]{inputenc}
\usepackage{graphicx}
\graphicspath{{image/}}
\usepackage{caption}
\usepackage{subcaption}
\usepackage{tgbonum}

\usepackage[paper = a4paper,
            inner = 2.54cm,
            outer = 2.54cm,
            bindingoffset = 6mm,
            top = 1.905cm]{geometry}
\usepackage{sectsty}
\usepackage{titlesec}
\titleformat{\section}
  {\fontsize{12}{15}\bfseries}{\thesection}{1em}{}

 \usepackage{lmodern}

\usepackage[x11names]{xcolor}

\titleformat{\chapter}[display]
  {\sffamily\fontsize{12}{15}\bfseries\color{black}}
  {\chaptertitlename\fontsize{12}{15} \thechapter}{1em}{}

\subsectionfont{\fontsize{12}{15}\selectfont}

\makeatletter
\renewcommand\chapter{\if@openright\cleardoublepage\else\clearpage\fi
                    \thispagestyle{fancy}%
                    \global\@topnum\z@
                    \@afterindentfalse
                    \secdef\@chapter\@schapter}
\makeatother

\usepackage{fancyhdr}
\usepackage[round]{natbib}

\usepackage{amsmath} 
\usepackage{amssymb} 

\RequirePackage{amsthm,amsmath, amsfonts, amssymb}

\usepackage[]{algorithm2e}

\usepackage{nomencl}
\makenomenclature

\newtheorem{thm}{Theorem}[section]
\newtheorem{lemma}{Lemma}[section]
\newtheorem{proposition}{Proposition}[section]
\newtheorem{corollary}{Corollary}[section]

\theoremstyle{definition}
\newtheorem{definition}{Definition}[section]
\begin{document}

\begin{titlepage}
    \begin{center}
    \vspace*{1cm}

    \text{Use Of Vapnik-Chervonenkis Dimension in Model Selection}

    \vspace{1cm}

    \text{ by}

    \vspace{1cm}

    \text{ Merlin T. Mpoudeu}

    \vspace{1.5cm}
    A DISSERTATION

    \vspace{1.5cm}

    Presented to the Faculty of \\

The Graduate College at the University of Nebraska \\

In Partial Fulfillment of Requirements \\

For the Degree of Doctor of Philosophy

\vspace{0.5cm}

 \text{Major: Statistics }

 \vspace{1.5cm}

 Under the Supervision of Professor \text{Bertrand Clarke}
 \vspace{0.5cm}

 Lincoln, Nebraska

  \vspace{0.5cm}

  August, 2017

    \end{center}

\end{titlepage}

\thispagestyle{plain}

\begin{titlepage}
\begin{center}
    \vspace*{1cm}

    \text{Use Of Vapnik-Chervonenkis Dimension in Model Selection}

    {Merlin T. Mpoudeu, Ph.D. \\ University of Nebraska, 2017}

    \end{center}
\end{titlepage}

\thispagestyle{plain}

\begin{titlepage}
\chapter*{Acknowledgements}
\thispagestyle{empty}
I would especially like to thank my advisor Bertrand Clarke.
Without his insight, encouragement, probing questions, excitement, and suggestions,
this thesis would have ended up much differently. I am very proud of this
work, and I owe him a tremendous amount for helping it to become
something valuable. I would also like to thank Professor Kent M. Eskridge for the help that he provided through his Friday consulting group.

I would like to thank Dr. Jennifer Clarke, Dr. Yumou Qiu and Dr Stephen Scott for being my
graduate committee members and reviewing my dissertation. Their help and suggestions have been invaluable on both academic and personal levels.

I am grateful to my friends (Cyrile Nzouda, Salfo Bikienda, and Kismiantini); you have been with me since the beginning, and your support through the last four years has been invaluable.

I would like to express my gratitude to my parents, Mbialeu Charles and Siewe Philomene, and my brother and sisters for their encouragement throughout this graduate school process.
I would also like to thank two and most import people in my life: My fiancee Anne Mbenya and my 20 month old son Charles Mpoudeu for their support and patience.

\end{titlepage}


\pagenumbering{roman}
\tableofcontents
\listoffigures
\listoftables
\printnomenclature[2cm]

\nomenclature[A]{\textbf{VCD}}{Vapnik-Chervonenkis Dimension}
\nomenclature[A]{\textbf{RHS}}{Right-Hand Side}
\nomenclature[A]{\textbf{LHS}}{Left-Hand Side}
\nomenclature[A]{\textbf{EMDBTEL}}{Expected Maximum Difference Between Two Empirical Losses}
\nomenclature[A]{\textbf{ERM}}{Empirical Risk Minimization}
\nomenclature[A]{\textbf{BIC}}{Bayesian Information Criterion}
\nomenclature[A]{\textbf{SCAD}}{Smoothly Clipped Absolute Deviation}
\nomenclature[A]{\textbf{ALASSO}}{Adaptive Least Angle Shrinkage Selector Operator}
\nomenclature[A]{\textbf{DG}}{Data Generator}

\setcounter{page}{1}

\pagenumbering{arabic}


\chapter*{INTRODUCTION}
\label{Intro}
\thispagestyle{fancy}
Given a dataset of size $n$, there are $2^{n}$ different dichotomies of the observations.
If we have a collection of functions $\mathcal{F}$ so that given any dichotomy of the dataset, there is a function $f \in \mathcal{F}$ that is consistent with the dichotomy then we say that $\mathcal{F}$ generates the $2^{n}$ datasets.
 The Vapnik-Chervonenkis Dimension (VCD) is the maximum value of data points all whose dichotomies can be generated by $\mathcal{F}$. This measures the richness of $\mathcal{F}$.

The earliest usage of what we now call the VCD seems
to be in \cite{Vapnik:Chervonenkis:1968}.   A translation into English was published as \cite{Vapnik:Chervonenkis:1971}.  The VCD was initially called an
index of a collection of sets with respect to a sample and was developed
to provide sufficient conditions for the uniform convergence of empirical distributions.
It was extended to provide a sense of dimension for function spaces, particularly for
functions that represented classifiers.  After two decades of development this
formed the foundation for the field of Statistical Learning Theory, summarized
in \cite{Vapnik:1999} and extensively treated in  \cite{Vapnik:1998}.
An intuitive treatment of VCD from the classification perspective
can be found in \cite{Moore:2008}.
Aside from the development in statistical learning, VCD plays an important role in
various real analysis settings.  For instance, VCD defines VC classes of functions that have
relatively small covering numbers, see \cite{Pollard:1984}, \cite{Devroye:etal:1996},
\cite{vanderVaart:1998}, and \cite{vandeGeer:2000} amongst others.
At the risk of oversimplification, the VCD of a space of functions is currently seen mostly
as a parameter characterizing bounds of cumulative risks in term of the size of the function class.

One of our goals is to estimate the true value $h$ of the VCD (noted $h_{T}$) of a data generator (DG) by, say,
$\hat{h}$ and use $\hat{h}$ as the estimated complexity of the model class; e.g, linear models.
Given $\hat{h}$, the next step is to perform model selection. In fact, we use the notion of `consistency at the true model' to do so. We define `consistency at the true model' to be
\begin{enumerate}
    \item consistency at the true model in the usual sense of the term, basically we expect that $\hat{h}\rightarrow h_{T}$ in probability as the sample size increases, coupled with
    \item  some kind of identifiable bad behaviour away from the true model, getting worse as the wrong model moves further from the true model.
\end{enumerate}
If we have a list of non-nested models with nested VCD's denoted by $h$, and we would like to perform model selection, we can estimate $h$ by setting

\begin{equation}
\label{ConsisNonNestedIntro}
\hat{h} = \arg\min_{k} \left|VCD\left(P_{k}(\cdot\mid\beta)\right) - \hat{h}_{k}\right|
\end{equation}
where $\left\{P_{k}\left(\cdot\mid\beta\right)|k = 1,2,\cdots, K\right\}$ is some set of models and $\hat{h}_{k}$ is calculated using model $k$. Writing  $VCD\left(P_{k}(\cdot\mid\beta)\right) = h_{k}$, we see that $\hat{h}_{k}$ is a function of the VCD of the conjectured model and $\arg\min_{k}$ in \eqref{ConsisNonNestedIntro} removes this dependence. In the special case that $P_{k}(\cdot\mid\beta)$'s are linear i.e. $P_{k}(\cdot\mid\beta) = \sum_{j=0}^{k}\beta_{j}x_{j}$, we will see that $h_{T} = k+1$ (here, $x_{0} \equiv 1$). In practice we an approximate form of \eqref{ConsisNonNestedIntro} by thresholding the difference within the $\arg\min$; see \eqref{treshold} in Subsec. \ref{Changesimulation}.

We can also order the inclusion of the covariates using a shrinkage method such as SCAD (see \cite{Fan&Li})
or the correlation (see \cite{fan2008sure}) and the use $\hat{h}$ to pick a model.
We also upper bounded the true but unknown risk by an expression that depends on $\hat{h}$. We used that expression for model selection. In fact, for each model on the list, we estimate a VCD for it and use the estimate in an upper bound of the unknown risk. Then we pick the model with the smallest upper bound on its risk. Asymptotically, the chosen model should perform well with higher probability. This technique is called Empirical Risk Minimization and is suitable for non-nested model lists.

While VCD is not in general equivalent to the real dimension, in the special case of linear models VCD and the real dimension are the same.  Thus, at a minimum, with linear models, we can cut down the plausible models to those that have at least $\hat{h}$ parameters -- and hopefully not too many more.

The estimation of $h_{T}$ is done in three steps: first we find the Expected Maximum Difference Between Two Empirical Losses (EMDBTEL). This value depends on the model class (the class of linear models for this case), the design points, the number of bootstrap samples we draw from our data, and the number of intervals used to discretrize the loss function. In fact, if the number of the design points is $10$, we will have $10$ different values of the EMDBTEL. Since, the bound on the EMDBTEL is not tight, the next step is to find the best factor on the RHS of \eqref{objfn} i.e. the value of  $c$ that will minimize the upper bound. The final step is the estimation of $h_{T}$. This is done by minimizing the squared distance between the EMDBTEL (the LHS of \eqref{objfn}) and the upper bound (the RHS of \eqref{objfn}).
The intuition is that this minimum
corresponds to the decision problem, or more precisely the model class, that has the fastest
convergence to zero of its expected supremal difference of cumulative empirical losses and hence
would be the `right' problem for us to solve if only we knew it.
Therefore, even though the model class is only implicitly defined, we can take its $h_T$ or
empirically, $\hat{h}$, as a lower bound on the VCD of model classes. in the examples presented here
our objective function has a unique minimum that can be taken as an estimate of $\hat{h}$. In practice, for finite sample sizes this is best regarded as a lower bound on $h_T$. Indeed, in practice, we can estimate several values of $\hat{h}$, take their average and standard deviation
as a way to get what is effectively a bootstrap confidence interval for $h_T$.

Our contribution rests on two pioneering earlier contributions.  First using zero-one loss functions,
Vapnik and collaborators \cite{Vapnik:etal:1994} developed the asymptotic expected supremal difference of the cumulative empirical losses in classification problems. Thus, it focused more on probabilities than expectations.
The asymptotic expression they obtained had three terms combined in an ad hoc way. We did not use the asymptotic form that they obtained to estimate $\hat{h}$. Nevertheless, our derivation of an asymptotic upper bound for the
expected supremal difference of empirical cumulative losses is based on their work. However, our derivation is more complex e.g., to allow for continuous loss functions, and we use the cross-validation form of the error. McDonald \cite{McDonald:2012} Chaps. \ref{chap:bounds} and \ref{chap:Numerical:Studies} provides a lucid discussion on this. Secondly, McDonald et al. \cite{McDonald:etal:2011} examine the technique from Vapnik et al.\cite{Vapnik:etal:1994} and tried to establish consistency of $\hat{h}$ for $h_{T}$, again in the classification context. The key result in their proof rested on a result from \cite{vandeGeer:2000}. However, this result does not correctly fill the gap in their proof. Our simulations with the objective function in \cite{Vapnik:etal:1994} leads us to believe that consistency will be very hard to prove --  if in fact it is true.

Our main contributions in this dissertation are:
 \begin{itemize}
        \item We re-derive the Vapnik et al. (1994) bounds for the regression case (not classification).
        \item We correct errors in the earlier derivation so we obtain a single objective function to estimate $h$.
        \item We improve the objective function by using a cross validation form of the errors.
        \item We improve the objective function by optimizing over a constant that appears in it.
        \item We correct the original Vapnik et al. (1994) algorithm for our case.
        \item We verify that our $\hat{h}$ works well in simulation and for three datasets.
    \end{itemize}

The motivation of this dissertation comes from the fact that in statistics, there are many different model selection techniques that have been developed. These include best subset selection, forward and backward model selection, Bayesian Information Criterion (BIC), Smoothly Clipped Absolute Deviation, Adaptive Least Angle Shrinkage Selector Operator, etc. However, none of these techniques use the complexity of the data generator (DG).

The remainder of this dissertation is structured as  follows.
In chapter \ref{chap:Vapnik:Cherv:Dim:Cov:Num}, we discuss the notion of complexity in the VCD sense in general. We give two definitions of the notion of VCD. We also prove that for the case of linear regression, VCD is just the number of parameters in the saturated model. We also present  summaries of \cite{Vapnik:etal:1994} and \cite{McDonald:etal:2011}. In Chapter \ref{chap:bounds} we present the main theoretical work we have accomplished to date. It rests on discretizing bounded loss functions
so that upper bounds for the distinct regions of  the expected supremal difference of empirical losses can be derived in Sec. \ref{sec:Ext:Vap:Bounds}.
In Chapter \ref{chap:Numerical:Studies} we implement the theory developed in Chapter \ref{chap:bounds} on synthetic datasets. In Chapters \ref{example} and \ref{Wheatdata} we implement the theory developed in Chapter \ref{chap:bounds} on three real datasets namely the Tour De France, the Abalone and the Wheat data.
In Chapter \ref{futurework}, we conclude our work, and we present ideas for future work, in particular what we can do to improve our estimator for $h_{T}$.


\chapter{VAPNIK-CHERVONENKIS DIMENSION, COVERING NUMBER AND COMPLEXITY}
\label{chap:Vapnik:Cherv:Dim:Cov:Num}
\thispagestyle{fancy}

There are a variety of intuitive concepts in Statistics that have been formalized in ways that make good sense to those who are familiar with them but are often impenetrable to those who are not. As example, we have information. A series of incisive articles including \cite{Soofi:1994},
 \cite{Ebrahimi:etal:1999} and \cite{Ebrahimi:etal:2010}  develop this concept admitting its `intangibility'.  Another example is model uncertainty. This topic has been addressed in many settings but remains relatively hard to quantify see \cite{Draper:1995} (and discussion) for one of the earliest contributions, \cite{Clyde:George:2004} for a more recent contribution from the Bayes perspective, and \cite{Berk:etal:2013} for a treatment of the effect of model uncertainty on inference. A third example, and the subject of this review, is complexity.

The aim of this chapter is to give some background on the notion of complexity. In fact, there are many ways to measure the complexity or the richness of a class of functions. In Statistical Learning Theory, the indexes to measure the richness of a collection of function are VCD, ND Covering Number, in Information Theory the notion of code length is used and this is an extension of Shannon Entropy and mutual information. The remainder of this chapter will be developed as follows: In Sec. \ref{VCdim}, we will give two definitions of the notion of VCD. In Sec. \ref{VCdimStat}, we will show that VCD for the class of linear models is just the number of parameters. In Sec. \ref{CoverNumberEntropy}, we will talk about Covering Number, Entropy, Growth Function and how they are related.

\section{Definition of VCD}
\label{VCdim}

The aim of the VCD is to measure the richness or capacity, the complexity of a collection of functions.
There are at least three distinct ways to approach the definition of VCD:
geometric, combinatorial, and via covering numbers.
We will begin with the geometric definition and then proceed to the combinatorial definition.  For the sake of simplicity, we only present here the first two definitions of VCD.

\subsection{Geometric definition of VC dimension}
\label{geodef}

Let $\cal{F}$ be a collection of classifiers on some domain $\mathcal{X}$. The domain of $\mathcal{X}$ is usually a finite $d$ dimensional space. Let $k\geq 1$ be an integer, and let $\mathcal{X}_{k}= \left\{x_{1},\cdots, x_{k}\right\} \subseteq \mathcal{X}$. A dichotomy of $\mathcal{X}_{k}$ is a partition of $\mathcal{X}_{k}$ into two disjoint sets. We say that $\mathcal{X}_{k}$ is identified by $\mathcal{F}$ if and only if for any dichotomy of $\mathcal{X}_{k}$, there exists a function $f\in\mathcal{F}$ consistent with the dichotomy. That is let $U$ be any dichotomy of $\mathcal{X}_{k}$, there exists $f\in\mathcal{F}$ such that $\forall u \in U$ $f(u) =1$ and $\forall u \in U^{c}\mathcal{X}_{k}$, $f(u) = 0$. Thus when we have a class of subsets of $\mathcal{X}$, we will say that the class of subset is generated by $\mathcal{F}$ if each element of the collection can be identified by $\mathcal{F}$. Succinctly, the \textsc{VCD} of $\mathcal{F}$ is the largest $k$ for which every single element of the power set of $\mathcal{X}_{k}$ ($P\left(\mathcal{X}_{k}\right)$) can be identified by $\mathcal{F}$ in other words, $P\left(\mathcal{X}_{k}\right)$ is generated by $\mathcal{F}$. An infinite VC dimension means that $\mathcal{F}$ maintains full richness for all sample sizes.

Before giving some examples on how to calculate the VCD using the geometric definition, we will first emphasize some aspects of VCD. Note that, in the definition of VCD we did not say that for all sets of size $k$, their power set should be generated by $\mathcal{F}$. It's enough to show that there exist a set of size $k$ such that its power set can be generated by $\mathcal{F}$, then we say that the VCD is at least $k$. So a simple way to prove that the VCD of a collection of function is $k$ can be done in two steps. First prove that the collection of functions can generate any set of size less or equal to $k$. Second prove that for any set of size $k+1$, we cannot identify a dichotomy for that set.
The following examples are toy examples that will serve for illustrative purposes.

Let $\mathcal{X}=\mathcal{R}$ and let $\mathcal{F}$ be a collection of indicator functions of the intervals in the real line of the form $\left[a,\infty\right)$ for some $a\in \mathcal{R}$. That's
\begin{equation*}
  f(x)=\chi_{\left[a,\infty\right)}(x)=\left\{
                                         \begin{array}{ll}
                                           1, & \hbox{if $x\in\left[a,\infty\right)$;} \\
                                           0, & \hbox{otherwise.}
                                         \end{array}
                                       \right.
\end{equation*}

The VCD is at least $1$. Let's select a sample of one point $\left(k = 1\right)$ and position it on the real line. We want to show that there are $a$'s which generate $\left\{\left\{\emptyset\right\},\left\{x\right\}\right\}$. By taking $a = x-1$, $f(x)=\chi_{\left[x-1,\infty\right)}(x)$ identifies $\left\{x\right\}$ and $f(x)= \chi_{\left[x+1,\infty\right)}(x)$ identifies $\emptyset$. Now the question becomes: is there any set of size $2$ that can be generated by $\mathcal{F}$? Choose two points $\mathcal{X}_{2}=\left\{x_{1}, x_{2}\right\}$, we have $P\left(\mathcal{X}_{2}\right)=\left\{\left\{\emptyset\right\},\left\{x_{1}\right\},\left\{x_{2}\right\},\left\{x_{1},x_{2}\right\}\right\}$. Let's assume without loss of generality that $x_{1}<x_{2}$. In fact, we cannot find any $a$ that identifies $x_{1}$ alone thus; the VCD of this collection of functions is $1$.

Let $\mathcal{X}=\mathcal{R}$ and let $\mathcal{F}$ be a collection of indicator functions of the intervals in the real line of the form $\left(a,b\right]$ for some $a,b\in \mathcal{R}^{2}$. That's

\begin{equation*}
  f(x)=\chi_{\left[a,b\right)}(x)=\left\{
                                         \begin{array}{ll}
                                           1, & \hbox{if $x\in\left(a,b\right]$;} \\
                                           0, & \hbox{otherwise.}
                                         \end{array}
                                       \right.
\end{equation*}

From the previous case, we can say that the VCD is at least $2$. Let's choose two points $\mathcal{X}_{2}=\left\{x_{1}, x_{2}\right\}$, we have $P\left(\mathcal{X}_{2}\right)=\left\{\left\{\emptyset\right\},\left\{x_{1}\right\},\left\{x_{2}\right\},\left\{x_{1},x_{2}\right\}\right\}$. Let's assume without loss of generality that $x_{1}<x_{2}$. We need to show that there are values of $a, b$ which identify each element of $P\left(\mathcal{X}_{2}\right)$. The following functions $f(x)= \chi_{\left(x_{2},x_{2}+1\right]}(x)$, $f(x)=\chi_{\left(x_{1}-1,\frac{x_{1}+x_{2}}{2}\right]}(x), f(x) = \chi_{\left(\frac{x_{1}+x_{2}}{2},x_{2}+1\right]}(x)$, $f(x)= \chi_{\left(x_{1}-1,x_{2}+1\right]}(x)$ respectively identify $\left\{\left\{\emptyset\right\},\left\{x_{1}\right\},\left\{x_{2}\right\},\left\{x_{1},x_{2}\right\}\right\}$. Now to show that the VCD is at most two, we need to show that any set of three points cannot be generated by $\mathcal{F}$. It is sufficient to show that one of the dichotomies is not identifiable. Let $\mathcal{X}_{3}= \left\{x_{1},x_{2},x_{3}\right\}$ and assume without loss of generality that $x_{1}<x_{2}<x_{3}$. There is no $\chi_{\left(a,b\right]}$ that identifies $\left\{x_{1},x_{3}\right\}$. Since the choice of points was arbitrary, we say that the VCD is $2$.

Another interesting example that is  more related to Statistics is the case of half space in the plane. Let $\mathcal{X}= \mathcal{R}^{2}$ and $\mathcal{F}$ is the collection of functions (half plane) of the form $f(x) = ax + b$ where $a,b \in \mathcal{R}^{2}$. For this case, the VCD is at least 3 because any 3 non co-linear points in $\mathcal{R}^{2}$ can be generated by $\mathcal{F}$. In fact, the power set will have 8 different sets (dichotomies) where four of them are the mirror of the other four. By placing one point on one side of the line (half plane) and the other two on the other side of the line we can identify 3 dichotomies and by placing all three points on one side of the plane we identify the fourth dichotomy. The other four are just a mirror image of what we just did. So we can identify them just by symmetric. Now we just have to prove that there is not any set of four points that can be generated by our collection of functions. Here, there are two cases to be considered. Firstly, let's assume without loss of generality that that the four points lie in the convex hull generated by them. In this case, the labelling which assigns a positive sign to the pair on one diagonal and a negative sign on the other diagonal cannot be identified by a straight line. Secondly, let's assume that three of the points form a convex hull and the fourth point is an interior point. The dichotomy where the points on the convex hull are all positives and the interior point is negative cannot be identified by a half space. So the VCD is 3.

VCD can be extended to ${\cal{F}}$'s containing functions that assume values
in $\mathbb{R}$ by looking at level sets.   In fact, we will define level sets as the collection of points in the domain of $f$ for which $f(x)\geq\beta$ where $\beta\in\cal{R}$. Let
\begin{eqnarray}
{\cal{F}}^* = \{ \chi_{\{ z \mid f(z) \geq \beta \} } (\cdot ) \mid f \in {\cal{F}}, \beta \in \mathcal{R} \}
\label{generalreal}
\end{eqnarray}
be the collection of indicator functions generated by ${\cal{F}}$.  Now, one can set
${\sf VCdim}({\cal{F}})={\sf VCdim}({\cal{F}}^*)$.

As an example of this consider the densities
${\cal{F}} = \{ \phi_\sigma(\cdot) \mid \sigma \in \mathbb{R}  \}$ where
$\phi_\sigma (\cdot)$ is the density of a $N(0, \sigma^2)$ distribution.  The collection of points for which $\phi_\sigma (x)\geq \beta$
i.e., the level set
is symmetric intervals around zero, i.e., ${\cal{F}}^* = \{ \chi_{[-a, a]} \mid a \in \mathbb{R}^+ \}$.
By reasoning similar to that above, ${\cal{F}}^*$ has VC dimension one.  (There is no
set of two distinct points on the real line that has a ${\cal{P}}({\cal{X}}_2)$ that can
be generated by the sets in ${\cal{F}}^*$ but any ${\cal{P}}({\cal{X}}_1)$ can be generated.)
Extending this, let
${\cal{F}} = \{ \phi_{\mu, \sigma}(\cdot) \mid \mu \in \mathbb{R}, \sigma \in \mathbb{R}  \}$ where
$\phi_{\mu, \sigma} (\cdot)$ is the density of a $N(\mu, \sigma^2)$ distribution.
Now, ${\cal{F}}^*$ consists of all closed intervals of the form $[a, b]$ in $\mathbb{R}$.
In this case, ${\sf VCD}({\cal{F}}) = 2$.

Now consider ${\cal{F}} = \{ \phi_{\mu, \Sigma}(\cdot) \mid \mu \in \mathbb{R}^2, \Sigma = \sigma^2 I_2, \sigma \in \mathbb{R}^+  \}$, densities of two independent normal with different means.  Then
${\cal{F}}^*$ consists of all circles in $\mathbb{R}^2$.  In this case, the circle will classify everything inside the circle as positive and everything outside aas negative. It is easy to see that any set of three points in the plane that forms a non-degenerate triangle gives a ${\cal{P}}({\cal{X}}_3)$ that can be generated
by ${\cal{F}}^*$.  However, no set of four points gives a ${\cal{P}}({\cal{X}}_4)$
that can be generated by ${\cal{F}}^*$.  Hence,  ${\sf VCD}({\cal{F}}) = 3$.

Analogously, we can see that if the variances are not equal,  i.e.,
${\cal{F}} = \{ \phi_{\mu, \Sigma}(\cdot) \mid \mu \in \mathbb{R}^2, \Sigma = {\sf diag}(\sigma_1^2, \sigma_2^2),  \sigma_1, \sigma_2 \in \mathbb{R}^2  \}$
consists of all ellipses, then it is possible to produce five distinct points whose
 ${\cal{P}}({\cal{X}}_5)$ can be generated by ${\cal{F}}$.
However, there is no set of six points in the plane giving a ${\cal{P}}({\cal{X}}_6)$ that can be
generated by ${\cal{F}}$.   So, ${\sf VCdim}({\cal{F}}) = 5$.

It is left to the reader to work out the VCD of the full set of normals in the plane,
i.e.,
${\cal{F}} = \{ \phi_{\mu, \Sigma}(\cdot) \mid ~\hbox{where} ~\mu \in \mathbb{R}^2,
 {\sf diag}(\Sigma) = (\sigma_1^2, \sigma_2^2),  \sigma_1, \sigma_2 \in \mathbb{R}^2 ~ \hbox{and}~
\sigma_{12} = \sigma_{21} \in \mathbb{R}  \}$.  (It is not hard to produce six points
in the plane whose ${\cal{P}}({\cal{X}}_6)$ can be generated from ${\cal{F}}$.
Arguing that no set of seven points leads to a ${\cal{P}}({\cal{X}}_7)$ that can be generated
by ${\cal{F}}$ is harder.)

Looking back to the examples studied so far, one might think that there is a one-to-one relationship between VCD of a collection of functions and the number of parameters in the model require to represent the class of functions. In fact, this observation is not always true. A one dimensional
family can have infinite VCD. Here is an example.

Let $\mathcal{X}=\mathcal{N}$, $\mathcal{F}=\left\{f_{\alpha}(x): \quad \hbox{where $\alpha\in\mathbb{R}$}\right\}$
$$
f_{\alpha}(x)=\left\{
                \begin{array}{ll}
                  1, & \hbox{if the $x^{th}$ bit in the binary representation of $\alpha$ is 1;} \\
                  0, & \hbox{if the $x^{th}$ bit in the binary representation of $\alpha$ is 0.}
                \end{array}
              \right.
$$

For any number $x$ that you give, I will always find a number $(\alpha)$ such that the $x^{th}$ bit in the binary representation is 1, for instance $x=5$, in this case, $\alpha = 125$ and the binary representation of $\alpha = 125 = 1111101$. Thus every $\mathcal{S}\subseteq\mathcal{X}$ can be generated by $\mathcal{F}$, therefore ${\sf VCdim(F)}=\infty$

\subsection{Combinatorial definition of VCD}
\label{combinatorial}

In Subsec. \ref{geodef} the VCD was associated with a collection of functions ${\cal{F}}$.
In this section, we will look at the VCD of a collection of measurable sets.

 Let ${\mathcal{C}}$ be a collection of measurable subsets of a space and let ${\mathcal{X}} = \{ x_1,...,x_n \}$ be a set
of $n$ points in the same space.  Consider the collection of sets
$C \cap \{x_1,..., x_n \}$  $\forall C \in \cal{C}$. As before, we say a subset  $A$ of $\mathcal{X}$
can be identified by ${\mathcal{C}}$ if there exists $C \in \mathcal{C}$ such that  $A = C \cap \{x_1,..., x_n \}$, and ${\cal{P}}( {\cal{X}} )$
is generated by ${\cal{C}}$ if and only  if $\forall A \in \mathcal{P}(\mathcal{X})$, there exists $C \in \mathcal{C}$ such that $A = C\cap \left\{x_{1},\ldots, x_{n}\right\}$. In statistical learning jargon, they use the terminology `picked out' for identified and
shattered for generated.  As before, the VC dimension $k$
of the class ${\cal{C}}$ of sets is the smallest $n$ for which no
set of size $n+1$ can be generated by ${\cal{C}}$.

Formally, \cite{Vaart:Wellner:1996} defines
\begin{eqnarray}
\Delta  ({\cal{C}}, x_1,..., x_n) =\# \left\{
\{x_1,...,x_n \} \cap C: C \in {\cal{C}} \right\} ,
\label{VWdef}
\end{eqnarray}
so that the VCD is
\begin{eqnarray}
k = \inf \left\{ n : \max_{ \{x_1,...,x_n \} }
\Delta ({\cal{C}}, x_1,..., x_n) < 2^n \right\} .
\label{VWVCdef}
\end{eqnarray}
The infimum over an empty set can be defined to be zero, so the VCD of a set is only infinity when ${\mathcal{C}}$ can generate ${\cal{P}}({\cal{X}}_k)$ for
arbitrarily large $k$. Devroye et al. \cite{Devroye:etal:2013} Chap. 12.4 uses a slight generalization
of this by defining
$$
s({\cal{C}}, n) = \max_{x_1, \ldots , x_n} \Delta  ({\cal{C}}, x_1,..., x_n)
$$

so that $s({\cal{C}}, n) \leq 2^n$ noting that if there is an $n$ for which $s({\cal{C}}, n) < 2^n$
then for all $k \geq n$, $s({\cal{C}}, k) < 2^k$.  The VC dimension $k$  is then the last
time equality holds, that is, the largest $k$ for which $s({\cal{C}}, k) = 2^k$.

Checking whether a subset of $\{x_1,...,x_n \} $ can be identified (or
`picked out') by some $C \in {\cal{C}}$ is much like checking whether the indicator
function for a set gives the right pattern of zeros and ones -- ones for the
$x_i$'s in the set identified and zero on the other $x_i$'s.
Consider the equivalence relation on sets defined by saying that two sets
in a class of sets are equivalent if and only if they identify the same subsets
of $\{x_1,...,x_n \} $.   Then, the number of equivalence classes in ${\cal{C}}$
depends on $\{x_1,...,x_n \} $. For fixed $n$, the equivalence classes heuristically
are like a data driven sub-$\sigma$-field from the $\sigma$ field generated by ${\cal{C}}$.
Thus the combinatorial interpretation of VC dimension is an effort to formalize the
idea that finitely many data points only permits finitely many events to be distinguished.

Let ${\cal{F}}$ be a class of measurable functions on a sample space
and let $f \in {\cal{F}}$.
The subgraph of $f: {\cal{X}} \rightarrow \mathbb{R}$ is
\begin{eqnarray}
SG_f =
\{ (x, t): t < f(x) \} .
\label{subgraph}
\end{eqnarray}
Now, ${\cal{F}}$ is called a VC class if and only if the collection of sets
$\{ SG_f \mid f \in {\cal{F}} \}$ has finite VC dimension as defined by \ref{VWVCdef}.
When ${\cal{F}}$ is a collection of indicator functions, it is straightforward to see that
${\cal{F}}$ is a VC class with ${\sf VCdim}({\cal{F}}) =k$ in the sense of
\ref{VWVCdef} if and only if ${\cal{F}}$  has VC dimension $k$ in the geometric sense in Subsec. \ref{geodef}.  Simply note that the
subgraphs from a collection of indicator functions is a class of sets that can generate a particular
${\cal{P}}({\cal{X}}_k)$ in the geometric sense if and only if it satisfies \ref{VWVCdef},
i.e., generates all the sets in the class ${\cal{P}}({\cal{X}}_k)$ by using $\Delta$ to pick them out.
More generally, results in \cite{Vaart:Wellner:1996} Sec. 2.6.5 suggest that in general
the definition of a VC class of functions (in terms of \ref{VWVCdef}) is equivalent to
the geometric definition by using the indicator functions in \ref{generalreal}.

Part of the reason \ref{VWVCdef} is regarded as combinatorial is from the following intuition
and result.
By definition, a VC class of functions ${\cal{F}}$ with ${\sf VCdim}({\cal{F}})=k$ identifies
fewer than $2^h$ subsets from any ${\cal{P}}({\cal{X}}_h)$ when $h >k$.  It is intuitive that
as $h$ increases the smaller the fraction of the possible sets that ${\cal{F}}$ can identify.
That is, if $h=k+\ell$ then even though ${\cal{F}}$ could identify $2^{h-\ell}$ sets the actual number
is polynomial as $h$ increases not exponential with a smaller exponent, i.e., the dropoff in how many sets
can be identified as $h$ increases is fast.  This
surprising fact is called Sauer's Lemma, see \cite{Sauer:1972}.   Let ${\cal{C}}$ be a class of functions
and define the `shatter' function
\begin{eqnarray}
s(n) = \max_{ \{ x_1,...,x_n \} } \Delta_n
({\cal{C}}, x_1,...,x_n) .
\nonumber
\label{shatterfn}
\end{eqnarray}
The shatter function is the largest number of subsets that ${\cal{C}}$ can identify from an
$n$-tuple of points.  It is naturally compared with $2^n$, the cardinality of the
power set from $n$ points.  If a set of size $n$ is
generated by ${\cal{C}}$ then ${\sf VCdim} ({\cal{C}}) \geq n$ and
$s(n) = 2^n$. If there is no such set then ${\sf VCdim}({\cal{C}}) < n$
and $s(n) < 2^n$.  Sauer's Lemma gives a polynomial bound on $s(n)$.  In particular,
Sauer's Lemma gives that if a collection
of sets has VC dimension bounded by $k$ then any set of $n$ elements can only be split
$n^k$ ways.  The proof requires
recognizing $s(n)$ can be represented as a sum of binomial coefficients to which
an inductive argument can be applied

\begin{thm}
 If ${\sf VCdim}({\cal{C}}) = k$, then
$s(n) \leq \left( \frac{en}{k} \right)^k$.
\end{thm}

\begin{proof}
  See \cite{Vaart:Wellner:1996}.
\end{proof}

This sort of sudden qualitative change in $s$ will be seen in other quantities
such as the growth function, see Sec. \ref{CoverNumberEntropy}.

\section{VC dimension in statistics}
\label{VCdimStat}
In section \ref{VCdim} we gave two definitions of VC dimension and we went through some examples of VCD where we derive the VCD of various sets of functions. In this section, we will formally prove the VCD of some collection of functions. In fact,we will derive the VC dimension of the class of linear functions.

\subsection{VCD for Regression}
\label{VCdimReg}
The VCD is a quantity that is defined for a collection of functions. It is the largest number of points the collection of functions can separate in all possible ways or equivalently the largest $n$ for which the growth function is $2^{n}$. Note that we really mean the largest number of points the collection of functions can separate and all possible ways; this does not imply that every set of $n$ points can be separate in possible ways by our collection of functions. Our goal here is to prove the following theorem:
\begin{thm}
\label{VCDLinear}
 Let $ \mathcal{R}^{d} = \mathcal{X}$, $\mathcal{S}= \left\{x_{1},x_{2},\cdots,x_{n}\right\}\subseteq \mathcal{R}^{d}$ and let

 $$
 \mathcal{F}=\left\{\left(w,\beta\right): f(x)=sign(wx-\beta)\right\}
 $$
 where $w\in\mathcal{R}^{d}$ and $\beta\in\mathcal{R}$. Then a set $\mathcal{S}=\left\{x_{1},x_{2},\cdots,x_{n}\right\}\subseteq \mathcal{R}^{d}$ is generated by $\mathcal{F}$ if and only if the set $\left\{(x_{1}^{'}-1),(x_{2}^{'}-1),\cdots,(x_{n}^{'}-1)\right\}$ is linearly independent in $\mathcal{R}^{d+1}$. It follows that VCdim($\mathcal{F}$)=$d+1$
\end{thm}

\begin{proof}
  The proof of this Theorem will follow the same stream of reasoning as the proof of ${\sf Theorem:~3.4}$ in \cite{anthony2009neural}. We will prove this theorem in two steps. Let's assume that if $\mathcal{S}$ is generated by $\mathcal{F}$, then $\left\{(x_{1}^{'}-1),(x_{2}^{'}-1),\cdots,(x_{n}^{'}-1)\right\}$ is linearly independent in $\mathcal{R}^{d+1}$. We will prove this by contradiction, that's we assume that $\mathcal{S}$ is generated by $\mathcal{F}$ but $\left\{(x_{1}^{'}-1),(x_{2}^{'}-1),\cdots,(x_{n}^{'}-1)\right\}$ is linearly dependent in $\mathcal{R}^{d+1}$. Let $V=\left(v_{1},v_{2},\cdots,v_{2^{n}}\right) = \left(\begin{array}{cccc}
                                                                                               w_{1} & w_{2} & \cdots & w_{2^{n}} \\
                                                                                               \beta_{1} & \beta_{2} & \cdots & \beta_{2^{n}}
                                                                                             \end{array}
  \right)$ be the total number of partitions. $\mathcal{S}$ is generated by $\mathcal{F}$ means that $\forall ~ v_{i}\in V, \exists ~(w_{i},\beta_{i})\in\mathcal{F}$ such that
  $$
  \left(
    \begin{array}{cc}
      x_{1}^{'} & -1 \\
    x_{2}^{'} & -1 \\
    \vdots & \vdots \\
    x_{n}^{'} & -1
    \end{array}
  \right)\left(
           \begin{array}{c}
             w_{i} \\
             \beta_{i} \\
           \end{array}
         \right)=v_{i}.
  $$
  Since $\left\{(x_{1}^{'}-1),(x_{2}^{'}-1),\cdots,(x_{n}^{'}-1)\right\}$ are linearly dependent, one can expresses one element of this vector as the linear combination of the others, that is
  $$
  \left(
    \begin{array}{cc}
      x_{n}^{'} & -1 \\
    \end{array}
  \right)=\sum_{k=1}^{n-1}\alpha_{i}\left(
                                     \begin{array}{cc}
                                       x_{k}^{'} & -1 \\
                                     \end{array}
                                   \right), \quad \hbox{where at least one $\alpha_{k}\neq 0$}.
  $$
  Using this, the $n^{th}$ coordinate of any dichotomy $v_{i}$ can be expressed as follows
\begin{eqnarray*}
y_{ni} &=& \left(
                 \begin{array}{cc}
                   x_{n}^{'} & -1 \\
                 \end{array}
               \right)\left(
                        \begin{array}{c}
                          w \\
                          \beta \\
                        \end{array}
                      \right) = \left(
           \begin{array}{cc}
             \sum_{i=1}^{n-1}\alpha_{i}x_{i}^{'} & -\sum_{i=1}^{n-1}\alpha_{i} \\
           \end{array}
         \right)\left(
                  \begin{array}{c}
                    w \\
                    \beta \\
                  \end{array}
                \right)\\
       &=& w\sum_{i=1}^{n-1}\alpha_{i}x_{i}^{'}-\beta\sum_{i=1}^{n-1}\alpha_{i} = \sum_{i=1}^{n-1}\left(\alpha_{i}x_{i}^{'}w-\beta\alpha_{i}\right) \\
      &=& \sum_{i=1}^{n-1}\alpha_{i}\left(wx_{i}^{'}-\beta\right) =\sum_{j=1}^{n-1}\alpha_{j}v_{ji}, \quad \hbox{where $v_{ji}=w_{i}x_{j}^{'}-\beta_{i}$}
\end{eqnarray*}
  Let's assume that $x_{j}$ with $\alpha_{j}\neq 0$ and let $y_{ji}=sign(\alpha_{j})$ for all $1\leq j \leq n-1$. Let $x_{j}$ and let $y_{ji}=-1$, we will prove that this dichotomy is not identifiable. $y_{ji}=sign(w_{i}x_{j}-\beta_{i})=sign(\alpha_{j})\Rightarrow \alpha_{j}\left(w_{i}x_{j}-\beta_{i}\right)\geq0 \Rightarrow \sum_{j=1}^{n-1}\alpha_{j}(w_{i}x_{j}-\beta_{i})\geq0\Rightarrow y_{ji}\geq0$, which is a contradiction since $y_{ji}=-1$. Then it follows that $VCD(\mathcal{F})\leq d+1$.

  For the second part, we assume that $\left\{(x_{1}^{'}-1),(x_{2}^{'}-1),\cdots,(x_{n}^{'}-1)\right\}$ is linearly independent and prove that $\mathcal{S}$ can be generated by $\mathcal{F}$. Because of the linear independence assumption, the matrix
  $$
  \left(
    \begin{array}{cc}
      x_{1}^{'} & -1 \\
         x_{2}^{'} & -1 \\
         \vdots & \vdots \\
         x_{n}^{'} & -1 \\
    \end{array}
  \right)
  $$
  has row-rank $n$, therefore, for any $v\in V$, there is a unique solution $(w,\beta)$ to
  the following equation
  \begin{equation*}
    \left(
       \begin{array}{cc}
         x_{1}^{'} & -1 \\
         x_{2}^{'} & -1 \\
         \vdots & \vdots \\
         x_{n}^{'} & -1 \\
       \end{array}
     \right)\left(
              \begin{array}{c}
                w \\
                \beta \\
              \end{array}
            \right)=v.
  \end{equation*}
  Thus it follows that $\mathcal{S}$ can be generated by $\mathcal{F}$, this implies that ${\sf VCD(\mathcal{F})}\geq d+1$. These two results together imply that ${\sf VCD(\mathcal{F})}=d+1$.
\end{proof}

We comment that the VCD for the regression function in a recursive partitioning model, i.e., a tree,
can be readily worked out as a function of the number of splits in the tree.  If there are, say, $K$
splits, it is easy to verify that the VCD is $k+1$.  Analogous results for neural networks -- even single
hidden layer neural networks -- are not so easy to derive.  Nevertheless, in the 1990's there were
many efforts, chiefly by computer scientists, to evaluate the VCD of classes of neural networks particularly in a classification context.
While many of these results were helpful in general terms, it is not clear how they can be
used in our present context
of regression because often only bounds or asymptotic expressions for
the VCD in terms of the number of hidden nodes were obtained.

\section{Covering Numbers, Entropy, Growth Function and VC-dimension}
\label{CoverNumberEntropy}
The goal is to assign a class of functions ${\cal{F}}$ a measure
that generalizes the concept of real dimension.  This will necessitate defining
covering numbers, their finite sample analogs called random entropy
numbers, and then generalizations of entropy numbers called the
annealed entropy and the growth function. The VCD will
emerge from a key shape property of the growth function. This
section is largely a revamping of parts of Vapnik (\cite{Vapnik:1998}, Chap.3).

To start, recall the simplest measure
of the size of a set is its cardinality, but since ${\cal{F}}$ is
typically uncountable that's no help.  What is a help is the concept
of a covering number since it is also a measure of size.  Since
${\cal{F}}$ is a subset of a normed space with norm $\| \cdot \|$,
given an $\epsilon$ one can measure ${\cal{F}}$'s size by the
minimal number of balls of radius $\epsilon$ required to cover
${\cal{F}}$.  This is called the covering number and is written $N(\epsilon,
{\cal{F}}, \| \cdot\|)$.   The $\epsilon$-entropy of ${\cal{F}}$, or
the Boltzmann entropy,  is $\ln  N(\epsilon, {\cal{F}}, \|
\cdot\|)$, its natural log. The covering number of a class of indicator functions ${\cal{F}}$ is
related to the number of subsets that can be picked out from
$\{ {x}_1,..., {x}_n\}$ by the functions in ${\cal{F}}$.
Indeed, write ${\cal{F}} = \{ 1_C : C \in {\cal{C}} \}$ so that
${\cal{C}}$ is the class of sets whose indicator functions are in
${\cal{F}}$.  Now, a complicated argument based on Sauer's Lemma
gives a bound on the covering number for ${\cal{F}}$.  The result is
that there is a universal constant $K$ so that for $1 > \epsilon >
0$ and $r \geq 1$, if $h$ is the (geometric) VCD of
${\cal{C}}$ then
$$
N(\epsilon, {\cal{F}}, L_r(\mu) ) =   N(\epsilon, {\cal{C}},
L_r(\mu) ) \leq \frac{K h (4 e)^h}{\epsilon ^{r(h-1)}}
$$
where $L_r$ is the Lebesgue space with $r$-th power norm with
respect to $\mu$, see \cite{Vaart:Wellner:1996} (Chap. 2.6).

The more typical definition of the entropy, also called the Shannon
entropy, is minus the expected log of the density; this is the
expectation of  a sort of random entropy.  The log density plays a
role similar to the Boltzmann entropy because both represent a
codelength.

While the relationships among the $\epsilon$-entropy of ${\cal{F}}$,
the VCD of ${\cal{F}}$, and other quantities are deep, all
of these are population quantities.  So, given a finite
sample, it is important to have an analog of covering number. This
is provided by counting the number of sets a sample of size $n$ can
distinguish.  Define two events as distinguishable using a sample
$z_1$,..., $z_n$, if and only if there is at least one $z_i$ that
belongs to one event and not the other.  It is immediate that not
all sets in a large class ${\cal{C}}$ are distinguishable given a
sample; different samples can have different collections of
distinguishable sets; and the number of distinguishable sets can
depend on the sample as well.

The next task is to identify which sets it is important for a sample
to be able to distinguish.  For a set of indicator functions $Q(z,
\alpha)$ with support sets $C_\alpha$ and $\alpha \in \Lambda$,
consider the vector of length $n$
$$
q(\alpha) = (Q(z_1, \alpha), ..., Q(z_n, \alpha) ). 
$$
For fixed $\alpha$, $q(\alpha)$ is a vertex of the unit cube in
$\mathbb{R}^n$.  As $\alpha$ ranges over $\Lambda$, $q(\alpha)$ hops from
vertex to vertex.  Let $N(\Lambda, z_1,...,z_n)$ be the number of
vertices $q(\alpha)$ lands on. Then, $N(\Lambda, z_1,...,z_n)$ is
the number of sets the sample $z_1$,...,$z_n$ can distinguish.

Let the $Z_i$ be IID.  Then, $N(\Lambda, Z_1,...,Z_n)$ is a
random variable bounded by $2^n$ for given $n$.  The random entropy
of the set ${\cal{F}}$ of indicator functions $Q(\cdot, \alpha)$ for
$\alpha \in \Lambda$ is $\ln N(\Lambda, z_1,...,z_n)$ and the entropy
of ${\cal{F}}$ is
$$
H(\Lambda, n) = \int  \ln N( \Lambda, z_1,...,z_n) d F(z_1,...,z_n). 
$$

The entropy of real valued functions in general is similar to the
indicator function case, but uses the concept of a covering number
more delicately.  First consider a set of bounded real valued
functions ${\cal{F}}_C$ defined by
$$
\forall\alpha\in\Lambda ~~ | Q(z, \alpha)  | \leq C .
$$
Again $q(\alpha, z_1,...,z_n)$ can be used to generate a region
${\cal{R}}$  in the $n$ dimensional cube of side length $2C$.  The
random covering number of ${\cal{R}}$ for a given $\epsilon$ is
$N(\Lambda, \epsilon, z_1,...,z_n)$ giving $\ln N(\Lambda, \epsilon,
z_1,...,z_n)$ as the random $\epsilon$-entropy.  (To define the
covering number, the norm must be specified as well as the
$\epsilon$ and region to be covered.  In this case, the natural
choice is the metric $\rho(q(\alpha), q(\alpha^\prime) )=
\max_{i=1,...,n} | Q(z_i, \alpha ) - Q(z_i, \alpha^\prime)|$.) The
$\epsilon$-entropy (not random) for ${\cal{F}}_C$ is
$$
H(\Lambda, \epsilon, n) = \int H(\Lambda, \epsilon, z_1,...,z_n) d
F(z_1,...,z_n) . 
$$

Equipped with the entropy, it is easy to specify two closely related
quantities that will arise in the bounds to be established that
involve the VCD. First, the annealed entropy is
$$
H_{ann}(\Lambda, n) = \ln E N(\Lambda, Z_1,...,Z_n) 
$$
and, second, the growth function  is
$$
G(\Lambda, n) = \ln \hbox{sup}_{z_1,...,z_n} N(\Lambda,
z_1,...,z_n). 
$$
Clearly, $H(\Lambda, n) \leq H_{ann}(\Lambda, n) \leq G(\Lambda,
n)$, the result of Jensen's inequality on the entropy $H(\Lambda,
n)$ and taking the supremum of the integrand.

The growth function does not depend on the probability measure; it
is purely data driven.

The key theorem linking VCD to a collection of indicator
functions is the following.  It provides an interpretation for $h$
separate from its definition in terms of shattering.
Although $h$ arises in the argument by using the geometric
definition of shattering, in principle, there is no reason not to
use the result of this Theorem as a third definition for $h$. Let
$C(n,i)$ be the number of combinations of $i$ items from $n$ items.

\begin{thm}
 \cite{Vapnik:1998}The growth function $G(\Lambda , n)$
for a set of indicator functions $Q(z,\alpha)$, $\alpha \in \Lambda$
is linear-logarithmic in the sense that
$$
G(\Lambda,n)=
\begin{cases}
= n \ln 2 & \text{if $n \leq h$;}\\
\leq \ln (\sum_{i=0}^h C(n, i) \leq  h \ln (e n /h) = h(1 + \ln(n/h)) & \text{if
$n>h$.}
\end{cases}
$$

{\bf Comment:}   If $h = \infty$ then the second case in the bound
never applies so $G(\Lambda, n) =n \ln 2$. The essence of the result
is that if a set of functions has finite VCD, then its
growth function is initially linear, meaning that the amount of
learning per datum accumulated is proportional to the sample size,
up to the VCD, after which the learning per datum drops off
to a log rate.
\end{thm}

\section{Review of \cite{Vapnik:etal:1994} and \cite{McDonald:etal:2011}}
\label{ReviewVapetal}

In this section we will present the work of Vapnik et al \cite{Vapnik:etal:1994} that leads to the empirical estimation of VCD. We will also present the work of \cite{McDonald:etal:2011} who tried to argue that using the procedure of \cite{Vapnik:etal:1994}, the estimated $\hat{h}$ will concentrate around the true value.

Let's first set the stage and notation that Vapnik used in his 1994 paper.
$f(x,\alpha), \alpha\in \Lambda$ is the set of binary classifiers.
The $n$ sample pairs $Z = (x,y)$ where $X\in\mathcal{R}^{p}$, $Y \in \{0,1\} $ is
$
\left(Z_{1},Z_{1},\cdots,Z_{n}\right).
$
The probability of error between $y$ and $f(x,\alpha)$ is

$$
p(\alpha) = E\left|Y - f(X,\alpha)\right|.
$$
The expectation is taken with respect to the joint distribution between $x$, and $y$ $(f(x,y))$. Since $f(x,y)$ is unknown, we use the empirical error
$$
\nu(\alpha) = \frac{1}{n}\sum_{i=1}^{n}\left|y_{i} - f(x_{i},\alpha)\right|.
$$
In Vapnik \cite{Vapnik:1991}, it was shown that a function that minimizes the empirical risk will be consistent if and only if the following one sided uniform convergence holds
$$
\lim_{n \rightarrow \infty}P\left(\sup_{\alpha\in\Lambda}\left|p(\alpha)-\nu(\alpha)\right|\geq\epsilon\right)=0.
$$
In the 1930's Kolmogorov and Smirnov found the law of distribution of the maximum deviation between a distribution function and the empirical function for any random variable. This can be stated as follows: for the collection of functions $f(x,\alpha),~\hbox{where $\alpha\in\Lambda$}$ the equality
$$
P\left(\sup_{\alpha\in\Lambda}\left(p(\alpha)-\nu(\alpha)\right)\geq\epsilon\right) = \exp\left(-2\epsilon^{2}n-2\sum_{j=2}^{\infty}(-1)^{j}\exp\left(-2\epsilon^{2}j^{2}n\right)\right)
$$
holds for large $n$ where $p(\alpha)= Ef(x,\alpha)$ and $\nu(\alpha) = \frac{1}{n}\sum_{i=1}^{n}f(x_{i},\alpha)$.

In Vapnik \cite{Vapnik:1998}, it was shown that
$$
P\left(\sup_{\alpha\in\Lambda}\left(p(\alpha)-\nu(\alpha)\right)\geq\epsilon\right) \le \min \left(1,\exp\left\{\left(c_{1}\frac{\ln\left(\frac{2n}{h}\right)+1}{\frac{n}{h}}-c_{2}\epsilon^{2}\right)n\right\}\right)
$$
As before, this inequality does not depend on the probability distribution.
Using these results, Vapnik et al. in \cite{Vapnik:etal:1994} showed that the expected maximum difference between two empirical losses can be bounded by a function which depends only on the VCD of the classifier. For comparison purposes we present this argument next.

Let $Z^{1} = Z_{1},Z_{2},\cdots,Z_{n}$ and $Z^{2}=\left(Z_{n+1},Z_{n+2},\cdots,Z_{2n}\right)$ be two independent identically distributed copies of $Z$. Let $\nu_{1}(Z^{1},\alpha)$ and $\nu_{2}(Z^{2},\alpha)$ be the empirical risk using $Z^{1}$ and $Z^{2}$. Vapnik et al \cite{Vapnik:etal:1994}. claims the following bound on the expectation of the maximum difference between $~\nu_{1}(z^{1},\alpha)~\hbox{and}~\nu_{2}(z^{2},\alpha)$,

\begin{equation}\label{VapnikBounds}
  \xi(n) = E\left(\sup_{\alpha\in\Lambda}\left(\nu_{1}(z^{1},\alpha)-\nu_{2}(z^{2},\alpha)\right)\right)\leq\left\{
                                    \begin{array}{ll}
                                      1, & \hbox{if $\frac{n}{2}\leq 0.5$;} \\  c_{1}\frac{\ln\left(\frac{2n}{h}\right)+1}{\frac{n}{h}}, & \hbox{if $\frac{n}{h}$ is small;} \\ c_{2}\sqrt{\frac{\ln\left(\frac{2n}{h}\right)+1}{\frac{n}{h}}}, & \hbox{if $\frac{n}{h}$ is large.}
                                                                        \end{array}                                                                     \right.
\end{equation}
Equation \eqref{VapnikBounds} can be bounded by
\begin{equation}\label{phi}
  \Phi_{h}(n) = \left\{
                  \begin{array}{ll}
                    1, & \hbox{if $\frac{n}{h}\leq 0.5$;} \\
                    a\frac{\ln\left(\frac{2n}{h}\right)+1}{\frac{n}{h}-k}\left(\sqrt{1 + \frac{b\left(\frac{n}{h}-k\right)}{\ln\left(\frac{2n}{h}\right)+1}}+1\right), & \hbox{otherwise.}
                  \end{array}
                \right.
\end{equation}
The constants $a = 0.16$ and $b=1.2$ are claimed to be universal and given in \cite{Vapnik:etal:1994}. These constants can be used for estimation of VCD for other classes of functions; they represent the trade off error between small and large $n/h$.
The constant $k = 0.14927$ was chosen so that $\Phi_{h}(0.5) =1$, although it is unclear what this means. If the LHS of \eqref{VapnikBounds} were known, we would just solve for $h$ and get an estimate $\hat{h}$. Because we do not have that information, we estimate $\hat{\xi}(n)$ numerically. The algorithm proposed for estimation $h$ in \cite{Vapnik:etal:1994} is as follows.

\begin{algorithm}[ht]
Result: Obtain $\widehat{\xi}(n_{j})$'s

Given:
\begin{itemize}
    \item A collection of classifiers $\mathcal{F}$;
    \item A set of design points $N_{L} = \left\{n_{1},n_{2},\cdots,n_{l}\right\}$;
    \item An integer $W$ for the number of bootstrap samples;
\end{itemize}
$\bf for ~ j= 1,2,\cdots, l$
\begin{enumerate}
    \item Generate a random sample of size $2n_{j}$ $Z^{2n} = \left(Z_{1},Z_{2},\cdots,Z_{2n_{j}}\right)$;
    \item Divide the random sample into two groups; $Z^{1}$ and $Z_{2}$;
    \item Flip the class labels for the second set: $Z^{2}$;
    \item Merge the two sets to train the binary classifier;
    \item Separate the sets and flip the labels on the second set back again;
    \item Calculate the training error of the estimated classifier $\widehat{f}$ respectively on $Z^{1}$ and $Z^{2}$ with the correct labels;
    \item Calculate
    $$
    \widehat{\xi}_{i}(n_{j}) = \left|\nu_{1}(Z^{1},\alpha)-\nu_{2}(Z^{2},\alpha)\right|
    $$
    \item Repeat steps 1-7 $W$ times
    \item Calculate
    $$
    \widehat{\xi}(n_{j}) = \frac{1}{W}\sum_{i=1}^{W}\widehat{\xi}_{i}(n_{j})
    $$
\end{enumerate}
$\bf \hbox{End for}$

 \caption{Generation of the  $\widehat{\xi}(n_{j})$'s using the Vapnik et al. \cite{Vapnik:etal:1994} algorithm.}
 \label{VapnikAlgorithm}
\end{algorithm}

Having enough values of $\widehat{\xi}(n_{j})$ we can use nonlinear regression to estimate $h$ via
$$
\widehat{\xi}(n) = \Phi_{h}(n) + \epsilon(n),
$$
where $\epsilon(n)$ has mean zero and an unknown distribution. Now we estimate $\hat{h}$ by
$$
\hat{h} = \min_{h}\sum_{l=1}^{|N_{L}|}\left(\hat{\xi}(n_{l})-\Phi_{h}(n)\right)^2.
$$
One of the problems with this method is that the derivation of $\Phi_{h}(n)$ uses a conditioning argument in which the conditioning set has probability going to zero so the derivation of $\Phi_{h}(n)$ is in question. Moreover, in our work with algorithm \ref{VapnikAlgorithm} (not shown here) we have found it to be unstable.

\cite{McDonald:etal:2011} tried to  prove the algorithm in \cite{Vapnik:etal:1994} for estimating $h$ was consistent.
Their main result was the following
\begin{thm}
\label{Theo1.3McDonal}
Let $\delta \geq \frac{4}{\sqrt{2Wk}}\max\left(24c_{1}, 29 \right)$ and suppose that $h\leq M$. Then $~\exists c_{1}, c_{2}, c_{3}$ and $ W, k$ so that
\begin{equation*}
    P\left(\left|\hat{h}-h\right|\geq\delta\right)\leq 13\exp\left(-\frac{Wkc_{2}\delta^{2}}{16c_{3}}\right).
\end{equation*}
\end{thm}
This result as stated is not true although the parts of the proof that are correct are a tour-de-force. What makes the proof incomplete is a gap in a sequence of upper bounds. One of the bounds that look intuitively reasonable does not in fact follow from the covering number argument the authors advance. So we believe that a version of Theorem \ref{Theo1.3McDonal} is true but will required a more careful treatment. Appendix \ref{AppenChap2} provides some details.

Finally for this chapter we credit these results for motivating us and providing us with a framework for our contributions. In particular, we have borrowed literally from the technique of proof of these results. However, we also remind readers that our results are for regression not classification; we have change the mathematical form of the error and adding an extract optimization, amongst other improvements. In short, we think we have taken a technique whose rudiments had been identified, have refined them, have assembled and have shown how to use the overall result for analysis.


\chapter{BOUNDS ON EXPECTED SUPREMAL DIFFERENCE OF EMPIRICAL LOSSES}
\label{chap:bounds}
\thispagestyle{fancy}
We are going to bound the Expected Maximum Difference Between Two Empirical Losses (EMDBTEL) for the case of regression. In fact, we convert the regression problem into $m~$ classification problems by discretizing the empirical loss.
We will use this bound to derive an estimator of the VCD for the class of linear functions. In Sec: \ref{sec:Ext:Vap:Bounds}, we will present the extension of Vapnik bounds and in Sec: \ref{ChangeExpect}, we will present adjustment in the computation of the EMDBTEL.

\section{Extension of Vapnik-Chervonenkis Bounds}
\label{sec:Ext:Vap:Bounds}
Let $Z=(X,Y)$ be a pair of observations and write $Z^{2n}=( Z_{1}, \dots, Z_{2n})$
be a vector of $2n$ independent and identically distributed (IID) copies of
$Z$.   Let
$$
Q(z,\alpha) = L\left(y, f\left(x,\alpha\right)\right)
$$
be a bounded real valued loss function, where $\alpha \in \Lambda$, an index set,
and assume $\forall~ \alpha, ~ 0\leq Q(Z, \alpha) \leq B$ for some $B \in \mathbb{R}$.
Consider the discretization of $Q$ using $m$  disjoint  intervals (with union $[0, B)$) given by
\begin{equation}
\label{E1}
  Q_{j}^{*}\left(z,\alpha, m\right)=\left\{
                                  \begin{array}{ll}
                                    \frac{\left(2j+1\right)B}{2m}, & \hbox{if $Q\left(z,\alpha\right)\in \emph{I}_{j}=\left[\frac{jB}{m},\frac{(j+1)B}{m}\right)$;} \\
                                    0, & \hbox{otherwise,}
                                  \end{array}
                                \right.
\end{equation}
and $j$ takes values from $0$ to $m-1$.
For $j\neq j^{'}$, $support  Q_{j}^{*}(Z,\alpha,m)\cap support Q_{j^{'}}^{*}(z,\alpha,m) = \emptyset$.
Now, consider indicator functions for $Q$ being in an interval of the same form.  That is, let
\begin{equation}
\label{E2}
  \chi_{I_{j}}(Q\left(z,\alpha\right)) = \left\{
               \begin{array}{ll}
                 1, & \hbox{if $Q\left(z,\alpha\right)\in \emph{I}_{j}=\left[\frac{jB}{m},\frac{(j+1)B}{m}\right)$;} \\
                 0, & \hbox{otherwise.}
               \end{array}
             \right.
\end{equation}
and write
$$
n_{1j}^{*} = \sum_{i=1}^{n}\chi_{I_{j}}\left(Q(z_{i}^{1},\alpha,m)\right)
\quad \hbox{and} \quad
n_{2j}^{*} = \sum_{i=n+1}^{2n}\chi_{I_{j}}\left(Q(z_{i}^{2},\alpha,m)\right)
$$
for the number of data points whose losses land inside the interval $I_{j}$ in the first and second
half of the sample of size $n$, respectively.

\begin{proposition}
    \label{Prop1}
  The sequence of measurable functions
$\langle Q^{*}( \cdot ,\alpha,m) \rangle \big{|}_{m\geq 1}=\sum_{j=0}^{m-1}Q_{j}^{*}(z,\alpha,m)$ converges to $Q(\cdot ,\alpha)\in \left[0,B\right)$
a.e. in the underlying probability of the measure space.
\end{proposition}

\begin{proof}
 we want to prove that $\forall \epsilon \ge 0$, there exists $m_{0}>0$ such that whenever $m>m_{0}$, $\left|Q^{*}(z,\alpha, m)\quad-\quad Q(z,\alpha)\right|\leq \epsilon$.
Let $\epsilon = \frac{B}{2m}$. $\forall z$ there exists $j$ such that
\begin{eqnarray*}
  \frac{jB}{m} &\leq&  Q(z,\alpha)\leq\frac{(j+1)B}{m} \\
  \frac{jB}{m} -\frac{(2j+1)B}{2m} &\leq& Q(z,\alpha)-Q^{*}(z,\alpha,m)\leq \frac{(j+1)B}{m}-\frac{(2j+1)B}{2m}\\
  -\frac{B}{2m} &\leq& Q^{*}(z,\alpha,m)-Q(z,\alpha)\leq\frac{B}{2m} \\
  \left|Q^{*}(z,\alpha,m)-Q(z,\alpha)\right| &\leq& \frac{B}{2m}=\epsilon \quad \quad
\end{eqnarray*}
\end{proof}

Let
$$
Z^{1} = (Z_{1}, \dots, Z_{n}) \quad \hbox{and} \quad
Z^{2} = (Z_{n+1},  \dots, Z_{2n})
$$
be the first and second half of the sample of size $2n$, respectively, and write
$$
\nu_{1j}^{*}\left(z^{1},\alpha,m\right)= \frac{n_{1j}^{*}Q^{*}_{j}\left(z^{1},\alpha,m\right)}{n}
\quad
\hbox{and}
\quad
\nu_{2j}^{*}\left(z^{2},\alpha,m\right)= \frac{n_{2j}^{*}Q^{*}_{j}\left(z^{2},\alpha,m\right)}{n},
$$
for the empirical risk using $Q^{*}\left(z,\alpha,m\right)$
on the first and second half of the sample, respectively for the $j^{th}$ interval.
Also,  let
$$
\nu_{1}\left(z^{1},\alpha\right)= \frac{1}{n}\sum_{i=1}^{n}Q\left(z^{1}_{i},\alpha\right)
\quad
\hbox{and}
\quad
\nu_{2}\left(z^{2},\alpha\right)= \frac{1}{n}\sum_{i=n+1}^{2n}Q\left(z^{2}_{i},\alpha\right)
$$
be the empirical risk function using $Q\left(z,\alpha\right)$ on the first and second half of the sample,
respectively.

To begin to control the expected supremal difference between bounded loss functions,
let $\epsilon > 0$ and define the events
\begin{equation}
\label{MaxDif}
A_{\epsilon,m} = \left\{z^{2n}: \sup_{\alpha \in \Lambda}\left(\nu_{1}^{*}\left(z^{1},\alpha\right)-\nu_{2}^{*}\left(z^{2},\alpha\right)\right)\geq \epsilon\right \}
\end{equation}
$$
\hbox{where}
\quad
\nu_{k}^{*}(z^{k},\alpha) = \sum_{j=0}^{m-1}\nu_{kj}^{*}(z^{k},\alpha,m)\quad \hbox{and} \quad k=1,2.
$$
Since $A_{\epsilon,m}$ is defined on the entire range of the loss function and we want to
partition the range into $m$ intervals, let

\begin{eqnarray*}
  A_{\epsilon,m} &=& \left\{z^{2n}\mid \sup_{\alpha \in\Lambda}\left\{\sum_{j=0}^{m-1}\nu_{1j}^{*}\left(z^{1},\alpha\right)-\sum_{j=0}^{m-1}\nu_{2j}^{*}\left(z^{2},\alpha\right)\right\}\geq \epsilon\right\} \\
   &=& \left\{z^{2n}| \sup_{\alpha\in\Lambda}\sum_{j=0}^{m-1}\left(\nu_{1j}^{*}(z^{1},\alpha,m)-\nu_{2j}^{*}(z^{2},\alpha,m)\right)\geq \epsilon\right\} \\
   &\subseteq& \left\{z^{2n}|~ \exists ~j ~\sup_{\alpha \in \Lambda}\left(\nu_{1j}^{*}(z^{1},\alpha,m)-\nu_{2j}^{*}(z^{2},\alpha,m)\right)\geq\frac{\epsilon}{m}\right\} \\
   &\subseteq& \bigcup_{j=0}^{m-1}\left\{z^{2n}| \sup_{\alpha\in\Lambda}\left(\nu_{1j}^{*}(z^{1},\alpha,m)-\nu_{2j}^{*}(z^{2},\alpha,m)\right)\geq\frac{\epsilon}{m}\right\}
\end{eqnarray*}

$$
\quad \hbox{where} \quad
A_{\epsilon,m,j} = \left\{z^{2n}: \sup_{\alpha \in \Lambda}\left(\nu_{1j}^{*}\left(z^{1},\alpha,m\right)-\nu_{2j}^{*}\left(z^{2},\alpha,m\right)\right)\geq \frac{\epsilon}{m}\right \}
$$
so that
$$
A_{\epsilon,m} \subseteq \bigcup_{j=0}^{m-1}A_{\epsilon,m,j}.
$$

Let suprema over $\Lambda$ within $A_{\epsilon, m, j}$ be achieved at
$$\alpha^{*}_{j}=\alpha^{*}_{j}(z^{2n})=\arg\sup_{\alpha\in\Lambda}\left(\nu_{1j}^{*}\left(z_{1},\alpha, m\right)-\nu_{2j}^{*}\left(z_{2},\alpha,m\right)\right).$$
Next, for any fixed $z^{2n}$,
and any given $\alpha_{j}$, form the vector
$$\left(Q^{*}(z_{1},\alpha_{j},m),Q^{*}(z_{2},\alpha_{j},m),\dots, Q^{*}(z_{2n},\alpha_{j},m)\right)$$
of the middle values of the intervals $\emph{I}_{j}$, for $j= 0,1,2,\dots, m-1$.
For any $\alpha_{j}$ and $\alpha_{j^\prime}$, write
\begin{eqnarray}
&& \alpha_{j} \sim \alpha_{j^\prime} \Longleftrightarrow
\nonumber \\
&& \left(Q^{*}(z_{1},\alpha_{j},m),\dots, Q^{*}(z_{2n},\alpha_{j}, m)\right) = \left(Q^{*}(z_{1},\alpha_{j^\prime},m),\dots,Q^{*}(z_{2n},\alpha_{j^\prime},m)\right).
\nonumber
\end{eqnarray}
So, for any fixed $Z^{2n} = z^{2n}$ it is seen that $\sim$ is an equivalence relation on $\Lambda$
and therefore partitions $\Lambda$ into disjoint equivalence classes.
Denote the number of these classes by $N^{\Lambda}_{j}$ and write
$$
N_j^\Lambda = N^{\Lambda}_{j}\left(Z^{2n}\right) = N^{\Lambda}_{j}(z_{1},z_{2},\dots,z_{2n}).
$$
We define values $\alpha^{*}_{jk}$ as the canonical representatives of the equivalence classes
where $k \in K_{j}$ is the $k^{th}$ equivalence class.
Clearly, $\#(K_{j})= N^{\Lambda}_{j}\left(Z^{2n}\right)$ and $K_{j}$ is treated simply as an index set.

To make use of the above partitioning of $A_{\epsilon}$,
consider mapping the space ${\cal{Z}}^{2n}$ onto itself using $(2n)!$ distinct permutations
$T = T_{i}$.  Then, if $f$ is integrable with respect to the distribution function of $Z_i$, its Riemann-Stieltjes
integral satisfies
$$
\int_{Z^{2n}}^{}f\left(Z^{2n}\right)\emph{dF}\left(Z^{2n}\right) = \int_{Z^{2n}}^{}f\left(T_{i}Z^{2n}\right)\emph{dF}\left(Z^{2n}\right),
$$
and this gives

\begin{equation}
\label{permid}
  \int_{Z^{2n}}^{}f\left(Z^{2n}\right)\emph{dF}\left(Z^{2n}\right) = \int_{Z^{2n}}^{}\frac{\sum_{i=1}^{(2n)!}f\left(T_{i}Z^{2n}\right)}{(2n)!}\emph{dF}\left(Z^{2n}\right).
\end{equation}
To make use of \ref{permid} we need the following results from \cite{Vapnik:1998}.

\begin{thm}
   \label{Theo1}
   In this theorem, we put together some background results from Vapnik \cite{Vapnik:1998}.
   \begin{enumerate}
     \item Let $H^\Lambda_{\sf ann}(n) = E \ln N^\Lambda (Z^n)$ denote the annealed entropy and $G^\Lambda(n) = \sup_{z^n} \ln N^\Lambda (Z^n)$ be the growth function. Then
    $$
    H_{ann}^{\Lambda}(n)\leq G^{\Lambda}(n)\le h\left(1+\ln\left(\frac{n}{h}\right)\right).
    $$
    \item Let $\Gamma ~= ~\sum_{k}^{}\frac{\binom{m}{k}\binom{2n-m}{n-k}}{\binom{2n}{n}}$
    where the summation is over $k$ so that
    $$
    \left|\frac{k}{n}~-~\frac{m~-~k}{n}\right|~>\epsilon, ~ \max(0,~m~-~n)\leq k\leq \min(m,~n)
    $$
    where $n$ and $m<~ 2n$ are arbitrary positive integers. Note that

    \begin{align*}
  \left|\frac{k}{n}~-~\frac{m~-~k}{n}\right| & > \epsilon \\
  \Leftrightarrow \left|\frac{2k}{n}~-~ \frac{m}{n}\right| & > \epsilon \\
  \Leftrightarrow \frac{2}{n}\left|k~-~\frac{m}{2}\right| & >\epsilon \\
  \Leftrightarrow \left|k~-~\frac{m}{2}\right| & >\frac{\epsilon n}{2}, ~\max(0,~m~-~n)\leq k \leq \min(m,~n) .
    \end{align*}
    Then, $\exists ~ C >0$ so that
    $
    \Gamma < C e^{-\epsilon n^2}.
    $
    \item  Let $Q(\cdot ,\alpha)$ for  $\alpha \in\Lambda$ be a set of real-valued(not necessary bounded) non-negative functions and let $H_{ann}^{\Lambda,}(n)$ be the annealed entropy of the level sets for
$\{ Q( \cdot ,\alpha) \mid \alpha \in \Lambda\}$. Then, for any $1< p \leq 2$, and a constant $D_p(\alpha)$ to be defined later, we have that
\begin{eqnarray*}
  P\left\{\sup_{\alpha\in\Lambda}\frac{\int_{}^{}Q(z,\alpha)\emph{dF}(z)-\frac{1}{n}\sum_{i=1}^{n}Q(z_{i},\alpha)}{D_{p}(\alpha)}\ge\epsilon\right\} &<& \nonumber\\ 4\exp\left\{\left(\frac{H_{ann}^{\Lambda,\mathbf{\beta}}(n)}{n^{2-\frac{2}{p}}}-\frac{\epsilon^{2}}{2^{1+\frac{2}{p}}}\right)n^{2-\frac{2}{p}}\right\}.
\end{eqnarray*}
   \end{enumerate}
\end{thm}

\begin{proof}
Clause 1 is proved in Theorem 4.3 in \cite{Vapnik:1998}, p. 145;
Clause 2 is proved in Sec.  4.13 of \cite{Vapnik:1998}, p. 163, and Clause 3 is proved
in Sec. 5.4 \cite{Vapnik:1998} p. 195.
\end{proof}
To present our first important result let
\begin{equation}
\label{diffnu}
  \Delta_{j}\left(T_{i}Z^{2n},\alpha^{*}_{j},m\right) = \nu_{1j}^{*}\left(T_{i}Z_{1},\alpha^{*}_{j},m\right) - \nu_{2j}^{*}\left(T_{i}Z_{2},\alpha^{*}_{j},m\right) .
\nonumber
\end{equation}
We have the following.

\begin{proposition}
\label{Prop2}
Let $\epsilon >0$, $m \in \cal N$, and $h= VCdim\left\{Q(\cdot,\alpha):\alpha\in\Lambda\right\}$, and $n$ be the sample size.
If $h$ is finite, then
$$
P\left(A_{\epsilon,m}\right) \leq 2m\left(\frac{2ne}{h}\right)^{h}\exp\left\{-\frac{n\epsilon^{2}}{m^2}\right\}.
$$

\end{proposition}
\begin{proof}
 First, by elementary manipulations we have
\begin{eqnarray}
  P\left(A_{\epsilon,m}\right) &\leq& P\left(\bigcup_{j=0}^{m-1}A_{\epsilon,m,j}\right)
   \leq \sum_{j=0}^{m-1}P\left(A_{\epsilon,m,j}\right)\nonumber\\
   &=& \sum_{j=0}^{m-1}P\left(\left\{Z^{2n}:\sup_{\alpha\in\Lambda}\left(\nu_{1j}^{*}(Z_{1},\alpha,m)-\nu_{2j}^{*}(Z_{2},\alpha,m)\right)\geq\frac{\epsilon}{m}\right\}\right).
\nonumber
\end{eqnarray}

Continuing the equality gives

\begin{eqnarray}
   &=& \sum_{j=0}^{m-1}P\left(\left\{Z^{2n}:\sup_{\alpha\in\Lambda}\left(\left(\nu_{1j}^{*}(T_{i}Z_{1},\alpha,m)-\nu_{2j}^{*}(T_{i}Z_{2},\alpha,m)\right)\right)\geq\frac{\epsilon}{m}\right\}\right)
\nonumber\\
   &=& \sum_{j=0}^{m-1}P\left(\left\{Z^{2n}:\left(\nu_{1j}^{*}(T_{i}Z_{1},\alpha_{j}^{*},m)-\nu_{2j}^{*}(T_{i}Z_{2},\alpha_{j}^{*},m)\right)\geq\frac{\epsilon}{m}\right\}\right)
\nonumber\\
   &=& \sum_{j=0}^{m-1}P\left(\left\{Z^{2n}: \Delta_{j}\left(T_{i}Z,\alpha_{j}^{*},m\right)\geq\frac{\epsilon}{m}\right\}\right)\nonumber\\
   &=& \frac{1}{(2n)!}\sum_{j=0}^{m-1}\sum_{i=1}^{(2n)!}P\left(\left\{Z^{2n}: \Delta_{j}\left(T_{i}Z,\alpha_{j}^{*},m\right)\geq\frac{\epsilon}{m}\right\}\right)
\nonumber \\
   &=& \frac{1}{(2n)!}\sum_{j=0}^{m-1}\sum_{i=1}^{(2n)!}\int_{}^{}\emph{I}_{\left\{Z^{2n}: \Delta_{j}\left(T_{i}Z,\alpha_{j}^{*},m\right)\geq\frac{\epsilon}{m}\right\}}\left(z^{2n}\right)\emph{dP}(Z^{2n}).
\label{permprobbd}
\end{eqnarray}

Let the equivalence classes in $\Lambda$ under $\sim$ be denoted $\Lambda_k$.
Then, the equivalence classes $\left[\alpha_{jk}^{*}\right]$ for the  $\alpha_{jk}^{*}$'s
provide a partition for $\Lambda$.  That is,
$\Lambda = \bigcup_{k=1}^{N_{j}^{\Lambda}(z^{2n})}[\alpha_{jk}^{*}]$
because $\alpha^{*}_{jk}\in \Lambda_k$ and hence $[\alpha_{jk}^{*}]=\Lambda_{k}$.
In addition, $\alpha^{*}_{jk}$ is the maximum value of $\alpha_{j}$ in the $k^{th}$ equivalence class.
So,
\begin{eqnarray}
\label{Indfnbd}
  \emph{I}_{\left\{Z^{2n}:\Delta_{j}\left(T_{i}Z,\alpha^{*}_{j},m\right)\geq\epsilon\right\}}(z^{2n})
&\leq&
\emph{I}_{\left\{Z^{2n}: \Delta_{j}\left(T_{i}Z,\alpha^{*}_{1j},m\right)\ge \frac{\epsilon}{m}\right\}}\left(z^{2n}\right) \nonumber \\
 &+&
\dots + \emph{I}_{\left\{Z^{2n}:\Delta_{j}\left(T_{i}Z,\alpha^{*}_{N^{\Lambda}_{j}(z^{2n})j},m\right)\ge\frac{\epsilon}{m}\right\}}\left(z^{2n}\right) \nonumber\\
   &=& \sum_{k=1}^{N^{\Lambda}_{j}(z^{2n})}\emph{I}_{\left\{Z^{2n}: \Delta_{j}\left(T_{i}Z,\alpha^{*}_{kj},m\right)\ge\frac{\epsilon}{m}\right\}}\left(z^{2n}\right)
\end{eqnarray}
where
\begin{eqnarray*}
  A_{\epsilon,m,j,k} &=& \left\{Z^{2n}: \Delta_{j}\left(T_{i}Z,\alpha^{*}_{kj},m\right)\ge\frac{\epsilon}{m}\right\} \\
   &=& \left\{Z^{2n}: \sup_{\alpha_{j}\in\Lambda_{k}}\left(\nu_{1j}^{*}(T_{i}Z_{1},\alpha_{j},m)-\nu_{2j}^{*}(T_{i}Z_{2},\alpha_{j},m)\right)\ge\frac{\epsilon}{m}\right\}
\end{eqnarray*}

Now, using \ref{Indfnbd}, \ref{permprobbd} is bounded by
\begin{eqnarray}
  P\left(A_{\epsilon,m}\right) &\leq& \frac{1}{(2n)!}\sum_{j=0}^{m-1}\sum_{i=1}^{(2n)!}\int_{}^{}\sum_{k=1}^{N_{j}^{\Lambda}(z^{2n})}\emph{I}_{\left\{Z^{2n}:\Delta_{j}(T_{i}Z,\alpha^{*}_{kj},m)\ge\frac{\epsilon}{m}\right\}}\left(z^{2n}\right) \emph{dP}\left(z^{2n}\right)
\nonumber\\
   &=& \int_{}^{}\sum_{j=0}^{m-1}\sum_{k=1}^{N_{j}^{\Lambda}(z^{2n})}\left[\frac{1}{(2n)!}\sum_{i=1}^{(2n)!}\emph{I}_{\left\{Z^{2n}:\Delta_{j}(T_{i}Z,\alpha_{kj}^{*},m)\geq \frac{\epsilon}{m}\right\}}\left(z^{2n}\right)\right]\emph{dP}\left(z^{2n}\right) .
\label{overallprobbd}
\end{eqnarray}
The expression in square brackets in \ref{overallprobbd} is the fraction of the number of the $(2n)!$ permutations $T_{i}$ of $Z^{2n}$ for which $A_{\epsilon,m,j,k}$ is closed under $T_{i}$ for any fixed equivalence class $\Lambda_{k}$. It is equal to
    \begin{equation}
    \label{E7}
      \Gamma_{j} = \sum_{k}^{}\frac{\binom{m^{*}_{j}}{k} \binom{2n-m^{*}_{j}}{m^{*}_{j}-k}}{\binom{2n}{n}}
    \nonumber
    \end{equation}
where
$$
\left\{
  \begin{array}{ll}
    \left\{k: \left|\frac{k}{n}-\frac{m^{*}_{j}-k}{n}\right|\ge\frac{\epsilon}{m}\right\} & \hbox{;} \\
    m^{*}_{j} = n_{1j}^{*} + n_{2j}^{*} & \hbox{.}
  \end{array}
\right.
$$
Here, $\Gamma_{j}$ is the probability of choosing exactly $k$ sample data points whose losses fall in interval $I_{j}$ respectively in the first and second half of the sample such that $A_{\epsilon,m,j,k}$ holds. $m_{j}^{*}$ is the number of data points from the first and the second half of the sample whose losses landed inside interval $j$.
Using  Theorem  \ref{Theo1}, Clause II, we have $\Gamma_{j} \le 2\exp\left(-\frac{n\epsilon^{2}}{m^2}\right)$.  So,
using this in equation \ref{overallprobbd} gives $P\left(A_{\epsilon}\right)$ is upper bounded by
\begin{eqnarray}
   && \int_{}^{}\sum_{j=0}^{m-1}\sum_{k=1}^{N_{j}^{\Lambda}(z^{2n})}2\exp\left(-\frac{n\epsilon^{2}}{m^2}\right)\emph{dP}(z^{2n})
   = 2\exp\left(-\frac{n\epsilon^{2}}{m^2}\right)\int_{}^{}\sum_{j=0}^{m-1}\sum_{k=1}^{N_{j}^{\Lambda}(z^{2n})}\emph{dP}(z^{2n})
\nonumber \\
   &=& 2\exp\left(-\frac{n\epsilon^{2}}{m^2}\right)\sum_{j=0}^{m-1}\int_{}^{}\sum_{k=1}^{N_{j}^{\Lambda}(z^{2n})}\emph{dP}(z^{2n})
   = 2\exp\left(-\frac{n\epsilon^{2}}{m^2}\right)\sum_{j=0}^{m-1}\int_{}^{}N_{j}^{\Lambda}(z^{2n})\emph{dP}(z^{2n})
\nonumber \\
   &=& 2\exp\left(-\frac{n\epsilon^{2}}{m^2}\right)\sum_{j=0}^{m-1}E\left(N_{j}^{\Lambda}(z^{2n})\right).
\label{probbdd}
\end{eqnarray}
Since Theorem \ref{Theo1}, Clause I gives
$$
H_{ann}(Z^{2n}) = \ln\left(E\left(N_{j}^{\Lambda}(Z^{2n})\right)\right)\leq G(2n)\leq h\ln\left(\frac{2ne}{h}\right) \Rightarrow
 E\left(N_{j}^{\Lambda}(z^{2n})\right)\leq \left(\frac{2ne}{h}\right)^{h}.
$$
Using this $m$ times in \ref{probbdd} gives the Proposition.
 \end{proof}
 Proposition \ref{Prop1} gives us an upper bound on the expected supremal difference. We will use this upper bound to bound the unknown true risk in the following proposition.
 Let $R\left(\alpha_{k}\right)$ be the true unknown risk at $\alpha_{k}$ and $R_{emp}\left(\alpha_{k}\right)$ be the empirical risk at $\alpha_{k}$.
\begin{proposition}
\label{subProp1}
   With probability $1-\eta$, the inequality
  \begin{eqnarray}
  \label{EquivProp1}
  R\left(\alpha_{k}\right) \leq R_{emp}\left(\alpha_{k}\right) + m\sqrt{\frac{1}{n}\log\left(\left(\frac{2m}{\eta}\right)\left(\frac{2ne}{h}\right)^h\right)}
  \end{eqnarray}
  holds simultaneously for all functions $Q\left(z,\alpha_{k}\right)$, $k = 1,2,\cdots, K$.
\end{proposition}
This inequality suggests that, the best model will be the one that minimizes the RHS of inequality \ref{EquivProp1}. The use of inequality \ref{EquivProp1} in model selection is called Empirical Risk Minimization (ERM).
\begin{proof}
  To obtain inequality \eqref{EquivProp1}, we have to equate the RHS of  Proposition \ref{Prop2} to some positive number $0\le \eta\leq 1$. That's
\begin{equation*}
  \eta = 2m\left(\frac{2ne}{h}\right)^{h}\exp\left(-\frac{n\epsilon^{2}}{m^{2}}\right)
\end{equation*}
and solve it for $\epsilon$, one will get
\begin{eqnarray}
\label{solution}
  \epsilon &=& m\sqrt{\frac{1}{n}\log\left(\left(\frac{2m}{\eta}\right)\left(\frac{2ne}{h}\right)^{h}\right)}.
\end{eqnarray}
 Proposition \ref{subProp1} can be obtained from the additive Chernoff bounds, expression $4.4$ in \cite{Vapnik:1998} as follows
\begin{eqnarray}
\label{proofsubProp1}
R\left(\alpha_{k}\right) &\leq& R_{emp}\left(\alpha_{k}\right) + \epsilon.
\end{eqnarray}
Using \eqref{solution} in inequality \eqref{proofsubProp1}, completes the proof.
\end{proof}
\begin{proposition}
\label{subProp2}
  With probability $1-\eta$, the inequality
\begin{equation}\label{EquivalentProp2}
  R\left(\alpha_{k}\right)\le R_{emp}\left(\alpha_{k}\right) + \frac{m^{2}}{2n}\log\left(\frac{2m}{\eta}\left(\frac{2ne}{h}\right)^{h}\right)\left(1+\sqrt{1+\frac{4nR_{emp}\left(\alpha_{k}\right)}{m^{2}\log\left(\frac{2m}{\eta}\left(\frac{2ne}{h}\right)^{h}\right)}}\right)
\end{equation}
holds simultaneously for all $K$ functions in the set $Q\left(z, \alpha_{k}\right)$, $k = 1,2,\dots, K$.
\end{proposition}
As we said before, this inequality suggests that, the best model will be the one that minimizes the RHS of \ref{EquivalentProp2}. The use of inequalities \eqref{subProp1} and \eqref{EquivalentProp2} in model selection are called Empirical Risk Minimization (ERM).
\begin{proof} Let $\epsilon, \eta >0$.  Then, inequality (4.18) in \cite{Vapnik:1998} gives, with probability at least $1-\eta$, that
\begin{eqnarray*}
  \frac{R\left(\alpha_{k}\right)-R_{emp}\left(\alpha_{k}\right)}{\sqrt{R\left(\alpha_{k}\right)}} \leq \epsilon.
  \end{eqnarray*}
  Routine algebraic manipulations and completing the square give
  \begin{eqnarray*}
    \left(R\left(\alpha_{k}\right)-0.5\left(\epsilon^{2}+2R_{emp}\left(\alpha_{k}\right)\right)\right)^{2}-0.25\left(\epsilon^{2}+2R_{emp}\left(\alpha_{k}\right)\right)^{2} \leq -R^{2}_{emp}\left(\alpha_{k}\right).
    \end{eqnarray*}
    Taking the square root on both sides and re-arranging gives
  \begin{eqnarray*}
  R\left(\alpha_{k}\right) \leq
R_{emp}\left(\alpha_{k}\right) + 0.5\epsilon^{2}\left(1 + \sqrt{1 + \frac{4R_{emp}\left(\alpha_{k}\right)}{\epsilon^{2}}}\right)
\end{eqnarray*}
Using \eqref{solution} in the last inequality completes the proof of the Proposition.
\end{proof}
\noindent
The next Theorem will bound the EMDBTEL using proposition \ref{Prop1}
\begin{thm}
\label{Theo2}
  Let $h = VCDim\left\{Q\left(\cdot,\alpha\right):\alpha\in\Lambda\right\}$, $m\in\cal{N}$ be the number of discretized intervals, $n\in \cal{N}$ be the sample size. If $h< \infty$, we have
\begin{equation*}
  E\left(\sup_{\alpha\in\Lambda}\left|\nu_{1}^{*}(z^{1},\alpha)-\nu_{2}^{*}(z^{2},\alpha)\right|\right) \leq m\sqrt{\frac{1}{n}\ln\left(2m^{3}\left(\frac{2ne}{h}\right)^h\right)}\\ + \frac{1}{m\sqrt{n\ln\left(2m^3\left(\frac{2ne}{h}\right)^{h}\right)}}
\end{equation*}
\end{thm}
\noindent

\begin{proof}
The LHS of the statement of the theorem equals
\begin{eqnarray}
  \int_{0}^{\infty}P\left(A_{\epsilon,m}\right)\emph{d}\epsilon
   \leq \int_{0}^{u}\emph{d}\epsilon + \int_{u}^{\infty}P\left(A_{\epsilon}\right)\emph{d}\epsilon
   \leq u +  2m\left(\frac{2ne}{h}\right)^{h}\int_{u}^{\infty}\exp\left(-\frac{n\epsilon^{2}}{m^{2}}\right)\emph{d}\epsilon .
\label{useid}
\end{eqnarray}
Observing that
$$
\epsilon>u \Rightarrow \frac{n\epsilon^{2}}{m^{2}}>\frac{nu\epsilon}{m^{2}}\Rightarrow \int_{u}^{\infty}\exp\left(-\frac{n\epsilon^{2}}{m^{2}}\right)\emph{d}\epsilon$$
$$\leq  \int_{u}^{\infty}\exp\left(-\frac{nu\epsilon}{m^{2}}\right)\emph{d}\epsilon=\frac{m^{2}}{nu}\exp\left(-\frac{nu^{2}}{m^{2}}\right),
$$
the bound in \eqref{useid} gives
\begin{equation}
E\left(\sup_{\alpha\in\Lambda}\left(\nu_{1}^{*}(z^{1},\alpha)-\nu_{2}^{*}(z^{2},\alpha)\right)\right)\leq u + 2m^{3}\left(\frac{2ne}{h}\right)^{h}\frac{1}{nu}\exp\left(-\frac{nu^{2}}{m^{2}}\right) .
\label{newbd}
\end{equation}
substituting
$$
u = m\sqrt{\frac{1}{n}\ln\left(2m^{3}\left(\frac{2ne}{h}\right)^{h}\right)}
$$
in \eqref{newbd}  gives the statement of the Theorem.
\end{proof}
\noindent
Now, we turn our attention to developing a second upper bound on the expected supremal difference.
We begin with two elementary lemmas.
\begin{lemma}
\label{Lemma1}
  Let $\epsilon >0$.
The probability of the supremal difference between two empirical losses is bounded by
$$
P\left(\sup_{\alpha \in \Lambda}\left(|\nu_{1}\left(Z^{1},\alpha\right)~-~\nu_{2}\left(Z^{2},\alpha\right)|\geq \epsilon\right)\right)\leq$$
$$ 2P\left(\sup_{\alpha \in \Lambda}\left(|\nu_{1}(Z^{1},\alpha)~-~E(Q(Z^{1},\alpha))|\geq \frac{\epsilon}{2}\right)\right) .
$$
\end{lemma}
\begin{proof}
  The left hand side is bounded by
  \begin{eqnarray*}
&&   P\left(\sup_{\alpha\in\Lambda}\left|\nu_{1}\left(Z^{1},\alpha\right)-E\left(Q(Z^{1},\alpha)\right)+E\left(Q(Z^{1},\alpha)-\nu_{2}\left(Z^{2},\alpha\right)\right)\right|\geq \epsilon\right)
\nonumber\\
   & \leq& P\left(\sup_{\alpha\in\Lambda}\left|\nu_{1}\left(Z^{1},\alpha\right)-E\left(Q(Z^{1},\alpha\right)\right| \geq\frac{\epsilon}{2}\right) +\\ && P\left(\sup_{\alpha\in\Lambda}\left|E\left(Q(Z^{1},\alpha)\right)-\nu_{2}\left(Z^{2},\alpha\right)\right|\geq\frac{\epsilon}{2}\right)
\nonumber\\
   &=&2P\left(\sup_{\alpha\in\Lambda}\left|\nu_{1}\left(Z^{1},\alpha\right)-E\left(Q(Z^{1},\alpha)\right)\right|\geq\frac{\epsilon}{2}\right).
\end{eqnarray*}
\end{proof}
To state  Lemma \ref{lemma2}, let
$$
\overline{Q(Z,\alpha)}^{-}=\frac{1}{n}\sum_{i=1}^{n}Q(z_{i},\alpha)^{-} \quad \hbox{and} \quad \overline{Q(Z,\alpha)}^{+} = \frac{1}{n}\sum_{i=1}^{n}Q(z_{i},\alpha)^{+}
$$
and assume
$$
D_{p}(\alpha) = \int_{0}^{\infty}\sqrt[p]{P\left\{Q(z,\alpha)\ge c\right\}}\emph{dc} < \infty .
$$
\begin{lemma}
  \label{lemma2}

   Let $Q(Z,\alpha)$ be a nonnegative function.  Then,

 \begin{eqnarray*}
  P\left(\sup_{\alpha\in\Lambda}\frac{\left(E\left(Q(Z,\alpha)\right)-\overline{Q(Z,\alpha)}\right)}{D_{p}(\alpha)}\geq\epsilon\right) &\leq& \nonumber\\ P\left(\sup_{\alpha\in\Lambda}\frac{\left( E\left(Q(Z,\alpha)^{-}\right)-\overline{Q(Z,\alpha)}^{-}\right)}{D_{p}(\alpha)}\geq\frac{\epsilon}{2}\right) +\nonumber\\ P\left(\sup_{\alpha\in\Lambda}\frac{\left(E\left(Q(Z,\alpha)^{+}\right)-\overline{Q(Z,\alpha)}^{+}\right)}{D_{p}(\alpha)}\geq\frac{\epsilon}{2}\right).
\end{eqnarray*}
\end{lemma}
\begin{proof}
The LHS is

\begin{eqnarray*}
  && P\left(\sup_{\alpha\in\Lambda}\frac{1}{D_{p}(\alpha)}\left(\left(E\left(\overline{Q(Z,\alpha)}^{+}\right)-\overline{Q(Z,\alpha)}^{+}\right)
  +\left(E\left(Q(Z,\alpha)^{-}\right)-\overline{Q(Z,\alpha)}^{-}\right)\right)\geq\epsilon\right)
  \nonumber \\
   & \leq & P\left(\left\{\left(\sup_{\alpha\in\Lambda}\frac{1}{D_{p}(\alpha)}\left(E\left(Q(Z,\alpha)^{+}\right)-\overline{Q(Z,\alpha)}^{+}\right)
   +\sup_{\alpha\in\Lambda}\frac{1}{D_{p}(\alpha)}\left(E\left(Q(Z,\alpha)^{-}\right)-\overline{Q(Z,\alpha)}^{-}\right)\right)\geq\epsilon\right\}
   \right.
   \nonumber \\
&& \quad
   \cap \left. \left\{\sup_{\alpha\in\Lambda}\frac{1}{D_{p}(\alpha)}\left(E\left(Q(Z,\alpha)^{+}\right)-Q(Z,\alpha)^{+}\right|)\geq\frac{\epsilon}{2}\right\}\right)
   \nonumber \\
   && + P\left(\left\{\left(\sup_{\alpha\in\Lambda}\frac{1}{D_{p}(\alpha)}\left(E\left(Q(Z,\alpha)^{+}\right)-\overline{Q(Z,\alpha)}^{+}\right)+
   \sup_{\alpha\in\Lambda}\frac{1}{D_{p}(\alpha)}\left(E\left(Q(Z,\alpha)^{-}\right)-\overline{Q(Z,\alpha)}^{-}\right)\right)\geq\epsilon\right\} \right.
   \nonumber \\
&& \quad  \left. \cap
   \left\{\sup_{\alpha\in\Lambda}\frac{1}{D_{p}(\alpha)}\left(E\left(Q(Z,\alpha)^{+}\right)-\overline{Q(Z,\alpha)}^{+}\right)\geq\frac{\epsilon}{2}\right\}\right)
   \nonumber \\
   & \leq & P\left(\sup_{\alpha\in\Lambda}\frac{1}{D_{p}(\alpha)}\left(E\left(Q(Z,\alpha)^{-}\right)-\overline{Q(Z,\alpha)}^{-}\right)\geq\frac{\epsilon}{2}\right)
   \nonumber \\
      &+&  P\left(\sup_{\alpha\in\Lambda}\frac{1}{D_{p}(\alpha)}\left(E\left(Q(Z,\alpha)^{+}\right)-\overline{Q(Z,\alpha)}^{+}\right)\geq\frac{\epsilon}{2}\right).
\end{eqnarray*}
\end{proof}
\begin{proposition}
  \label{Prop4}
Let $h = VCDim\left\{Q\left(\cdot,\alpha\right):\alpha\in\Lambda\right\}$, where $Q\left(\cdot,\alpha\right)$ is unbounded, $n\in \cal{N}$ be the sample size. If $h< \infty$, and
$$
D_{p}\left(\alpha\right) = \int_{0}^{\infty}\sqrt[p]{P\left\{Q(z,\alpha)\ge c\right\}} dc \leq \infty
$$
where $1 < p \leq 2$ is some fixed parameter,
we have
\begin{align}
  P\left(\sup_{\alpha\in\Lambda}\left| \nu_{1}\left(Z^{1},\alpha\right)-\nu_{2}\left(Z^{2},\alpha\right)\right|\geq\epsilon\right) & \leq & 16\left(\frac{ne}{h}\right)^{h}\exp\left\{-\left(\frac{\epsilon n^{1-\frac{1}{p}}}{D_{p}(\alpha^{*})2^{2.5+\frac{1}{p}}}\right)^{2}\right\},
\label{probbd}
\end{align}
where $\alpha^{*}=\sup_{\alpha\in\Lambda}D_{p}(\alpha)$.
\end{proposition}

\begin{proof}
By Lemma \ref{Lemma1}, we have
\begin{align}
\label{LHSLemma1}
  P\left(\sup_{\alpha\in\Lambda}\left| \nu_{1}\left(Z^{1},\alpha\right)-\nu_{2}\left(Z^{2},\alpha\right)\right|\geq\epsilon\right) &\leq& 2P\left(\sup_{\alpha\in\Lambda}\left|\nu_{1}\left(z^{1},\alpha\right)-E\left(Q\left(z^{1},\alpha\right)\right)\right|\geq\frac{\epsilon}{2}\right)
\end{align}
To bound the RHS of \ref{LHSLemma1}, use Lemma \ref{lemma2}  to observe that
\begin{align}
\label{posneg}
  P\left(\sup_{\alpha\in\Lambda} \left| \nu_{n}\left(z,\alpha\right)-E\left(Q(z,\alpha)\right)\right|
\geq\frac{\epsilon}{2}\right)
& \leq
P\left(\sup_{\alpha\in\Lambda}\left|E\left(Q(z,\alpha)^{+}\right)-\frac{1}{n}\sum_{i=1}^{n}Q(z_{i},\alpha)^{+}\right| \geq\frac{\epsilon}{4}\right)
\nonumber\\
& \quad + P\left(\sup_{\alpha\in\Lambda}\left| E\left(Q(z,\alpha)^{-}\right)-\frac{1}{n}\sum_{i=1}^{n}Q(z,\alpha)^{-}\right| \geq\frac{\epsilon}{4}\right).
\end{align}
\noindent
Each probability on the right hand side of \eqref{posneg} can be bounded.  Since
$\alpha^*$ is a maximum, we have $1/D_{p}(\alpha^{*}) \leq 1/D_{p}(\alpha)$.  Thus, for either the positive or negative parts in \eqref{posneg} we have
\begin{eqnarray*}
  \sup_{\alpha\in\Lambda}\frac{E\left(Q(z,\alpha)\right)-\nu_{n}(z,\alpha)}{D_{p}(\alpha^{*})}
&\leq& \sup_{\alpha\in\Lambda}\frac{E\left(Q(z,\alpha)\right)-\nu_{n}(z,\alpha)}{D_{p}(\alpha)}
 \\
 P\left(\sup_{\alpha\in\Lambda}\frac{E\left(Q(z,\alpha)\right)-\nu_{n}(z,\alpha)}{D_{p}(\alpha^{*})}\geq\frac{\epsilon}{4D_{p}(\alpha^{*})}\right)
&\leq&  \nonumber \\
P\left(\sup_{\alpha\in\Lambda}\frac{E\left(Q(z,\alpha)\right)-\nu_{n}(z,\alpha)}{D_{p}(\alpha)}\geq\frac{\epsilon}{4D_{p}(\alpha^{*\emph{}})}\right) \Rightarrow
\\
  P\left(\sup_{\alpha\in\Lambda}\left(E\left(Q(z,\alpha)\right)-\nu_{n}(z,\alpha)\right)\geq\frac{\epsilon}{4}\right)
&\leq& \nonumber \\
 P\left(\sup_{\alpha\in\Lambda}\frac{E\left(Q(z,\alpha)\right)-\nu_{n}(z,\alpha)}{D_{p}(\alpha)}\geq\frac{\epsilon}{4D_{p}(\alpha^{*})}\right)
\end{eqnarray*}
Letting $\delta = \frac{\epsilon }{4D_{p}(\alpha^{*})}$ and using Theorem \ref{Theo1}, Clause 3, the last inequality gives
\begin{align}\label{useclauseIII}
  P\left(\sup_{\alpha\in\Lambda} \left(E\left(Q(z,\alpha)\right)-\nu_{n}(z,\alpha)\right)\geq\frac{\epsilon}{4} \right) & \leq 4\exp\left\{\left(\frac{H_{ann}^{\Lambda,\mathbb{\beta}}(n)}{n^{2-\frac{2}{p}}}-\frac{\delta^{2}}{2^{1+\frac{2}{p}}}\right)n^{2-\frac{2}{p}}\right\}.
\end{align}
Using \eqref{useclauseIII}, both terms on the right in \eqref{posneg} can be bounded.  This gives
\begin{align}
\label{8factor}
  P\left(\sup_{\alpha\in\Lambda}\left| \nu_{n}(z,\alpha)-E\left(Q(z,\alpha)\right)\right| \geq\frac{\epsilon}{2}\right) & \leq 8\exp\left\{\left(\frac{H_{ann}^{\Lambda,\mathbb{\beta}}(n)}{n^{2-\frac{2}{p}}}-\frac{\delta^{2}}{2^{1+\frac{2}{p}}}\right)n^{2-\frac{2}{p}}\right\}.
\end{align}
\noindent
The presence of $\mathcal{\beta}$ in the the exponent of the annealed entropy will not change the validity of Clause I of  Theorem \ref{Theo1}, so we have,
$
H_{ann}^{\Lambda,\mathbb{\beta}}(n) \leq G^{\Lambda}(n)\leq \ln\left(\frac{en}{h}\right)^{h}
$
therefore,
$
\exp\left(H_{ann}^{\Lambda, \mathbb{\delta}}(n)\right)\leq \exp\left(G(n)\right)\leq \left(\frac{en}{h}\right)^{h}.
$
Using this in \eqref{8factor} gives
\begin{align*}
  P\left(\sup_{\alpha\in\Lambda}\left( \nu_{n}(z,\alpha)-E\left(Q(z,\alpha)\right)\right) \geq\frac{\epsilon}{2}\right) & \leq 8\left(\frac{ne}{h}\right)^{h}\exp\left\{-\left(\frac{\delta n^{1-\frac{1}{p}}}{2^{0.5+\frac{1}{p}}}\right)^{2}\right\}
\nonumber\\
  & = 8\left(\frac{ne}{h}\right)^{h}\exp\left\{-\left(\frac{\epsilon n^{1-\frac{1}{p}}}{D_{p}(\alpha^{*})2^{2.5+\frac{1}{p}}}\right)^{2}\right\}.
\end{align*}
Recalling the extra factor of 2 in  Inequation \eqref{LHSLemma1} gives the statement of the Proposition.
\end{proof}
\begin{thm}
  \label{Theo3}
  Let $h = VCDim\left\{Q\left(\cdot,\alpha\right):\alpha\in \Lambda\right\}$ where $Q\left(\cdot,\alpha\right)$ is unbounded, $n \in \mathcal{N}$ is the sample size. If $h \le \infty$, and
  $$
  D_{p}\left(\alpha\right) = \int_{0}^{\infty}\sqrt[p]{P\left\{Q\left(z,\alpha\right)\geq c\right\}}dc \leq \infty
  $$ where $1< p\leq 2$ is some fixed parameter, we have
\begin{equation}
  E\left(\sup_{\alpha\in\Lambda}\left|\nu_{1}(Z_{1},\alpha)-\nu_{2}(Z_{2},\alpha)\right|\right) \leq \frac{D_{p}(\alpha^{*})2^{2.5+\frac{1}{p}}\sqrt{h\ln\left(\frac{ne}{h}\right)}}{n^{1-\frac{1}{p}}} + \frac{16D_{p}(\alpha^{*})2^{2.5+\frac{1}{p}}}{n^{1-\frac{1}{p}}\sqrt{h\ln\left(\frac{ne}{h}\right)}}.
\label{Hannbd234}
\end{equation}
\end{thm}

\begin{proof}
Using the integral of probabilities identity, the left-hand side of \eqref{Hannbd234} equals
\begin{eqnarray*}
&&  \int_{0}^{\infty}P\left(\sup_{\alpha\in\Lambda}\left|\nu_{1}(Z_{1},\alpha)-\nu_{2}(Z_{2},\alpha)\right|\geq\epsilon\right) \emph{d}\epsilon\nonumber\\
   &\leq& \int_{0}^{\infty}16\left(\frac{ne}{h}\right)^{h}\exp\left\{-\left(\frac{\epsilon n^{1-\frac{1}{p}}}{D_{p}(\alpha^{*})2^{2.5+\frac{1}{p}}}\right)^{2}\right\}\emph{d}\epsilon \nonumber\\
   &\equiv& \int_{0}^{u}\emph{d}\epsilon +16\left(\frac{ne}{h}\right)^{h}\int_{u}^{\infty}\exp\left\{-\left(\frac{\epsilon n^{1-\frac{1}{p}}}{D_{p}(\alpha^{*})2^{2.5+\frac{1}{p}}}\right)^{2}\epsilon^{2}\right\}\emph{d}\epsilon \nonumber\\
   &\leq& u + 16\left(\frac{ne}{h}\right)^{h}\int_{u}^{\infty}\exp\left\{-\left(\frac{n^{1-\frac{1}{p}}}{D_{p}(\alpha^{*})2^{2.5+\frac{1}{p}}}\right)^{2}u\epsilon\right\}\emph{d}\epsilon \nonumber\\
   &=&  u  + \left(\frac{D_{p}(\alpha^{*})2^{2.5+\frac{1}{p}}}{n^{1-\frac{1}{p}}}\right)^{2}\frac{16\left(\frac{ne}{h}\right)^{h}}{u}
\exp\left\{-\left(\frac{un^{1-\frac{1}{p}}}{D_{p}(\alpha^{*})2^{2.5+\frac{1}{p}}}\right)^{2}\right\}.\\
\end{eqnarray*}
Choosing
$$
u = D_{p}(\alpha^{*})2^{2.5+\frac{1}{p}}\frac{\sqrt{h\ln\left(\frac{ne}{h}\right)}}{n^{1-\frac{1}{p}}}
$$
gives the statement of the Theorem.
\end{proof}
\noindent
Finally, we summarize the results of  Theorem  \ref{Theo2},  \hbox{and}~ \ref{Theo3}.  It identifies the conditions
under which the tightest upper bound we are able to derive for the expected supremal difference holds.
\begin{thm}
  \label{Theo5}
Assume that $h \rightarrow \infty$, $\frac{n}{h} \rightarrow \infty$, $m \rightarrow \infty$, $\ln\left(m\right)=o(n)$,
$$
D_{p}\left(\alpha\right) = \int_{0}^{\infty}\sqrt[p]{P\left\{Q(z,\alpha)\ge c\right\}} dc\leq \infty,
$$
here $p=2$, we have that
\begin{equation}
   E\left(\sup_{\alpha\in\Lambda}\left|\nu_{1}(Z_{1},\alpha)-\nu_{2}(Z_{2},\alpha)\right|\right) \leq \min\left(1,8D_{p}(\alpha^{*})\right)\sqrt{\frac{h}{n}\ln\left(\frac{2ne}{h}\right)} .
\label{objfn}
\end{equation}
\end{thm}

\begin{proof}
 Theorems \ref{Theo2}~ \hbox{and}~ \ref{Theo3} give
\begin{eqnarray}\label{E16}
  E\left(\sup_{\alpha\in\Lambda}\left|\nu_{1}^{*}\left(Z_{1},\alpha,m\right)-\nu_{2}^{*}\left(Z_{2},\alpha,m\right)\right|\right) &\leq& \sqrt{\frac{\ln(2m)}{n} + \frac{h}{n}\ln\left(\frac{2ne}{h}\right)} \nonumber\\
   &\overset{\infty}{=}& \sqrt{\frac{h}{n}\ln\left(\frac{2ne}{h}\right)},
\end{eqnarray}
where $\overset{\infty}{=}$ indicates a limit as $n \rightarrow \infty$ has been taken.
Similarly, Theorem \ref{Theo3} gives
\begin{equation}\label{E17}
  E\left(\sup_{\alpha\in\Lambda}\left|\nu_{1}\left(Z_{1},\alpha\right)-\nu_{2}\left(Z_{2},\alpha\right)\right)|\right)\leq 8D_{p}\left(\alpha^{*}\right)\sqrt{\frac{h}{n}+\frac{h}{n}\ln\left(\frac{2ne}{h}\right)}+E(n,h)
\end{equation}
where $E(n,h)$ is of smaller order than the first term on the right in \ref{E17}. Thus,
\begin{equation}\label{OVERALLBOUND}
 E\left(\sup_{\alpha\in\Lambda}\left|\nu_{1}\left(Z_{1},\alpha\right)-\nu_{2}\left(Z_{2},\alpha\right)\right)|\right)\leq 8D_{p}\left(\alpha^{*}\right)\sqrt{\frac{h}{n}\ln\left(\frac{2ne}{h}\right)}.
\end{equation}
Taking the minimum over the RHS of equations \eqref{E16} and \eqref{OVERALLBOUND} gives the Theorem.
\end{proof}
\noindent
In Chap. \ref{chap:Numerical:Studies}, we will see that this more or less direct extension of Vapnik's bounds performs badly because, typically the bound is far too loose. This happens because the error criterion is too "small". So, in the next sub-section we change the error criterion to make it 'bigger' while preserving the structure of the derivation of the upper bounds in Proposition \ref{Prop2}, and in Theorem \ref{Theo2}, \ref{Theo3} and \ref{Theo5}. This makes the upper bound tighter and as we will see in Chap. \ref{chap:Numerical:Studies}, we can optimize over the constant factor  in the upper bound of Theorem \ref{Theo5} to get a better estimate for $h$.


\section{Change in the Expected Maximum Deviation Quantity}
\label{ChangeExpect}
The numerical effect of the change in the calculation of EMDBTEL will be observed in Sec. 
. In this section, we will restate all results using a cross-validation form of error. Here, we will present only the main changes. In fact, the computation of the maximum difference is done with the idea of cross-validation in mind. Thus, the losses whose difference we take are found using data that was not used to build the model. Moreover, instead of using one model, we use two, one for each half sample. The two models are identical in that they have the same covariates.


The stream of reasoning used to develop the results in the previous section will remain the same here. However, there are some changes in detail that will be noted.

Here we identify main changes so that the remainder of the proof will follow.
Let $Z = (x,y)$ be a pair of observations and write $Z^{1} = \left(z_{1},z_{2},\cdots, z_{n}\right)$ and $Z^{2}= \left(z_{n+1},z_{n+2},\cdots,z_{2n}\right)$  two vectors of $n$ independent and identically distributed (IID) copies of Z. Let
$$
Q_{1}\left(z_{1},\alpha_{1}\right)= L\left(y,f\left(x,\alpha_{1}\right)\right)\quad \hbox{and}\quad Q_{2}\left(z_{2},\alpha_{2}\right)= L\left(y,f\left(x,\alpha_{2}\right)\right)
$$
be two bounded real valued loss functions where $\alpha_{i}\in\Lambda, ~ i=1,2$, an index set, and assume $\forall ~ \alpha_{i}, ~ 0\leq Q_{i}\left(z_{i},\alpha_{i}\right)\leq B_{i}$  for some $B_{i}\in \mathbb{R}$ , $i=1,2$.
 Consider the discretization of $Q_{i}$ using $m$ disjoint intervals (with union $[0; B_{i})$) given by
 \begin{equation}
 \label{E11}
 Q^{*}_{ij}\left(z_{i},\alpha_{i},m\right)=\left\{
  \begin{array}{ll}
    \frac{(2j+1)B_{i}}{2m}, & \hbox{if $Q_{i}\left(z,\alpha\right)$ $\in~I_{j}=\left[\frac{jB_{i}}{m},\frac{(j+1)B_{i}}{m}\right),~i=1,2$;} \\
    0 & \hbox{otherwise.}
  \end{array}
\right.
 \end{equation}
Where $j = 0, 1,\cdots, m-1$.
Now, consider indicator functions for $Q_{i}$ being in an interval of the same form.  That is, let
\begin{equation}
\label{E22}
  \chi_{I_{ij}}(Q_{i}\left(z,\alpha,m\right)) = \left\{
               \begin{array}{ll}
                 1, & \hbox{if $Q_{i}\left(z,\alpha\right)\in \emph{I}_{j}=\left[\frac{jB_{i}}{m},\frac{(j+1)B_{i}}{m}\right),~i=1,2$;} \\
                 0, & \hbox{otherwise.}
               \end{array}
             \right.
\end{equation}
and write
$$
n_{1j}^{*} = \sum_{i=1}^{n}\chi_{I_{1j}}\left(Q_{1}(z_{i}^{2},\alpha_{1},m)\right)
\quad \hbox{and} \quad
n_{2j}^{*} = \sum_{i=n+1}^{2n}\chi_{I_{2j}}\left(Q_{2}(z_{i}^{1},\alpha_{2},m)\right)
$$
for the number of data points whose losses land inside the interval $I_{j}$ in the first and second
half of the sample of size $n$, respectively.
The empirical loss for each model will now be written as follows:
$$
\nu_{1j}^{*}\left(z^{2}, \alpha_{1}, m\right)= \frac{n^{*}_{2j}Q^{*}_{2j}\left(z^{2}, \alpha_{1}, m\right)}{n},~
\nu_{2j}^{*}\left(z^{1}, \alpha_{2}, m\right)= \frac{n^{*}_{1j}Q^{*}_{1j}\left(z^{1}, \alpha_{2}, m\right)}{n}.
$$
We see that in the computation of the empirical loss, we use datasets that were not used to fit the model.
The event $A_{\epsilon,m}$ will  now be defined as
\begin{equation}
\label{NewAepsilon}
A_{\epsilon,m} = \left\{z^{2n}: \sup_{\alpha_{1},\alpha_{2}}\left(\nu_{1}^{*}\left(z^{2},\alpha_{1}\right)-\nu_{2}^{*}\left(z^{1},\alpha_{2}\right)\right)\geq\epsilon\right\}
\end{equation}
$$
\hbox{where}~\nu^{*}_{1}\left(z^2,\alpha_{1}\right)=\sum_{j=0}^{m-1}\nu_{1j}^{*}\left(z^{2},\alpha_{1},m\right),~
\nu^{*}_{2}\left(z^1,\alpha_{2}\right)=\sum_{j=0}^{m-1}\nu_{2j}^{*}\left(z^{1},\alpha_{2},m\right).
$$

Since $A_{\epsilon}$ is defined on the entire range of our loss function, and we want to partition the range into $m$ disjoint intervals, let
\begin{eqnarray*}
   A_{\epsilon,m} &=& \left\{z^{2n}~|\sup_{\alpha_{1}, ~\alpha_{2}\in \Lambda}\left\{\sum_{j=0}^{m-1}\nu_{1j}^{*}\left(z^{2},\alpha_{1}\right)-\sum_{j=0}^{m-1}\nu_{2j}^{*}\left(z^{1},\alpha_{2}\right)\right\}\ge\epsilon\right\}\\
   &\subseteq& \left\{z^{2n}~|~ \exists~ j : \sup_{\alpha_{1},~\alpha_{2}\in ~ \Lambda}\left(\nu_{1j}^{*}\left(z^{2}, \alpha_{2},m\right)-\nu_{2j}^{*}\left(z^{1},\alpha_{2},m\right)\right)\ge\frac{\epsilon}{m}\right\}\\
   &\subseteq& \bigcup_{j=0}^{m-1}\left\{z^{2n}~|~ \sup_{\alpha_{1},~\alpha_{2}~\in\Lambda}\left(\nu_{1j}^{*}\left(z^{2}, \alpha_{1},m\right)-\nu_{2j}^{*}\left(z^{1}, \alpha_{2},m\right)\right)\ge\frac{\epsilon}{m}\right\}\\
   &\subseteq& \bigcup_{j=0}^{m-1} A_{\epsilon,m,j}
\end{eqnarray*}
where $ A_{\epsilon,m,j} = \left\{z^{2n}~|~ \sup_{\alpha_{1},~\alpha_{2}~\in\Lambda}\left(\nu_{1j}^{*}\left(z^{2}, \alpha_{1},m\right)-\nu_{2j}^{*}\left(z^{1}, \alpha_{2},m\right)\right)\ge\frac{\epsilon}{m}\right\}$.
The suprema over $\Lambda$ within $A_{\epsilon,m,j}$ will be achieved at
$$
\alpha_{j}^{*} = \alpha_{j}^{*}\left(z^{2n}\right) = \arg\sup_{\alpha_{1},~\alpha_{2}}\left(\nu_{1j}^{*}\left(z^{2},\alpha_{1},m\right)-\nu_{2j}^{*}\left(z^{1},\alpha_{2},m\right)\right).
$$
Every development done in page \pageref{permid} will stay unchanged. Using the cross-validation form of the error and using two different models will not affect  Theorem  \ref{Theo1}. However, since the use of these changes are intense in the proof of  Proposition \ref{Prop2}, we will restate it and redo the proof taking into account these changes.
\begin{proposition}
  \label{NewProp2}
  Let $\epsilon \ge 0,\quad m \in  \mathcal{N}, \hbox{and},\quad h = VCD \left\{Q\left(\cdot,\alpha\right): \alpha \in  \Lambda\right\}$. If $h$ is finite, then
  \begin{equation}
  \label{NewUPBO}
  P\left(A_{\epsilon,m}\right) \leq 2m\left(\frac{2ne}{h}\right)^h\exp\left\{-\frac{n\epsilon^2}{m^{2}}\right\}.
  \end{equation}
\end{proposition}
This proposition is very similar in spirit to Proposition \ref{Prop2}. The difference is mainly in the setting. This being so, in our proof of  Proposition \ref{NewProp2}, we will emphasize the steps where changes to the argument are important.

\begin{proof}
  Let $\Delta^{*}_{j}\left(T_{i}Z^{2n},\alpha_{j}^{*},m\right)=\nu_{1j}^{*}\left(T_{i}Z_{2},\alpha_{1j}^{*},m\right)-\nu_{2j}^{*}\left(T_{i}Z_{1},\alpha_{2j}^{*},m\right)$, where $\alpha_{j} = \left(\alpha_{1j},~\alpha_{2j}\right)$ and $\alpha_{j}^{*} = \left\{\alpha_{1j},~\alpha_{2j}: \arg\max\Delta_{j}\left(T_{i}Z,\alpha_{j},m\right)\right\}$. Using some manipulations, we have
  \begin{eqnarray}
  \label{UpAep}
  P\left(A_{\epsilon,m}\right) &\leq& P\left(\bigcup_{j=0}^{m-1}A_{\epsilon,m,j}\right)\nonumber\\
   &\leq& \sum_{j=0}^{m-1}P\left(A_{\epsilon,m,j}\right)\nonumber\\
   &=& \sum_{j=0}^{m-1}P\left(\left\{Z^{2n}:\sup_{\alpha_{1},~\alpha_{2}\in\Lambda}\left(\nu_{1j}^{*}(Z_{2},\alpha_{1},m)-\nu_{2j}^{*}(Z_{1},\alpha_{2},m)\right)\geq\frac{\epsilon}{m}\right\}\right).
\nonumber
\end{eqnarray}
Continuing the equality gives that the RHS equals
\begin{eqnarray}
\sum_{j=0}^{m-1}P\left(\left\{Z^{2n}:\sup_{\alpha_{1},~ \alpha_{2}\in\Lambda}\left(\left(\nu_{1j}^{*}(T_{i}Z_{2},\alpha_{1},m)-\nu_{2j}^{*}(T_{i}Z_{1},\alpha_{2},m)\right)\right)\geq\frac{\epsilon}{m}\right\}\right) &=&\nonumber\\ \sum_{j=0}^{m-1}P\left(\left\{Z^{2n}:\left(\nu_{1j}^{*}(T_{i}Z_{2},\alpha_{1j}^{*},m)-\nu_{2j}^{*}(T_{i}Z_{1},\alpha_{2j}^{*},m)\right)\geq\frac{\epsilon}{m}\right\}\right)
\nonumber
\end{eqnarray}
\begin{eqnarray}
   &=& \sum_{j=0}^{m-1}P\left(\left\{Z^{2n}: \Delta_{j}^{*}\left(T_{i}Z,\alpha_{j}^{*},m\right)\geq\frac{\epsilon}{m}\right\}\right)\nonumber\\
   &=& \frac{1}{(2n)!}\sum_{j=0}^{m-1}\sum_{i=1}^{(2n)!}P\left(\left\{Z^{2n}: \Delta_{j}^{*}\left(T_{i}Z,\alpha_{j}^{*},m\right)\geq\frac{\epsilon}{m}\right\}\right)
\nonumber \\
   &=& \frac{1}{(2n)!}\sum_{j=0}^{m-1}\sum_{i=1}^{(2n)!}\int_{}^{}\emph{I}_{\left\{Z^{2n}: \Delta_{j}^{*}\left(T_{i}Z,\alpha_{j}^{*},m\right)\geq\frac{\epsilon}{m}\right\}}\left(z^{2n}\right)\emph{dP}(Z^{2n}).
\label{permprobbdd}
\end{eqnarray}
Let the equivalence classes in $\Lambda$ under $\sim$ be denoted $\Lambda_k$.
Then, the equivalence classes $[\alpha_{jk}^{*}]$ for the  $\alpha_{jk}^{*}$'s
provide a partition for $\Lambda$.  That is,
$\Lambda = \bigcup_{k=1}^{N_{j}^{\Lambda}(z^{2n})}[\alpha_{jk}^{*}]$
because $\alpha^{*}_{jk}\in \Lambda_k$ and hence $[\alpha_{jk}^{*}]=\Lambda_{k}$.
In addition, $\alpha^{*}_{jk}$ is the maximal value of $\alpha_{1j}$ and $ \alpha_{2j}$ in the $k^{th}$ equivalence class.
So,
\begin{eqnarray}
\label{Indfnbdd}
  \emph{I}_{\left\{Z^{2n}:\Delta\left(T_{i}Z,\alpha^{*}_{j},m\right)\geq\epsilon\right\}}(z^{2n})
&\leq&
\emph{I}_{\left\{Z^{2n}: \Delta_{j}^{*}\left(T_{i}Z,\alpha^{*}_{1j},m\right)\ge \frac{\epsilon}{m}\right\}}\left(z^{2n}\right) \nonumber \\
 &+&
\dots + \emph{I}_{\left\{Z^{2n}:\Delta_{j}^{*}\left(T_{i}Z,\alpha^{*}_{N^{\Lambda}_{j}(z^{2n})j},m\right)\ge\frac{\epsilon}{m}\right\}}\left(z^{2n}\right) \nonumber\\
   &=& \sum_{k=1}^{N^{\Lambda}_{j}(z^{2n})}\emph{I}_{\left\{Z^{2n}: \Delta_{j}^{*}\left(T_{i}Z,\alpha^{*}_{kj},m\right)\ge\frac{\epsilon}{m}\right\}}\left(z^{2n}\right)
\end{eqnarray}
where
\begin{eqnarray*}
  A_{\epsilon,m,j,k} &=& \left\{Z^{2n}: \Delta_{j}^{*}\left(T_{i}Z,\alpha^{*}_{kj},m\right)\ge\frac{\epsilon}{m}\right\} \\
   &=& \left\{Z^{2n}: \sup_{\alpha_{1j}, \alpha_{2j}\in\Lambda_{k}}\left(\nu_{1j}^{*}(T_{i}Z_{2},\alpha_{1j},m)-\nu_{2j}^{*}(T_{i}Z_{1},\alpha_{2j},m)\right)\ge\frac{\epsilon}{m}\right\}
\end{eqnarray*}

Now, using \eqref{Indfnbdd}, \eqref{permprobbdd} is bounded by
\begin{eqnarray}
  P\left(A_{\epsilon,m}\right) &\leq& \frac{1}{(2n)!}\sum_{j=0}^{m-1}\sum_{i=1}^{(2n)!}\int_{}^{}\sum_{k=1}^{N_{j}^{\Lambda}(z^{2n})}\emph{I}_{\left\{Z^{2n}:\Delta_{j}^{*}(T_{i}Z,\alpha^{*}_{kj},m)\ge\frac{\epsilon}{m}\right\}}\left(z^{2n}\right) \emph{dP}\left(z^{2n}\right)
\nonumber\\
   &=& \int_{}^{}\sum_{j=0}^{m-1}\sum_{k=1}^{N_{j}^{\Lambda}(z^{2n})}\left[\frac{1}{(2n)!}\sum_{i=1}^{(2n)!}\emph{I}_{\left\{Z^{2n}:\Delta_{j}^{*}(T_{i}Z,\alpha_{kj}^{*},m)\geq \frac{\epsilon}{m}\right\}}\left(z^{2n}\right)\right]\emph{dP}\left(z^{2n}\right) .
\label{overallprobbdd}
\end{eqnarray}
The expression in square brackets in \eqref{overallprobbdd} is the fraction of the number of the $(2n)!$ permutations $T_{i}$ of $Z^{2n}$ for which $A_{\epsilon,m,j,k}$ is closed under $T_{i}$ for any fixed equivalence class $\Lambda_{k}$. It is equal to
\begin{equation}
\label{E77}
  \Gamma_{j} = \sum_{k}^{}\frac{\binom{m^{*}_{j}}{k} \binom{2n-m^{*}_{j}}{m^{*}_{j}-k}}{\binom{2n}{n}}
\nonumber
\end{equation}
where
$$
\left\{
  \begin{array}{ll}
    \left\{k: \left|\frac{k}{n}-\frac{m^{*}_{j}-k}{n}\right|\ge\frac{\epsilon}{m}\right\} & \hbox{;} \\
    m^{*}_{j} = n_{1j}^{*} + n_{2j}^{*} & \hbox{.}
  \end{array}
\right.
$$
Here, $\Gamma_{j}$ is the probability of choosing exactly $k$ sample data points whose losses fall in interval $I_{j}$ respectively in the first and second half of the sample such that $A_{\epsilon,m,j,k}$ holds. $m_{j}^{*}$ is the number of data points from the first and the second half of the sample whose loss land inside interval $j$.
Using  Theorem \ref{Theo1} Clause II, we have $\Gamma_{j} \le 2\exp\left(-\frac{n\epsilon^{2}}{m^2}\right)$. So, using this in \ref{overallprobbdd} gives $P\left(A_{\epsilon, m}\right)$ is upper bounded by
\begin{eqnarray}
   && \int_{}^{}\sum_{j=0}^{m-1}\sum_{k=1}^{N_{j}^{\Lambda}(z^{2n})}2\exp\left(-\frac{n\epsilon^{2}}{m^2}\right)\emph{dP}(z^{2n})
   = 2\exp\left(-\frac{n\epsilon^{2}}{m^2}\right)\int_{}^{}\sum_{j=0}^{m-1}\sum_{k=1}^{N_{j}^{\Lambda}(z^{2n})}\emph{dP}(z^{2n})
\nonumber \\
   &=& 2\exp\left(-\frac{n\epsilon^{2}}{m^2}\right)\sum_{j=0}^{m-1}\int_{}^{}\sum_{k=1}^{N_{j}^{\Lambda}(z^{2n})}\emph{dP}(z^{2n})
   = 2\exp\left(-\frac{n\epsilon^{2}}{m^2}\right)\sum_{j=0}^{m-1}\int_{}^{}N_{j}^{\Lambda}(z^{2n})\emph{dP}(z^{2n})
\nonumber \\
   &=& 2\exp\left(-\frac{n\epsilon^{2}}{m^2}\right)\sum_{j=0}^{m-1}E\left(N_{j}^{\Lambda}(z^{2n})\right).
\label{probbddd}
\end{eqnarray}
Since Theorem  \ref{Theo1} Clause I gives
$$
H_{ann}(Z^{2n}) = \ln\left(E\left(N_{j}^{\Lambda}(Z^{2n})\right)\right)\leq G(2n)\leq h\ln\left(\frac{2ne}{h}\right) \Rightarrow
 E\left(N_{j}^{\Lambda}(z^{2n})\right)\leq \left(\frac{2ne}{h}\right)^{h}.
$$
Using this $m$ times in \eqref{probbddd} gives the  Proposition.
\end{proof}
\noindent
Because changes mentioned in the computation of the EMDBTEL do not affect Propositions \ref{subProp1}, \ref{subProp2} and Theorems \ref{Theo2}, \ref{Theo3}, and \ref{Theo5}, we restate them here for convenience.
\begin{proposition}
\label{Prop1and2}
  \begin{enumerate}
    \item With probability $1-\eta$, the inequality
    \begin{equation}
    R\left(\alpha_{k}\right)\leq R_{emp}\left(\alpha_{k}\right) + m\sqrt{\frac{1}{n}\log\left(\left(\frac{2m}{\eta}\right)\left(\frac{2ne}{h}\right)^{h}\right)}
    \end{equation}
    holds simultaneously for all $K$ functions in the set $\left\{Q\left(z, \alpha_{k}\right),k = 1,2,\cdots, K\right\}$.
    \item With probability $1-\eta$, the inequality
    \begin{equation}
    R\left(\alpha_{k}\right) \leq R_{emp}\left(\alpha_{k}\right) + \frac{m^2}{2n}\log\left(\frac{2m}{\eta}\left(\frac{2ne}{h}\right)^h\right)\left(1 + \sqrt{1 + \frac{4nR_{emp}\left(\alpha_{k}\right)}{m^{2}\log\left(\frac{2m}{\eta}\left(\frac{2ne}{h}\right)^h\right)}}\right)
    \end{equation}
    holds simultaneously for all $K$ functions in the set $\left\{Q\left(z, \alpha_{k}\right),k = 1,2,\cdots, K\right\}$.
  \end{enumerate}
\end{proposition}

\begin{thm}
\label{Theo23and4}
  \begin{enumerate}
    \item If $h< \infty$, we have
    \begin{equation}
     E\left(\sup_{\alpha_{1}, \alpha_{2}\in\Lambda}\left|\nu_{1}^{*}(z^{2},\alpha_{1})-\nu_{2}^{*}(z^{1},\alpha_{2})\right|\right) \leq m\sqrt{\frac{1}{n}\ln\left(2m^{3}\left(\frac{2ne}{h}\right)^h\right)}\\ + \frac{1}{m\sqrt{n\ln\left(2m^3\left(\frac{2ne}{h}\right)^{h}\right)}}
    \end{equation}
    \item If $h \le \infty$, and
  $$
  D_{p}\left(\alpha\right) = \int_{0}^{\infty}\sqrt[p]{P\left\{Q\left(z,\alpha\right)\geq c\right\}}dc \leq \infty
  $$ where $1< p\leq 2$ is some fixed parameter, we have
    \begin{equation}
      E\left(\sup_{\alpha_{1},\alpha_{2}\in\Lambda}\left|\nu_{1}(Z_{2},\alpha_{1})-\nu_{2}(Z_{1},\alpha_{2})\right|\right) \leq \frac{D_{p}(\alpha^{*})2^{2.5+\frac{1}{p}}\sqrt{h\ln\left(\frac{ne}{h}\right)}}{n^{1-\frac{1}{p}}} + \frac{16D_{p}(\alpha^{*})2^{2.5+\frac{1}{p}}}{n^{1-\frac{1}{p}}\sqrt{h\ln\left(\frac{ne}{h}\right)}}.
    \label{Hannbd}
    \end{equation}
    \item Assume that $h \rightarrow \infty$, $\frac{n}{h} \rightarrow \infty$, $m \rightarrow \infty$, $\ln\left(m\right)=o(n)$, \hbox{and}
$$
D_{p}\left(\alpha\right) = \int_{0}^{\infty}\sqrt[p]{P\left\{Q(z,\alpha)\ge c\right\}} dc\leq \infty
$$
where $p=2$. Then we have that
\begin{equation}
   E\left(\sup_{\alpha_{1}, \alpha_{2}\in\Lambda}\left|\nu_{1}(Z_{2},\alpha_{1})-\nu_{2}(Z_{1},\alpha_{2})\right|\right) \leq \min\left(1,8D_{p}(\alpha^{*})\right)\sqrt{\frac{h}{n}\ln\left(\frac{2ne}{h}\right)} .
\label{objfn2}
\end{equation}

  \end{enumerate}
\end{thm}
\begin{figure}%
    \centering
    \includegraphics[width=10cm, height=10cm]{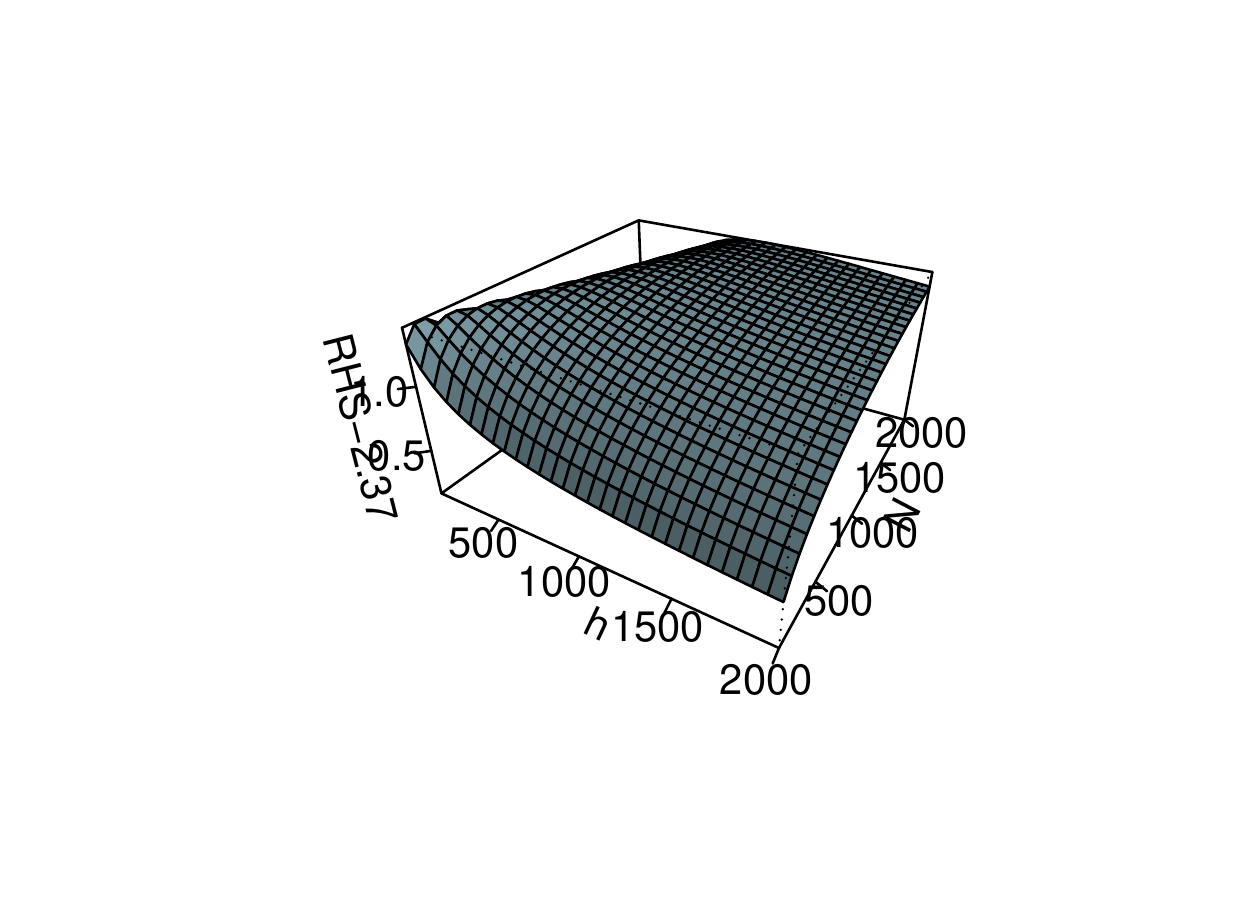}
    \caption{Perspective Plot of the RHS of equation \eqref{objfn2}}
    \label{perps}%
\end{figure}
The  expression in the RHS of inequality \eqref{objfn2} is a decreasing function of $h$ for fixed $n$ and an increasing function of $n$ for fixed $h$. This observation is confirmed in figure \eqref{perps}
Now that we have all the results that we want, we can estimate the VCD and use it in model selection. We show how to do this for simulated and real dataset in Chap. \ref{chap:Numerical:Studies}.
\noindent



\chapter{SIMULATION STUDIES}
\label{chap:Numerical:Studies}
\thispagestyle{fancy}
Our goal here is to implement the theory developed in Chapter \ref{chap:bounds}. We start by giving a definition of what we call `consistency at the true model' in Sec. \ref{DefConcistency}. In Sec. \ref{ImpleTheo5}, we implement Theorem \ref{Theo5} computationally. We develop this implementation in the simulation setting of Subsec. \ref{SimSettings}. The model class we use for simulated data is the class of linear models. In Subsec. \ref{DirectImplementation}, we begin by using the original algorithm of Vapnik et al., \cite{Vapnik:etal:1994} to estimate the LHS of Theorem \ref{Theo5} and see that often it does not give satisfactory results. In particular we will see that the estimates are not accurate. We also see a sense in which Vapnik et al., \cite{Vapnik:etal:1994} original algorithm --Algorithm \ref{VapnikAlgorithm} in Chap. \ref{chap:Vapnik:Cherv:Dim:Cov:Num}-- might not be too bad.

We can do better.

Specifically, we suggest that the inaccuracy in the estimate of VCD using Vapnik's original algorithm \ref{VapnikAlgorithm} arises from insufficient variation in the computation of the Maximum Difference Between Two Empirical Losses and from the fact that $k, a, b$ in equation \eqref{phi} were chosen in sub-optimal ways. To correct for these, we use a data-driven estimate for a constant in the RHS of \eqref{objfn} from Theorem \ref{Theo5}. We also use a cross-validated form of the error, and instead of taking the maximum difference, we take the mean difference between two empirical losses. This is done in  Secs. \ref{ImpleTheo5} and \ref{estimator}. In Sec. \ref{Simulation}, we present cases where our method from Sec. \ref{estimator} works properly. In Subsec. \ref{Dependency}, we look at cases where our method does not work and show how tuning some parameters dramatically improves our estimates.

\section{Definition of `consistency at the true model'}
\label{DefConcistency}
\begin{definition}
\label{Defcon}
We define `consistency at the true model' to be
\begin{enumerate}
    \item consistency at the true model in the usual sense of the term i.e $\hat{h}\rightarrow h_{T}$ in probability coupled with
    \item  some kind of identifiable bad behaviour away from the true model, usually getting worse as the wrong model moves further from the true model.
\end{enumerate}
\end{definition}

We argue that the $2^{nd}$ clause is important because if we are in the unfortunate case of having a wrong model, clause 2 can help us move toward a better model. Moreover, since in practice no models are correct to infinite precision, clause 2 is essential for good modeling.
Our proposed methodology for using estimate $\hat{h}$ of $h_{T}$ is based on this definition. We begin by assuming that there is a true model. The model need not be true to infinite precision, but hopefully it is plausible enough that when we try to discredit it, it will not in fact be discredited.

For any model, we can evaluate the LHS of equation \eqref{objfn} from Theorem \ref{Theo5} by Algorithm \ref{Algo1} below. Then, we can use nonlinear regression in \eqref{objective} to get $\hat{h}$. So, it is seen that $\hat{h}$ is a function of the conjectured model. In principle, for any given model class that we might take as true, the VCD can be found. This is particularly easy for linear models. -- Theorem \ref{VCDLinear} shows that the VCD for linear models is just the number of parameters for a saturated model.

Since our goal is to estimate the true VCD, when a conjectured model $P\left(\cdot\mid\beta\right)$ is linear, we expect $VCD\left(P\left(\cdot\mid\beta\right)\right)\cong\hat{h}$. By the same logic, if $P\left(\cdot\mid\beta\right)$ is far from the true model, we expect $VCD\left(P\left(\cdot\mid\beta\right)\right)\gg
\hat{h}$ or $VCD\left(P\left(\cdot\mid\beta\right)\right)\ll
\hat{h}$. That is, there are conditions that would satisfy
clause 2 of our definition of `consistency at the true model'. Taken together, clause 1 and two of \ref{Defcon} suggest that we estimate $h_{T}$ by seeking
\begin{equation}
\label{Consis}
\hat{h} = \arg\min_{k}\left|VCD\left(P_{k}\left(\cdot\mid\beta\right)\right)- \hat{h}_{k}\right|
\end{equation}
where $\left\{P_{k}\left(\cdot\mid\beta\right)|k = 1,2,\cdots, K\right\}$ is some set of models and $\hat{h}_{k}$ is calculated using model $k$.
In the case of linear model, with $q = 1,2,\cdots, Q$ explanatory variables, we get
\begin{equation}
\label{eqcon}
\hat{h} = \arg\min_{q}\left|q - \hat{h}_{q}\right|
\end{equation}
where $\hat{h}_{q}$ is the estimated VCD for model of size $q$.
Note that \eqref{Consis} can identify a good model even when consistency at the true model fails. The reason is that \eqref{Consis} only requires a minimum at the VCD not convergence to the true VCD. That is, \eqref{Consis} is a slightly weaker condition than consistency at the true model in that clause 1 need not be satisfied. Indeed, clause 1 can be regarded as what we what, while \eqref{Consis} is what we can find. Both criteria narrow the search for the true model.

It would be seen in the simulation that, often, clause 2 is not perfectly satisfied. For instance there may be two values of $h$ which achieve \eqref{Consis} at least approximately. For instance if the true value is $p=30$, there may be a larger value $p$ (see the four panels in Fig. \ref{diffseed30}). In these cases, usually the smallest value of $\hat{h}$ that approximately achieves the $\arg\min$ is usually the best choice. This means that from a strictly pragmatic standpoint, imposing some level of parsimony on the model selection is helpful
-- whence the introduction of a threshold in \eqref{Consis} or \eqref{eqcon} as used in \eqref{treshold}.
For ease of exposition, we usually drop the explicit use of the threshold $t$. However, in practice, we must
remember that taking a precise $\arg\min$ is sub-optimal compared to allowing a little flexibility via $t$.

\section{Implementation of Theorem \ref{Theo5}}
\label{ImpleTheo5}
Before presenting our simulations, let's first re-examine Theorem  \ref{Theo5}. In fact, Theorem \ref{Theo5}  has two parts. The Left-Hand Side (LHS) that represents the Expected Maximum Difference Between Two Empirical Losses (EMDBTEL) and the Right-hand side (RHS) that represents the upper bound of the EMDBTEL. The connection between these two quantities is the VCD that we would like to estimate.
The VCD appears explicitly  in the expression of the RHS of  \eqref{objfn} from Theorem  \ref{Theo5}, whereas in the LHS of  \eqref{objfn} from Theorem \ref{Theo5}, the VCD is implicit in the computation. More details on the calculation of the LHS of \eqref{objfn} from Theorem  \ref{Theo5} are given in Sect \ref{estimator}.

Our objective is to find the value of VCD that will minimize the squared distance between the RHS and the LHS of \eqref{objfn} of Theorem \ref{Theo5}.

After many simulations, we observed that the RHS bound of \eqref{objfn} from Theorem \ref{Theo5} was neither tight nor stable. This finding was the same as \cite{McDonald:etal:2011}. So, to make the bound tighter, instead of taking $\min\left(1, D_{p}(\alpha^{*})\right)$ (the constant factor at the RHS of \eqref{objfn} from Theorem \ref{Theo5}), we replace it by $c$, and use our data to find a value of $\hat{c}$ that makes the upper bound as tight as possible. This approach is different from Vapnik et al., \cite{Vapnik:etal:1994} who asserted (without justification) that the constants in the objective function were universal.

Our estimation of $\hat{h}$ has three steps. The first step is to estimate the Expected Maximum Difference Between Two Empirical Losses (EMDBTEL) i.e the RHS of \eqref{objfn} in Theorem \ref{Theo5}. This step requires that we define the model class and hence its complexity in VCD. We also need to choose the number of design points ($N_{L} = n_{1}, n_{2}, \cdots, n_{l}$). Design points represent different sample sizes that will be used to compute each value of the EMDBTEL use in the non-linear regression see \eqref{objective}. They are chosen so they will cover $[0,n]$ roughly uniformly, where $n$ is the sample size. Thus, we have used a uniform grid, and this has worked well in our examples. However, we note that Shao et al. \cite{shao2000measuring} proposed a design procedure for specifying the $N_{L}$'s and the number of repeated experiments at a given sample size. We have not assessed how well their recommendations perform; we leave this as future work.
Also, in this first step, we use the number $m$ of disjoint intervals as defined in Chap. \ref{chap:bounds} to convert our regression problem into $m$ classification problems.

The second step is the calculation of $c$. This step is very important because a good value of $c$ will make the bound tighter and the probability of getting a good estimate $\hat{h}$ will be higher. To find $c$, we define a relatively fine grid of values, e.g. let $c$ range from $0.1$ to $100$, in steps of size $0.01$. For each value of $c$, we compute the squared distance between the RHS and the LHS of  Theorem \ref{Theo5}, and pick the $\hat{c}$ that gives the minimum squared distance. This is possible because we have a predefined model complexity.

The third step is to find $\hat{h}$. Once we have calculated the LHS of Theorem \ref{Theo5} and we have the value of $\hat{c}$, the only unknown value in Theorem \ref{Theo5} is $h$. To estimate the assumed true value $h_{T}$ of $\hat{h}$, we find the value $h$ that minimizes the squared distance between the RHS and the LHS of Theorem \ref{Theo5} by non-linear regression.

The point of the rest of this section is to argue that the algorithm proposed by Vapnik et al. \cite{Vapnik:etal:1994}, i.e Algorithm \ref{VapnikAlgorithm}, to estimate the LHS of \eqref{objfn} from Theorem \ref{Theo5} and thereby estimate the VCD can behave poorly even on simulated data. In fact, as we said before, to estimate the VCD we first need to find a good way to estimate the LHS of inequality \eqref{objfn} in Theorem \ref{Theo5}. For our first set of simulations, using the computational procedure of  Vapnik et al. \cite{Vapnik:etal:1994} i.e Algorithm \ref{VapnikAlgorithm}  EMDBTEL in the regression setting gives minimally adequate results that we can improve. In particular Algorithm \ref{VapnikAlgorithm} does not show variability in the estimate of VCD across a wide range of models.

\subsection{Simulation Settings}
\label{SimSettings}

Here, our simulations are based on linear models, since for these, we know the VCD equals the number of parameters in the model. To establish notation, we write the regression function as a linear combination of the covariates $X_{j}$, $j=0,1,\cdots, p$,
$$
y = f(x,\beta) = \beta_{0} + \beta_{1}x_{1}+\beta_{2}x_{2}+\cdots+\beta_{p}x_{p}=\sum_{j=0}^{p}\beta_{j}x_{j}.
$$
Given a dataset, $\left\{(x_{i},Y_{i},~ i =1,2,\cdots,n)\right\}$, the matrix representation is
$$
Y = X^{'}_{nxp}\beta_{px1} + \epsilon_{nx1} \quad \hbox{where}\quad X^{'}_{nxp} = \left(
                                                                    \begin{array}{ccccc}
                                                                      1 & x_{11} & x_{21} & \cdots & x_{p1} \\
                                                                      1 & x_{12} & x_{22} & \cdots & x_{p2} \\
                                                                      \vdots & \vdots & \vdots & \vdots & \vdots \\
                                                                      1 & x_{1n} & x_{2n} & \cdots & x_{pn} \\
                                                                    \end{array}
                                                                  \right)
\quad \hbox{and}\quad \beta = \left(
                                 \begin{array}{c}
                                   \beta_{0} \\
                                   \beta_{1} \\
                                   \vdots \\
                                   \beta_{p} \\
                                 \end{array}
                               \right)
.
$$
and $\epsilon_{n\times 1}$ is an n-dimension column vector of mean zero.
Now, the least squared estimator $\widehat{\beta}$ is given by
$$
\widehat{\beta} = \left(X^{'}X\right)^{-1}X^{'}Y.
$$
Our simulated data is analogous. We write
$$
Y = \beta_{0}x_{0} + \beta_{1}x_{1} + \beta_{2}x_{2} + \cdots + \beta_{p}x_{p} + \epsilon
$$
$$
\quad \hbox{where} ~ \epsilon ~ \sim N(0,\sigma_{\epsilon}=0.4)\quad x_{0} = 1, ~ \beta_{j}\sim N(\mu = 5, \sigma_{\beta} = 3), ~ \hbox{for $j=1,2,\cdots p$ and}
$$
$$
 \quad x_{j} \sim N(\mu=5, \sigma_{x} = 2).
$$ and are all independent
\noindent
We center and scale all our variables including the response. Initially, we use a nested sequence of model lists. If our covariates are highly correlated, what we can do before applying our method is to de-correlated our variables by studentized or sphere the data. If our interest in only on prediction, we will not be interested in VCD, however, we can use VCD to narrow down the class of predictors we have to consider.

\subsection{Direct Extension of the Algorithm in Vapnik et al. \cite{Vapnik:etal:1994} }
\label{DirectImplementation}
The objective here is to argue that direct implementation of Theorem \ref{Theo5} using Algorithm \ref{VapnikAlgorithm} to estimate the LHS of \eqref{objfn} in Theorem \ref{Theo5} often gives poor results on simulated datasets. Later, in Sec. \ref{AnalTour5}, we also show this for the real datasets. Indeed, we were unable to replicate the simulation results given in Vapnik et al. \cite{Vapnik:etal:1994} using their computational technique. To implement  Theorem \ref{Theo5}, we took the constant $\min\left(1, D_{p}\left(\alpha^{*}\right)\right)$ to be $1$, since it does not affect the minimization over $h$.

We implemented simulations for two different model sizes $p = 15\hbox{ and} 60$.
For both model sizes, we implemented two sets of models. The first set of models used only a subset of variables that was used to generate our response. For instance, out of $15$ explanatory variables used to generate our response, our modeling started by using only $9$ of them. We added covariates one at a time with their corresponding parameters. Once we get to $15$ explanatory variables, we were at the true model. So the explanatory variables added after this were decoys and have corresponding true coefficients zero. More formally, we started with models of the form
\begin{equation*}
  y = \beta_{0} + \sum_{j=1}^{l}\beta_{j}x_{j},
\end{equation*}
where $~l \leq p~$ and $~p~$ is the number of parameters used to simulate the data. When we proceeded to include decoys explanatory variables, we used models of the form
\begin{equation*}
  y = \beta_{0} + \sum_{j=1}^{m}\beta_{j}x_{j}, ~\hbox{where, $~m > p,~$ and $\forall~ m > p,~ \beta_{j} = 0$.}
\end{equation*}
Taken together, this gives a simple nested sequence of models in which the correct terms are added first. In practice, we suggest that, any collection of models built using the same explanatory variables can be ordered by any shrinkage method such as ALASSO \cite{efron2004least}, SCAD \cite{Fan&Li}, elastics net \cite{zou2005regularization}, etc. that satisfies the oracle property.

Thus for $p = 15$, we fitted $14$ different models centered at the true model and each time we estimated the VCD. In Table \ref{TableTheorem5set15}, the first model used $9$ out of the $15$ covariates, and we added covariates one at a time up to 15. For the second group of $8$ models, we used all $15$ covariates and adding decoys one at a time. In Table \ref{TableTheorem5set15}, the first model in the second group has the $15$ correct covariates and one decoy. Later models used more decoys. The left column represents the model size used to estimate the VCD. The right column shows the estimate of VCD.

\begin{table}[ht]
    \begin{subtable}[h]{0.4\textwidth}
        \centering
        \begin{tabular}{l | l}
        \hline
  Model size & $\hat{h}$ \\
  \hline\hline
  $9$  & 50 \\
  $10$ & 50 \\
  $11$ & 50 \\
  $12$ & 50 \\
  $13$ & 50 \\
  $14$ & 50 \\
  $15$ & 2 \\
  $16$ & 2 \\
  $17$ & 2 \\
  $18$ & 2 \\
  $19$ & 2 \\
  $20$ & 2 \\
  $21$ & 2 \\
  $22$ & 2 \\
  $23$ & 2 \\
  \hline
\end{tabular}
\caption{Estimates of $\widehat{h}$, $\sigma_{\epsilon}=0.4 $,$~N_{L} = 50, 100, 150, 200, 250, 300, 350 $, $~N = 400 $,$~m=15$, $p=15$. Using Vapnik's algorithm, here Algorithm     \ref{VapnikAlgorithm}}
    \label{TableTheorem5set15}
    \end{subtable}

    \begin{subtable}[h]{0.4\textwidth}
    \centering
    \begin{tabular}{l | l}
  \hline
  Model Size & $\hat{h}$ \\
  \hline\hline
  $53$ & 75 \\
  $54$ & 75 \\
  $55$ & 75 \\
  $56$ & 75 \\
  $57$ & 75 \\
  $58$ & 75 \\
  $59$ & 75 \\
  $60$ & 2 \\
  $61$ & 2 \\
  $62$ & 2 \\
  $63$ & 2 \\
  $64$ & 2 \\
  $65$ & 2 \\
  $66$ & 2 \\
  $67$ & 2 \\
  $68$ & 2 \\
  \hline
\end{tabular}
\caption{Estimation of $\widehat{h},~\sigma=0.4,~ N_{L} = 75, 150, 225, 300, 375, 450, 525, 600$, $~N=600~$, $m=10$, $p=60$}
\label{TableTheorem5set60}

    \end{subtable}
  \caption{Direct implementation of Using Vapnik's algorithm, here Algorithm \ref{VapnikAlgorithm} for $p = 15 $ and 60.}
  \label{VapAlgoP1560}

\end{table}

From Table \ref{TableTheorem5set15}, we observe that, there is no variability in the estimate of $\hat{h}$, except for the large sudden decrease at the true model $p=15$. For the cases where the conjectured model is a subset of the true model, $\hat{h}$ is just equal to the minimum value of the design points and is far from the true value ($h_{T} = 15$). When adding decoy covariates in the model, the estimate that we get is the same ($\hat{h} = 2$) for all model sizes.
We performed the analogous simulation for $p=60$. Results of these simulations are contained in Table \ref{TableTheorem5set60}. The observations are qualitatively the same as for $p = 15$.

From Tables \ref{TableTheorem5set15} and \ref{TableTheorem5set60}, we can also observe that the estimated $\hat{h}$ equals $\min N_{L}$ for subsets of the true model and suddenly drops to a much lower value that is different from $\min N_{L}$ when the conjectured model is exactly equal to the true model and is constant thereafter. Moreover, neither the large nor the small values of $\hat{h}$ are true. Thus, even though we can identify the location of the drop as the point where the true model occurs, $\hat{h}$ is not a direct estimate of the VCD, and so is not consistent at the true model. Otherwise put, $\hat{h}$  merely permits us to locate the true model on a nested list of models, contrary to the claims made by Vapnik et al.,\cite{Vapnik:etal:1994}.

Overall, the implementation of Theorem \ref{Theo5} on synthetic data faces problems such as lack of variability and inaccuracy. The lack of variability merely makes it hard to tell in a complex setting whether the true model is more or less complicated than the conjectured model. That is, clause 2 fails. The inaccuracy is obvious in that estimates of $\hat{h}$ were far from the true number of parameters. On the positive side, we do retain the ability to identify a true model and the true value of $h$ is within the range of values of $\hat{h}$ found by Vapnik et al.'s method.
In an effort to improve the  Vapnik et al. \cite{Vapnik:etal:1994} technique, we consider an alternative formulation in the next section.

\section{An Estimator of the VCD}
\label{estimator}

Recall that from Theorem \ref{Theo5}, the upper bound is
\begin{equation}
\label{UPPTHEO6}
  \Phi_{h}(n)=\min\left(1,D_{p}(\alpha^{*})\right)\sqrt{\frac{h}{n}\log\left(\frac{2ne}{h}\right)}.
\end{equation}
This is slightly different from the form derived in \cite{Vapnik:etal:1994} and studied
in \cite{McDonald:etal:2011}. Moreover, although  $\min\left(1,D_{p}(\alpha^{*})\right)$ does not affect the optimization, it might not be the best constant for the inequality in \eqref{objfn}. So, we replace it with an arbitrary constant $c$ and optimize over it to make our upper bound as tight as possible.
In contrast to Vapnik et al. \cite{Vapnik:etal:1994}, our computation of $c$ is data driven not `universal'. We let $c$ vary from $ 0.01 $ to $100 $ in steps of size $0.01$. However, we have observed in practice, the best value of $\hat{c}$ is usually between $1$ and $8$. The technique that we use to estimate $\hat{h}$ is also different from that in Vapnik et al. \cite{Vapnik:etal:1994}. Indeed, our Algorithm \ref{Algo1} below accurately encapsulates the way the LHS of \eqref{OVERALLBOUND} is formed unlike Algorithm \ref{VapnikAlgorithm}. In particular, we use two large bootstrapping procedures, one as a proxy for calculating expectations and  the second as a proxy for calculating a maximum.
In Vapnik et al.'s algorithm (Algorithm \ref{VapnikAlgorithm}), to account for the maximum variability they have to change the label of the second data. In our case we use cross-validation and bootstrapping.
In fact, we split the data into two groups, using the first group of data, we fitted model I and with the second group of data we fitted model II. To estimate the error from the first model, we compared the predictions from the first model and the second group of data and the other way around for the second model.
We comment that Jun Shao et al. \cite{shao1993linear} suggested in a cross-validation context that as much as data as possible should be reserved for testing rather than fitting.

To explain how we find our estimate of the RHS of \eqref{objfn} from Theorem \ref{Theo5}, we start by replacing the sample size $n$ in \eqref{UPPTHEO6} with a specified value of design point, so that the only unknown is $h$. Thus, formally, we replace \eqref{UPPTHEO6} by
$$
\Phi_{h}^{*}(n_{l}) = \widehat{c}\sqrt{\frac{h}{n_{l}}\log\left(\frac{2n_{l}e}{h}\right)},
$$
where $\hat{c}$ is the optimal data driven constant.
If we knew the LHS of \eqref{OVERALLBOUND} even computationally, we could use it to estimate $h$. However, in general we don't know the LHS of \eqref{OVERALLBOUND}.
Instead, we generate one observation of
\begin{equation}
    \widehat{\xi}(n_{l})=E\left(\sup_{\alpha_{1}, \alpha_{2}\in\Lambda}\left(\nu_{1}^{*}(z_{2},\alpha_{1})-\nu_{2}^{*}(z_{1},\alpha_{2})\right)\right)=\Phi_{h}^{*}(n_{l}) + \epsilon(n_{l})
\label{regmodel}
\end{equation}
for each design point $n_{l}$ by bootstrapping. In \eqref{regmodel}, we assume $\epsilon(n_{l})$ has a mean zero
but an otherwise unknown distribution.  We can therefore obtain a list of values of $\widehat{\xi}(n_{l})$
for the elements of $N_{L}$. To generate $\widehat{\xi}(n_{l})$ for each $n_{l}$ our algorithm uses a cross-validation error defined in step 5. We use our computation procedure for finding the LHS of \eqref{OVERALLBOUND} as in algorithm \ref{Algo1}.

\begin{algorithm}
\SetAlgoLined
\KwResult{An estimate of a vector of values containing $\hat{\xi}(n_{l})$. }
{\bf Inputs:}
 \begin{itemize}
         \item A collection of regression models $\mathcal{G}=\left\{g_{\beta}:\beta\in\beta\right\}$;
         \item A dataset;
         \item Two integers $b_{1}$ and $b_{2}$ for the number of bootstrap samples;
         \item Integer $m$ for the number of disjoint intervals use to discretize the losses;
         \item A set of  design points $N_{L} = \left\{n_{1}, n_{2}, \cdots, n_{l}\right\}$.
       \end{itemize}
\begin{enumerate}
  \item For each $g = 1, 2,\cdots, l$ do;
  \item Take a bootstrap sample of size $2n_{g}$ (with replacement) from our dataset;
  \item Randomly divide the bootstrap data into two groups $~G_{1}~ \hbox{and}~ G_{2}$ of size $n_{g}$ each;
  \item Fit two models one for $G_{1}$ and one for $G_{2}$;
  \item The mean square error of each model is calculated using the covariate andthe  response from the other model, thus: For instance $$ MSE_{1} = \left(predict(Model_{1} ,X_{2})-Y_{2}\right)^2,~ \hbox{and}~ $$
  $$MSE_{2} = \left(predict(Model_{2} ,X_{1})-Y_{1}\right)^2;$$
  \item Discretize the loss function ie discretize $MSE_{1} ~ \hbox{and} ~ MSE_{2}$ into $m$ disjoint intervals;
  \item Estimate $\nu_{1j}^{*}(Z_{2},\alpha_{1},m)$ and $\nu_{2j}^{*}(Z_{1},\alpha_{2},m)$ using $MSE_{1} ~\hbox{and}~ MSE_{2}$ respectively for each interval;
  \item Compute the difference $\left|\nu_{1j}^{*}(z_{2},\alpha_{1},j) - \nu_{2j}^{*}(z_{1},\alpha_{2},j)\right|$
  \item Repeat steps $1-7$ $b_{1}$ times, take the mean intervalwise and sum it across all intervals so we have:
  $$
  \hat{\xi_{i}}\left(n_{l}\right) = \sum_{j=0}^{m-1}mean\left|\nu_{1j}^{*}(z_{2},\alpha_{1},j) - \nu_{2j}^{*}(z_{1},\alpha_{2},j)\right|;
  $$
  \item Repeat step $1-8$ $b_{2}$ times and calculate
   $$
  \hat{\xi}(n_{l})= \frac{1}{b_{2}}\sum_{i=1}^{b_{2}}\hat{\xi}_{i}(n_{l}).
  $$
  \item End for.
\end{enumerate}
 \caption{Generate $\hat{\xi}(n)$'s.}
 \label{Algo1}
\end{algorithm}

Note that this algorithm is parallelizable because different $n_{l}$ can be sent to different nodes to speed up the process of estimating $\hat{\xi}\left(\right)$ for all $n_{l}$.
After obtaining $\hat{\xi}(n_{l})$ for each value of $n_{l}$, we estimate $h_{T}$ by minimizing the squared distance between $\hat{\xi}(n_{l})$ and $\Phi\left(n_{l}\right)$. Our objective function is
\begin{eqnarray}
f_{n_{l}}(h) = \sum_{l=1}^{|N_{L}|}\left(\hat{\xi}(n_{l}) - \hat{c}\sqrt{\frac{h}{n_{l}}\log\left(\frac{2n_{l}e}{h}\right)}\right)^{2}.
\label{objective}
\end{eqnarray}
We comment that optimizing \eqref{objective} does not in general lead to closed form solutions. For instance, differentiating \eqref{objective} with respect to $h$ yields
$$
\frac{\partial f_{n_{l}}(h)}{\partial h}=\sum_{l=1}^{|NL|}\frac{\hat{c}\left(\log\left(\frac{2n_{l}e}{h}\right)-1\right)\left(\hat{c}\sqrt{\frac{h}{n_{l}}\log\left(\frac{2n_{l}e}{h}\right)}-\hat{\xi}(n_{l})\right)}{n_{l}\sqrt{\frac{h}{n_{l}}\log\left(\frac{2n_{l}e}{h}\right)}}
$$
Setting this equal to zero only leads to an implicit expression for $\widehat{h}$ that can only be solved numerically.
Doing this would give our estimate of the VCD. However, for ease of computation, we use grid search on \eqref{objective} to estimate $h$, even though a quadratic optimization might be more efficient.
To evaluate the performance of our $\hat{h}$, we first present some simulation studies using linear models.
\section{Analysis of Synthetic data}
\label{Simulation}
The simulations in this section are based on the setting in Sec. \ref{SimSettings}.
As required for any valid estimation, the technique that we use to estimate the VCD does not require knowledge of the true number of parameters in the model.
In Subsec. \ref{PropoWorks}, we present simulation results to show that our estimator for VCD is consistent for the VCD of the true model. Of course, since our results are only simulations, we do not get clause 1 of consistency at the true model perfectly on all occasions: sometimes our $\hat{h}$ is off by one. (As noted later, this can often be corrected if the sample size is larger.) So we are led to believe that, with enough data our $\hat{h}$ is consistent in the sense of Def. \ref{Defcon}. On the other hand, for the sake of good prediction, it may sometimes be good to add one or two explanatory variables beyond what a model selection or principle identifies as optimal. In Subsec. \ref{Dependency}, we will look at simulations where results do not initially appear to be consistent with the theory. Roughly, we show that for larger values of $p$, larger values of $n$ are needed. Also, as $p$ increases, we must choose $n_{l}$'s that are properly spread out over $[0,n]$.
These studies help us calibrate the estimating procedure for $h_{T}$ given $n$ and a reasonable range for the number of parameters. We suggest this is necessary because \eqref{UPPTHEO6} is only an upper bound (see Theorem \ref{Theo5}) that tightens as $n$ increases.

To recap: Our estimation procedure has three steps. First, we obtain $\hat{\xi}\left(n_{l}\right)$'s; The algorithm for its estimation was given in Sec. \ref{estimator}. Second step we find the value of $~\hat{c}$ by minimizing \eqref{objective}. This is possible because $h$ is known in the conjectured model. It is the complexity of the model used to estimate $\hat{\xi}\left(n_{l}\right)$. Third, we find the value of $~\hat{h}~$ that minimizes \eqref{objective}.

\subsection{Simulation cases where our proposed method works}
\label{PropoWorks}
Simulations are implemented for model sizes $p = 15, 30, 40, \hbox{and}~ 50$ and we present the results for all cases.
For $p = 15, 30$, the choice of parameters in our simulation are all the same.
When the sample size is $N=400$; the design points are $N_{L} = \left\{50, 100, 150, 200, 250, 300, 400\right\}$; $m=10$; and the number of bootstrap samples is $b_{1} = b_{2}=50$. For $p = 40, 50 ~\hbox{and}~60$, the choice of the parameters in our simulation is all the same. When the sample size is $N=600$; the design points are $N_{L} = \left(75, 150, 225, 300, 375, 450, 525, 600\right)$; $m=10$; and the number of bootstrap samples is $b_{1} = b_{2}=50$.
For these cases, we fitted two sets of models, the first set uses a subset of our covariates to estimate the VCD, and in the second set, we added some decoys (their corresponding $\beta$ in the generation of the response are zeros).
Outputs of simulations are given in Figures \ref{Hhat15}-\ref{Hhat60}. 
We emphasize that we use Def. \ref{Defcon} for our $\hat{h}$. $\widehat{ERM}_{1}$ and $\widehat{ERM}_{2}$ use the point where the sharpest decreases occurs, and $BIC$ is simply minimized to identify a good model.
\noindent

\begin{figure}
    \centering
    \begin{subfigure}[b]{0.5\textwidth}
        \centering
        \includegraphics[width = \textwidth]{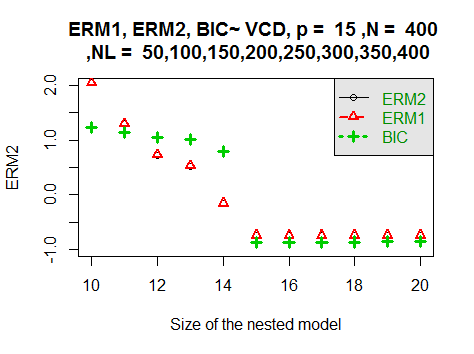}
        \caption{Values of $\widehat{ERM}_{1}$, $\widehat{ERM}_{2}$ and $BIC$}
        \label{erm15}
    \end{subfigure}
    \hfill
    \begin{subfigure}[b]{0.5\textwidth}
        \centering
        \includegraphics[width = \textwidth]{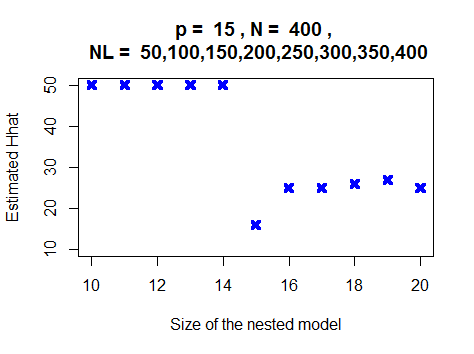}
        \caption{Estimates of $\hat{h}$}
        \label{hat15}
    \end{subfigure}
    \caption{Estimates of $\hat{h}$, $\widehat{ERM}_{1}$, $\widehat{ERM}_{2}$ and $BIC$ for $p=15$, $\sigma_{\epsilon}=0.4$, $\sigma_{\beta}=3$, $\sigma_{x} = 2$}
    \label{Hhat15}
\end{figure}

\begin{figure}
    \centering
    \begin{subfigure}[b]{0.5\textwidth}
        \centering
        \includegraphics[width = \textwidth]{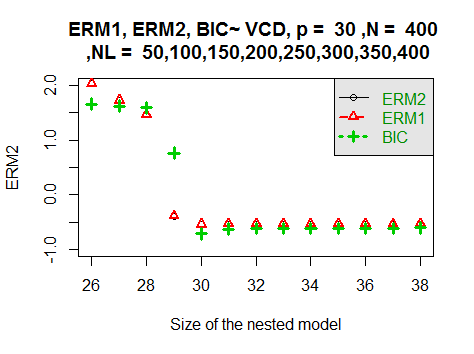}
        \caption{Values of $\widehat{ERM}_{1}$, $\widehat{ERM}_{2}$ and $BIC$}
        \label{erm30}
    \end{subfigure}
    \hfill
    \begin{subfigure}[b]{0.5\textwidth}
        \centering
        \includegraphics[width = \textwidth]{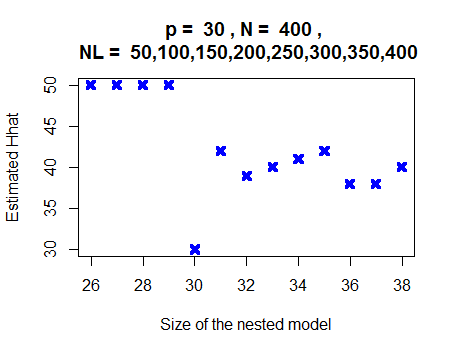}
        \caption{Estimates of $\hat{h}$}
        \label{hat30}
    \end{subfigure}
    \caption{Estimates of $\hat{h}$, $\widehat{ERM}_{1}$, $\widehat{ERM}_{2}$ and $BIC$ for $p=30$,$\sigma_{\epsilon}=0.4$, $\sigma_{\beta}=3$, $\sigma_{x} = 2$}
    \label{Hhat30}
\end{figure}

\begin{figure}
    \centering
    \begin{subfigure}[b]{0.4\textwidth}
        \centering
        \includegraphics[width = \textwidth]{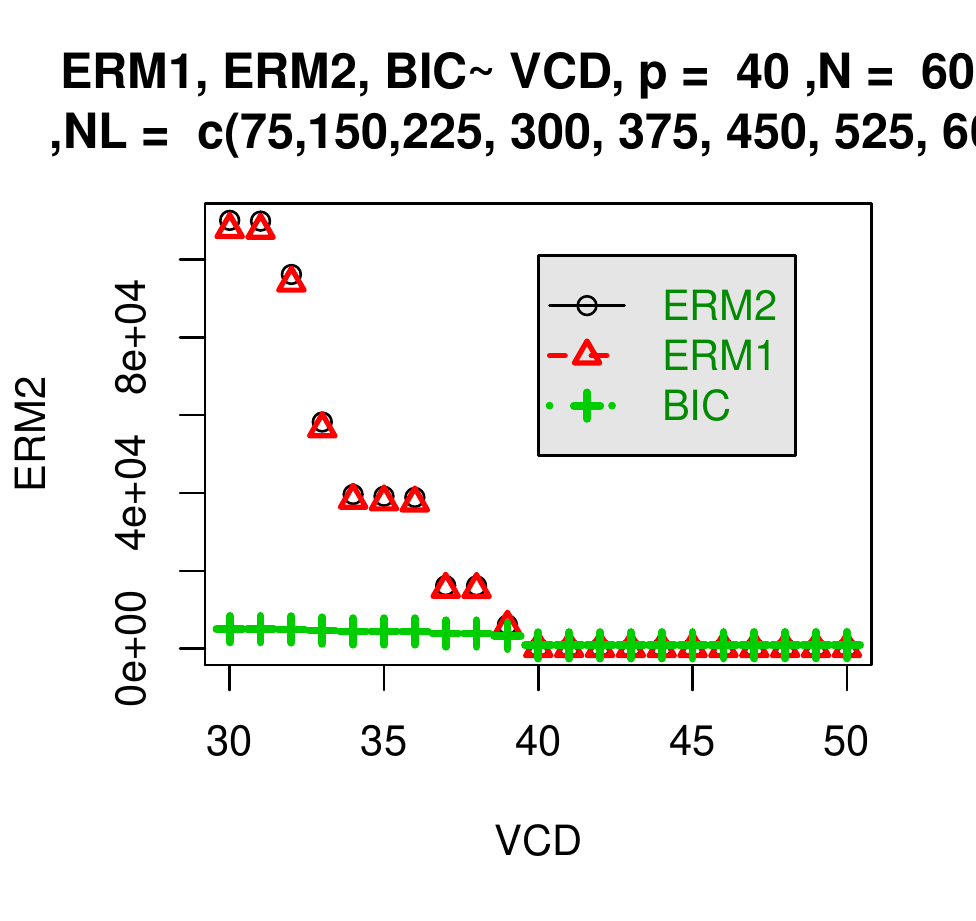}
        \caption{Values of $\widehat{ERM}_{1}$, $\widehat{ERM}_{2}$ and $BIC$}
        \label{erm40}
    \end{subfigure}
    \hfill
    \begin{subfigure}[b]{0.4\textwidth}
        \centering
        \includegraphics[width = \textwidth]{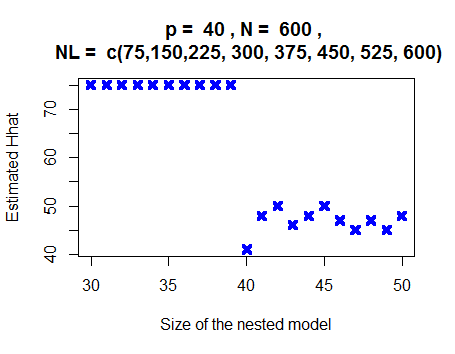}
        \caption{Estimates of $\hat{h}$}
        \label{hat40}
    \end{subfigure}
    \caption{Estimates of $\hat{h}$, $\widehat{ERM}_{1}$, $\widehat{ERM}_{2}$ and $BIC$ for $p=40$}
    \label{Hhat40}
\end{figure}

\noindent

\begin{figure}
    \centering
    \begin{subfigure}[b]{0.4\textwidth}
        \centering
        \includegraphics[width = \textwidth]{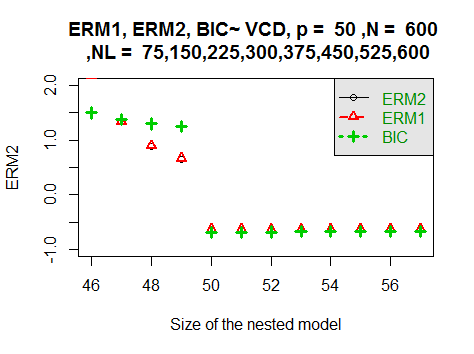}
        \caption{Values of $\widehat{ERM}_{1}$, $\widehat{ERM}_{2}$ and $BIC$}
        \label{erm50}
    \end{subfigure}
    \hfill
    \begin{subfigure}[b]{0.4\textwidth}
        \centering
        \includegraphics[width = \textwidth]{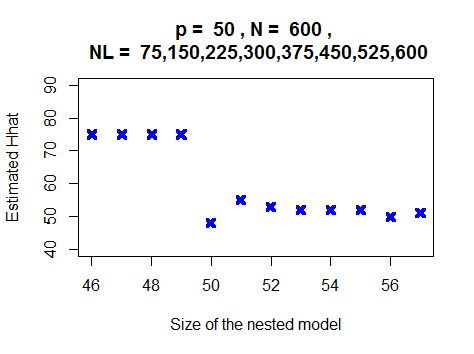}
        \caption{Estimates of $\hat{h}$}
        \label{hat50}
    \end{subfigure}
    \caption{Estimates of $\hat{h}$, $\widehat{ERM}_{1}$, $\widehat{ERM}_{2}$ and $BIC$ for $p=50$}
    \label{Hhat50}
\end{figure}

\begin{figure}
    \centering
    \begin{subfigure}[b]{0.4\textwidth}
        \centering
        \includegraphics[width = \textwidth]{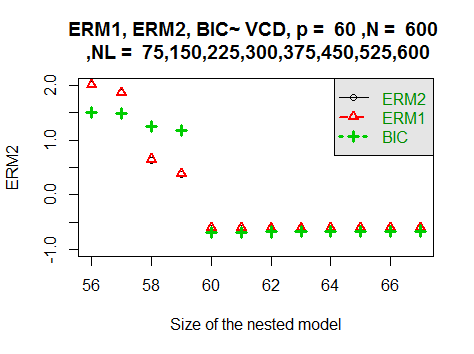}
        \caption{Values of $\widehat{ERM}_{1}$, $\widehat{ERM}_{2}$ and $BIC$}
        \label{erm60}
    \end{subfigure}
    \hfill
    \begin{subfigure}[b]{0.4\textwidth}
        \centering
        \includegraphics[width = \textwidth]{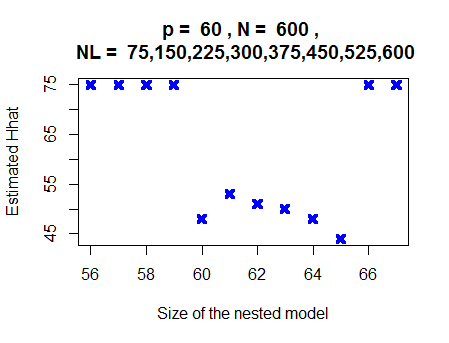}
        \caption{Estimates of $\hat{h}$}
        \label{hat60}
    \end{subfigure}
    \caption{Estimates of $\hat{h}$, $\widehat{ERM}_{1}$, $\widehat{ERM}_{2}$ and $BIC$ for $p=60$}
    \label{Hhat60}
\end{figure}

By examining Figures \ref{Hhat15} to \ref{Hhat60},
we see that, for each given true model of pre-specified size, we fitted a list of nested models. We start with models that used a subset of covariates to generate the response, we added one variable at a time up to the correct size of the model, and after that, we added decoys. These tables also contain information on the upper bound of the true unknown risk implemented using Proposition \ref{subProp1} ($\widehat{ERM}_{1}$) and Proposition \ref{subProp2} ($\widehat{ERM}_{2}$), and also the Bayesian Information Criterion (BIC) of the corresponding model. We comment that $\widehat{ERM}_{1}$ and $\widehat{ERM}_{2}$ are the standard style of probabilistic upper bounds for the empirical risk
under two senses of distance, one is additive (additive Chernoff bound \cite{Vapnik:1998}) and the other is multiplicative (multiplicative Chernoff bounds \cite{Vapnik:1998}). Note that $\widehat{ERM}_{1}$ and $\widehat{ERM}_{2}$ from Propositions \ref{subProp1} and \ref{subProp2} are of exactly the same form as would be derived using our cross-validated form of the errors namely the clause 1 and 2 from Proposition \ref{Prop1and2}.

To ensure $\widehat{ERM}_{1}$, $\widehat{ERM}_{2}$ and $BIC$ will fit on the same graph, we have scaled their values. Their exact values are recorded in the tables of Appendix \ref{AppendiceChapter3}.

We see that when the size of the conjectured model is strictly less than that of the true model, the estimated VCD is equal to the minimum value of the design points, and the values of the $\widehat{ERM}_{1}$, $\widehat{ERM}_{2}$ and $BIC$ are extremely high. Furthermore, these latter values typically decrease as the conjectured models become similar to the true model. For this range of model sizes, when the conjectured model exactly matches the true model, the estimated VCD ($\hat{h}$) is closest to the true value. The biggest discrepancy (of size 2) occurs for $p=50$; by contrast, for every other case the difference between the true value and the estimated VCD is at most one.

From Figures \ref{Hhat15}-\ref{Hhat50},
the behaviour of $\widehat{ERM}_{1}$, $\widehat{ERM}_{2}$ or $BIC$ is the same. In fact, when the conjectured model is a subset of the true model, we see a consistent and substantial decrease of $\widehat{ERM}_{1}$, $\widehat{ERM}_{2}$ or $BIC$; and a sudden drop of these statistics when the conjectured model perfectly matches the true model. This sudden drop can be considered as an indicator of the true model. After this point, $\widehat{ERM}_{1}$ usually flatlines; in fact, $\widehat{ERM}_{1}$ does not have the ability to discriminate well over models in the sense of clause 2 of Def. \ref{Defcon}. However, $\widehat{ERM}_{2}$ can still identify the true model since its smallest value occurs at the sudden drop. Also, $\widehat{ERM}_{2}$ often increases (albeit slowly) from its minimum as $p$ increases, thereby often satisfying clause 2 of Def. \ref{Defcon} at least in a minimal sense. We also see that
$BIC$ has good power of discrimination since its smallest values occur at the true model thereafter, when we add decoys the values of $BIC$ are bigger than those of the true model although sometimes not by much.
\newline
Finally, the smallest discrepancy between the size of the model ($h$) and $\hat{h}$ usually occurs at the true model. This indicates that $\hat{h}$ satisfies clause 1 of Def. \ref{Defcon}, i.e appears to be consistent for the true model.
In addition, $\hat{h}$ generally increases as the size of the model becomes bigger although pass a certain point it may flatline as well. The problem with flatlining or even decreasing past a certain value of $h$ is worse when $N$ is not large enough relative to $p$.

In Figure \ref{Hhat60},
$\widehat{ERM}_{1}$, $\widehat{ERM}_{2}$ and $BIC$ behave as before. In fact, we observe a decrease as the conjectured model becomes similar to the true model and there is a big drop as the conjectured model exactly matches the true model. However, at the true model, there is a big discrepancy between $\hat{h}$ and the size of the true model. We suggest that this discrepancy occurs because the sample size is too small comparatively to $p$ and the choice of the design points is poor. In Sec. \ref{Dependency} we will observe the effect of the sample size and make tentative recommendations for how to choose design points well.

For the present, we note that Figure \ref{Hhat60}
gives us the estimates of $\widehat{ERM}_{1}$ and $\widehat{ERM}_{2}$ when the sample size is $N=600$, and $N_{L}$ takes on values from $75$ to $600$ in steps of size $75$. We observe that $\widehat{ERM}_{1}$ and $\widehat{ERM}_{2}$ have a very low power of discrimination between models in the sense of clause 2 of Def. \ref{Defcon}. However, from Figure \ref{Hhat60_2000} where the sample size is $N = 2000$ and $N_{L} = 500, 700, 1000, 1500, 2000$,
we see that after the sudden drop in the estimate of $\widehat{ERM}_{1}$ and $\widehat{ERM}_{2}$, (that is an indicator of the true model), that $\widehat{ERM}_{2}$ tends to discriminate better than $\widehat{ERM}_{1}$.

\begin{table}[ht]
  \centering
  \caption{Relative increase of the sample size given the size of the model}\label{NoverP}
\begin{tabular}{|c|c|c|c|c|c|}
  \hline
   & Fig. \ref{hat15} & Fig. \ref{hat30}  & Fig. \ref{hat40} & Fig. \ref{hat50} & Fig. \ref{hat60} \\
  \hline\hline
 $\frac{N}{p}$ & 27 & 13 & 15 &   12 & 10  \\
 \hline
\end{tabular}
\end{table}
Table \ref{NoverP} gives the ratio of the sample size to the size of the model. In fact, from Fig \ref{hat15}  to Fig \ref{hat60}, we see that as $\frac{N}{p}$ increases, we have a good estimate of $\hat{h}$ and clause 2 of Def. \ref{Defcon} seems to be satisfied better. That is, the higher $\frac{N}{p}$ is, the better clause 2 of Def. \ref{Defcon} is satisfied. From Fig. \ref{hat15} to Fig. \ref{hat40}, we see that clause 2 of Def. \ref{Consis} is satisfied, in Fig. \ref{hat50}, clause 2 is minimally satisfied, and in Fig. \ref{hat60}, clause 2 is not satisfied.

We argue that estimating VCD directly is better than using $\widehat{ERM}_{1}$ or $\widehat{ERM}_{2}$. There are several reasons. First, the computation of $\widehat{ERM}_{1}$ and $\widehat{ERM}_{2}$ requires $\hat{h}$. It also requires a threshold $\eta$ be chosen (see Proposition \ref{Prop1and2}) and is more dependent on $m$ than $\hat{h}$ is. Being more complicated than $\hat{h}$, $\widehat{ERM}_{1}$, $\widehat{ERM}_{2}$ will break down faster than $\hat{h}$. This is seen, for instance in Table \ref{AdjTableP60_700}, where $\hat{h}$ identifies the correct model (even if it does not estimate $h_{T}$ accurately due to small sample sizes) while $\widehat{ERM}_{1}$ fails clause 2 of Def. \ref{Defcon}. $\widehat{ERM}_{2}$ behaves better than $\widehat{ERM}_{1}$ because the former is multiplicative (a stronger criterion) but can still give smaller values for larger model sizes, see Table \ref{AdjTableP60_2000_NL2}. In these cases when $\widehat{ERM}_{1}$ or $\widehat{ERM}_{2}$ performs poorly, $\hat{h}$ successfully identifies the true model and usually gives a value a value relatively close to the correct VCD. As a further point, Tables \ref{AdjTableP60_2000} and \ref{AdjTableP60_2000_NL2} show that $\widehat{ERM}_{1}$ often performs notably worse than $\widehat{ERM}_{2}$ or $\hat{h}$, and this may be attributed to the design points. More generally, we propose the following based on efficiency.

Specifically, we argue that $\widehat{ERM}_{1}$, and $\widehat{ERM}_{2}$ break down faster than $\hat{h}$ (under Def. \ref{Defcon}) with increasing $p$, if the sample size is held constant. Otherwise put, $\widehat{ERM}_{1}$ and $\widehat{ERM}_{2}$ are inefficient compared to $\hat{h}$, in that to do as well as $\hat{h}$ or ($BIC$) for fixed $p$, the sample size for $\widehat{ERM}_{1}$ and $\widehat{ERM}_{2}$ must be larger. We see this phenomenon in Tables \ref{AdjTableP15}-\ref{AdjTableP50} where we let $p=15, 30, 40$, and $p=50$ while $n$ ranges from 400 to 600. In these cases we see that the inefficiency makes $\widehat{ERM}_{1}$ and $\widehat{ERM}_{2}$ less variable, though they still oscillate somewhat when decoys are included. So $\widehat{ERM}_{1}$, and $\widehat{ERM}_{2}$ do not satisfy clause 2 of consistency at the true model very well in contrast to $\hat{h}$ using \eqref{eqcon}.

\begin{figure}
    \centering
    \begin{subfigure}[b]{0.4\textwidth}
        \centering
        \includegraphics[width = \textwidth]{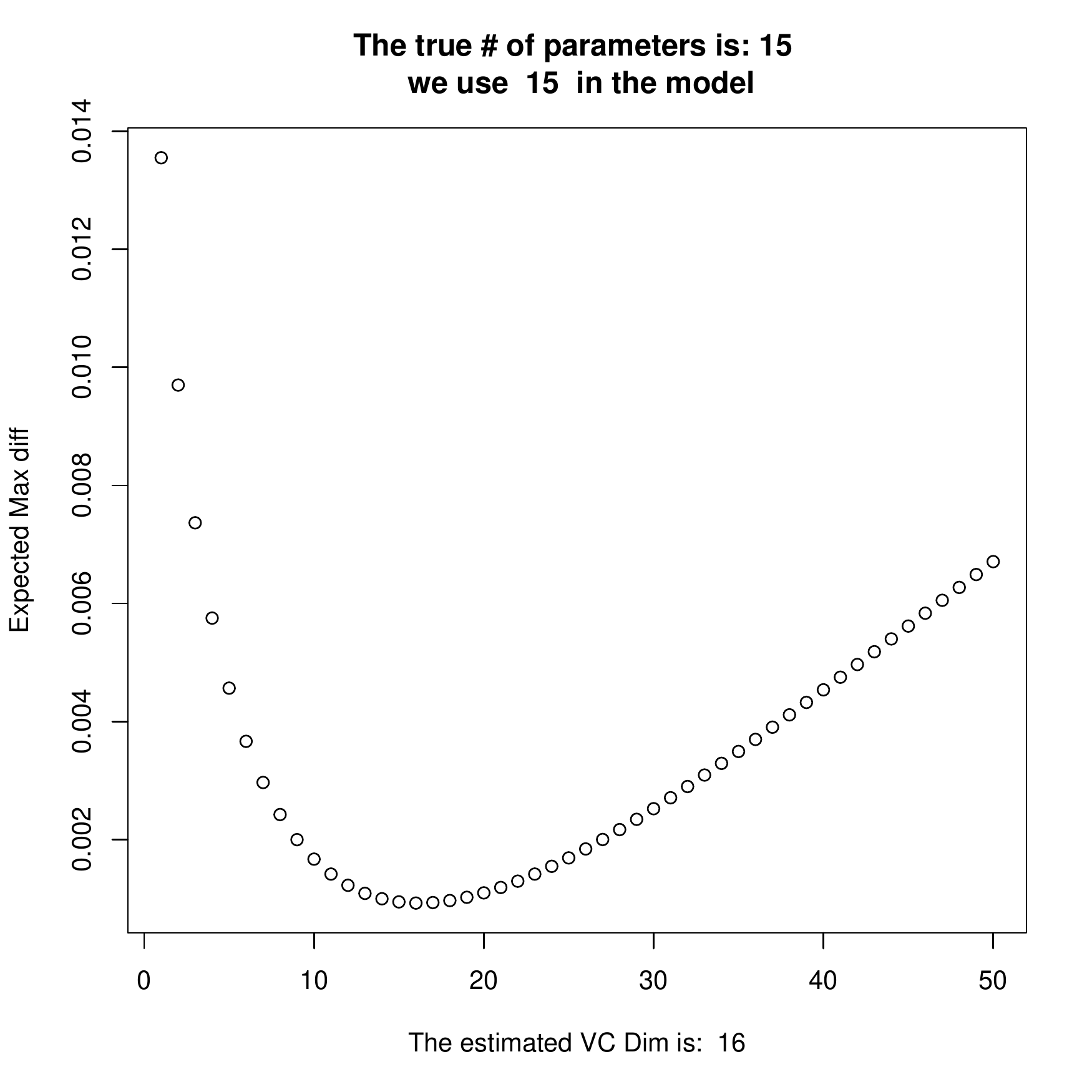}
        \caption{The true Number of Parameters is $~p=15$, $N=400$, $N_{L} = 50, 100, 150, 200, 250, 300, 350,400$}
        \label{Graph15_400}
    \end{subfigure}
    \hfill
    \begin{subfigure}[b]{0.4\textwidth}
        \centering
        \includegraphics[width = \textwidth]{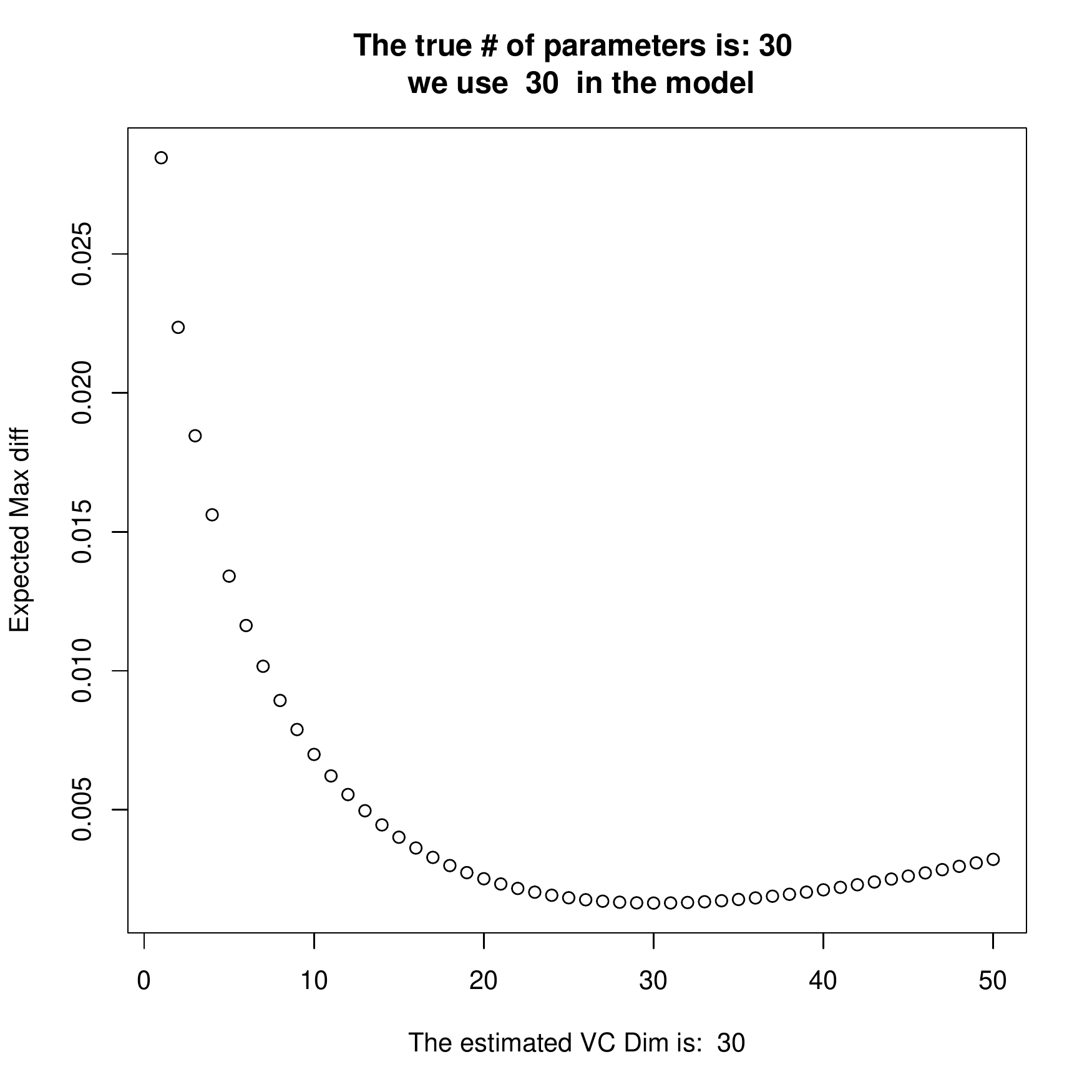}
        \caption{The true Number of Parameters is $~p=30$, $N=400$, $N_{L} = 50, 100, 150, 200, 250, 300, 350,400$}
        \label{Graph30_400}
    \end{subfigure}
    \caption{Estimates of VCD for $p=15, 30$}
    \label{EstVCD1530}
\end{figure}

\begin{figure}
    \centering
    \begin{subfigure}[b]{0.4\textwidth}
        \centering
        \includegraphics[width = \textwidth]{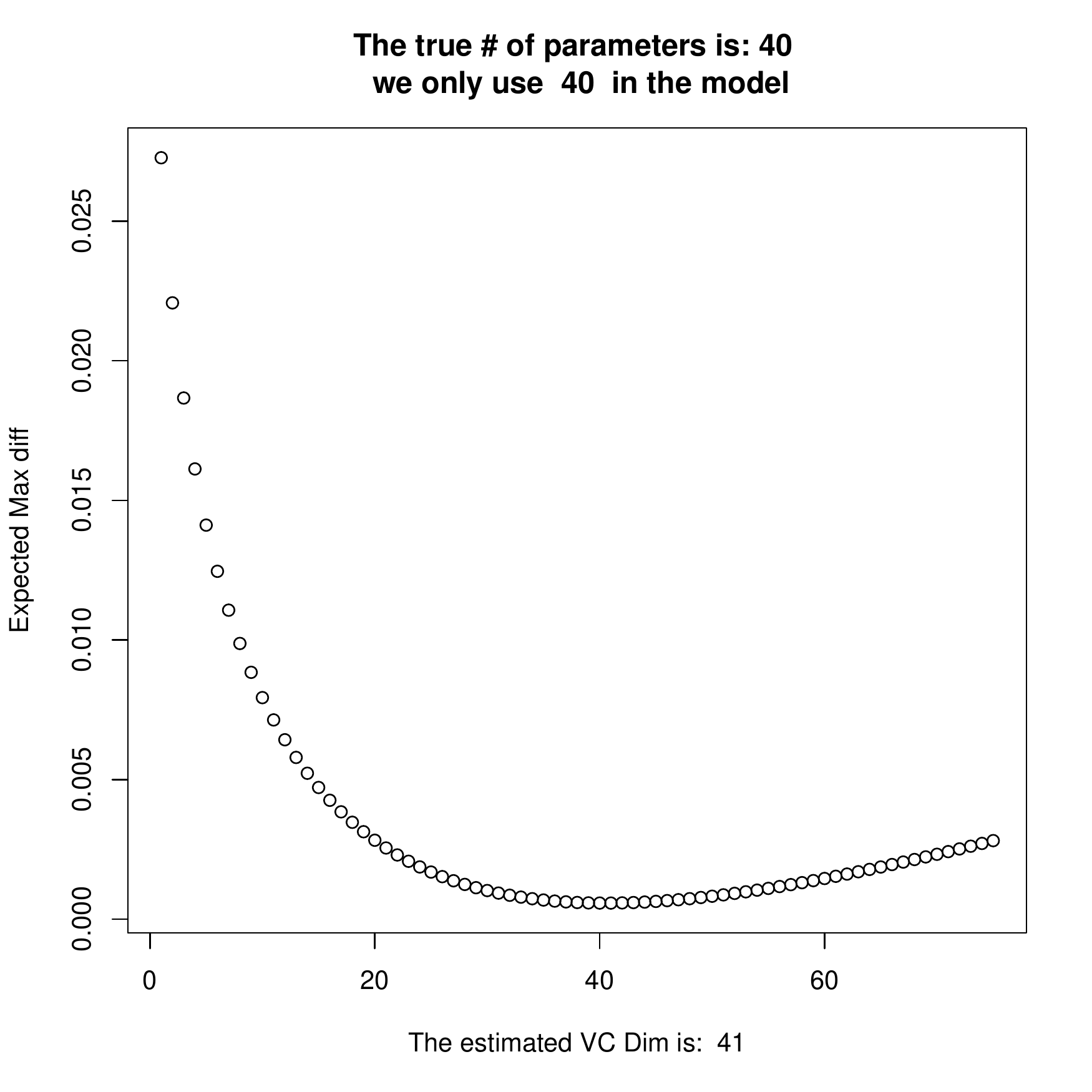}
        \caption{The true Number of Parameters is $p=40$, $N=600$, $N_{L} = 75, 150, 225, 300, 375, 450, 525, 600$}
        \label{Graph40_600}
    \end{subfigure}
    \hfill
    \begin{subfigure}[b]{0.4\textwidth}
        \centering
        \includegraphics[width = \textwidth]{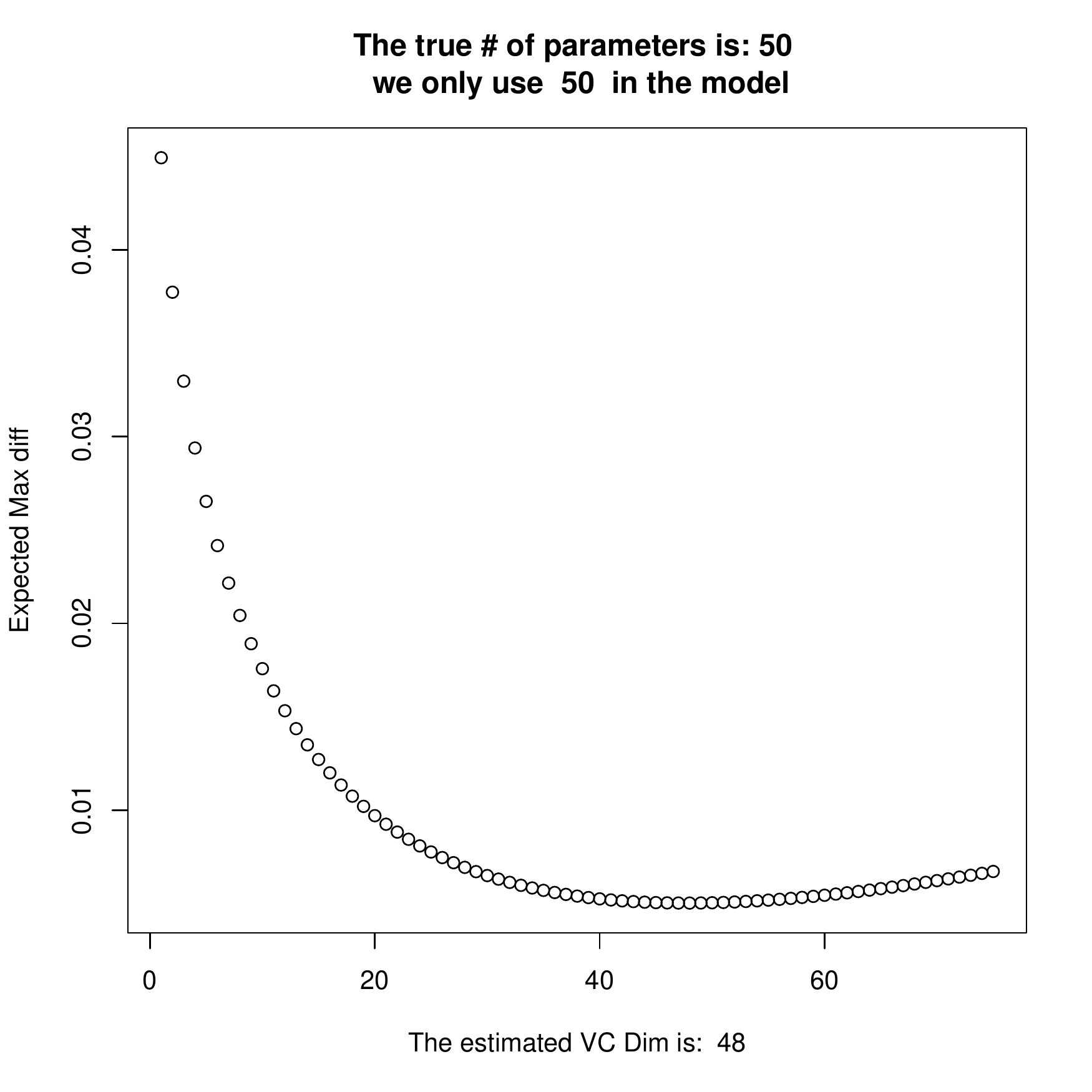}
        \caption{The true Number of Parameters is $~p=50$, $N=600$, $N_{L} = 75, 150, 225, 300, 375, 450, 525, 600$}
        \label{Graph50_600}
    \end{subfigure}
    \caption{Estimates of VCD for $p=40, 50$}
    \label{EstVCD4050}
\end{figure}

\begin{figure}
    \centering
    \begin{subfigure}[b]{0.4\textwidth}
        \centering
        \includegraphics[width = \textwidth]{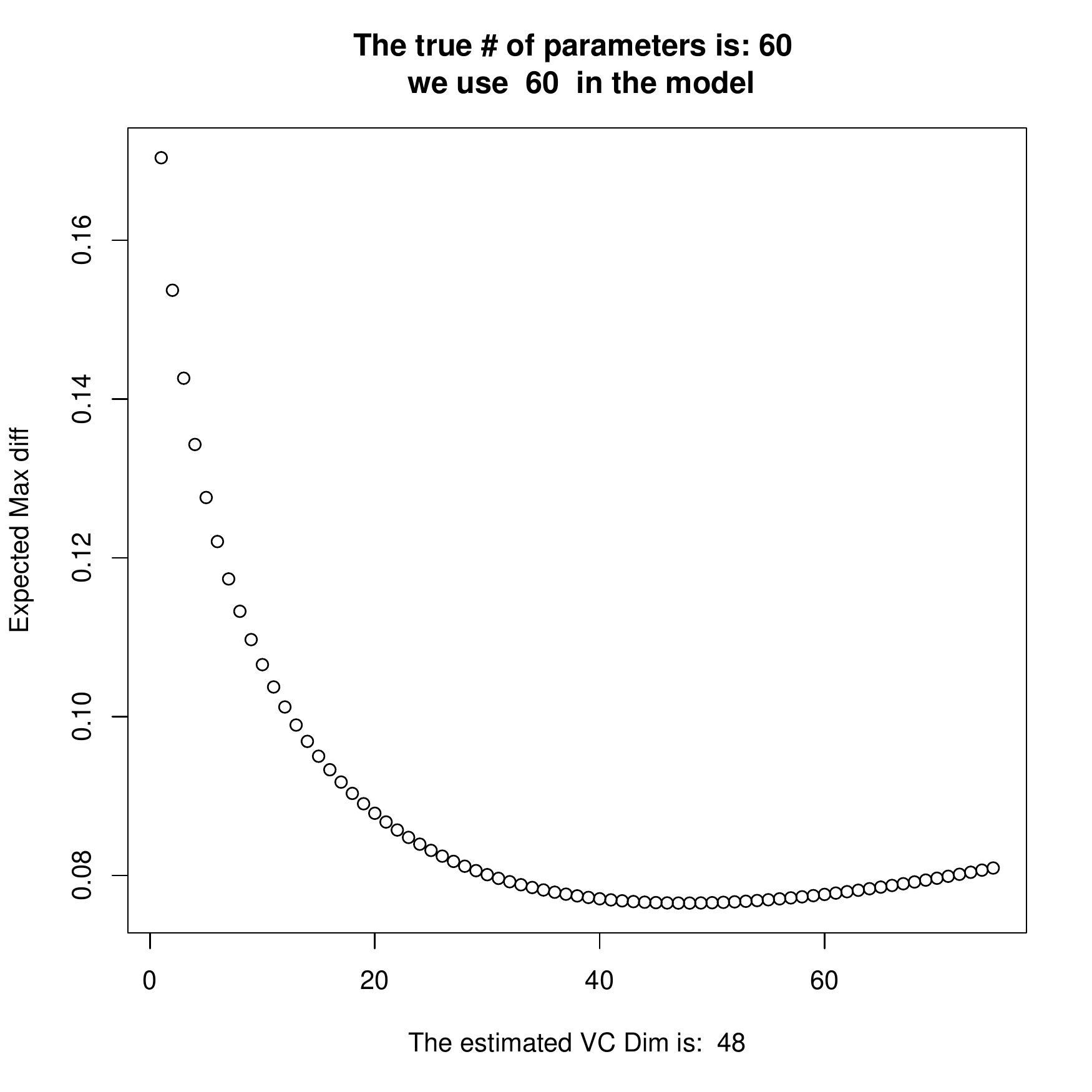}
        \caption{The true Number of Parameters is $~p=60$, $N= 600$, $N_{L} = 5, 150, 225, 300, 375, 450, 525, 600$}
        \label{Graph60_600}
    \end{subfigure}
    \hfill
    \begin{subfigure}[b]{0.4\textwidth}
        \centering
        \includegraphics[width = \textwidth]{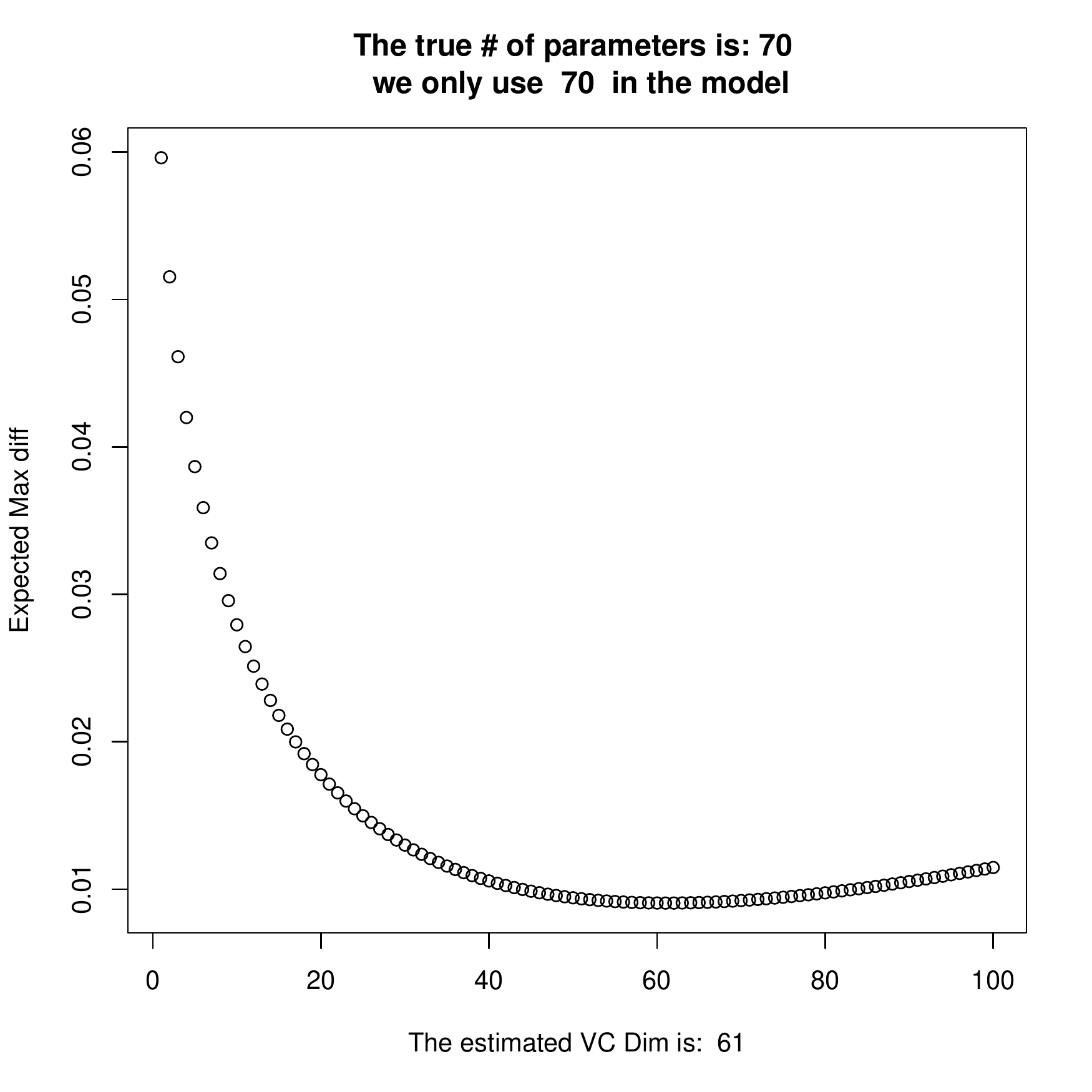}
        \caption{The true Number of Parameters is $~p=70$, $N = 700$, $N_{L} = 100, 200, 300, 400, 500, 600, 700$}
        \label{Graph70_700}
    \end{subfigure}
    \caption{Estimates of VCD for $p=60, 70$}
    \label{EstVCD6070}
\end{figure}

Graphs \ref{Graph60_600} and \ref{Graph70_700} from Figure \ref{EstVCD1530}-\ref{EstVCD6070} are the graphs of the objective function \eqref{objfn}. On the x-axis, we have the VCD and on the y-axis we have the value of the objective function. The minimum of these graphs is the estimate $\hat{h}$.
\subsection{Changes in simulation settings}
\label{Changesimulation}
Our goal in this subsection is to see how our estimates will behave when we change some of our simulation settings. In fact, we change the values of $\sigma_{\epsilon}$ from 0.4 to 0.8. We also change $\sigma_{x}$ and $\sigma_{\beta}$ and we estimate $\hat{h}$ when the size of the conjectured model exactly matches the true model. We repeated this 67 times and each time we only changed the seed used to simulate the response and the covariates. (We want to do this 100 times, but time did not allow this.) By doing this, we can see how our estimates will vary given different seeds.

\begin{figure}
    \centering
    \begin{subfigure}[b]{0.5\textwidth}
        \centering
        \includegraphics[width = \textwidth]{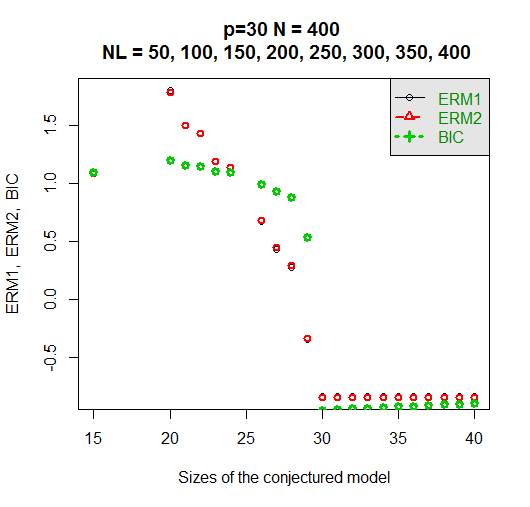}
        \caption{The true Number of Parameters is $p=30$, $N=400$, $N_{L} = 50, 100, 150, 200, 250,300, 350, 400$}
        \label{erm30sig08}
    \end{subfigure}
    \hfill
    \begin{subfigure}[b]{0.5\textwidth}
        \centering
        \includegraphics[width = \textwidth]{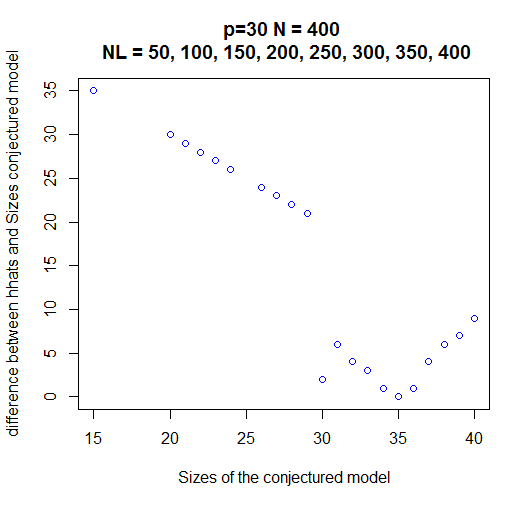}
        \caption{The true Number of Parameters is $p=30$, $N=400$, $N_{L} = 50, 100, 150, 200, 250,300, 350, 400$}
        \label{hhat30sig08}
    \end{subfigure}
    \caption{Estimates of $\hat{h}$, $\widehat{ERM}_{1}$, $\widehat{ERM}_{2}$, $BIC$ for $p=30$, $\sigma_{\epsilon}=0.8$}
    \label{ERmhhatsig08}
\end{figure}

\begin{figure}
    \centering
    \begin{subfigure}[b]{0.5\textwidth}
        \centering
        \includegraphics[width = \textwidth]{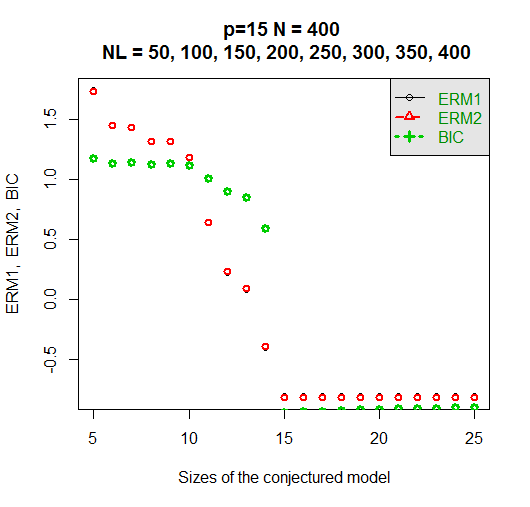}
        \caption{The true Number of Parameters is $p=15$, $N=400$, $N_{L} = 50, 100, 150, 200, 250,300, 350, 400$}
        \label{erm15sig08}
    \end{subfigure}
    \hfill
    \begin{subfigure}[b]{0.5\textwidth}
        \centering
        \includegraphics[width = \textwidth]{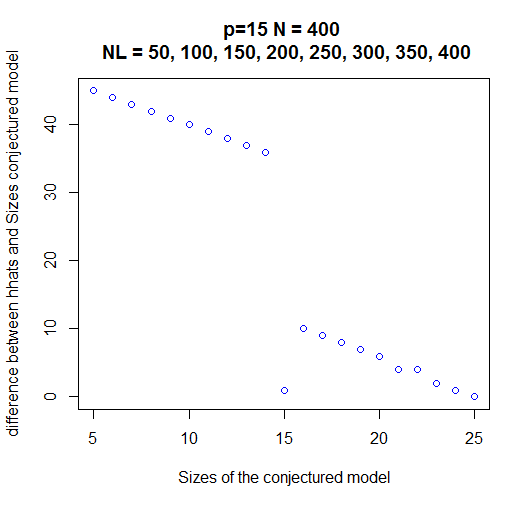}
        \caption{The true Number of Parameters is $p=15$, $N=400$, $N_{L} = 50, 100, 150, 200, 250,300, 350, 400$}
        \label{hhat15sig080}
    \end{subfigure}
    \caption{Estimates of $\hat{h}$, $\widehat{ERM}_{1}$, $\widehat{ERM}_{2}$, $BIC$ for $p=15$, $\sigma_{\epsilon}=0.8$}
    \label{ERmhhat15sig08}
\end{figure}

Figs \ref{ERmhhatsig08} to \ref{ERmhhat15sig08} give the estimate of $\hat{h}$, $\widehat{ERM}_{1}$, $\widehat{ERM}_{2}$, and $BIC$ when we increase $\sigma_{\epsilon}$ form 0.4 to 0.8, the true number of parameters is $p=30$ and $p=15$ respectively. Figs \ref{erm15sig08} and \ref{erm30sig08} are the graphs of $\widehat{ERM}_{1}$, $\widehat{ERM}_{2}$ and $BIC$ against the sizes of the conjectured model. Again, the behaviour of  $\widehat{ERM}_{1}$, $\widehat{ERM}_{2}$ and $BIC$ is pretty much the same of before; in fact, we observe a big drop in their values when the size of the conjecture model exactly matches the size of the true and do not discriminate well when the size of the conjectured model passes the size of true model.

Figs \ref{hhat15sig080} and \ref{hhat30sig08} are graphs of the difference between $\hat{h}$ and the size of the conjectured model for $p=15$ and $p=30$ respectively with $\sigma_{\epsilon}=0.8$. We observed that the minimum distance between the size of the conjectured model and $\hat{h}$ occurs at  models of size 15 and  25 when the true $p=15$ and a model of size $35$ when the true $p = 30$. However, at the true model $\hat{h} = 16$ and $\hat{h}=28$ respectively for $p=15$ and $p=30$. We comment that a parsimony argument possibly using a threshold as in \eqref{treshold} applied to Fig. \ref{hhat30sig08} could lead to the choice $\hat{h} = 30$.

\begin{figure}
  \centering
  \includegraphics[width=15cm]{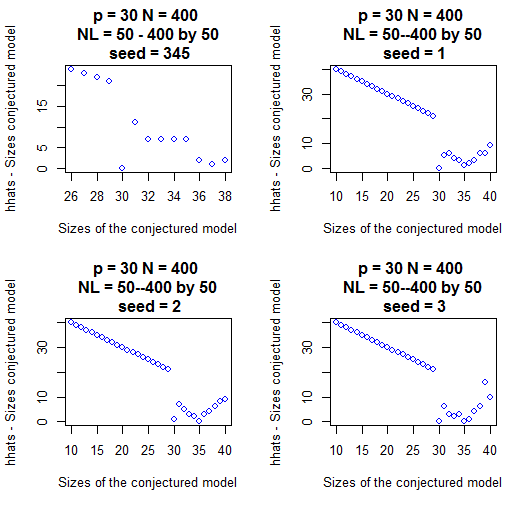}
  \caption{Estimates of $\hat{h}$ for $p=30$ for different seeds}\label{diffseed30}
\end{figure}
Fig \ref{diffseed30} is the graph of the estimated $\hat{h}$ for $p=30$ using different seeds in the simulate of the X's. In these case, we retain $\sigma_{\epsilon}=0.4$, $\sigma_{\beta}=3$ and $\sigma_{X} = 2$. We observe that when we change the seed, the model that we pick using $\hat{h}$ is not always the true model. For instance, when the seed is 2, the minimum discrepancy between $\hat{h}$ and the size of the conjectured model occurs when the size of the conjectured model is 35. When the seed is 3, the minimum discrepancy between $\hat{h}$ and the size of the conjectured model occurs at either 30 or 35. Again, we observe that parsimony arguments may improve our selection of $\hat{h}$ i.e., make $\hat{h}$ reliably closer to $h_{T}$

\begin{figure}
  \centering
  \includegraphics[width=15cm]{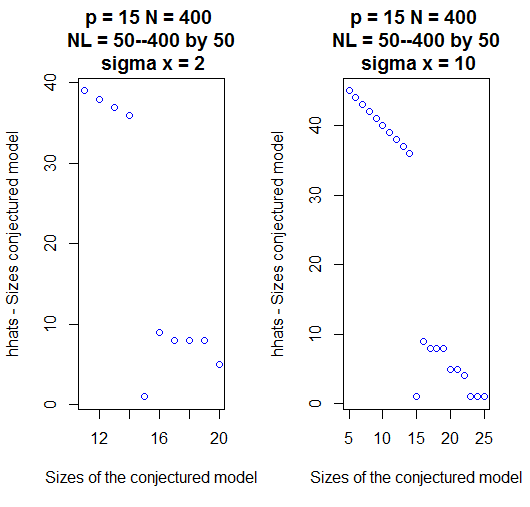}
  \caption{Estimates of $\hat{h}$ for $p=15$ with different  $\sigma_{x}$'s use to simulate the covariates}\label{Incressigx}
\end{figure}
Fig. \ref{Incressigx} estimates of $\hat{h}$ when the standard deviation used to simulate the covariates increases from 2 to 10. We observe that with a small variability in our covariates our method still picks the true $h$. (the small discrepancy between $\hat{h}$ and the size of the conjectured model occurs at the true model.) However, we note that as $\sigma_{X}$ increases, we may have to invoke parsimony to make $\hat{h}$ closer to $h_{T}$. Indeed, the right panel of Fig. \ref{Incressigx} shows that when there is a high variability in the covariates the method does not discriminate well --  the method did not discriminate amongst models of size 15, 23, 24, and 25.

\begin{figure}
    \centering
    \begin{subfigure}[b]{0.4\textwidth}
    \centering
    \includegraphics[width = \textwidth]{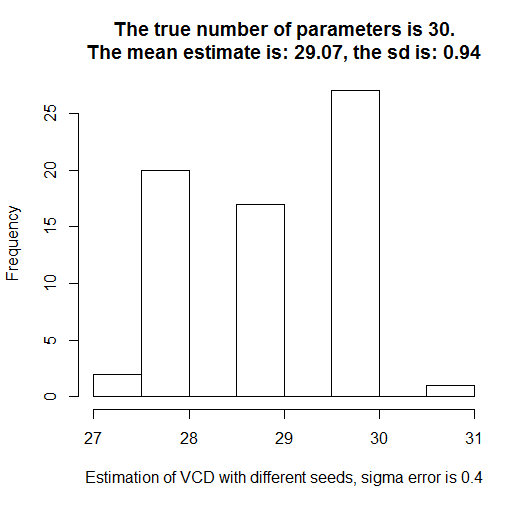}
    \caption{Sampling distribution of  $\hat{h}$, $\sigma_{\epsilon}=0.4$, $\sigma_{\beta}=3$, $\sigma_{X}=2$.}
    \label{Boot30}
    \end{subfigure}
    \hfill
    \begin{subfigure}[b]{0.4\textwidth}
    \centering
    \includegraphics[width = \textwidth]{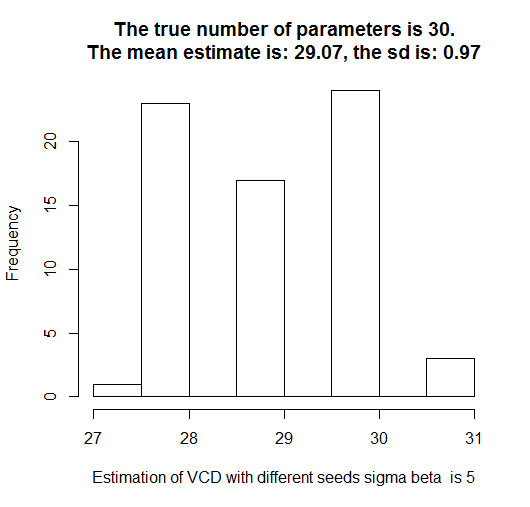}
    \caption{Sampling distribution of  $\hat{h}$,  $\sigma_{\epsilon}=0.4$, $\sigma_{\beta}=5$, $\sigma_{X}=2$.}
    \label{BootBeta}
    \end{subfigure}
    \hfill
    \begin{subfigure}[b]{0.4\textwidth}
    \centering
    \includegraphics[width = \textwidth]{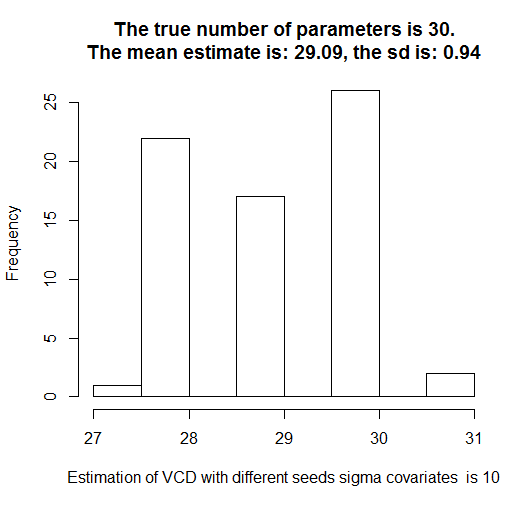}
    \caption{Sampling distribution of  $\hat{h}$, $\sigma_{\epsilon}=0.4$, $\sigma_{\beta}=3$, $\sigma_{X}=10$.  }
    \label{XBoot}
    \end{subfigure}
    \caption{Effect of changing the seed on $\hat{h}$ at the true model. The true number of parameters is 30}
    \label{30seedboot}
\end{figure}
Fig. \ref{30seedboot} gives the estimate of $\hat{h}$ at the true model when the seed used to simulate the response and covariates varies from 1 to 67. Fig. \ref{Boot30} shows a sampling distribution for $\hat{h}$. Here, there is no change in the simulation settings, we merely  computed different values of $\hat{h}$ with the seed changing from 1 to 67. We observe that most of the estimates of $\hat{h}$ are around 30. The average value of $\hat{h}$ is 29.07 with a standard deviation of 0.94. Fig. \ref{BootBeta} gives the same observation, albeit with a little more variability. In fact, the average value of $\hat{h}$ is 29.07 with a standard deviation of 0.97. Again, Fig. \ref{XBoot} is a sampling distribution for $\hat{h}$ for $\sigma_{X}=10$. The analysis is as before. The average value of $\hat{h}$ is 29.09 with a standard deviation of 0.94. For the case of varying $\sigma_{X}$, we expect at most only slight change we have studentized the covariates.

\begin{figure}
  \centering
  \includegraphics[width=\textwidth]{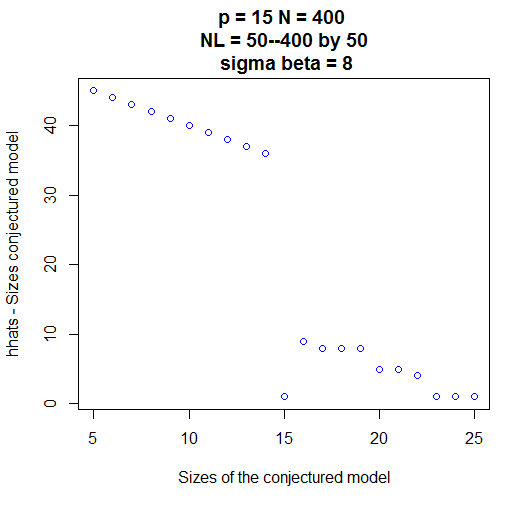}
  \caption{Estimates of $\hat{h}$ for $p=15$, $\sigma_{\beta}=8$}\label{8beta15}
\end{figure}

In Fig. \ref{8beta15}, we set $\sigma_{\beta}=8$. We see that $\hat{h}$ does not discriminate well. In fact, $\hat{h}$ does not detect a difference between models of size 15, 23, 24, and 25.
As in other settings, we suggest parsimony arguments and thresholding can improve the performance of $\hat{h}$. In the present case, simply choosing the smallest value of $\hat{h}$ (parsimony) identifies the true model.

Based on our simulation results, we suggest that in practice, $\hat{h}$ be chosen by a criterion of the form

\begin{equation}
\label{treshold}
\hat{h} = \arg\min_{k}\left\{\left|\hat{h}_{k}- VCD\left(p_{k}(\cdot,\beta)\right)\right|\leq t\right\}
\end{equation}
for some reasonable threshold $t$. Often can be taken to be 0, but not always. In practice, we recommend $t=2$.

\subsection{Dependency on The Sample Size and Design Points}
\label{Dependency}
Our goal here is to show how we can improve the quality of our estimate by increasing $N$ or tuning the design points. At the beginning of this section, we started observing the effect of sample size on the quality of our estimates. Here, we emphasize both the sample size and the effect of design points. We perform simulations for model sizes $p = 60, \hbox{~and } 70$. For $p=70$, the sample size is set to be $N = 700$, design points vary from $100$, to $700$, by $100$, everything else is still the same as in the previous subsection. To observe the effect of sample size and design points on $\hat{h}$, $\widehat{EMR}_{1}$ and $\widehat{ERM}_{2}$ we set the sample size to be $N = 2000~$ and $N_{L} = 500, 700, 1000, 1500, 2000$ for either $p=60$ or $p=70 $ and keep everything else unchanged.

Figures \ref{Hhat60} and \ref{hhat70}
give estimates of  $\hat{h}$, the upper bound of the true unknown risk using  Propositions \ref{subProp1} ($\widehat{ERM}_{1}$) and \ref{subProp2} ($\widehat{ERM}_{2}$), and the $BIC$ as before, for small sample sizes $N=600, 700$ and their corresponding design points. Given that the size of the conjectured model is strictly less than the size of the true model, $\hat{h}$ is equal to the smallest design point. However, when the conjectured model exactly matches the true model, $\hat{h} \approx 50, 61$ underestimates $h_{T}$. When the conjectured model is more complex than the true model, we see that $\hat{h}$ still underestimates $h_{T}$ in most cases. Our observations about $\widehat{ERM}_{1}$, $\widehat{ERM}_{2}$ and $BIC$ are still the same as before.

Figure \ref{Hhat60_700} gives estimates of $\hat{h}$, $\widehat{ERM}_{1}$, $\widehat{ERM}_{2}$ and $BIC$ when $N=700$, $N_{L} = 100, 200$ $, 300, 400, 500, 600$, and $700$ and the model size is $p=60$. We see that $\hat{h}=57$; this estimate is closer to the true value than that from Figure \ref{Hhat60}. We do not observe any change in the qualitative behaviour of $\widehat{ERM}_{1}$, $\widehat{ERM}_{2}$ and $BIC$. This shows that small changes in sample size or design points may have large numerical effects on the values of $\hat{h}$, $\widehat{ERM}_{1}$ and $\widehat{ERM}_{2}$ and $BIC$.

Figure \ref{Hhat60_2000} is qualitatively the same as Figure \ref{Hhat60}. The difference is the sample size and design points. In fact, in Figure \ref{Hhat60_2000}, the sample size is $N=2000$, design points are
$N_{L} = 500, 700, 1000, 1500, 2000$, whereas in Figure \ref{Hhat60}, $N=600$, design points vary from $75$ to $600$ in steps of $75$. The behaviour of $\hat{h}$, $\widehat{ERM}_{1}$, $\widehat{ERM}_{2}$ and $BIC$ are still the same as previously described. We infer from this that as the sample size increases, $\hat{h}$ moves closer to its true value. In these figures, the design points have also shifted.  This leads us to suggest that to get the optimal estimate $\hat{h}$, not only must $N$ increase, the value design points must also increase so that the range of values in the set of design points covers $\left[0, N\right]$.

When comparing Figure \ref{Hhat60_2000NL2} to Figure \ref{Hhat60},
the difference is that in
Figure \ref{Hhat60_2000NL2} $N=2000$, whereas (in Figure \ref{Hhat60}, $N=600$).
We observe that when there is a big enough increase in $N$ relative to the design points ($N_{L}$ being constant, but small compare to sample size) $\hat{h}$ converges to the true value, but at a lower rate even though $\hat{h}$ only moved from 48 to 50, in both cases, the model identified is very close to the true model. Indeed, in Figure \ref{Hhat60}, $\hat{h}$ identifies $p=61$ or $62$ and in Figure \ref{Hhat60_2000NL2}, $\hat{h}$ identifies $p=61$.

Figure \ref{Hhat60_700} and Figure \ref{Hhat60_2000NL}
are almost the same despite the substantial increase in $N$. In Figure \ref{Hhat60_2000NL},
the sample size is $N=2000$, and in Figure \ref{Hhat60_700}
the sample size is $N=700$, design points $N_{L} = 100, 200, 300, 400, 500, 600, 700$ for the two figures. We observe that the qualitative behaviour of $\widehat{ERM}_{1}$, $\widehat{ERM}_{2}$ and $BIC$ is unchanged and $\hat{h}$ at the true model is nearly the same for both simulations and close to the true VCD. In both cases, the model identified by $\hat{h}$ is qualitatively the same and close to the true model. $\hat{h}$ tends to increase as the size of the wrong model increases i.e clause 2 of Def. \ref{Defcon} is approximately satisfied. This comparison suggests that if the design points are large relative to $p$ then $\hat{h}$ can be found accurately without necessitating large sample sizes, i.e., for well chosen design points, $\frac{N}{p}
\approx 15$ will be sufficient. (Usually, the rule of thumb in parameter estimate is to set 10 datas per parameter. Here, we recommend a higher sample size because we are doing model selection as well as parameter estimation. Often, this method will not work with smaller sample sizes. However, few model selection techniques work with small sample sizes outside very restricted settings)

These observations show that, with enough data points, and by choosing design points well (usually half of the size of the data), we can obtain good estimates of $\hat{h}$. Therefore, we can do model selection using clause 1 of Def. \ref{Defcon} and a form of clause 2. Usually, with this adjustment, $\hat{h}$ will perform better than either $\widehat{ERM}_{1}$ or $\widehat{ERM}_{1}$. In contrast to $BIC$, for a given sample size, our method performs roughly as well as BIC for model selection when the design points are well chosen, but also provides an estimate of the VCD.

We leave the question of optimally choosing the design points as future work even though we have conjectured that design points should broadly cover $\left[0,N\right]$ with more design points in the upper half of the interval as $N$ increases, i.e the size of the design points should track $N$. We also recommend that a set of design points be chosen so that the shape of the objective function \eqref{objfn} vs VCD is convex and has a well defined minimum.

\begin{figure}
    \centering
    \begin{subfigure}[b]{0.5\textwidth}
        \centering
        \includegraphics[width = \textwidth]{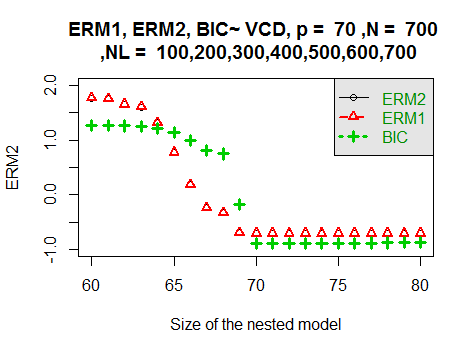}
        \caption{Values of $\widehat{ERM}_{1}$, $\widehat{ERM}_{2}$ and $BIC$ for $p=70$}
        \label{ERM70}
    \end{subfigure}
    \hfill
    \begin{subfigure}[b]{0.5\textwidth}
        \centering
        \includegraphics[width = \textwidth]{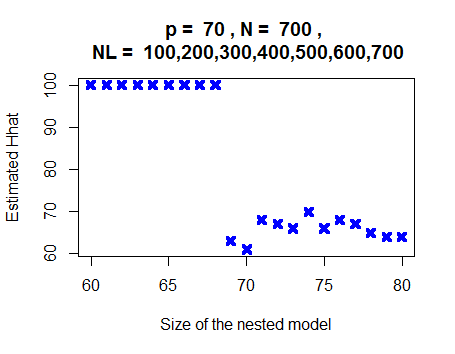}
        \caption{Estimate of $\hat{h}$ for $p=70$}
        \label{Hhat70}
    \end{subfigure}
    \caption{Estimates of $\hat{h}$, $\widehat{ERM}_{1}$, $\widehat{ERM}_{2}$ and $BIC$ for $p=70$}
    \label{hhat70}
\end{figure}

\begin{figure}
    \centering
    \begin{subfigure}[b]{0.5\textwidth}
        \centering
        \includegraphics[width = \textwidth]{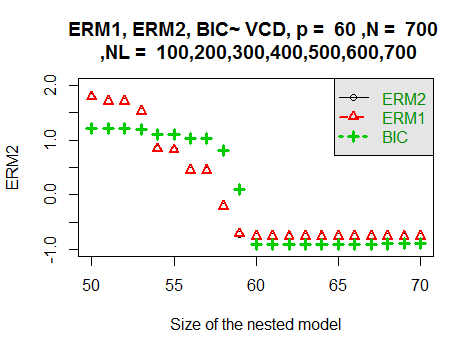}
        \caption{Values of $\widehat{ERM}_{1}$, $\widehat{ERM}_{2}$ and $BIC$ for $p=60$}
        \label{ERM60700}
    \end{subfigure}
    \hfill
    \begin{subfigure}[b]{0.5\textwidth}
        \centering
        \includegraphics[width = \textwidth]{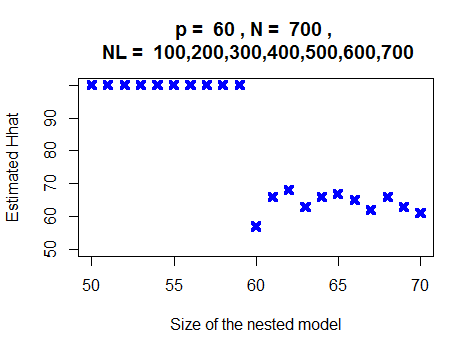}
        \caption{Estimate of $\hat{h}$ for $p=60$}
        \label{Hhat60700}
    \end{subfigure}
    \caption{Estimates of $\hat{h}$, $\widehat{ERM}_{1}$, $\widehat{ERM}_{2}$ and $BIC$ for $p=60$, $\hat{h} = 57$}
    \label{Hhat60_700}
\end{figure}

\begin{figure}
    \centering
    \begin{subfigure}[b]{0.5\textwidth}
        \centering
        \includegraphics[width = \textwidth]{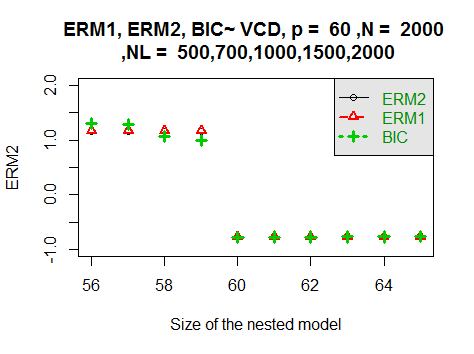}
        \caption{Values of $\widehat{ERM}_{1}$, $\widehat{ERM}_{2}$ and $BIC$ for $p=60$}
        \label{ERM602000}
    \end{subfigure}
    \hfill
    \begin{subfigure}[b]{0.5\textwidth}
        \centering
        \includegraphics[width = \textwidth]{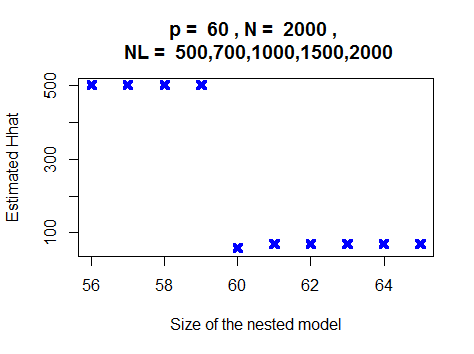}
        \caption{Estimate of $\hat{h}$ for $p=60$}
        \label{Hhat602000}
    \end{subfigure}
    \caption{Estimates of $\hat{h}$, $\widehat{ERM}_{1}$, $\widehat{ERM}_{2}$ and $BIC$ for $p=60$, $\hat{h} = 59$}
    \label{Hhat60_2000}
\end{figure}

\begin{figure}
    \centering
    \begin{subfigure}[b]{0.4\textwidth}
        \centering
        \includegraphics[width = \textwidth]{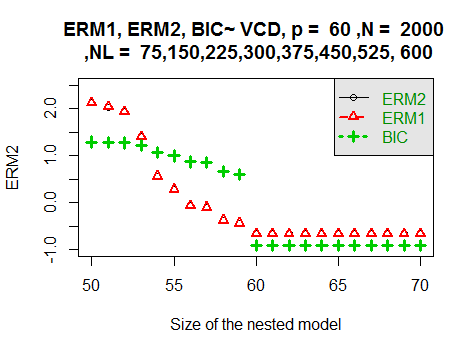}
        \caption{Values of $\widehat{ERM}_{1}$, $\widehat{ERM}_{2}$ and $BIC$ for $p=60$}
        \label{ERM602000NL2}
    \end{subfigure}
    \hfill
    \begin{subfigure}[b]{0.4\textwidth}
        \centering
        \includegraphics[width = \textwidth]{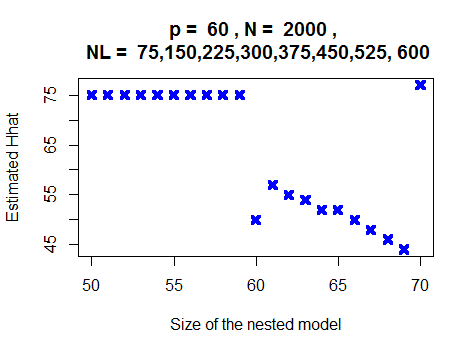}
        \caption{Estimate of $\hat{h}$ for $p=60$}
        \label{Hhat602000NL2}
    \end{subfigure}
    \caption{Estimates of $\hat{h}$, $\widehat{ERM}_{1}$, $\widehat{ERM}_{2}$ and $BIC$ for $p=60$}
    \label{Hhat60_2000NL2}
\end{figure}

\begin{figure}
    \centering
    \begin{subfigure}[b]{0.5\textwidth}
        \centering
        \includegraphics[width = \textwidth]{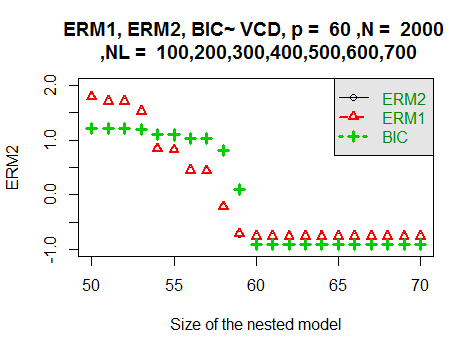}
        \caption{Values of $\widehat{ERM}_{1}$, $\widehat{ERM}_{2}$ and $BIC$ for $p=60$}
        \label{ERM602000NL}
    \end{subfigure}
    \hfill
    \begin{subfigure}[b]{0.5\textwidth}
        \centering
        \includegraphics[width = \textwidth]{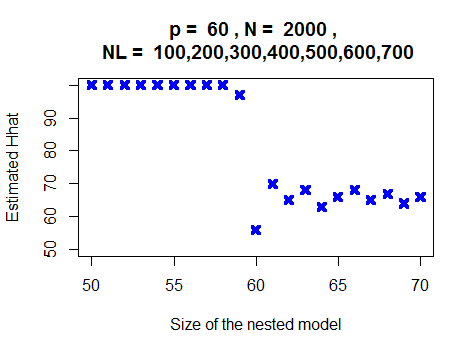}
        \caption{Estimate of $\hat{h}$ for $p=60$}
        \label{Hhat602000NL}
    \end{subfigure}
    \caption{Estimates of $\hat{h}$, $\widehat{ERM}_{1}$, $\widehat{ERM}_{2}$ and $BIC$ for $p=60$}
    \label{Hhat60_2000NL}
\end{figure}

\begin{figure}
    \centering
    \begin{subfigure}[b]{0.5\textwidth}
        \centering
        \includegraphics[width = \textwidth]{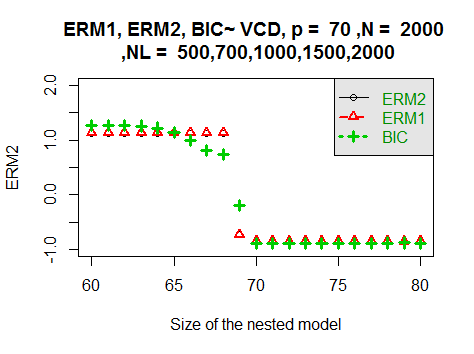}
        \caption{Values of $\widehat{ERM}_{1}$, $\widehat{ERM}_{2}$ and $BIC$}
        \label{ERM702000}
    \end{subfigure}
    \hfill
    \begin{subfigure}[b]{0.5\textwidth}
        \centering
        \includegraphics[width = \textwidth]{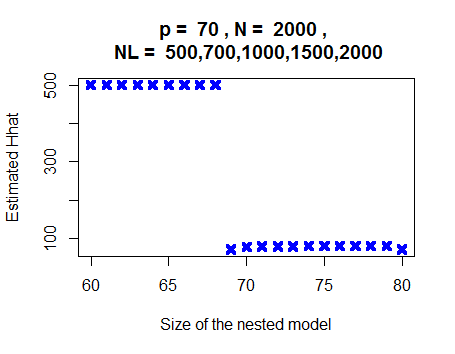}
        \caption{Estimate of $\hat{h}$ for $p=70$}
        \label{Hhat702000}
    \end{subfigure}
    \caption{Estimates of $\hat{h}$, $\widehat{ERM}_{1}$, $\widehat{ERM}_{2}$ and $BIC$ for $p=70$}
    \label{Hhat70_2000}
\end{figure}

\subsubsection{Summary}
In Subsect. \ref{DirectImplementation} we looked at Vapnik's method for Theorem \ref{Theo5} on synthetic datasets. We observed that this did not give good results reliably. We assume that this inaccuracy on the estimate of $\hat{h}$ is due first to the lack of variability in the error criterion, second to the fact that the bound is not tight enough. To fix these problems, we used the estimator of $h$ in Sec. \ref{estimator} that accounts for these. In Sec. \ref{Simulation}, we estimated $h$ using our new estimator $\hat{h}$, and our estimates were closer to the true value. To understand the behaviour of our method, we introduced the concept of `consistency at the true model'. We also observed how $\widehat{ERM}_{1}$, $\widehat{ERM}_{2}$ and $BIC$ decrease as the conjectured model moves closer to the true model and how they drop suddenly when the two models are equal. In Subsect. \ref{Dependency}, we investigated how we can tune the design points and the sample size to improve the estimate of $\hat{h}$. We observed that, given a sample size, if we choose a good set of design points we can improve the estimation of $h$ by $\hat{h}$. We also saw that as the sample size increases, $\hat{h}$ is consistent for the true VCD, and this may be independent of the selection of the design points. So, we suggest that the Theorem in \cite{McDonald:etal:2011} is true, but at a very slow rate and only when $c = \hat{c}$ and for a cross-validated form of the error is used. We also note that the design points must also be uniformly chosen such that its range cover the size of the dataset, and more design points are better than fewer design points.


\chapter{ANALYSIS OF TOUR DE FRANCE DATASET}
\label{example}
\thispagestyle{fancy}

After estimating $\hat{h}$ on synthetic datasets, the natural question is: How do we  estimate $~\hat{h}~$ with real data and use it for model selection?
We will see that the estimate of $\hat{h}$ and the upper bounds of Propositions \ref{subProp1} and \ref{subProp2} will guide our way to the choice of the best model. Since the goal of this chapter is to evaluate our method on a real dataset, we have chosen \sf {Tour De France}.\footnote{Tour De France Data was collected by Bertrand Clarke. More information can be found at  http://www.letour.fr/}

We start this chapter by giving some information about our dataset in Sec. \ref{DescripStat}. Then, in Sec. \ref{AnalTour5} we analyze our dataset using two classes of models with a model list based on $Year$ and $Distance$; the first class is a sequence of nested models and the second class consists of non-nested models. We evaluate our method by comparing $\hat{h}$ to $BIC$, $\widehat{ERM}_{1}$ and $\widehat{ERM}_{2}$. In the last Section, we look at the effect of outliers in the estimates $\hat{h}$, $\widehat{ERM}_{1}$ and $\widehat{ERM}_{2}$. In Sec. \ref{AnalTour4} we re-analyze the data using a model list made up of variables $Year$, $Distance$, $Age$ of the winner and $Stages$ won. A natural question is why not combine the two model lists into one. The answer is that the resulting model would have 9 variables, but our dataset would have size 103 (including outliers) or 99 (with outliers removed). Thus, we would have enough data for parameters estimation, but not for model selection, which is the point of using $h$.

We admit that in the present context, the \sf{Tour De France} dataset is a toy dataset, because it is small in terms of the number of observations compared to the number of quantities that we will be computing from it. However, the results seem reasonable and this chapter is important mostly to demonstrate how $\hat{h}$ can be used and compared to other methods. In Chap. \ref{Wheatdata} we will use  much larger datasets and avoid problems with over-analysis.

\section{Descriptive Analysis of Tour De France Data}
\label{DescripStat}

\begin{figure}
    \centering
    \begin{subfigure}[b]{0.4\textwidth}
        \centering
        \includegraphics[width = \textwidth]{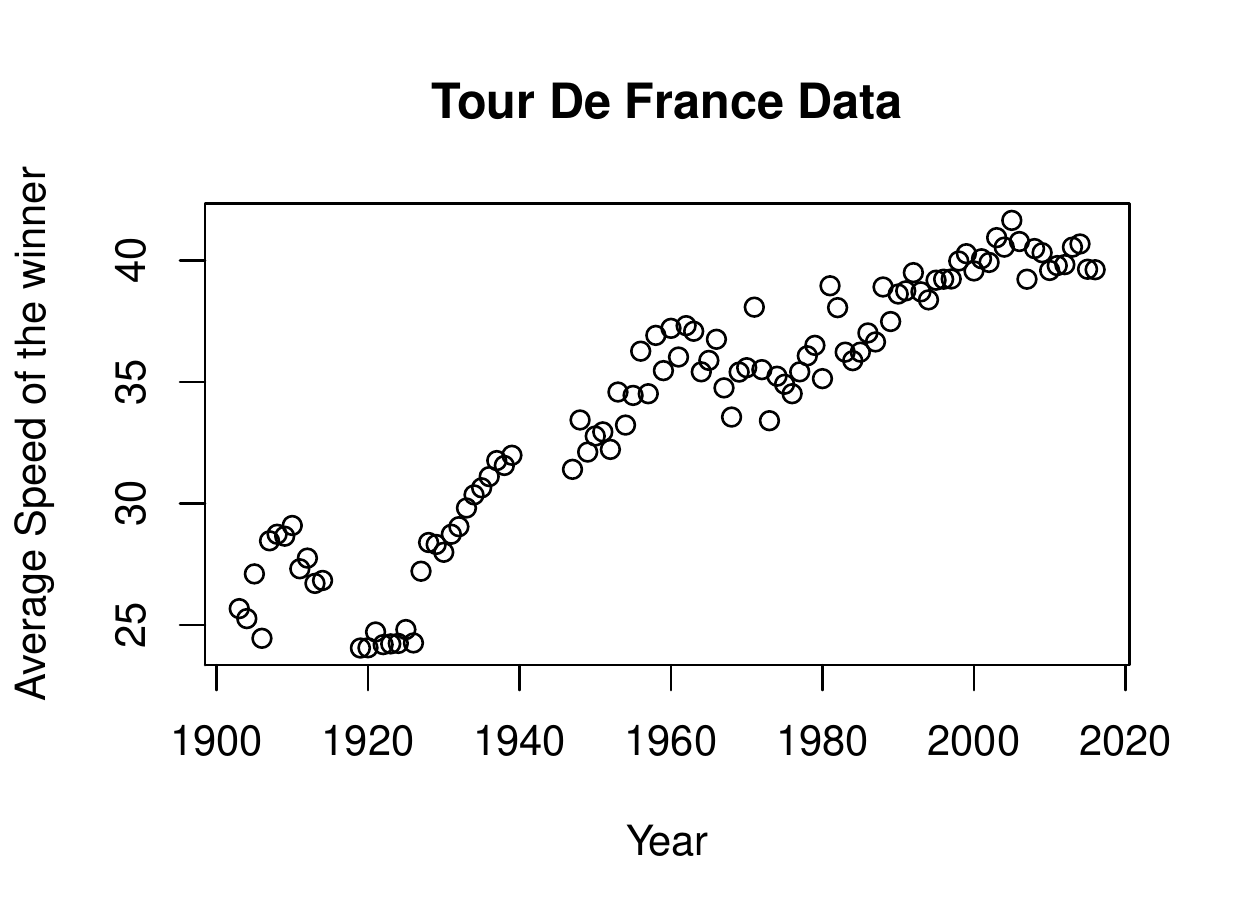}
        \caption{Speed VS Year}
        \label{SpeedYear}
    \end{subfigure}
    \hfill
    \begin{subfigure}[b]{0.4\textwidth}
        \centering
        \includegraphics[width = \textwidth]{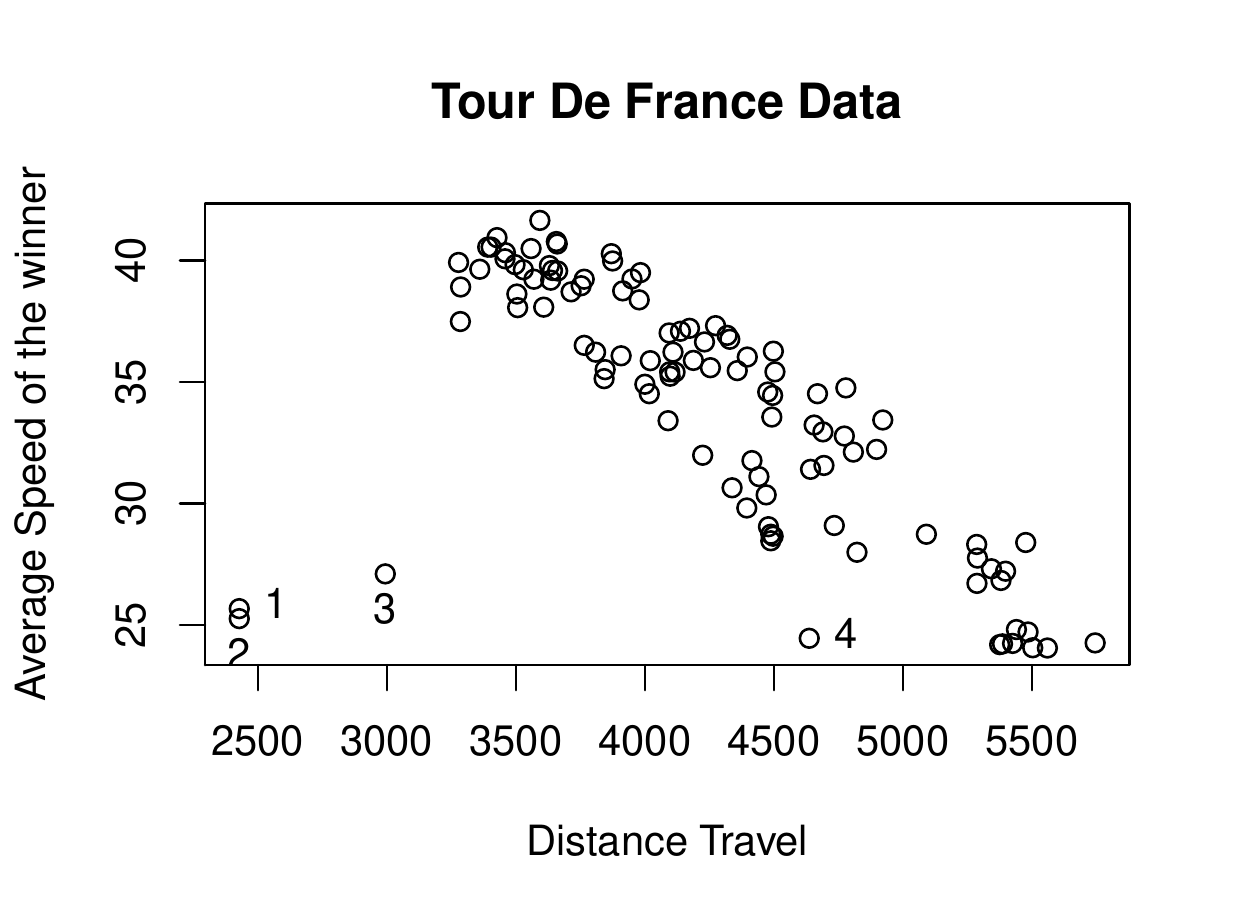}
        \caption{Speed VS Distance}
        \label{SpeedDist}
    \end{subfigure}
    \caption{Tour De France from 1903 to 2016}
    \label{Descriptive}
\end{figure}

The full data set has $n=103$ dependent data points. The data points in the set are dependent because many cyclists competed in the Tour for more than one year. Indeed, one may argue that it is better analyzed as an auto-regressive time series. Here we ignore the dependence structure for the sake of comparing our method with independent data. Each data point has a value of the response variable, the average speed in kilometer per hour (km/h) of the winner (Speed) of the Tour from 1903 to 2016. However during World War 1 and 2 there was no Tour De France, so we do not have data points for those periods of time. We also see the effect of World War 1 on the speed of the winner of the tour.
The lowest speeds were after World War 1. This is probably due to the death of the young men. These data points are potential outliers. After World War 2, there was also a decrease in average winning speed, but the decrease was less than that after World War 1.
From Fig. \ref{SpeedYear}, we see that there is a curvilinear relationship between Speed and Year (Y). In Fig. \ref{SpeedDist}, we see that there is a linear relationship between Speed and Distance and that the variability of speed increases with the Distance (D). The four observations labeled 1-4 that are away from the bulk of the data are from the early years of the Tour and are also considered as potential outliers. The \sf{Tour De France} dataset also has information on the Age of the winner (A), the number of stages won by the winner (S), and the distance (Km) (D) of the \sf {Tour De France}. We use these data in a separate analysis with a different model list in Sec. \ref{AnalTour4}.

\begin{table}[ht]
  \centering
  \caption{Correlation between covariates}\label{Corr}

\begin{tabular}{|c|c|c|c|c|c|}
  \hline
  & Speed &Year & Distance  & Stages & Age \\
  \hline\hline
  Speed & 1 & 0.94 & -0.69 & -0.18 & 0.07 \\
    Year & 0.94 & 1  & -0.63 & -0.22 &  0.21 \\
    Distance & -0.69 & -0.63  & 1 & 0.11 & -0.08 \\
    Stages & -0.18 & -0.22 & 0.11 & 1 & -0.20 \\
    Age & 0.07 & 0.21  &  0.08 & -0.20 & 1 \\
    \hline
\end{tabular}
\end{table}

\begin{table}[h]
  \centering
  \caption{Correlation between Speed and $D^2$, $Y^2$ and $Y:D$}\label{SpeedCorr}

\begin{tabular}{|c|c|c|c|}
  \hline
   &$Distance^2$ & $Year^2$  & Year:Distance \\
  \hline\hline
  Speed & -0.53 & -0.24 & 0.02 \\
  \hline
\end{tabular}
\end{table}

Tables \ref{Corr} and \ref{SpeedCorr} give us Pearson correlation between our variables. We see that there is a strong correlation between Speed and Year (0.94). There is a negative but strong linear relationship between speed and distance (-0.69) and weak linear relationship between speed and stages won (-0.18), whereas, there is a positive and very weak relationship between speed and Age of the winner (0.07). We also notice that the correlation between speed and the squared distance is not strong (-0.53), quite weak between speed and $Year^2$ $(-0.24)$ and almost absent between speed and the interaction between Year and distance $(0.02)$. Tables similar to \ref{Corr} and \ref{SpeedCorr} can be computed using Kendall's $\tau$ for instance; the results are qualitatively the same.
\newline

\begin{figure}
    \centering
    \begin{subfigure}[b]{0.4\textwidth}
        \centering
        \includegraphics[width = \textwidth]{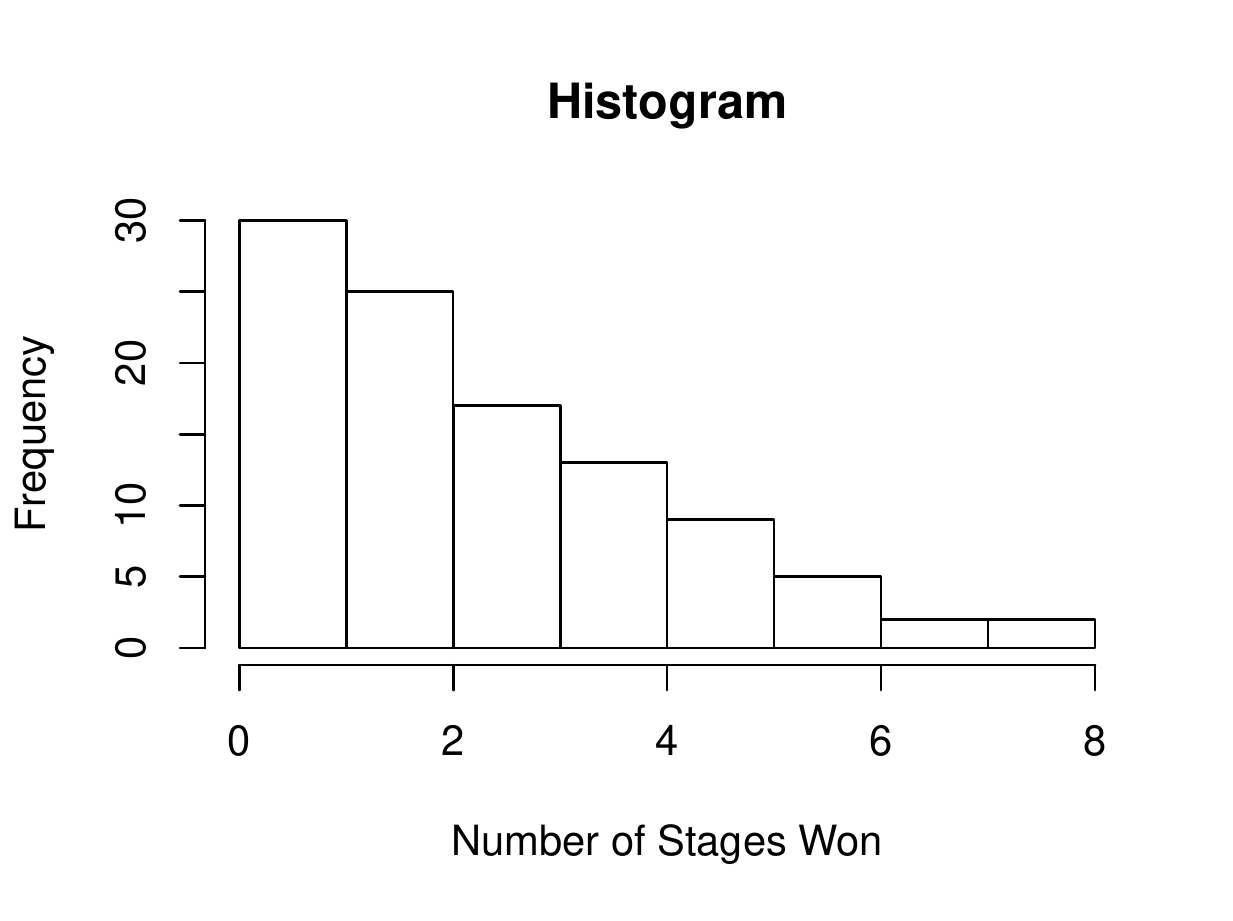}
        \caption{Number of Stages won by the winner}
        \label{Stages}
    \end{subfigure}
    \hfill
    \begin{subfigure}[b]{0.4\textwidth}
        \centering
        \includegraphics[width = \textwidth]{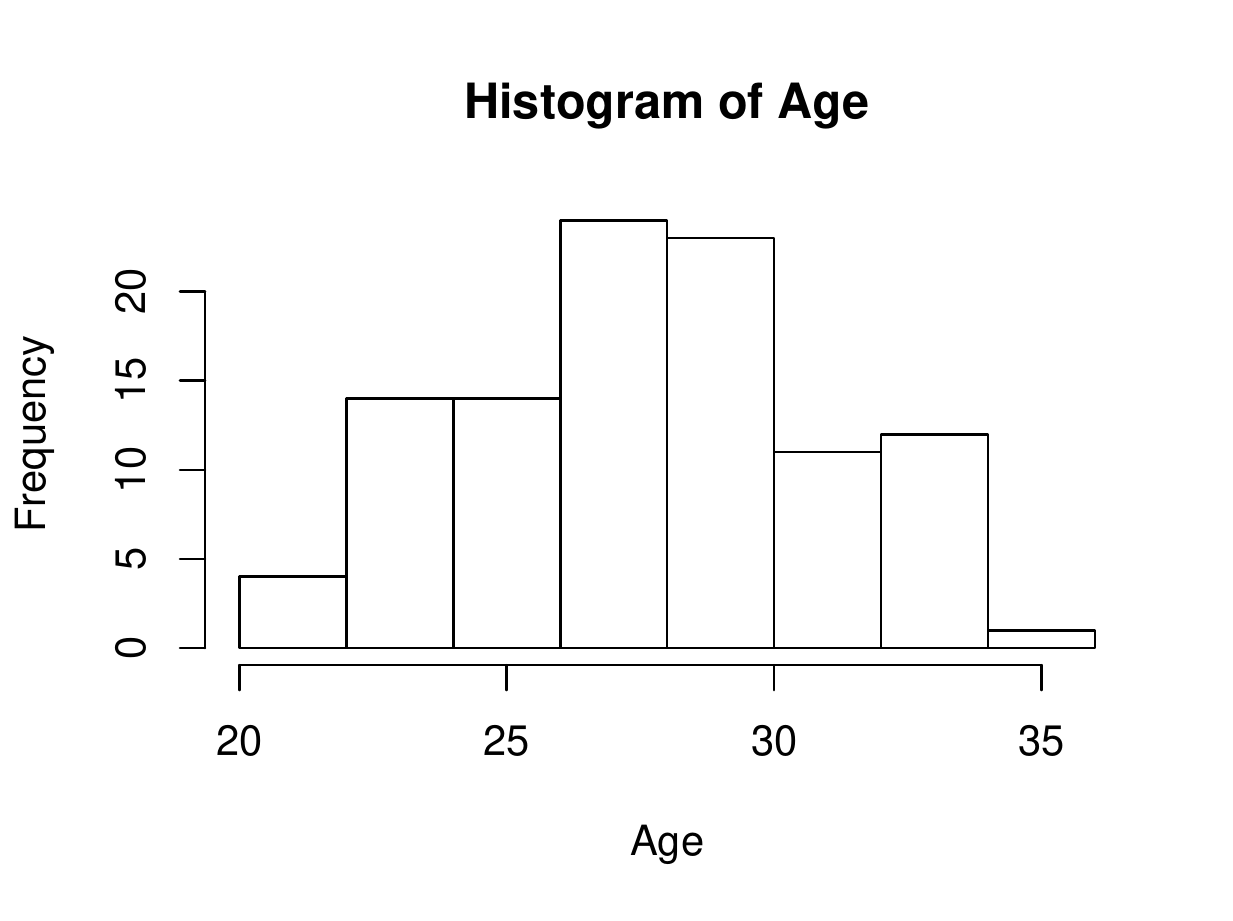}
        \caption{Estimates of $\hat{h}$}
        \label{Age}
    \end{subfigure}
    \caption{Histogram of Stages won and Age of the winners}
    \label{Hist}
\end{figure}

From Figure \ref{Hist}, we see that most of the winners of the tour had age ranging from 25 to 30 years and the number of stages that they won ranged between 0 and 1. This suggest that the winners were the ones who went fast enough to do well but not so as to over exert themselves, becoming tired or injured as a result.

\section{Analysis of a nested and non-nested collection of model lists of Tour De France }
\label{AnalTour5}
\subsection{A nested model list}\label{AnalTour5Nested}
We identify a nested model list using $Y$, $D$, $Y^2$, $D^2$ and $Y:D$ as covariates. Because the size of the dataset is small, we can only use a small model list. We order the variables using the Smoothly Clipped Absolute Deviation (SCAD) \cite{Fan&Li} shrinkage method because it perturbs parameter estimates the least and satisfies an oracle property.

Under SCAD\cite{Fan&Li}, the order of inclusion of variables is $Y$, $D$, $D^2$, $Y^2$, and $Y:D$. We therefore fit five different models.
\begin{table}[ht]
  \centering
  \caption{Direct Implementation of Vapnik's method for the nested models}\label{NestTheo5Direct}

        \begin{tabular}{|c|c|c|c|c|}
  \hline
   Model & $\hat{h}$ & $\widehat{ERM}_{1}$  & $\widehat{ERM}_{2}$ & $BIC$\\
  \hline\hline
  $Y$ & 20 & 316.74 & 495.22 & 419.14 \\
    $Y,~D$ & 20  & 218.11 & 451.06 &  411.13 \\
    $Y,~D,~D^2$ & 20 & 176.09  & 317.27 & 365.68  \\
    $Y, D,~D^2,~Y^2$ & 20 & 172.75 & 312.89 & 368.24 \\
    $Y, D,~D^2,~Y^2,~Y:D$ & 20 & 172.05  &  311.97 & 372.43 \\
    \hline
\end{tabular}
\end{table}

\begin{table}[ht]
  \centering
  \caption{Estimates of  $\hat{h}$, $\widehat{ERM}_{1}$, $\widehat{ERM}_{2}$ and $BIC$ of the nested model using our method}\label{NestourMethod}

\begin{tabular}{|c|c|c|c|c|}
  \hline
   Model & $\hat{h}$ & $\widehat{ERM}_{1}$  & $\widehat{ERM}_{2}$ & $BIC$\\
  \hline
  $Y$ & 4 & 16.42 & 44.95 & 79.67 \\
    $Y,~D$ & 4  & 15.10 & 42.83 &  71.66 \\
    $Y,~D,~D^2$ & 4 & 11.21  & 36.37 & $\bf 26.21$  \\
    $Y, D,~D^2,~Y^2$ &\bf{4}  & 11.09 & 36.16 & 28.77 \\
    $Y, D,~D^2,~Y^2,~Y:D$ & 4 & $\bf 11.06$  &  $\bf 36.11$ & 32.96 \\
    \hline
\end{tabular}
\end{table}

The estimate of VCD requires that we choose some values to use in the analysis. Since the size of our dataset is 103, we set $m=10$, we choose to vary $N_{L}$  from $20$ to $100$ by $10$. And, we set the number of bootstrap samples to be $b_{1}=b_{2} = 50$.

Tables \ref{NestTheo5Direct} and \ref{NestourMethod} give us estimates of $\hat{h}$, $\widehat{ERM}_{1}$, $\widehat{ERM}_{2}$ and $BIC$ of the nested models using Algorithm \ref{VapnikAlgorithm} and Algorithm \ref{Algo1} respectively. It is seen that Vapnik's original method is helpful only if it is reasonable to surmise that there are exactly $15$ missing variables. Our method uniquely identifies one of the models on the list. Even though there is likely no model for Tour De France dataset that is accurate to infinite precision, our method is giving a useful result.

From Table \ref{NestTheo5Direct}, we observe that $\hat{h} = 20$, the smallest design point. This indicates that the objective function used to estimate $\hat{h}$ did not give a useful answer. This is the problem that we face most of the time we directly implement Algorithm \ref{VapnikAlgorithm} to estimate the LHS of \ref{objfn} in Theorem \ref{Theo5}.

In Table \ref{NestourMethod}, we see $\hat{h}=4$ no matter which model is chosen. This indicates that we need 4 explanatory variables to explain our response. Thus, the best model is the one with $Y$, $D$, $D^2$, and $Y^2$. However, if we use $BIC$ for model selection we select the model with $Y$, $D$, and $D^2$. We prefer the model chosen by $\hat{h}$ because there does not appear to be as strong a curvilinear relationship between speed and distance as there appears to be between speed and Year e.g via $Y^2$. We attribute the lesser performance of $BIC$ to the fact that its derivation rests heavily on the assumption that the data are independent. Minimizing $\widehat{ERM}_{1}$ and $\widehat{ERM}_{2}$ leads to the model with five parameters. There is nothing a priori wrong with this, but a smaller model (of size 4 using $\hat{h}$) is preferred when justifiable. Alternatively, we may regard the difference among the $3^{th}$, $4^{th}$ and $5^{th}$ models as trivial for $\widehat{ERM}_{1}$, $\widehat{ERM}_{2}$, so they effectively lead to the model with $Y$, $D$ and $D^2$ as terms, since $\widehat{ERM}_{1}$ and $\widehat{ERM}_{2}$ both have a large decrease from the $2$ term to the 3 term model. That is, they give the same result as $BIC$ which we think is inferior to the model chose by $\hat{h}$.

\subsection{The non-nested cases}
Now, we replicate what we did above, but without using SCAD to nest the models. Since VCD and real dimension of the parameters in the regression function coincide for linear models, we considered all first and second order terms for this case. We fitted $31~$ different models, $10~$ models with only two variables, $10$ with $3$ variables, $5$ with $4$ variables, $5$ with $1$ variable and $1$ with $5$ variables. For each model in the set, we estimate $\hat{h} $, the upper bound of Propositions \ref{subProp1} and \ref{subProp2} and $BIC$.
\newline

\begin{table}[ht]
  \centering
  \caption{Models of size one using Year and Distance as covariates}\label{Set1YD1}

\begin{tabular}{|c|c|c|c|c|}
  \hline
  Model Size &$\widehat{h}$ & $\widehat{ERM}_{1} $ & $\widehat{ERM}_{2}$ & BIC \\
  \hline\hline
  $Y$ & 4 & 16.42  & 44.95 & 79.67 \\
     $D$ & 4 & 59.05  & 105.39 & 239.79 \\
    $D^2$ & 4 & 78.08  & 130.13 & 270.86 \\
    $Y^2$ & 4 & $\bf 14.56 $  & $\bf 41.94 $ & $\bf 65.96 $ \\
    $Y:D$ & 4 & 59.05  & 105.39 & 239.79 \\
    \hline
\end{tabular}
\end{table}

\begin{table}[ht]
  \centering
  \caption{Models of size two using Year and Distance as covariates}\label{Set1YD2}

\begin{tabular}{|c|c|c|c|c|}
  \hline
  Model Size &$\widehat{h}$ & $\widehat{ERM}_{1} $ & $\widehat{ERM}_{2}$ & BIC \\
  \hline\hline
  $Y,D $ & 4 & 15.10  & 42.83 & 71.66 \\
    $Y, Y^2 $ & 4 & 16.42  & 44.95 & 79.67 \\
    $Y, D^2$ & 4 & 15.10  & 42.83 & 71.66 \\
    $Y, Y:D$ & 4 & $\bf 14.56 $  & $\bf 41.94 $ & $\bf 65.96 $ \\
    $D, Y^2$ & 4 & 46.46  & 88.51 & 217.12 \\
    $D, D^2$ & 4 & 35.86  & 73.83 & 186.69 \\
    $D, Y:D$ & 4 & 46.14  & 88.07 & 216.30 \\
    $Y^2, D^2$ & 4 & 100.97  & 159.03 & 298.92 \\
    $Y^2, Y:D$ & 4 & 100.91  & 158.96 & 303.50 \\
    $D^2, Y:D$ & 4 & 78.03  & 130.06 & 275.42 \\
    \hline
\end{tabular}
\end{table}

\begin{table}[ht]
  \centering
  \caption{Models of size three using Year and Distance as covariates}
\label{Set1YD3}

\begin{tabular}{|c|c|c|c|c|}
  \hline
  Model Size &$\widehat{h}$ & $\widehat{ERM}_{1} $ & $\widehat{ERM}_{2}$ & BIC \\
  \hline\hline
  $Y$, $Y:D$, $D^2$ & 4 & 12.91 & 37.90 & 40.16\\
    $D$, $Y:D$, $D^2$ & 4 & 24.82 & 57.86 & 145.74 \\
    $Y$, $D^2$, $Y^2$ & 4 & 13.78 & 40.67 & 61.87 \\
    $Y$, $Y^2$, $Y:D$ & 4 & 14.26 & 41.46 & 67.33\\
  $Y,D, Y^2 $ & 4 & 13.48  & 40.17 & 58.25 \\
    $Y, D,D^2 $ & 4 & $\bf 11.21 $  & $\bf 36.37$ & $\bf 26.21 $\\
    $Y, D, Y:D$ & 4 & 14.51  & 41.86 & 70.05 \\
    $D, Y^2, D^2$ & 4 & 34.91  & 72.49 & 188.12 \\
    $D, Y^2, Y:D$ & 4 & 23.79  & 56.31 & 140.20 \\
    $Y^2, D^2, Y:D$ & 4 & 77.59  & 129.50 & 279.43 \\
  \hline
\end{tabular}
\end{table}

\begin{table}[ht]
  \centering
  \caption{Models of size four using Year and Distance as covariates}
\label{Set1YD4}

\begin{tabular}{|c|c|c|c|c|}
  \hline
  Model Size &$\widehat{h}$ & $\widehat{ERM}_{1} $ & $\widehat{ERM}_{2}$ & BIC \\
  \hline\hline
  $Y,D, Y^2, D^2 $ & 4 & $\bf 11.09 $  & $\bf 36.16 $ & $\bf 28.77$ \\
    $Y, D, D^2, Y:D$ & 4 & 11.20  & 36.35 & 30.67 \\
    $Y, Y^2, D^2, Y:D$ & 4 & 11.99  & 37.70 & 43.09 \\
    $D, Y^2, D^2, Y:D$ & 4 & 19.03  & 49.06 & 114.79 \\
    $ Y, D, Y^2, Y:D$  & 4 & 11.93 & 39.12 & 45.79\\
  \hline
\end{tabular}
\end{table}

Next, we compare $\hat{h}$, $\widehat{ERM}_{1}$, $\widehat{ERM}_{2}$ and $BIC$, for the case of non-nested model lists of size $31$. We emphasize here that for the cases of non-nested model lists, common practice is to make $\widehat{ERM}_{1}$ and $\widehat{ERM}_{2}$ as small as possible. For all these cases, Tables \ref{Set1YD1}-\ref{Set1YD4} show how  $\hat{h}$, $\widehat{ERM}_{1}$, $\widehat{ERM}_{2}$ and $BIC$ behave.
From Tables  \ref{Set1YD1}-\ref{Set1YD4}, we observe that, for any model in this set of models, the estimated VCD is $4$. This indicates the robustness in the estimated VCD for this dataset. We can also see that the behavior of the estimated upper bounds from using either Propositions \ref{subProp1} or \ref{subProp2} ($\widehat{ERM}_{1}$ or $\widehat{ERM}_{2}$) are the same within each model size. More specifically within each size class, the minimum of $\widehat{ERM}_{1}$ and $\widehat{ERM}_{2}$ occurs at the same model. We note that the overall minima of $\widehat{ERM}_{1}$ and $\widehat{ERM}_{2}$ occurs for $Y$, $D$, $Y^2$, $D^2$, $Y:D$ (See the last line of Table \ref{NestourMethod}, 11.06 and 36.11 respectively.), but the minimum of the models of size four (see Table \ref{Set1YD4} is 11.09 and 36.16 respectively) are nearly the same and nearly lead to the same model as in the nested case (See Table \ref{NestourMethod} where the values are 11.09 and 36.16 respectively.).

The behavior of the $BIC$ is similar to that of the estimated upper bounds $\widehat{ERM}_{1}$ and $\widehat{ERM}_{2}$. The smallest BIC model always matches the model with the smallest estimated upper bound, for each model size. However, the smallest $BIC$ over all model sizes occurs for a model of size three, namely $Y,~ D, $ and $D^2$. This is not the case when we look at $\widehat{ERM}_{1}$ or $\widehat{ERM}_{2}$.
As observed before, if we look at graph (a) on Figure \ref{Descriptive}, we see that there is a curvature when we regress Speed on Year; this indicates that the best model should be non-linear in Year. So the best model appears to be the model given by $\hat{h}$, $\widehat{ERM}_{1}$ or $\widehat{ERM}_{2}$ (since they give nearly identical values for the model of size 5 and the best model of size 4) not $BIC$. We can improve the result from $\widehat{ERM}_{1}$ and $\widehat{ERM}_{2}$ if we eliminate the extra term $Y:D$ that they include. Looking at Table \ref{SpeedCorr}, the correlation between $Y:D$ and the average speed is nearly zero. So adding this information to $\widehat{ERM}_{1}$ or $\widehat{ERM}_{2}$ improves them to $\hat{h}$ by allowing us to eliminate the extra term $Y:D$. In simulations, we argued that $\hat{h}$ was better than $\widehat{ERM}_{1}$, $\widehat{ERM}_{2}$. For efficiency reason; here we argue that $\hat{h}$ is more robust against deviation from independence than $BIC$.

\subsection{ Analysis of The Tour De France dataset with outliers removed for nested and non-nested cases.}

The observations just after World War One may be outliers. Our goal is to see how our estimate will behave after we remove these observations.
\newline
As before, we analyze this reduced dataset in two steps. In the first step, we identify the nested model lists by SCAD \cite{Fan&Li}. Then, for each model in the class, we will estimate $\hat{h}$, and obtain $\widehat{ERM}_{1}$, $\widehat{ERM}_{2}$ and $BIC$. In the second step, without defining any structure, we fit all models of size 1, 2, 3, 4, and 5.

\begin{table}[ht]
  \centering
  \caption{Correlation between covariates with outliers removed}\label{CorrNoOutliers}

\begin{tabular}{|c|c|c|c|c|c|}
  \hline
  & Speed &Year & Distance  & Stages won& Age \\
  \hline\hline
  Speed & 1 & 0.96 & -0.56 & -0.24 & 0.37 \\
    Year & 0.96 & 1  & -0.57 & -0.25 &  0.40 \\
    Distance & -0.56 & -0.57  & 1 & 0.15 & -0.15 \\
    Stages won & -0.24 & -0.25 & 0.15 & 1 & -0.18 \\
    Age & 0.37 & 0.40  &  -0.15 & -0.18 & 1 \\
    \hline
\end{tabular}
\end{table}

\begin{table}[ht]
  \centering
  \caption{Correlation between Speed and $D^2$, $Y^2$ and $Y:D$ with outliers removed}\label{SpeedCorrSqrtNoOutliers}

\begin{tabular}{|c|c|c|c|}
  \hline
   &$Distance^2$ & $Year^2$  & Year:Distance \\
  \hline\hline
  Speed & -0.5 & -0.38 & -0.08 \\
  \hline
\end{tabular}
\end{table}

We first investigate the correlation between the average speed and our covariates before and after removing observations after world war one. Table \ref{CorrNoOutliers} gives us the same information as Table \ref{Corr}, but with outliers removed. When comparing the two tables, we see that the correlation between Speed and Year, Speed and Stage won, and between speed and Age increases when outliers are removed. However, it decreases between speed and Distance. Moreover, when we look at correlation between Speed and $D^2$ (-0.53), we see a slight decrease when outliers are removed, and a slight increase on the correlation between speed and $Y^2$ (-0.38, Table \ref{SpeedCorrSqrtNoOutliers}). We still did not have much correlation between speed and the interaction between Year and Distance(-0.08).

\begin{table}[ht]
  \centering
  \caption{Nested models using $Year$ and $Distance$ as covariates with outliers removed}
\label{Set2YDNest}

\begin{tabular}{|c|c|c|c|c|}
  \hline
  Model Size &$\widehat{h}$ & $\widehat{ERM}_{1} $ & $\widehat{ERM}_{2}$ & BIC \\
  \hline
    $Y$ & 4 & 12.87 & 40.72 & 44.55\\
    $Y$, $D^2$ & 4 & 12.01 & 39.26 & 37.84\\
    $Y$, $D^2$, $D$ & 4 & 11.66 & 38.36 & 37.41 \\
    $Y$, $D^2$, $D$, $Y^2$ & \bf{4} & 11.48 & 38.34 & 39.20 \\
    $Y$, $D^2$, $D$, $Y^2$, $Y:D$ & 4 & 11.35 & 38.13 & 41.83 \\
  \hline
\end{tabular}
\end{table}

\begin{table}[ht]
  \centering
  \caption{Models of size one using Year and Distance as covariates with outliers removed}
\label{Set2YD1}

\begin{tabular}{|c|c|c|c|c|}
  \hline
  Model Size &$\widehat{h}$ & $\widehat{ERM}_{1} $ & $\widehat{ERM}_{2}$ & BIC \\
  \hline\hline
    $Y$ & 4 & $\bf 12.87$ & $\bf 40.72$ & $\bf 44.55$\\
    $D$ & 4 & 69.60 & 121.45 & 246.36 \\
    $Y^2$ & 4 & 85.40 & 141.85 & 267.20 \\
    $Y:D$ & 4 & 98.61 & 158.60 & 281.68\\
   $D^2 $ & 4 & 75.14 & 128.66 & 254.21 \\
  \hline
\end{tabular}
\end{table}

\begin{table}[ht]
  \centering
  \caption{Models of size two using Year and Distance as covariates with  outliers removed}
\label{Set2YD2}

\begin{tabular}{|c|c|c|c|c|}
  \hline
  Model Size &$\widehat{h}$ & $\widehat{ERM}_{1} $ & $\widehat{ERM}_{2}$ & BIC \\
  \hline
    $Y$, $D$ & 4 & 12.83 & 40.64 & 48.57\\
    $Y$,$Y^2$ & 4 & 12.39 & 39.91 & 43.04 \\
    $Y$, $Y:D$ & 4 & 12.76 & 40.52 & 47.69 \\
    $Y$, $D^2$ & 4 & $\bf 12.01 $ & $\bf 39.26 $ & $\bf 37.84$ \\
    $D, Y^2$ & 4 & 49.11  & 94.14 & 214.56 \\
    $D, Y:D $ & 4 & 52.92  & 99.31 & 222.46\\
    $D, D^2$ & 4 & 44.22  & 87.43 & 203.35 \\
    $Y^2, Y:D$ & 4 & 84.23  & 140.35 & 270.36 \\
    $Y^2, D^2$ & 4 & 74.07  & 127.27 & 257.29 \\
    $Y:D, D^2$ & 4 & 75.14  & 128.65 & 258.75 \\
  \hline
\end{tabular}
\end{table}

\begin{table}[ht]
  \centering
  \caption{Models of size three using Year and Distance as covariates with  outliers removed}
\label{Set2YD3}

\begin{tabular}{|c|c|c|c|c|}
  \hline
  Model Size &$\widehat{h}$ & $\widehat{ERM}_{1} $ & $\widehat{ERM}_{2}$ & BIC \\
  \hline
    $Y$, $D$, $Y^2$ & 4 & 12.09 & 39.40 & 43.54\\
    $Y$, $D$, $Y:D$ & 4 & 12.76 & 40.52 & 52.22 \\
    $Y$, $D$, $D^2$ & 4 &$\bf 11.66 $ & $\bf 38.66$ & $\bf 37.41$ \\
    $D$, $Y^2$, $Y:D$ & 4 & 23.93 & 58.24 & 138.23 \\
    $D, Y^2, D^2$ & 4 & 40.50  & 82.25 & 198.38 \\
    $Y^2, Y:D, D^2 $ & 4 & 73.98  & 127.15 & 261.72\\
    $Y, Y:D, D^2$ & 4 & 11.83  & 38.94 & 39.80 \\
    $D, Y:D, D^2$ & 4 & 32.26  & 71.26 & 174.88 \\
    $Y, Y^2, D^2$ & 4 & 11.95  & 39.16 & 41.57 \\
    $Y, Y^2, Y:D$ & 4 & 12.34 & 39.81 & 46.79 \\
  \hline
\end{tabular}
\end{table}

\begin{table}[ht]
  \centering
  \caption{Models of size four and five using Year and Distance as covariates with outliers removed}
\label{Set2YD4}

\begin{tabular}{|c|c|c|c|c|}
  \hline
  Model Size &$\widehat{h}$ & $\widehat{ERM}_{1} $ & $\widehat{ERM}_{2}$ & BIC \\
  \hline
    $Y$, $D$, $Y^2$, $D^2$ & 4 & $\bf 11.48$ & $\bf 38.34$ & $ \bf 39.20$\\
    $Y$, $Y^2$, $D$, $Y:D$ & 4 & 11.93 & 39.12 & 45.79\\
    $Y^2$, $D$, $D^2$, $Y:D$ & 4 & 21.99 & 55.28 & 132.40 \\
    $Y$, $D$, $D^2$, $Y:D$ & 4 & 11.66 & 38.66 & 41.97 \\
    $Y, Y^2, D^2, Y:D$ & 4 & 11.81 & 38.90 & 44.02 \\
    $Y, D, Y^2, D^2, Y:D$ & 4 & $\bf 11.35$ & $\bf 38.13$ &$\bf 41.83$\\
  \hline
\end{tabular}
\end{table}
Under SCAD\cite{Fan&Li}, the order of inclusion of our covariates is: $Y$, $D^2$, $D$, $Y^2$ and $Y:D$. This order is different from when we use all data points. Recall that, when we used all data points, $D$ was included before $D^2$ and $D^2$ was included after $y^2$. With this new ordering we fit 5 different models.

From Table \ref{Set2YDNest}, if we choose a model using $\hat{h}$, we get the same answer as in Sec. \ref{AnalTour5Nested}, the model with four variables: $Y$, $D^2$, $D$, $Y^2$. The interaction between Year and distance (Y:D) is not included because of the low correlation between Speed and $Y:D$ (-0.08).  Also as before, $BIC$ indicates a model of size 3 having $Y$, $D^2$, $D$ as covariates and  $\widehat{ERM}_{1}$ and $\widehat{ERM}_{2}$ chose a model of size 5. The reasoning in Subsec. \ref{AnalTour5Nested} for why we think that the $\hat{h}$ chosen model is best continues to hold.

Next, we turn to the non-nested cases. Tables \ref{Set2YD1}-\ref{Set2YD4} give analogous results to those in Tables \ref{Set1YD1}--\ref{Set1YD4}  and our observations about $\hat{h}$, $\widehat{ERM}_{1}$, $\widehat{ERM}_{2}$ and $BIC$ are unchanged. In fact, we see that removing outliers did not affect $\hat{h}$, however, there is a change in the values of  $\widehat{ERM}_{1}$, $\widehat{ERM}_{2}$ and $BIC$; this suggests that $\hat{h}$ may be relatively robust against outliers; Alternatively it may only reflect that $\hat{h}$ must be an integer.

For the non-nested case, if we use either $\widehat{ERM}_{1}$ or $\widehat{ERM}_{2}$ for model selection, we will choose the model with the smallest value. Now, from Tables \ref{Set2YD1}-\ref{Set2YD4}, it is seen that both $\widehat{ERM}_{1}$ and $\widehat{ERM}_{2}$ pick the most complex model with all five parameters. $BIC$ continues to choose $Y$, $D^2$ $D$. Again we invoke the argument of Subsec. \ref{AnalTour5Nested} to argue that the model chosen using $\hat{h}$ is the best.


\section{Analysis of a nested and non-nested collection of model lists of Tour De France using different covariates}
\label{AnalTour4}

\subsection{Outliers retained}\label{OutRetained}
In this Subsection, we use $\left(Y, D, S, A\right)$ as covariates, estimate $\hat{h}$, obtain $\widehat{ERM}_{1}$, $\widehat{ERM}_{2}$ and $BIC$ using for two cases: When covariates are nested and when they are not.
\newline
We start by identifying the nested structure of the covariates. Under SCAD \cite{Fan&Li}, the order of inclusion is $Year$, $Age$, $Distance$ and $Stages won$. With this order, we fit 4 different models. For the non-nested case, we fit $15$ models. In this list, we fit $6$ with $2$ variables, $4$ with $3$ variables, $4$ with $1$ variable and $1$ with all four variables.

In Table \ref{NestedAll}, we have the nested models. We see that $\hat{h} = 4$ for all models. This indicates that a model with all four variables should give good predictions of the average speed of the winner of the Tour. However, if we use $\widehat{ERM}_{1}$, $\widehat{ERM}_{2}$ or $BIC$, the best model will have $Year$, $Age$ and  $Distance$ as covariates. However, on panel (a) of Figure \ref{Hist}, we see that most of the winners of the Tour won between zero and one stage, so stages won is nearly constant. This suggests that the model chosen by $\widehat{ERM}_{1}$, $\widehat{ERM}_{2}$ or $BIC$ is better than the one chose by $\hat{h}$.
\newline

\begin{table}[ht]
  \centering
  \caption{Nested models using all covariates}
\label{NestedAll}

\begin{tabular}{|c|c|c|c|c|}
  \hline
  Model Size &$\widehat{h}$ & $\widehat{ERM}_{1} $ & $\widehat{ERM}_{2}$ & BIC \\
  \hline\hline
  $Y$ & 4 & $ 16.42 $ &  44.95  &  79.67 \\
    $Y$, $A$ & 4 & 14.79  & 42.32 & 68.45 \\
    $Y$, $A$, $D$ & 4 &  $ \bf 14.11$ & $\bf 41.22 $ &  $\bf 65.71 $\\
    $Y$, $A$, $D$, $S$ & 4 & 14.31  & 41.22 & 70.33 \\
    \hline
\end{tabular}
\end{table}

\begin{table}[ht]
  \centering
  \caption{Models of size one using all covariates}
\label{Set1ALL1}

\begin{tabular}{|c|c|c|c|c|}
  \hline
  Model Size &$\widehat{h}$ & $\widehat{ERM}_{1} $ & $\widehat{ERM}_{2}$ & BIC \\
  \hline\hline
  $Y$ & 4 & $\bf 16.42 $ & $\bf  44.95 $ & $\bf 79.67$ \\
    $D$ & 4 & 59.05  & 105.39 & 239.79 \\
    $S$ & 4 &  103.70   & 162.43 &  301.81 \\
    $A$ & 4 & 106.42  & 165.82 & 304.62 \\
    \hline
\end{tabular}
\end{table}

\begin{table}[ht]
  \centering
  \caption{Models of size two using all covariates}
\label{Set1ALL2}

\begin{tabular}{|c|c|c|c|c|}
  \hline
  Model Size &$\widehat{h}$ & $\widehat{ERM}_{1} $ & $\widehat{ERM}_{2}$ & BIC \\
  \hline\hline
  $Y,D $ & 4 & 15.10  & 42.82 & 71.66 \\
    $Y,S$ & 4 & 16.36  & 44.85 & 83.79 \\
    $Y, A$ & 4 & $\bf 14.79 $  & $\bf 42.32 $ & $\bf 68.45$ \\
    $D, S$ & 4 & 57.85  & 103.80 & 242.12 \\
    $D, A $& 4 & 57.32  & 103.09 & 241.08\\
    $S, A$ & 4 & 103.54  & 162.23 & 306.28\\
  \hline
\end{tabular}
\end{table}

\begin{table}[ht]
  \centering
  \caption{Models of size three using all covariates}
\label{Set1ALL3}

\begin{tabular}{|c|c|c|c|c|}
  \hline
  Model Size &$\widehat{h}$ & $\widehat{ERM}_{1} $ & $\widehat{ERM}_{2}$ & BIC \\
  \hline\hline
  $Y,D, S $ & 4 & 15.06  & 42.77 & 75.91 \\
    $Y,D, A $ & 4 & $\bf 14.11 $ & $\bf 41.22$ & $\bf 65.71 $\\
    $Y, S, A$ & 4 & 14.79  & 42.32 & 73.07 \\
    $D, S, A$ & 4 & 56.64  & 102.19 & 244.35 \\
  \hline
\end{tabular}
\end{table}

\begin{table}[ht]
  \centering
  \caption{Models of size four using all covariates}
\label{Set1ALL4}

\begin{tabular}{|c|c|c|c|c|}
  \hline
  Model Size &$\widehat{h}$ & $\widehat{ERM}_{1} $ & $\widehat{ERM}_{2}$ & BIC \\
  \hline\hline
  $Y,D, S, A  $ & 4 & 14.31  & 41.22 & 70.33 \\
  \hline
\end{tabular}
\end{table}
For the non-nested case, we see from Tables \ref{Set1ALL1} -- \ref{Set1ALL4} that, for any conjectured model, the estimated VCD for the Tour de France data is $4$. Thus, if we choose a model using $\hat{h}$, the model that we will choose will again contain Year, Distance, Stages won and Age as covariates. $\widehat{ERM}_{2}$ behaves as before, and suggests that the best model can either have three parameters $Y$, $D$ and $A$, or four parameters $Y$, $D$, $S$ and $A$, since both models have same value for $\widehat{ERM}_{2}$. $BIC$ and $\widehat{ERM}_{1}$ confirm the smallest model, i.e with $Y$, $D$ and $A$.

We comment that the qualitative appearance of graphs of expected maximum difference VS VCD for $Y, D, S, A$ model list is the same in Figs. \ref{EstVCD1530} -- \ref{EstVCD6070}.

\subsection{Analysis of Tour De France dataset with outliers removed using Y, D, S, and A as covariates}
Recall that the outliers are the observations just after World War One. As we did earlier, we find the best model using this set of covariates in two settings. First, we use SCAD \cite{Fan&Li} to order the inclusion of variables into the model. Second, we fit all models of size 1, 2, 3, and 4. Under SCAD, the order of inclusion is $Year$, $Distance$, $Age$, and the number of $Stage~won$. This order is different from that with all observations where we found $Y, A, D, S$. With this order, we fit 4 models.

\begin{table}[ht]
  \centering
  \caption{Nested models using all covariates with outliers removed}
\label{NestedNoOutliers}

\begin{tabular}{|c|c|c|c|c|}
  \hline
  Model Size & $\widehat{h}$ & $\widehat{ERM}_{1} $ & $\widehat{ERM}_{2}$ & BIC \\
  \hline\hline
  $Y$ & 4 & 16.42  & 44.95 & 79.67 \\
  $Y$, $D$ & 4 & 12.83 & 40.64 & $\bf 48.57$ \\
  $Y$, $D$, $A$ & 4 & $\bf 12.81 $ & $\bf 40.61 $ & 52.89 \\
  $Y$, $D$, $A$, $S$ & 4 & $\bf 12.81 $ & $ \bf 40.61 $ & 57.41 \\
  \hline
\end{tabular}
\end{table}
From Table \ref{NestedNoOutliers}, as before, we observed $\hat{h} = 4$. So the best model under $\hat{h}$ has $Y$, $D$, $A$, and the number of Stages won ($S$). $\widehat{ERM}_{1}$ and $\widehat{ERM}_{2}$ do not discriminate between model of size three and four since they have the same value. In fact, on physical grounds, we can choose the model with $Y, D, A$ since most of the winners of the Tour won between zero and one stage. The $BIC$ picks the model with $Y$ and $D$, omitting $A$ and $S$.
In fact, under $\hat{h}$, $\widehat{ERM}_{1}$, $\widehat{ERM}_{2}$, the model with $Y$, $D$, $A$ is probably the best since $S$ cannot be very important and the value for it for $\hat{h}$, $\widehat{ERM}_{1}$, and $\widehat{ERM}_{2}$ are very similar. The $BIC$ model with $Y$ and $D$as variables means in particular that $A$ and $S$ have little effect. This is likely true for $S$ (see Fig \ref{Hist} (a)) but less so for $A$ (see Fig \ref{Hist} (b)). Thus in these cases, $\hat{h}$, $\widehat{ERM}_{1}$, $\widehat{ERM}_{2}$ lead with some extra reasoning to the same model $Y, D, A$. Since $\widehat{ERM}_{1}$, $\widehat{ERM}_{2}$ are derived using $\hat{h}$, this is not particularly surprising.

\begin{table}[ht]
  \centering
  \caption{Models of size one using all covariates with no outliers}
\label{Size1YDSA}

\begin{tabular}{|c|c|c|c|c|}
  \hline
  Model Size & $\widehat{h}$ & $\widehat{ERM}_{1} $ & $\widehat{ERM}_{2}$ & BIC \\
  \hline\hline
  $A$ & 4 & 86.32  & 143.10 & 268.35 \\
  $S$ & 4 &  93.64  &  152.33  & 276.50 \\
  \hline
\end{tabular}
\end{table}

\begin{table}[ht]
  \centering
  \caption{Models of size two using all covariates with outliers removed}
\label{Size2YDSA}

\begin{tabular}{|c|c|c|c|c|}
  \hline
  Model Size & $\widehat{h}$ & $\widehat{ERM}_{1} $ & $\widehat{ERM}_{2}$ & BIC \\
  \hline\hline
  $Y$, $D$ & 4 & $\bf12.82$  & $\bf40.64$ & $\bf 48.57 $ \\
  $Y$, $A$ & 4 & 12.84  &  40.67  & 48.79 \\
  $Y$, $S$ & 4 & 12.87 & 40.71 & 49.09 \\
  $D$, $S$ & 4 & 67.16 & 118.26 & 247.25 \\
  $D$, $A$ & 4 & 61.88 & 111.30 & 238.80 \\
  $S$, $A$ & 4 & 83.38 & 139.26 & 269.33\\
  \hline
\end{tabular}
\end{table}

\begin{table}[ht]
  \centering
  \caption{Models of size three using all covariates with outliers removed}
\label{Size3YDSA}

\begin{tabular}{|c|c|c|c|c|}
  \hline
  Model Size & $\widehat{h}$ & $\widehat{ERM}_{1} $ & $\widehat{ERM}_{2}$ & BIC \\
  \hline\hline
  $Y$, $D$, $A$ & 4 & $\bf 12.81 $  & $\bf 40.61$ & $\bf 52.89 $ \\
  $Y$, $D$, $S$ & 4 & $ 12.83 $ & $ 40.64 $ & 53.11 \\
  $D$, $A$, $S$ & 4 & 60.62 & 109.62 & 241.21 \\
  $Y$, $A$, $S$ & 4 & 12.85 & 40.67 & 53.31 \\
  $Y$, $D$, $A$, $S$ & 4 & $\bf 12.81 $ & $\bf 40.61 $ & 57.41 \\
  \hline
\end{tabular}
\end{table}

Thus, with and without outliers, analysis from Table \ref{Size1YDSA} to \ref{Size3YDSA} is qualitatively the same as the one we have done before.

\subsubsection{Summary}

In this chapter -- Chapter \ref{example} we implemented our theory on Tour de France data.
We fitted two set of models. The first set is based on the following covariates $Y, D, D^2, Y^2$, and  $Y:D$, the second contains $Y, D, S$, and $A$. Whichever of the two model lists is chosen, $\hat{h} = 4$. The best model to predict the average speed of the winner of the \sf{Tour De France} dataset should have $4$ variables.

Table \ref{SummaryYDAll} puts together the best model for each size across all sizes when using Year, distance and their combination as covariates. Table \ref{SummaryYDNoOutliers} gives us the same information but with outliers removed. A closer look at these tables informs us that outliers did not have any effect on $\hat{h}$, whereas it did have an effect on $\widehat{ERM}_{1}$, $\widehat{ERM}_{2}$ and $BIC$ since models of size one and two differ from both tables. Over those two tables, $BIC$ picks the model with 3 covariates, whereas $\widehat{ERM}_{1}$, and $\widehat{ERM}_{2}$ pick the biggest model. Note also that with $\widehat{ERM}_{1}$, and $\widehat{ERM}_{2}$ there is a graduate decrease not a sudden drop to indicate the location of the true model.

Table \ref{SummaryYDAS} and Table \ref{SummaryYDASNoOutliers} respectively put together the best model for each size across all sizes when using $Y, D, A$ and $S$ as covariates with and without outliers removed. We see that outliers did not have any effect on $\hat{h}$, however, it did have an effect on $\widehat{ERM}_{1}$, $\widehat{ERM}_{2}$ and $BIC$ since models of size two differ from both tables. Over those two tables, $BIC$ picks the model with only $Y$ as covariate. When outliers are not removed, $\widehat{ERM}_{1}$ picks the model of size 3, whereas when outliers are removed, $\widehat{ERM}_{1}$ did not discriminate between models of size three and four. Across both tables, $\widehat{ERM}_{2}$ did not discriminate between models of size three and four, its values are the same. We do note, however, that the nesting in Tables \ref{SummaryYDAS} and \ref{SummaryYDASNoOutliers} may indicate a sensitivity to outliers of SCAD.

If we have to choose between the best models from both set of models, with or without outliers, and base our choice either on $\widehat{ERM}_{1}$ or $\widehat{ERM}_{2}$, we will have to choose the best model from the first set, since $\widehat{ERM}_{1}$ or $\widehat{ERM}_{2}$ is smaller than that of the second set; this suggests that we will need $5$ explanatory variables to predict the Average speed of the winner of the Tour, and those variables are: $Y, D, Y^2, D^2$. and $Y:D$. Moreover, the correlation between the average speed of the winner of the tour and $Y:D$ is almost not existent, also because there is not a big difference between models of size four and five (with and without outliers) in $\widehat{ERM}_{1}$ and $\widehat{ERM}_{2}$. We will say that the best model is the model with $Y, D, Y^2, D^2$ -- The model chosen by using $\hat{h}$ in the nested case.

Concerning the second model class $(Y, D, A, S)$, with outliers (see Table \ref{NestedAll}), $\widehat{ERM}_{1}$, $\widehat{ERM}_{2}$ and $BIC$ pick the model containing $Y$, $A$ and $D$. However, when outliers are removed (see Table \ref{NestedNoOutliers}), $\widehat{ERM}_{1}$ and $\widehat{ERM}_{2}$ did not discriminate between models containing $Y$, $D$, $A$ and the model containing $Y$, $D$, $A$ and $S$. Since, on physical ground, stages won is just like a decoy, we will say that the model containing $Y$, $D$, and $A$ is the one that we should choose and we can use $\hat{h}$ to get this model by invoking the extra information that speed is independent of stages won.

We conclude that the best model is the model chosen by $\hat{h}$ with the following variables: $Y$, $D$, $Y^2$, $D^2$.


\begin{table}[ht]
  \centering
  \caption{Best models across all model sizes using $Y$, $D$ based on $\widehat{ERM}_{1}$, $\widehat{ERM}_{2}$ and $BIC$ }
\label{SummaryYDAll}

\begin{tabular}{|c|c|c|c|c|}
  \hline
  Model Size & $\widehat{h}$ & $\widehat{ERM}_{1} $ & $\widehat{ERM}_{2}$ & BIC \\
  \hline\hline
  $Y^2$ & 4 & 14.56  & 41.94 & 65.96 \\
  $Y$, $Y:D$ & 4 & 14.56 & 41.94 & 65.96 \\
  $Y$, $D$, $D^2$ & 4 & 11.21 & 36.37 & 26.21 \\
  $Y$, $D$, $Y^2$, $D^2$ & 4 & 11.09 & 36.16 & 28.77 \\
  $Y$, $D$, $Y^2$, $D^2$, $Y:D$ & 4 & 11.06 & 36.11 & 32.96 \\
  \hline
\end{tabular}
\end{table}

\begin{table}[ht]
  \centering
  \caption{Best models across all model sizes using $Y$, $D$ based on $\widehat{ERM}_{1}$, $\widehat{ERM}_{2}$ and $BIC$ with outliers removed}
\label{SummaryYDNoOutliers}

\begin{tabular}{|c|c|c|c|c|}
  \hline
  Model Size & $\widehat{h}$ & $\widehat{ERM}_{1} $ & $\widehat{ERM}_{2}$ & BIC \\
  \hline\hline
  $Y$ & 4 & 12.87  & 40.72 & 44.55 \\
  $Y$, $D^2$ & 4 & 12.01 & 39.26 & 37.84 \\
  $Y$, $D$, $D^2$ & 4 & 11.66 & 38.66 & 37.41 \\
  $Y$, $D$, $Y^2$, $D^2$ & 4 & 11.48 & 38.34 & 39.20 \\
  $Y$, $D$, $Y^2$, $D^2$, $Y:D$ & 4 & 11.35 & 38.13 & 41.83 \\
  \hline
\end{tabular}
\end{table}

\begin{table}[ht]
  \centering
  \caption{Best models across all model sizes using $Y$, $D$, $A$, $S$ based on $\widehat{ERM}_{1}$, $\widehat{ERM}_{2}$ and $BIC$}
\label{SummaryYDAS}

\begin{tabular}{|c|c|c|c|c|}
  \hline
  Model Size & $\widehat{h}$ & $\widehat{ERM}_{1} $ & $\widehat{ERM}_{2}$ & BIC \\
  \hline\hline
  $Y$ & 4 & 16.42  & 44.95 & 79.67 \\
  $Y$, $A$ & 4 & 14.79 & 42.32 & 68.45 \\
  $Y$, $A$, $D$ & 4 & 14.11 & 41.22 & 65.71 \\
  $Y$, $D$, $A$, $S$ & 4 & 14.31 & 41.22 & 70.33 \\
  \hline
\end{tabular}
\end{table}

\begin{table}[ht]
  \centering
  \caption{Best models across all model sizes using $Y$, $D$, $A$, $S$ based on $\widehat{ERM}_{1}$, $\widehat{ERM}_{2}$ and $BIC$ with outliers removed}
\label{SummaryYDASNoOutliers}

\begin{tabular}{|c|c|c|c|c|}
  \hline
  Model Size & $\widehat{h}$ & $\widehat{ERM}_{1} $ & $\widehat{ERM}_{2}$ & BIC \\
  \hline\hline
  $Y$ & 4 & 12.87  & 40.72 & 44.55 \\
  $Y$, $D$ & 4 & 12.82 & 40.64 & 48.57 \\
  $Y$, $A$, $D$ & 4 & 12.81 & 40.61 & 52.89 \\
  $Y$, $D$, $A$, $S$ & 4 & 12.81 & 40.61 & 57.41 \\
  \hline
\end{tabular}
\end{table}


\chapter{ANALYSIS OF MORE COMPLEX DATASETS}
\label{Wheatdata}
\thispagestyle{fancy}
In Chapter \ref{chap:Numerical:Studies}, we used synthetic datasets to look at the implementation of theories developed in Chapter  \ref{chap:bounds}. In that context, we estimated VCD for different sizes of linear models and we observed some variability in the estimate of VCD. The biggest discrepancy occurred when the size of the true model was $p=60$. However, this discrepancy was reduced when $N$ was increased. In Chapter \ref{example}, we followed essentially the same steps for the Tour De France data. We argue that the Tour de France data set was a toy data because the size (103 observations) was small compared to the number of models that we fitted. In this chapter, we will use larger and more complex datasets to demonstrate the effectiveness of our new methodology.

We start by analyzing the Abalone dataset \cite{AbaloneData} in Sec. \ref{Abalone}. We compare the model chosen by our method to that of SCAD \cite{Fan&Li}, ALASSO \cite{zou2006adaptive} and BIC. We will also replicate what we have done on Tour de France dataset in Chapter \ref{example} on the Wheat dataset in Sec. \ref{AnalWheat}. In Sec.\ref{WheatDesc}, we describe the wheat dataset. In Sec. \ref{Phenotype} we estimate the VCD using phenotype covariates by location. In Sec. \ref{MultiLoc}, we pool together all location and perform the analysis. In Sec. \ref{Design}, we add variables representing the design structure of the dataset to estimate $\hat{h}$. 

\section{Analysis of the Abalone dataset \cite{AbaloneData}}
\label{Abalone}
The ${\sf Abalone}$ dataset has been widely used in statistics and in machine learning as a benchmark dataset. It is known to be very difficult to analyze as either a classification or as a regression problem. Our goal in this section is to see how our method will perform on the $\sf Abalone$ data set and compare the result to other model selection techniques such as SCAD \cite{Fan&Li}, Adaptive Lasso (ALASSO) \cite{zou2006adaptive}, and BIC.

\subsection{Descriptive Analysis of Abalone dataset}
\label{AbaDesc}
\begin{figure}[h]
    \centering
    \begin{subfigure}[b]{0.6\textwidth}
        \centering
        \includegraphics[width = \textwidth]{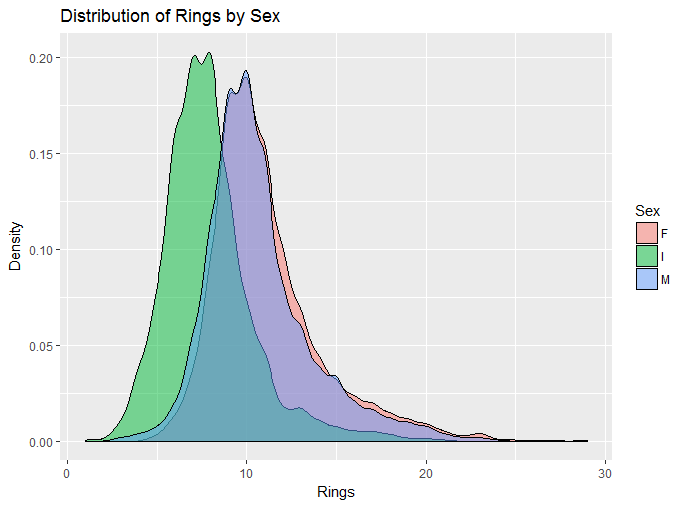}
        \caption{Distribution of $Rings$ by $Sex$}
        \label{AbaDist}
    \end{subfigure}
    \hfill
    \begin{subfigure}[b]{0.6\textwidth}
        \centering
        \includegraphics[width = \textwidth]{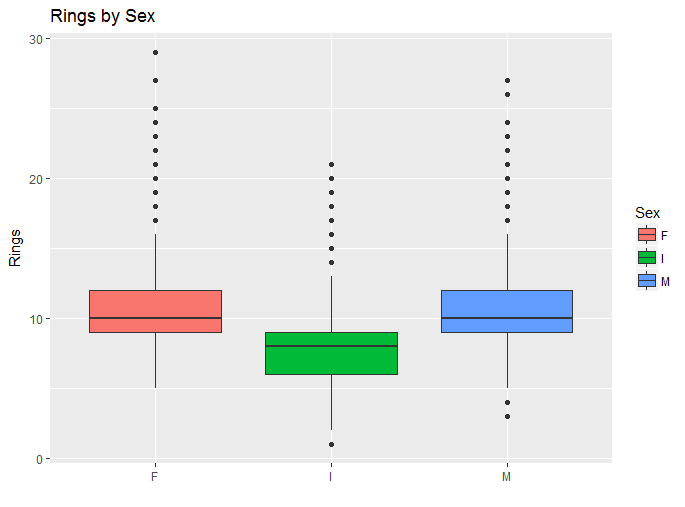}
        \caption{Boxplot of $Rings$ by $Sex$}
        \label{AbaBoxplot}
    \end{subfigure}
    \caption{Density and Boxplot of Abalone}
    \label{DescriptiveAba}
\end{figure}

    \begin{table}
    \centering
      \begin{tabular}{l | l | l | l | l}
              & $Length$ & $Diameter$ & $Height$ & $WholeWT$\\
            \hline \hline
            $Length$ & 1 & 0.99 & 0.83 & 0.93 \\
            $Diameter$ & 0.99 & 1 & 0.83 & 0.93\\
            $Height$ & 0.83 & 0.83 & 1 & 0.82\\
            $WholeWT$ & 0.93 & 0.93 & 0.82 & 1 \\
            $ShuckedWT$ & 0.90 & 0.89 & 0.77 & 0.97\\
            $VisceraWT$ & 0.90 & 0.90 & 0.80 & 0.97\\
            $ShellWT$ &  0.90 & 0.91 & 0.82 & 0.96\\
            $Rings$ & 0.56 & 0.57 & 0.56 &  0.54
            \end{tabular}
            \caption{First set of correlation}
            \label{Abacor1}
    \end{table}

    \begin{table}
      \centering
      \begin{tabular}{l | l | l | l | l}
              & $ShuckedWT$ & $VisceraWT$ & $ShellWT$ & $Rings$\\
            \hline \hline
            $Length$ & 0.90 & 0.90 & 0.90 & 0.56 \\
            $Diameter$ & 0.89 & 0.90 &   0.91 & 0.57 \\
            $Height$ &  0.77 & 0.80 &   0.82 & 0.56 \\
            $WholeWT$ & 0.97 & 0.97 &   0.96 & 0.54 \\
            $ShuckedWT$ & 1 &     0.93 & 0.88 & 0.42 \\
            $VisceraWT$ & 0.93 & 1 & 0.91 & 0.50 \\
            $ShellWT$ & 0.88 &     0.91 & 1 & 0.63 \\
            $Rings$ & 0.42 & 0.5 &   0.63 & 1
            \end{tabular}
            \caption{Second set of correlation}
            \label{Abacor2}
    \end{table}

The Abalone dataset \cite{AbaloneData} has 4177 observation 8 covariates. Sex is a nominal variable with 3 categories: Male, Female and Infant. Length (mm) is the longest shell measurement, Diameter (mm), Height (mm) is the height measures with the meat, Whole weight (grams) is the whole weight of the abalone, Shucked weight (grams) is the weight of the meat, Viscera weight (grams) is the gut weight after bleeding, Shell weight is the shell weight after being dried. The response variable, i.e., the $Y$, is rings. The number of rings is roughly the age of an abalone, since abalone typically grow one ring per year. This dataset is known to be very hard to analyze because many of these covariates are function of each other.

Fig. \ref{AbaDist} gives the distribution of the rings by Sex. We see that Female and Male abalone have essentially the same distribution. The distribution of Infant abalone is shifted left from the common distribution of Female and Male. All three distributions are skewed to the right but are close to a normal distribution over much of their support. From this graph, we see that if $Rings \leq 5$, the abalone is likely to be an Infant. However if your $Rings \geq 15$, it is more probable that the abalone is a Female or a Male.

Fig. \ref{AbaBoxplot} is the box plot of Rings by Sex. We still observe that Female and Male Abalone have the same distribution. We observe some data points outside of the whisker plot; these observations are potential outliers but they are less than $1\%$ of the sample..

Tables \ref{Abacor1} and \ref{Abacor2} give all pairwise correlations of the variables. None of the correlation is obviously low. In fact, the smallest correlation (0.42) is between $Rings$ and $Shucked~weight$. Moreover, we see relatively high correlation between our explanatory variables. In linear models this would leads us to expect problems from multicollinearity. Overall, it is not clear which, if any, of the explanatory can be omitted when explaining the response ($rings$). We also observe that $Length$ is highly correlated with other explanatory variables but not so much with the response ($Rings$).

\begin{figure}
    \centering
    \begin{subfigure}[b]{0.4\textwidth}
        \centering
        \includegraphics[width = \textwidth]{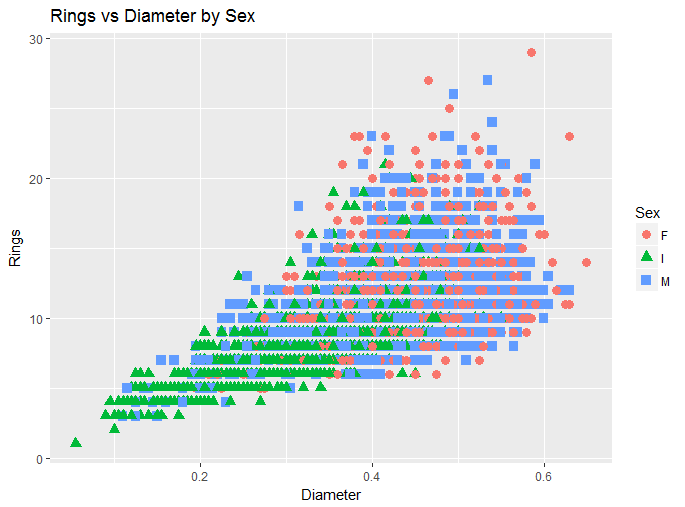}
        \caption{Scatter plot of $Rings$ VS $Diameter$}
        \label{ScatterDiameter}
    \end{subfigure}
    \hfill
    \begin{subfigure}[b]{0.4\textwidth}
        \centering
        \includegraphics[width = \textwidth]{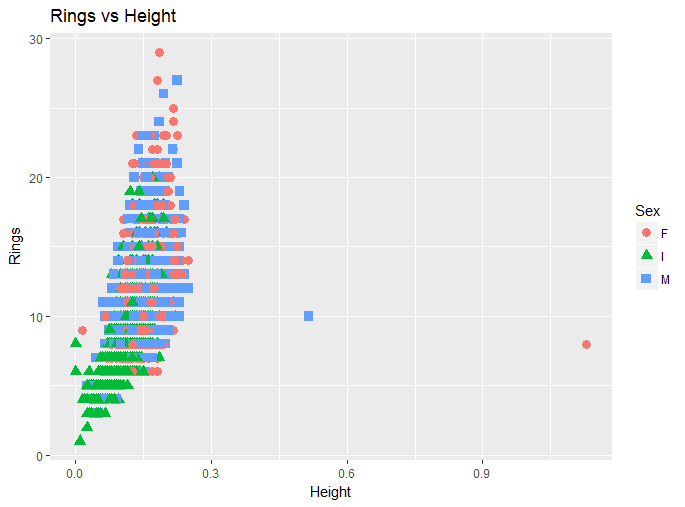}
        \caption{Scatter plot of Rings VS Height}
        \label{ScatterHeight}
    \end{subfigure}
    \begin{subfigure}[b]{0.4\textwidth}
        \centering
        \includegraphics[width = \textwidth]{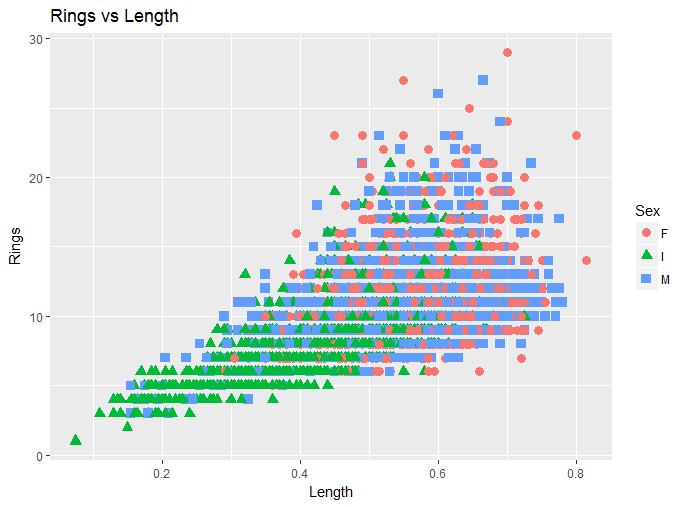}
        \caption{Scatter plot of Rings VS Length}
        \label{ScatterLength}
    \end{subfigure}
    \hfill
    \begin{subfigure}[b]{0.4\textwidth}
        \centering
        \includegraphics[width = \textwidth]{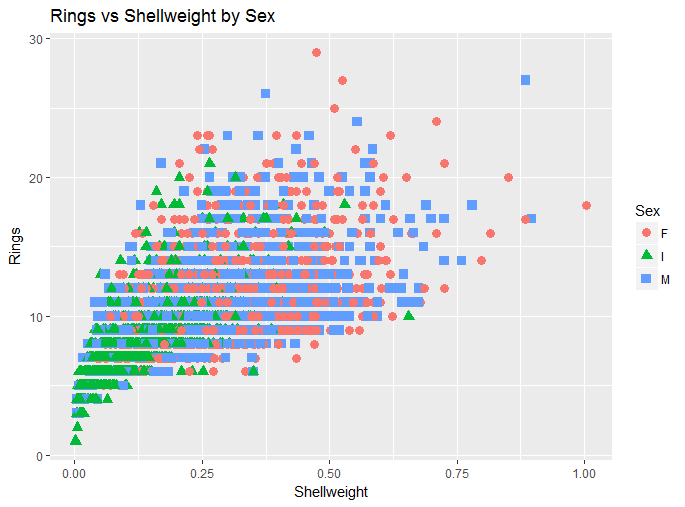}
        \caption{Scatter plot of $Rings$ VS $Shell~weight$}
        \label{ScatterShellweight}
    \end{subfigure}
    \caption{Scatter plot of $Rings$ VS $Diameter, Height,
    Length$ by $Sex$}
    \label{ScatterPlot}
\end{figure}

\begin{figure}
    \centering
    \begin{subfigure}[b]{0.4\textwidth}
        \centering
        \includegraphics[width = \textwidth]{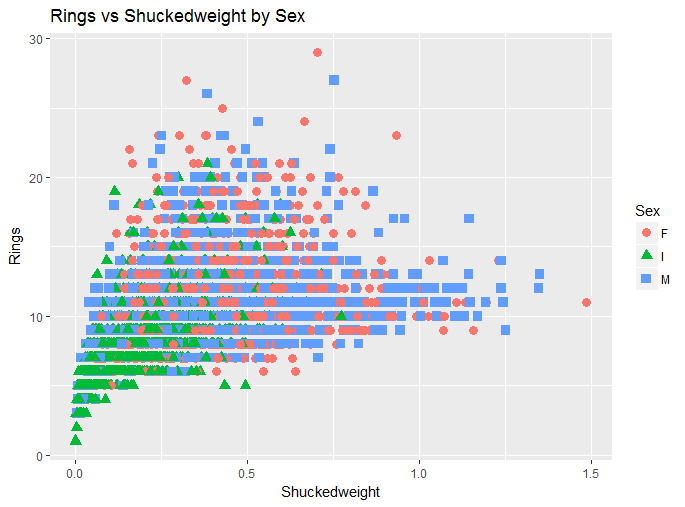}
        \caption{Scatter plot of Rings VS Shucked weight}
        \label{ScatterShuked}
    \end{subfigure}

    \hfill
    \begin{subfigure}[b]{0.4\textwidth}
        \centering
        \includegraphics[width = \textwidth]{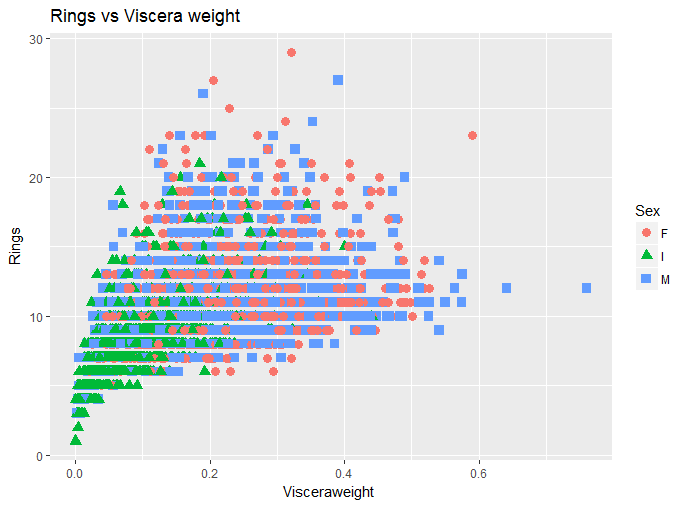}
        \caption{Scatter plot of Rings VS Viscera weight}
        \label{ScatterViscera}
    \end{subfigure}
    \hfill
    \begin{subfigure}[b]{0.4\textwidth}
        \centering
        \includegraphics[width = \textwidth]{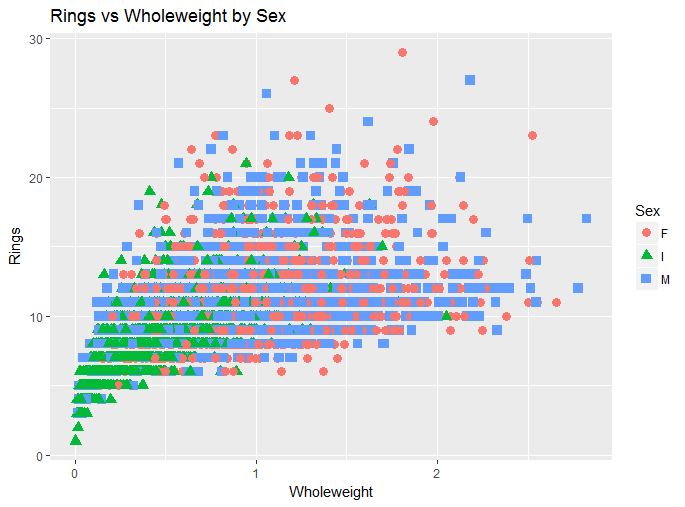}
        \caption{Scatter plot of Rings VS Whole weight}
        \label{ScatterWholeweight}
    \end{subfigure}
    \caption{Scatter plot of $Rings$ VS $Shell~ weight$, $Shucked ~weight$ and $Whole~ weight$ by $Sex$}
    \label{Scatterplot2}
\end{figure}

Figs. \ref{ScatterPlot} and \ref{Scatterplot2} show all pairwise plots of the form $Rings$ versus covariates (color coded by Sex). We see that no matter which covariates are chosen, the variability in the Rings increases as the size of the covariates increases. We also observe that there is likely to be a curvilinear relationship between Rings and the covariates. However it is so  weak for $Rings$ vs $Diameter$ and $Rings$ vs $Length$ that linear terms in these covariates may be adequate.

\subsection{Statistical Analysis of the Abalone data}

Our goal here is to evaluate our method on a more complex dataset. The model that we use to estimate the complexity of the response `Rings' is a linear combination of all the variables. To accomplish this, we first order the inclusion of variables in the model using correlation \cite{fan2008sure}. Under correlation between $Rings$ and each of the explanatory variables, the order of inclusion of variables is as follows: $Shell~weight$, $Diameter$, $Height$, $Length$, $Whole~ weight$, $Viscera~weight$ and $Shucked~weight$. Using this ordering, we fit seven different models, estimated $\hat{h}$, and found values of $\widehat{ERM}_{1}$, $\widehat{ERM}_{2}$ and $BIC$. These values are in Table \ref{AbaNested}. We also compare our method to other model selection techniques based on sparsity such as SCAD \cite{Fan&Li} and Adaptive Lasso \cite{zou2006adaptive}.

From Table \ref{AbaNested}, we observe that $\hat{h} = 8$ except for model of size 5. We regard $\hat{h} = 9$ for a model of size 5 as a random fluctuation since it is close to 8 and our method while stable is not perfectly so. The model chosen using $\hat{h}$ is the biggest model because this model has the smallest distance between its size and $\hat{h}$. However, the fact that model is first order of size 7 and $\hat{h}=8$ suggest there may be a missing variable in the dataset i.e., unavoidable bias. In fact, observational data is prone to bias than design of experiments, but bias is always present in practice. Note that a missing variable in a model may correspond to many actual variables that would have to be measured.  We see some variability in the estimate of $\widehat{ERM}_{1}$ and $\widehat{ERM}_{2}$ as we include variables in the model. We see a drop when $Diameter$ is included, and the values go back up when $height$ and $Length$ are included. There is a big decrease at the $6^{th}$ model and a slight increases at the $7^{th}$ model. This observation is similar for $BIC$. So $\widehat{ERM}_{1}$, $\widehat{ERM}_{2}$ and $BIC$ pick the model of size $6$ while our method picks the model of size seven and suggest there is a bias from at least a missing variable. We regard the results from using $\hat{h}$ as more plausible physically in the present case.

\begin{table}[h]
  \centering
  \caption{Nested models using covariates for Abalone}
\label{AbaNested}

\begin{tabular}{|c|c|c|c|c|}
  \hline
  Model Size & $\widehat{h}$ & $\widehat{ERM}_{1} $ & $\widehat{ERM}_{2}$ & BIC \\
  \hline\hline
  $Shellweight$ & 8 & 26315 & 26524  & 19567\\
  $Shellweight$, $Diameter$ & 8 & 22839 & 23033 &  18983 \\
  $Shellweight$, $Diameter$, $Height$ & 8 & 26045 & 26253 & 19540 \\
  $Shellweight$, $Diameter$, $Height$, $Length$ & 8 & 25745 & 25952 & 19500\\
  5 & 9 & 23477 & 23684 & 19124\\
  6 & 8 & 20389 & 20573 & 18551\\
  7 & 8 & 20507 & 20696 & 18575  \\
  \hline
\end{tabular}
\end{table}

\begin{figure}
    \centering
    \begin{subfigure}[b]{0.4\textwidth}
    \centering
    \includegraphics[width = \textwidth]{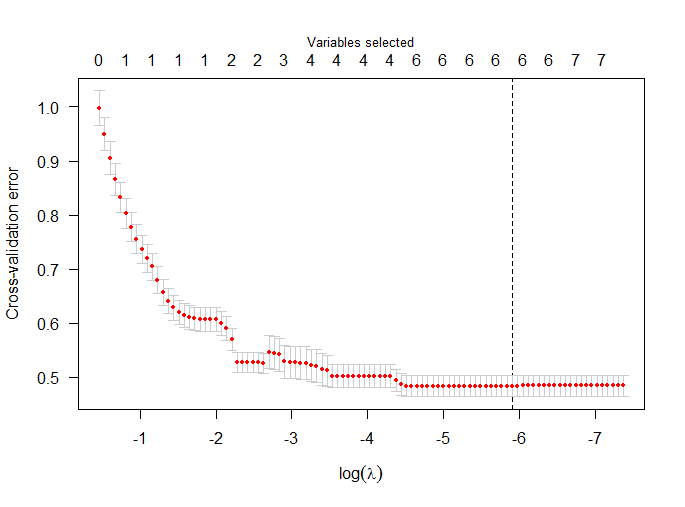}
    \caption{Optimal value of $\lambda$}
    \label{Lambda}
    \end{subfigure}
    \hfill
    \begin{subfigure}[b]{0.4\textwidth}
    \centering
    \includegraphics[width = \textwidth]{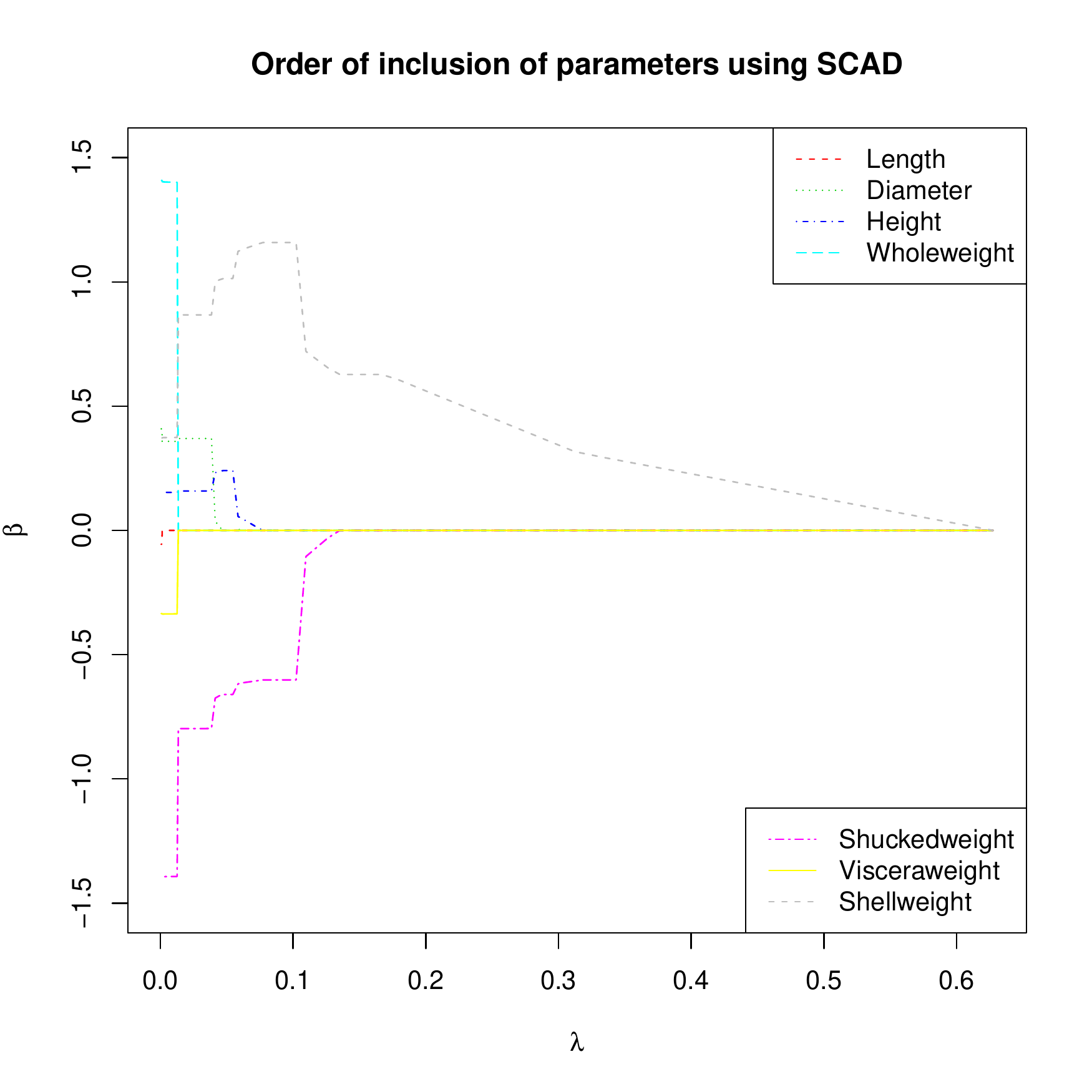}
    \caption{Order of Inclusion of variables in model}
    \label{AbaSCAD}
    \end{subfigure}
    \caption{Analysis of Abalone dataset using SCAD}
    \label{Aba}
\end{figure}

Next we turn to the results of a sparsity driven analysis. We comment that since there are seven explanatory variables and $N = 4177$, sparsity per se is not necessarily an important property for a model to have. We present these results (for comparative purpose only).

Our first comparison uses SCAD \cite{Fan&Li} as a model selection technique. Fig \ref{Lambda} gives the value of the cross-validation error for each value of $\log\left(\lambda\right)$. The optimal value of $\lambda$ is seen to be $\hat{\lambda} = 0.0027$. With this value of $\hat{\lambda}$, the best model must have 6 variables.
Fig \ref{AbaSCAD} is the trace of parameters for different value of $\lambda$. Using the optimal value of $\hat{\lambda}$, the variables that get into the model in order are $Shell~weight$, $Shucked ~ weight$, $Height$, $Diameter$, $Viscera ~ weight$, and $Whole~weight$. Thus, under SCAD, we are led to the model
    \begin{equation}
    \label{AbaEqua}
    \begin{split}
    \widehat{Rings} & = 0.36\cdot Diameter + 0.15\cdot Height + 1.40\cdot wholeweight \\
     & - 1.39S\cdot shuckedweight -0.34\cdot Visceraweight  \\
     & + 0.37\cdot Shellweight.
    \end{split}
    \end{equation}

    Analogous analysis under Adaptive LASSO \cite{zou2006adaptive} leads us to the same six terms and the model is:

    \begin{equation} \label{AbaALASSO}
    \begin{split}
    \widehat{Rings} & = 3 + 11.62\cdot Diameter + 11.69\cdot Height + 9.21\cdot wholeweight \\
     & - 20.24\cdot Shuckedweight -9.79\cdot Visceraweight  \\
     & + 8.63\cdot Shellweight.
    \end{split}
    \end{equation}
Even though \eqref{AbaEqua} and \eqref{AbaALASSO} have the same terms, the coefficients are very different. This may occur because there is a high correlation between covariates as observed in Tables \ref{Abacor1} and \ref{Abacor2}; a prediction based analysis might resolve this question however that is not our point here. Note that the models in \eqref{AbaEqua} and \eqref{AbaALASSO} both included $Shucked~weight$ but neither included $Length$, whereas the models chosen by $BIC$, $\widehat{ERM}_{1}$ and $\widehat{ERM}_{2}$ include $Length$ but not $Shucked ~weight$. That is, the sparsity models use the same variables, the $\widehat{ERM}_{1}$, $\widehat{ERM}_{2}$ and $BIC$ used the same variables (albeit a different set) and $\hat{h}$ includes all the variables, suggesting that some are missing. As before, we regard the model chosen by $\hat{h}$ as the most reasonable.

\section{Analysis of the Wheat dataset}
\label{AnalWheat}
The Wheat dataset has 2912 observation, 104 varieties. The experimental study was conducted in seven locations; Lincoln, NE,
in 1999 to 2001, and Mead and Sidney, NE, in 2000 and 2001. The design used in Lincoln, NE, in 1999 was a Randomized Complete Block Design (RCBD) with four replicates. In others years, the design was an Incomplete Block design with four replicates where each replication consisted of eight incomplete blocks of thirteen entries.  The environments are diverse and representative of wheat producing areas of Nebraska. More information concerning the dataset and the design structure can be found in \cite{campbell2003identification}.
The response variable is $Yield$ (MG/ha), the covariates that we used are 1000 kernel weight (TKWT), kernels per spike (KPS), Spikes per square meter (SPSM), Height of the plant (HT), Test weight (TSTWT(KG/hl)), and Kernels per squared meter (KPSM) in wheat. $Yield$ was measured by harvesting all four rows with a small plot combine.

The analysis of the wheat dataset will be done using phenotype data location by location first. Then we will provide an analysis combining the data from all locations primarily for comparison purposes.

\subsection{Description of the Wheat dataset}
\label{WheatDesc}

\begin{figure}
    \centering
    \begin{subfigure}[b]{0.5\textwidth}
        \centering
        \includegraphics[width = \textwidth]{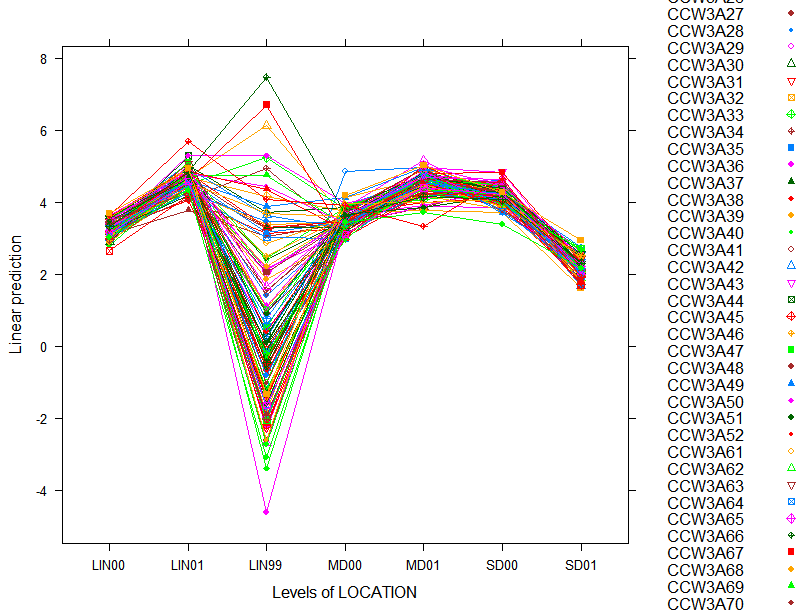}
        \caption{Interaction between Locations and Varieties}
        \label{InteractionPlot}
    \end{subfigure}
    \hfill
    \begin{subfigure}[b]{0.5\textwidth}
        \centering
        \includegraphics[width = \textwidth]{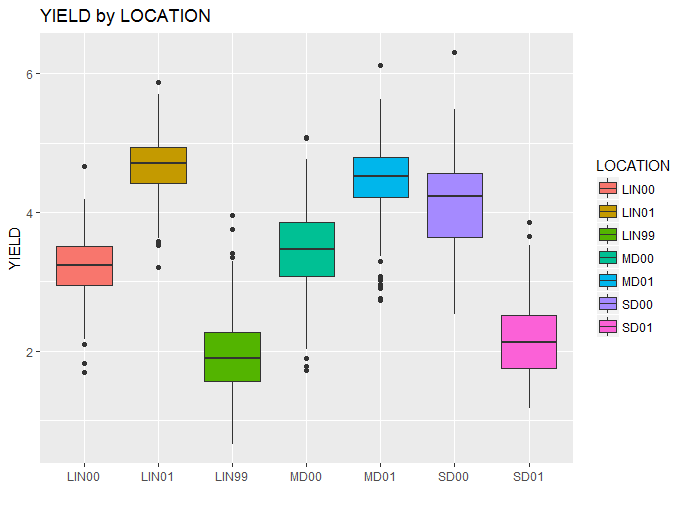}
        \caption{Boxplot of Yield by Location}
        \label{BoxPlotYield}
    \end{subfigure}
    \caption{Interaction plot and boxplot of the Yield. Panel \ref{InteractionPlot} shows symbols for each variety over the 7 locations. The symbols are connected by lines to shows interaction. Panel \ref{BoxPlotYield} pools over varieties at each location.}
    \label{WheatPlotDesc}
\end{figure}

From Fig. \ref{InteractionPlot}, the variable of interest is $Yield$. We see that lines connecting locations by varieties are crossing; this indicates there is interaction between varieties and Locations. We can also observed that there is a lot of variations in Lincoln 1999. On graph \ref{BoxPlotYield}, we see how the yield varies from one location to another. The smallest value of the yield occurs in Lincoln in 1999, and the highest yield occurs at Lincoln in 2001. The most variable location is Sidney in 2000. All data points outside of the whisker plot can be considered as potential outliers.

From Fig. \ref{kpsm}, we see that there is a reasonably linear relationship between YIELD and KPSM --although the variance appears to increase slowly with $KPSM$. On Fig. \ref{spsm}, there is fairly good linear relationship between YIELD and SPSM although again the variance increases with $SPSM$. However, on Fig. \ref{tstwt}, there is a curvilinear relationship between YIELD and TSTWT, and the variance starts small (as a function of $TSTWT$), increases rapidly and then appears to decrease somewhat. Figs. \ref{ht}, \ref{kps}, and \ref{tkwt} suggest that the data do not reflect a strong relationship between $YIELD$ and any of $HT$, $KPS$, and $TKWT$.

\begin{figure}
    \centering
    \begin{subfigure}[b]{0.4\textwidth}
    \centering
    \includegraphics[width = \textwidth]{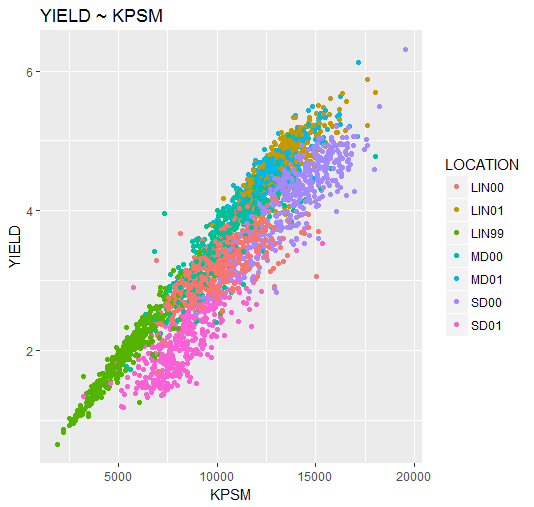}
    \caption{Scatter plot of YIELD vs KPSM}
    \label{kpsm}
    \end{subfigure}
    \hfill
    \begin{subfigure}[b]{0.4\textwidth}
    \centering
    \includegraphics[width = \textwidth]{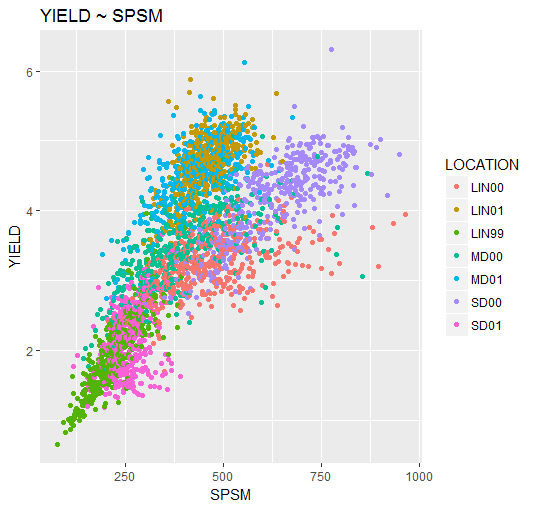}
    \caption{Scatter plot of YIELD vs SPSM}
    \label{spsm}
    \end{subfigure}
    \hfill
    \begin{subfigure}[b]{0.4\textwidth}
    \centering
    \includegraphics[width = \textwidth]{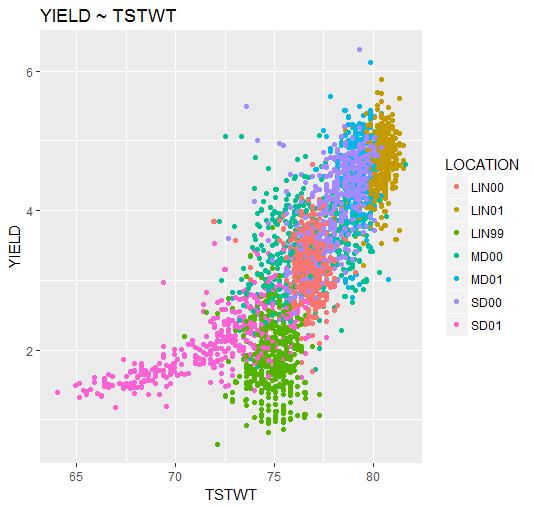}
    \caption{Scatter plot of YIELD vs TSTWT}
    \label{tstwt}
    \end{subfigure}
    \hfill
    \begin{subfigure}[b]{0.4\textwidth}
    \centering
    \includegraphics[width = \textwidth]{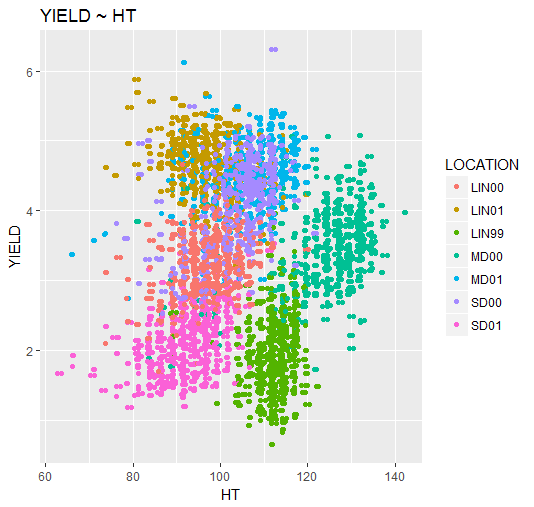}
    \caption{Scatter plot of YIELD vs HT}
    \label{ht}
    \end{subfigure}
    \hfill
    \begin{subfigure}[b]{0.4\textwidth}
    \centering
    \includegraphics[width = \textwidth]{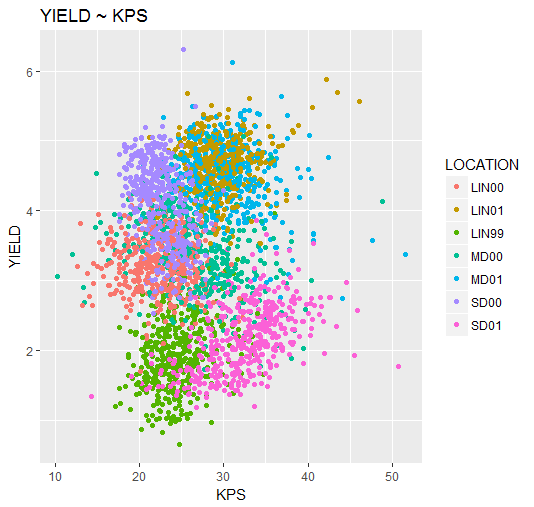}
    \caption{Scatter plot of YIELD vs KPS}
    \label{kps}
    \end{subfigure}
    \hfill
    \begin{subfigure}[b]{0.4\textwidth}
    \centering
    \includegraphics[width = \textwidth]{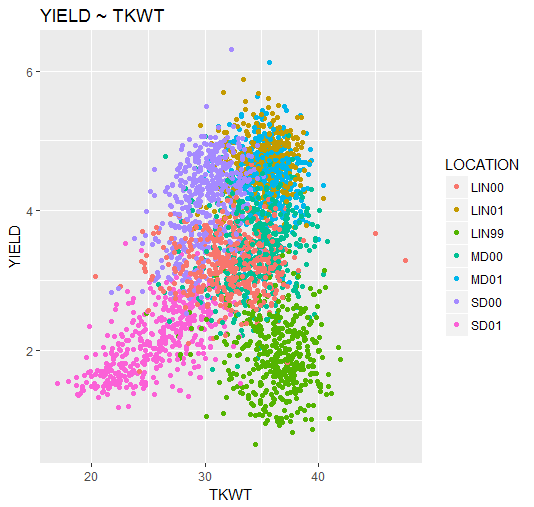}
    \caption{Scatter plot of YIELD vs TKWT}
    \label{tkwt}
    \end{subfigure}
    \caption{Scatter plot of Yield versus phenotype covariates}
    \label{scatterwheat}
\end{figure}

\begin{table}[h]
  \centering
  \caption{Correlation between phenotype covariates over all locations}
\label{CorrPheno}

\begin{tabular}{|c|c|c|c|c|c|c|c|}
  \hline
  &$YIELD$ & $HT$ & $TSTWT $ & $TKWT$ & $SPSM$ & $KPS$ & $KPSM$ \\
  \hline
  $YIELD$ & 1.00 & 0.04 & 0.80 & 0.27 & 0.74 & 0.022 & 0.93\\
$HT$ &    0.04 & 1.00 & 0.09 & 0.38 & 0.02 & -0.22 & -0.072\\
$TSTWT$ & 0.80 &  0.09 &  1.00 & 0.53 & 0.53 & -0.06 &  0.64\\
$TKWT$ &  0.27 &  0.38 &  0.53 &  1.00 & -0.11 & -0.09 & -0.09\\
$SPSM$ &  0.74 &  0.02 & 0.53 & -0.11 &  1.00 & -0.5 & 0.83\\
$ KPS$ &   0.022 & -0.22 & -0.06 & -0.09 & -0.5 &  1.00 & 0.014\\
$ KPSM $ &  0.93 & -0.07 &  0.64 & -0.09 &  0.83 &  0.01 & 1.00\\
\hline
\end{tabular}
\end{table}
From Table \ref{CorrPheno}, we see a strong linear relationship between Yield and $TSTWT$ (0.80), between $Yield$ and $SPSM$ (0.74) and between Yield and $KPSM$ (0.93). There is also a very weak linear relationship between $YIELD$ and $HT$ (0.04), $YIELD$ and $KPS$ (0.022). We also notice some strong correlations between the phenotypic variables. For instance the correlation between $KPSM$ and $TSTWT$ is 0.64, between $KPSM$ and $SPSM$ is 0.83, between $TSTWT$ and $TKWT$ is 0.53. There is also weak correlation between $TSTWT$ and $HT$ (0.09), and between $TSTWT$ and $KPS$ (-0.06).

Intuitively, we expect
$$
YIELD = \beta_{0} + \beta_{1}\cdot TKWT\cdot KPSM + \epsilon
$$
to be a good model because yield is essentially the product of the number of kernels and their average weight. Similarly,
$$
YIELD = \beta_{0} + \beta_{1}\cdot TKWT\cdot KPS\cdot SPSM + \epsilon
$$
should be a good model as well because $ KPSM \cong KPS\cdot SPSM$. Thus using our phenotypic variables there is no unique good model. Moreover, these two models probably only capture the major effect of the explanatory variables. So both are over simplifications.

\subsection{Estimation of VCD using phenotype covariates by Location}
\label{Phenotype}


We analyze the data first for each location, so that we can see how our methodology performs.
Here, we use the six phenotypic variables to estimate $\hat{h}$ and obtain $\widehat{ERM}_{1}$, $\widehat{ERM}_{2}$ and $BIC$. In fact, the analysis is done in two settings. First, we find the order of inclusion of these variables in the model using correlation \cite{fan2008sure} with $Yield$. Second, we estimate $\hat{h}$ and obtain $\widehat{ERM}_{1}$, $\widehat{ERM}_{2}$ and $BIC$. We will present the analysis of Lincoln 1999 and the across locations, the rest can be found in Appendix \ref{AppenChap5}

\subsubsection{Analysis of Wheat data in Lincoln 1999}

\begin{figure}
    \centering
    \begin{subfigure}[b]{0.4\textwidth}
    \centering
    \includegraphics[width = \textwidth]{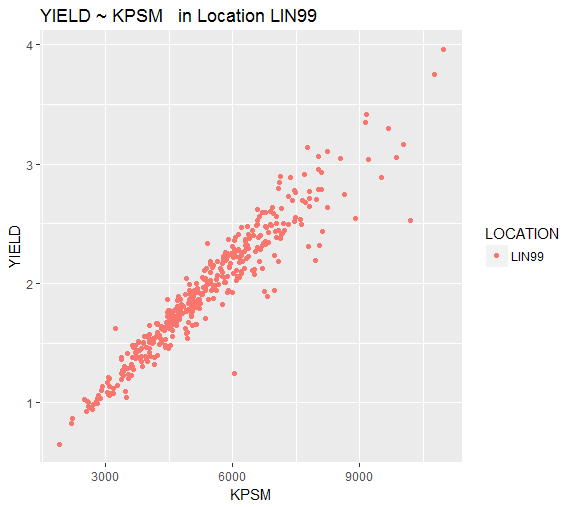}
    \caption{Scatter plot of YIELD vs KPSM}
    \label{kpsm3}
    \end{subfigure}
    \hfill
    \begin{subfigure}[b]{0.4\textwidth}
    \centering
    \includegraphics[width = \textwidth]{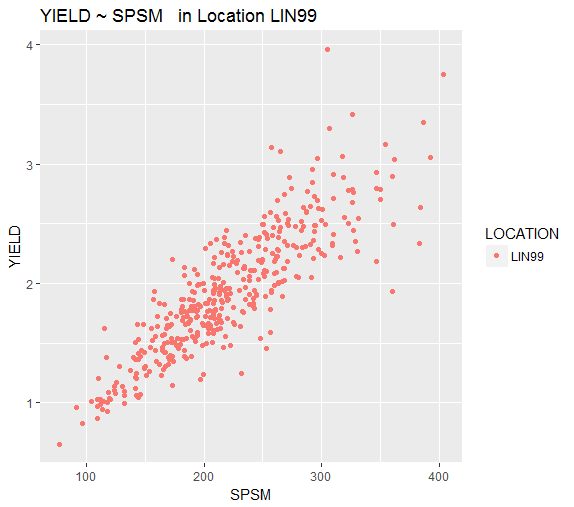}
    \caption{Scatter plot of YIELD vs SPSM}
    \label{spsm3}
    \end{subfigure}
    \hfill
    \begin{subfigure}[b]{0.4\textwidth}
    \centering
    \includegraphics[width = \textwidth]{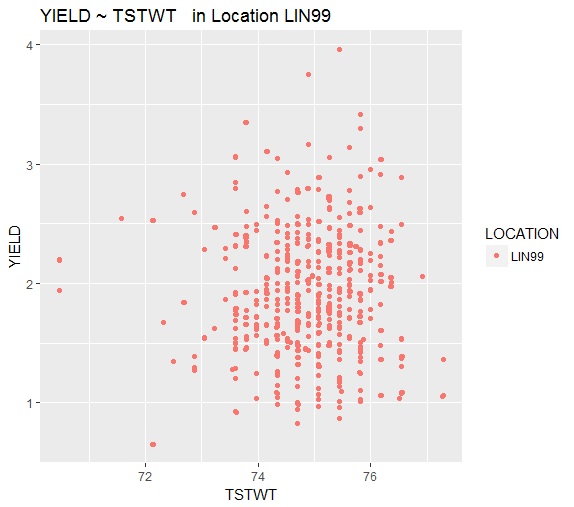}
    \caption{Scatter plot of YIELD vs TSTWT}
    \label{tstwt3}
    \end{subfigure}
    \hfill
    \begin{subfigure}[b]{0.4\textwidth}
    \centering
    \includegraphics[width = \textwidth]{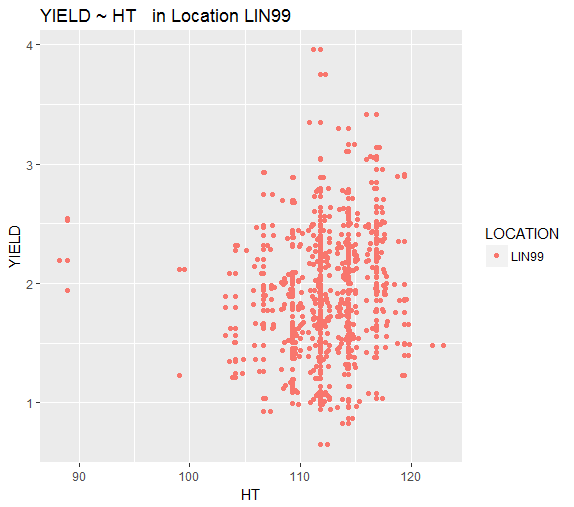}
    \caption{Scatter plot of YIELD vs HT}
    \label{ht3}
    \end{subfigure}
    \hfill
    \begin{subfigure}[b]{0.4\textwidth}
    \centering
    \includegraphics[width = \textwidth]{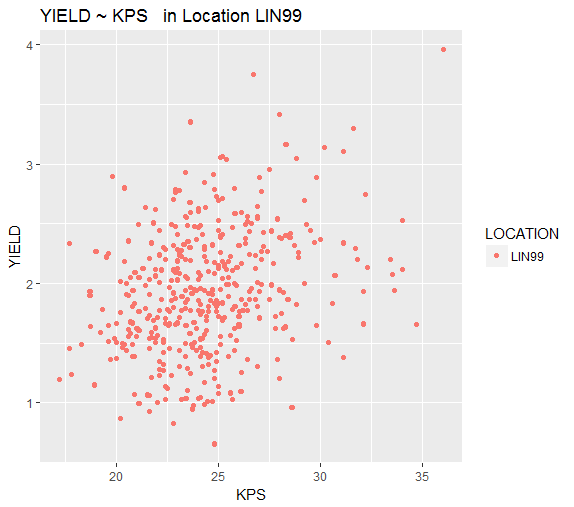}
    \caption{Scatter plot of YIELD vs KPS}
    \label{kps3}
    \end{subfigure}
    \hfill
    \begin{subfigure}[b]{0.4\textwidth}
    \centering
    \includegraphics[width = \textwidth]{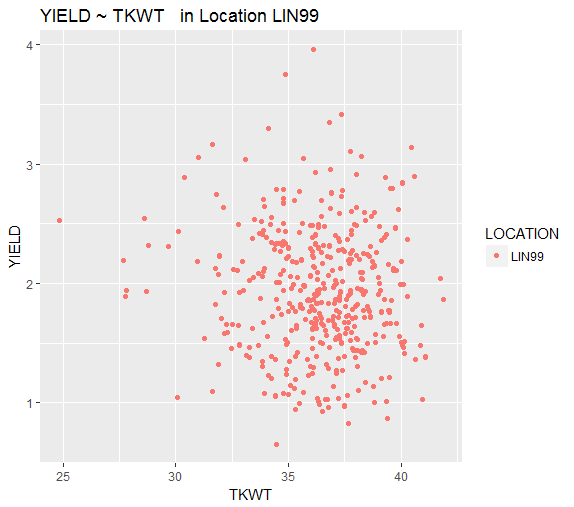}
    \caption{Scatter plot of YIELD vs TKWT}
    \label{tkwt3}
    \end{subfigure}
    \caption{Scatter plot of Yield versus phenotype data in Lincoln 1999}
    \label{scatterwheat3}
\end{figure}
From Fig. \ref{kpsm3}, we see that there is a strong linear relationship between $YIELD$ and $KPSM$. The variance is increasing very slowly with $KPSM$ so for practical purposes we regard it as relatively stable i.e essentially constant. There are also some data points not close to the majority of data but the overall trend is linear. From Fig. \ref{spsm3}, there is a weak linear relationship between $YIELD$ and $SPSM$. The variance increases as a function of $SPSM$. From Figs. \ref{tstwt3} to \ref{tkwt3}, there does not appear to be any non-trivial relationship between between $YIELD$ and any of $TSTWT$, $HT$, $KPS$, and $TKWT$. The observations are just spread around and do not appear to show any particular pattern.

We compare our method to $BIC$ and also to sparse model selection technique as SCAD \cite{Fan&Li} and ALASSO \cite{zou2006adaptive}. We use correlation \cite{fan2008sure} between $YIELD$ and the phenotypic covariates to order the inclusion of phenotype variables and their products in the model. In fact, there are six linear terms, six squared terms, and fifteen cross products, this leads to a total of 27 variables and therefore, 27 nested models to fit. Using correlation, the order of inclusion of terms is:
$TKWT\cdot KPSM, KPSM\cdot HT, KPSM\cdot TSTWT,$
$ SPSM\cdot KPS, KPSM, KPSM^2, KPSM\cdot SPSM, SPSM \cdot TKWT,$ $ KPSM\cdot KPS, SPSM,$ $ SPSM\cdot TSTWT,$
$ SPSM\cdot HT, SPSM^2, KPS\cdot HT, KPS\cdot TSTWT,$
$ KPS^2, KPS, KPS\cdot TKWT, HT^2, HT, HT\cdot TSTWT,$
$ TSTWT, TSTWT^2, TKWT\cdot HT, TKWT\cdot TSTWT,$
$ TKWT^2, TKWT$.

\begin{table}[h]
    \begin{subtable}[h]{0.45\textwidth}
        \centering
        \begin{tabular}{l | l | l | l | l}
          Size& $\hat{h}$ & $\widehat{ERM}_{1}$ & $\widehat{ERM}_{2}$ & $BIC$ \\
        \hline \hline
        1 & 18 & 3.33 &  6.64 & -1221.52\\
        2 & 17 & 3.28 &  6.43 & -1215.85\\
        3 & 16 & 3.23 &  6.22 & -1210.61\\
        4 & 17 & 3.28 &  6.42 & -1205.50\\
        5 & 19 & 3.52 &  6.80 & -1206.86 \\
        6 & 19 & 3.35 &  6.80 & -1201.31\\
        7 & 19 & 3.35 &  6.80 & -1195.29\\
        8 & 20 & 3.40 &  7.00 & -1189.421 \\
        9  & 19 & 3.35 & 6.80 & -1184.31
        \end{tabular}
        \caption{Model sizes 1 to 9}
        \label{Loc3set1}
    \end{subtable}
    \hfill
    \begin{subtable}[h]{0.45\textwidth}
        \centering
        \begin{tabular}{l | l | l | l | l}
          Size& $\hat{h}$ & $\widehat{ERM}_{1}$ & $\widehat{ERM}_{2}$ & $BIC$ \\
        \hline \hline
        10 & 19 & 3.35 & 6.80 & -1178.39\\
        11 & 20 & 3.39 & 6.99 & -1172.38\\
        12 & 20 & 3.90 & 6.99 & -1167.28\\
        13 & 20 & 3.99 & 6.99 & -1161.4\\
        14 & 19 & 3.33 & 6.76 & -1159.72\\
        15 & 20 & 3.38 & 6.96 & -1154.10\\
        16 & 20 & 3.76 & 6.96 & -1148.37 \\
        17 & 21 & 3.42 & 7.16 & -1142.42\\
        18 & 21 & 3.42 & 7.16 & -1136.79\\
        \end{tabular}
        \caption{Model sizes 10 to 18}
        \label{Loc3set2}
    \end{subtable}
    \begin{subtable}[h]{0.45\textwidth}
        \centering
        \begin{tabular}{l | l | l | l | l}
          Size& $\hat{h}$ & $\widehat{ERM}_{1}$ & $\widehat{ERM}_{2}$ & $BIC$ \\
        \hline \hline
        19 & 21 & 3.42 & 7.16 & -1130.82\\
        20 & 21 & 3.42 & 7.15 & -1125.49\\
        21 & 22 & 3.46 & 7.35 & -1119.55\\
        22 & 22 & 3.45 & 7.34 & -1115.6 \\
        23 & 12 & 2.96 & 5.26 & -1109.60 \\
        24 & 22 & 3.45 & 7.33 & -1104.06 \\
        25 & 22 & 3.45 & 7.34 & -1098.50 \\
        26 & 12 & 2.95 & 5.25 & -1092.47\\
        27 & 12 & 2.95 & 5.25 & -1087.2
        \end{tabular}
        \caption{Model sizes 19 to 27}
        \label{Loc3set3}
    \end{subtable}
    \caption{Estimation of $\hat{h}$, $\widehat{ERM}_{1}$, $\widehat{ERM}_{1}$ and $BIC$ in LIN99. Size indicates the dimension of the parameter space of the linear model.}
    \label{Loc3}
\end{table}

Table. \ref{Loc3}, shows some variability in the estimate of $\hat{h}$, also in the values of $\widehat{ERM}_{1}$, $\widehat{ERM}_{2}$ and $BIC$. The smallest difference between the size of the conjectured model and $\hat{h}$ occurs when the size of the conjectured model is 22; this indicates that the model chosen by $\hat{h}$ has 22 variables. $\widehat{ERM}_{1}$ and $\widehat{ERM}_{2}$ choose the most complex model with all parameters. $BIC$ picks the smallest model with $TKWT\cdot KPSM$ as variable.
\begin{figure}
    \centering
    \begin{subfigure}[b]{0.5\textwidth}
    \centering
    \includegraphics[width = \textwidth]{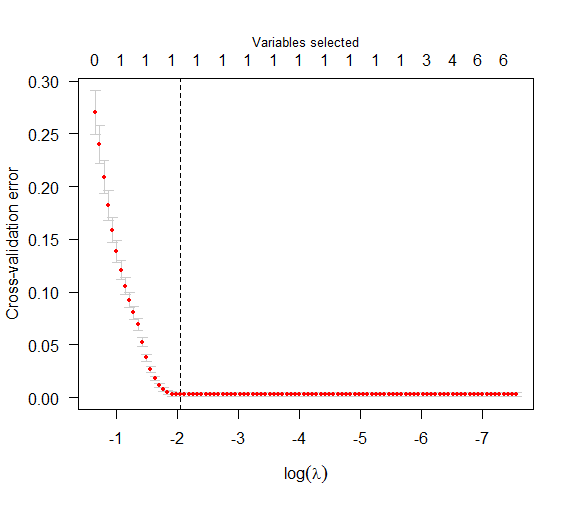}
    \caption{Optimal value of $\lambda$ for Lincoln 1999}
    \label{lambda3}
    \end{subfigure}
    \hfill
     \begin{subfigure}[b]{0.5\textwidth}
    \centering
    \includegraphics[width = \textwidth]{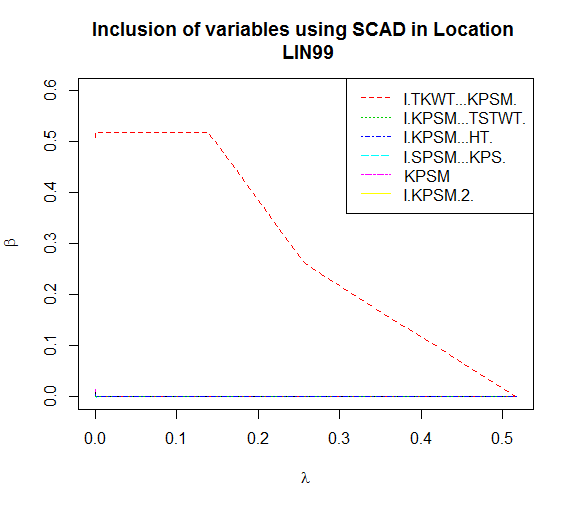}
    \caption{Trace of variables for Lincoln 1999 using SCAD \cite{Fan&Li}}
    \label{trace3}
    \end{subfigure}
    \caption{Analysis of Wheat data set in Lincoln 1999 using SCAD \cite{Fan&Li}}
    \label{Lin99}
\end{figure}

 Now, we turn to analyzing the Lincoln 1999 dataset using shrinkage methods. From Fig. \ref{lambda3}, the optimal value of $\lambda$ using SCAD is $\hat{\lambda} = 0.1281894$; with this value, the optimal model under SCAD should have one variable. From Fig. \ref{trace3}, and using the optimal $\hat{h}$, the variable chosen is $KPSM\cdot TKWT$, that is, the model is

\begin{equation}
\label{Loc3SCAD}
\widehat{YIELD} =  1.93 + 0.52\cdot KPSM\cdot TKWT.
\end{equation}

Under Adaptive LASSO, a similar analysis leads to
\begin{equation}
\label{Loc3ALASSO}
\widehat{YIELD} =  1.93 + 0.48\cdot KPSM\cdot TKWT + 0.038\cdot SPSM\cdot KPS + 0.0043\cdot TKWT.
\end{equation}

Note that for the Lincoln 1999, SCAD and $BIC$ give the same smallest model that is also intuitively reasonable. However, it is known there are effects on $Yield$ from explanatory variables other than $TKWT\cdot KPSM$. Since, we ignore incomplete blocks for the analysis, the estimated VCD could be mostly the incomplete block effects.

\section{Multilocation analysis}
\label{MultiLoc}
 For this analysis, we pool over locations, and we did not include location as a class variable inside the model. We start by ordering the inclusion of terms in the model using correlation \cite{fan2008sure}. Under correlation, the nesting of our variables is as follows:
$TKWT\cdot KPSM, KPSM\cdot TSTWT, KPSM, $
 $SPSM\cdot KPS, KPSM^2, KPSM\cdot HT, $
 $SPSM\cdot TKWT, KPSM\cdot SPSM$
 $TSTWT^2, TSTWT, KPSM\cdot KPS, SPSM\cdot TSTWT,$
 $ SPSM, $
 $SPSM\cdot HT, $ $SPSM^2, TKWT\cdot TSTWT, TKWT, HT\cdot TSTWT, $
 $TKWT^2, KPS\cdot TKWT, KPS\cdot TSTWT, TKWT\cdot HT, KPS\cdot HT$
 $HT, HT^2, KPS, KPS^2$. With this order, we fit 27 models and each time we estimate $\hat{h}$ and find the the values of $\widehat{ERM}_{1}$, $\widehat{ERM}_{2}$ and $BIC$.

\begin{table}[h]
    \begin{subtable}[h]{0.45\textwidth}
        \centering
        \begin{tabular}{l | l | l | l | l}
          Size& $\hat{h}$ & $\widehat{ERM}_{1}$ & $\widehat{ERM}_{2}$ & $BIC$ \\
        \hline \hline
        1 & 6 & 7 & 10  & -9750\\
        2 & 6 & 7 & 10  & -9743\\
        3 & 6 & 7 & 10  & -9737\\
        4 & 6 & 7 & 10  & -9731\\
        5 & 6 & 7 & 10  & -9725\\
        6 & 6 & 7 & 10  & -9718\\
        7 & 6 & 7 & 10  & -9711\\
        8 & 6 & 7 & 10  & -9703\\
        9 & 6 & 7 & 10  & -9699
        \end{tabular}
        \caption{Model sizes from 1 to 9}
        \label{Loc8set1}
    \end{subtable}
    \hfill
    \begin{subtable}[h]{0.45\textwidth}
        \centering
        \begin{tabular}{l | l | l | l | l}
          Size& $\hat{h}$ & $\widehat{ERM}_{1}$ & $\widehat{ERM}_{2}$ & $BIC$ \\
        \hline \hline
        10 & 6 & 7 & 10 & -9691\\
        11 & 6 & 7 & 10 & -9683\\
        12 & 6 & 7 & 10 & -9675\\
        13 & 6 & 7 & 10 & -9668\\
        14 & 6 & 7 & 10 & -9660\\
        15 & 6 & 7 & 10 & -9652\\
        16 & 6 & 7 & 10 & -9645\\
        17 & 6 & 7 & 10 & -9638\\
        18 & 6 & 7 & 10 & -9632
        \end{tabular}
        \caption{Model sizes from 10 to 18}
        \label{Loc8set2}
    \end{subtable}
    \begin{subtable}[h]{0.45\textwidth}
        \centering
        \begin{tabular}{l | l | l | l | l}
          Size& $\hat{h}$ & $\widehat{ERM}_{1}$ & $\widehat{ERM}_{2}$ & $BIC$ \\
        \hline \hline
        19 & 6 & 7 & 10 &  -9625\\
        20 & 6 & 7 & 10 & -9617\\
        21 & 6 & 7 & 10 & -9610\\
        22 & 6 & 7 & 10 & -9603\\
        23 & 6 & 7 & 10 & -9596\\
        24 & 6 & 7 & 10 & -9590\\
        25 & 6 & 7 & 10 & -9582\\
        26 & 6 & 7 & 10 & -9583\\
        27 & 6 & 7 & 10 & -9575
        \end{tabular}
        \caption{Model sizes from 19 to 27}
        \label{Loc8set3}
    \end{subtable}
    \caption{Estimation of $\hat{h}$, $\widehat{ERM}_{1}$, $\widehat{ERM}_{1}$ and $BIC$ for wheat data}
    \label{Wheat8}
\end{table}
From Table \ref{Wheat8}, we see that no matter which model you pick, $\hat{h} = 6$ and $\widehat{ERM}_{1} = 7$ and $\widehat{ERM}_{2} = 10$. We observe variability in the values of $BIC$. The smallest value of $BIC$ occurs with the first model. Also, $\widehat{ERM}_{1}$ and $\widehat{ERM}_{2}$ pick a model of size 1. The fact that $\widehat{ERM}_{1}$ and $\widehat{ERM}_{2}$ is constant might indicate that variables included after the first one did not improve the prediction power of the model. However, $\hat{h}$ picked the model of size 6, a value reflecting a sort of middle point among $\hat{h}$'s found for the seven locations separately.
\begin{figure}
    \centering
    \begin{subfigure}[b]{0.4\textwidth}
    \centering
    \includegraphics[width = \textwidth]{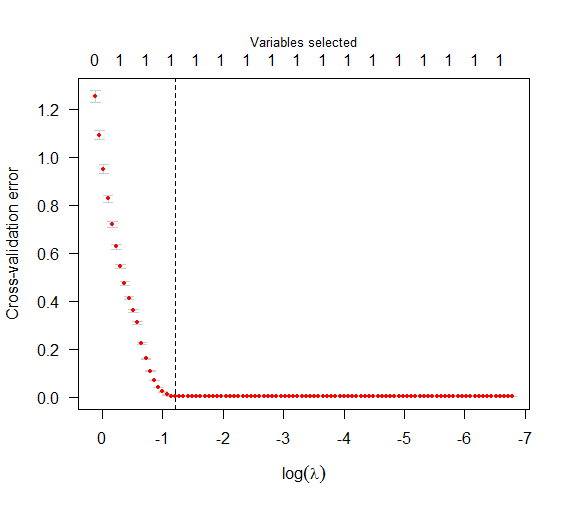}
    \caption{Optimal value of $\lambda$ across all locations}
    \label{lambda8}
    \end{subfigure}
    \hfill
     \begin{subfigure}[b]{0.4\textwidth}
    \centering
    \includegraphics[width = \textwidth]{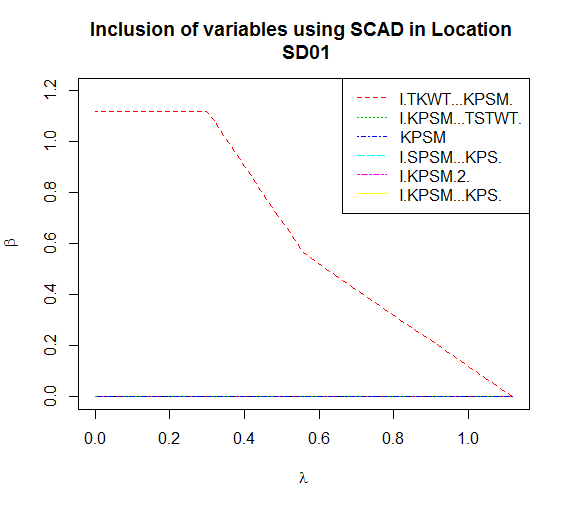}
    \caption{Trace of variables across all location using SCAD \cite{Fan&Li}}
    \label{trace8}
    \end{subfigure}
    \caption{Analysis of Wheat data set with all locations combined using SCAD \cite{Fan&Li}}
    \label{Loc8}
\end{figure}
Turning our attention to the sparsity methods, we observe from Fig. \ref{lambda8} that the optimal value of $\lambda$ is $\hat{\lambda} = 0.29736$; this indicates that the optimal number of parameters is 1. From Fig. \ref{trace8}, under SCAD, the variable chosen is $TKWT\cdot KPSM$, and the model can be written as:

\begin{equation}
\label{Loc8SCAD}
    \widehat{YIELD}  = 3.43 + 1.12\cdot KPSM\cdot TKWT.
\end{equation}
Similarly, using ALASSO we have
\begin{equation}
\label{Loc8ALASSO}
    \widehat{YIELD}  = 3.43 + 1.12\cdot KPSM\cdot TKWT
\end{equation}

We note here that ALASSO and SCAD are identical. $BIC$, $\widehat{ERM}_{1}$ and $\widehat{ERM}_{2}$ also pick $TKWT\cdot KPSM$ as the unique variable inside their model.
the model chosen by $\hat{h}$ is the most complex one with 6 variables. Thus, all techniques except $\hat{h}$ simply choose the single most important term across locations while $\hat{h}$ tried to choose a model that was overall representative of the data.

\begin{table}[h]
  \centering
  \caption{Nested models using all phenotype covariates}
\label{NestedPheno}

\begin{tabular}{|c|c|c|c|c|}
  \hline
  Model Size & $\widehat{h}$ & $\widehat{ERM}_{1} $ & $\widehat{ERM}_{2}$ & $BIC$ \\
  \hline\hline
  $KPSM$ & 6 & 1226.38 & 1273.28  & 5774.36\\
  $KPSM$, $TKWT$ & 6 & 56.88 & 66.56 &  -3234.779 \\
  $KPSM$, $TKWT$, $KPS$ & 6 & 50.75 & 59.86 & -3567.844 \\
  $KPSM$, $TKWT$, $KPS$, $HT$ & 6 & 49.03 & 57.98 & -3662.68\\
  $KPSM$, $TKWT$, $KPS$, $HT$, $TSTWT$ & 6 & 48.41 & 57.30 &-3692.71\\
  $KPSM$, $TKWT$, $KPS$, $HT$, $TSTWT$, $SPSM$ & 6 & 48.02 & 56.86 & -3709.31\\

  \hline
\end{tabular}
\end{table}

We redid our analysis using first order linear models for the sake of comparison.
From Table \ref{NestedPheno}, we see that $\hat{h} = 6$ for all models in the list. This indicates that we should use all phenotypic variables to predict the yield. This model is not accurate because we see from Table \ref{CorrPheno} that there is a strong correlation between some of our covariates and we suspect the terms $SPAM\cdot TKWT$ is important. We also see a sudden drop in the values of $\widehat{ERM}_{1}$, $\widehat{ERM}_{2}$ and $BIC$ when we move from model of size 1 to model of size 2, and a gradual decrease thereafter.

\begin{table}[h]
  \centering
  \caption{Models of size two using all phenotypes covariates}
\label{Phenosize2}

\begin{tabular}{|c|c|c|c|c|}
  \hline
  Model Size & $\widehat{h}$ & $\widehat{ERM}_{1} $ & $\widehat{ERM}_{2}$ & BIC \\
  \hline\hline
  $HT$, $TSTWT$ & 6 & 1329.03 &  1377.88  & 6008.69 \\
  $HT$, $TKWT$ & 6 & 3364.70  & 3442.70  & 8715.42 \\
  $HT$, $SPSM$ & 6 & 1676.28  & 1731.20  & 6685.26 \\
  $HT$, $KPSM$ & 6 & 482.73  & 511.98  & 3054.36 \\
  $HT$, $KPS$ & 6 & 3647.63 & 3728.86 & 8950.62 \\
  $TKWT$, $KPS$ & 6 & 3374.08 & 3452.19 & 8723.53 \\
  $TKWT$, $KPSM$ & 6 & $ 57 $ & $ 67 $ & $-3235 $ \\
  $TKWT$, $SPSM$ & 6 & 1218.70 & 1265.45 & 5756.05\\
  $TSTWT$, $KPS$ & 6 & 1314.87 & 1363.45 & 5977.47\\
  $TSTWT$, $KPSM$ & 6 & 265.95 & 287.53 & 1311.74\\
  $TSTWT$, $SPSM$ & 6 & 825.04 & 863.42 & 4618.51\\
  $TSTWT$, $TKWT$ & 6 & 1226.38 & 1273.28 & 5774.36\\
  $KPS$, $ KPSM$ & 6 & 524.57 & 555.08 & 3297.08 \\
  $SPSM$, $KPS$ & 6 & 941.17 & 982.20 & 5002.59\\
  $SPSM,$ $KPSM$ & 6 & 510.22 & 540.31 & 3216.12\\
  \hline
\end{tabular}
\end{table}

\begin{table}[h]
  \centering
  \caption{Models of size three using all phenotypes covariates}
\label{Phenosize3}

\begin{tabular}{|c|c|c|c|c|}
  \hline
  Model Size & $\widehat{h}$ & $\widehat{ERM}_{1} $ & $\widehat{ERM}_{2}$ & $BIC$ \\
  \hline\hline
  $HT$, $KPS$,$KPSM$ & 6 & 478.89 & 508.02 & 3039.01  \\
  $HT$, $SPSM$, $KPS$ & 6 & 877.51 & 917.11 & 4806.32 \\
  $HT$, $SPSM$, $KPSM$ & 6 & 459.07 & 487.58 & 2915.55 \\
  $HT$, $TKWT$, $KPS$ & 6 & 3360.45 & 3438.40 & 8719.71 \\
  $HT$, $TKWT$, $KPSM$ & 6 & $ 53.75 $ & $ 63.15 $ & $ -3395.81 $ \\
  $HT$, $TKWT$, $SPSM$ & 6 & 1163.39 & 1209.06 & 5628.62 \\
  $HT$, $TSTWT$, $KPS$ & 6 & 1314.02 & 1362.59 & 5983.56 \\
  $HT$, $TSTWT$, $KPSM$ & 6 & 251.95 & 272.94 & 1161.43 \\
  $HT$, $TSTWT$, $SPSM$ & 6 & 823.61 & 861.97 & 4621.44\\
  $HT$, $TSTWT$, $TKWT$ & 6 & 1220.27 & 1267.05 & 5767.78\\
  $TSTWT$, $SPSM$, $KPS$ & 6 & 430.42 & 458.01 & 2727.35 \\
  $TSTWT$, $SPSM$, $KPSM$ & 6 & 253.69 & 274.75 & 1181.48\\
  $TSTWT$, $TKWT$, $KPS$ & 6 & 1214.51 & 1261.19 & 5754.00\\
  $TSTWT$, $TKWT$, $KPSM$ & 6 & 56.53 & 66.18 & -3245.16\\
  $TSTWT$, $TKWT$, $SPSM$ & 6 & 821.48 & 859.78 & 4613.88 \\
  $TKWT$, $KPS$, $KPSM$ & 6 & 50.75 & 59.86 & -3567.844 \\
  $TKWT$, $SPSM$, $KPS$ & 6 & 247.12 & 267.90 & 1104.71\\
  $TKWT$, $SPSM$, $KPSM$ & 6 & 50.36 & 59.44 & -3590.39\\
  $SPSM$, $KPS$, $KPSM$ & 6 & 449.08 & 477.27 & 2851.33 \\
  \hline
\end{tabular}
\end{table}

\begin{table}[h]
  \centering
  \caption{Models of size four using all phenotypes covariates}
\label{Phenosize4}

\begin{tabular}{|c|c|c|c|c|}
  \hline
  Model Size & $\widehat{h}$ & $\widehat{ERM}_{1} $ & $\widehat{ERM}_{2}$ & BIC \\
  \hline\hline
  $HT$, $SPSM$, $KPS$, $KPSM$    & 6 & 416.29 & 443.41 & 2637.80\\
  $TSTWT$, $TKWT$, $SPSM$, $KPS$ & 6 & 244.28 & 264.94 & 1078.83 \\
  $TSTWT$, $TKWT$, $SPSM$, $KPSM$ & 6 & 50.12 & 59.17 & -3596.89\\
  $HT$, $TKWT$, $SPSM$, $KPS$ & 6 & 244.49 & 265.16 & 1081.35\\
  $HT$, $TKWT$, $SPSM$, $KPSM$ & 6 & 48.64 & 57.55 & -3686.69\\
  $HT$, $TKWT$, $KPS$, $KPSM$ & 6 & 49.03 & 57.98 & -3662.68\\
  $HT$, $TSTWT$, $KPS$, $KPSM$   & 6 & 244.28 & 264.94 & 1078.83 \\
  $HT$, $TSTWT$, $SPSM$, $KPS$  & 6 & 411.54 & 438.51 & 2604.33 \\
  $HT$, $TSTWT$, $SPSM$, $KPSM$ & 6 & 234.78 & 255.02 & 962.65 \\
  $HT$, $TSTWT$, $TKWT$, $KPS$  & 6 & 1204.10 & 1250.58 & 5736.89\\
  $HT$, $TSTWT$, $TKWT$, $KPSM$ & 6 & 52.87 & 62.19 & -3436.98\\
  $HT$, $TSTWT$, $TKWT$, $SPSM$ & 6 & 816.28 & 854.46 & 4603.33\\
  \hline
\end{tabular}
\end{table}

\begin{table}[h]
  \centering
  \caption{Models of size five and six using all phenotypes covariates}
\label{Phenosize5}

\begin{tabular}{|c|c|c|c|c|}
  \hline
  Model Size & $\widehat{h}$ & $\widehat{ERM}_{1} $ & $\widehat{ERM}_{2}$ & BIC \\
  \hline\hline
  $HT$, $TKWT$, $SPSM$, $KPS$, $KPSM$ & 6 & 48.61 & 57.51 & -3680.81\\
  $HT$, $TSTWT$, $TKWT$, $KPS$, $KPSM$  &6 &  48.41 & 57.30 &-3692.71\\
  $HT$, $TSTWT$, $TKWT$, $SPSM$, $KPS$ & 6 & 233.90 & 254.11 & 959.71\\
  $HT$, $TSTWT$, $TKWT$, $SPSM$, $KPSM$ & 6 & 48.05 & 56.90 & -3715.12\\
  $HT$, $TSTWT$, $SPSM$, $KPS$, $KPSM$ & 6 & 226.56 & 246.45& 866.46\\
  $TSTWT$, $TKWT$, $SPSM$, $KPS$, $KPSM$ & 6 & 50.05 & 59.09 & -3593.21\\
  $KPSM$, $TKWT$, $KPS$, $HT$, $TSTWT$, $SPSM$ & 6 & 48.02 & 56.86 & -3709.31\\
  \hline
\end{tabular}
\end{table}
from Tables \ref{Phenosize2} -- \ref{Phenosize5}, $\hat{h} = 6$, we see variability in the values of $\widehat{ERM}_{1}$, $\widehat{ERM}_{2}$ and $BIC$. The smallest values of $\widehat{ERM}_{1}$, $\widehat{ERM}_{2}$ and $BIC$ occur with the model of size 6. in agreement with $\hat{h}$. This model is not the best we can have because $KPSM\cdot TKWT$ is important to include. Note that here we believe the model list excludes many good models but $\hat{h}$, $\widehat{ERM}_{1}$, $\widehat{ERM}_{2}$ and $BIC$ give the same model and $\hat{h}$ is a constant equal to its value on the earlier model list of size 27 when all locations are combined.

\section{Estimate of VCD using the design structure of the model}
\label{Design}
Our objective in section is to see the impact of taking into account the design structure in our estimate. We implement our method on Lincoln 2000 and 2001 datasets. We are aware that the bootstrap technique that we used in our method will destroy the design used to collect this dataset, however, we took account of the blocking factor for completeness. Since the blocking variable is a factor variable with 32 different categories, we cannot use it to compute the correlation; so to include variable in the model, we use the order of inclusion from the analysis of the Lincoln 2000 and 2001 datasets. For each model, we add the blocking variable. For instance, in Lincoln 1999, the first term that gets into the model is $TKWT\cdot KPSM$, the model that we fit has $TKWT\cdot KPSM$ plus $IBLK_{j}, j = 1,2, \cdots, 32$ ($IBLK_{j}$ is the $j^{th}$ incomplete block). So that is the size of each model increases by 32. As another example, the $3^{rd}$ model indicated in Table \ref{Loc1} becomes

\begin{equation*}
    \widehat{YIELD} = \beta_{0} + \beta_{1}TKWT\cdot KPSM + \beta_{2}KPSM\cdot TSTWT
    +\beta_{3}KPSM\cdot HT +\sum_{k=1}^{32}\theta_{k}IBLK_{k}
\end{equation*}

\subsubsection{Analysis of Lincoln 2000 data with the design structure included in the model}
\begin{table}[h]
    \begin{subtable}[h]{0.45\textwidth}
        \centering
        \begin{tabular}{l | l | l | l | l}
          Size& $\hat{h}$ & $\widehat{ERM}_{1}$ & $\widehat{ERM}_{2}$ & $BIC$ \\
        \hline \hline
        1 & 13 & 50 & 63  & 490\\
        2 & 13 & 50 & 63  & 493\\
        3 & 13 & 25 & 34  & 191\\
        4 & 13 & 25 & 33  & 195\\
        5 & 15 & 18 & 26  & 49\\
        6 & 15 & 17 & 25  & 33\\
        7 & 16 & 11 & 17  & -173\\
        8 & 16 & 9 & 15  & -267\\
        9 & 16 & 9 & 14  & -274
        \end{tabular}
        \caption{Model sizes from 1 to 9}
        \label{Loc1set11}
    \end{subtable}
    \hfill
    \begin{subtable}[h]{0.45\textwidth}
        \centering
        \begin{tabular}{l | l | l | l | l}
          Size& $\hat{h}$ & $\widehat{ERM}_{1}$ & $\widehat{ERM}_{2}$ & $BIC$ \\
        \hline \hline
        10 & 16 & 9 & 14 & -271\\
        11 & 16 & 9 & 14 & -267\\
        12 & 16 & 9 & 14 & -264\\
        13 & 16 & 9 & 14 & -259\\
        14 & 16 & 8 & 14 & -276\\
        15 & 16 & 8 & 14 & -270\\
        16 & 16 & 8 & 14 & -271\\
        17 & 16 & 5 & 10 & -519\\
        18 & 15 & 5 & 9 & -518
        \end{tabular}
        \caption{Model sizes from 10 to 18}
        \label{Loc1set21}
    \end{subtable}
    \begin{subtable}[h]{0.45\textwidth}
        \centering
        \begin{tabular}{l | l | l | l | l}
          Size& $\hat{h}$ & $\widehat{ERM}_{1}$ & $\widehat{ERM}_{2}$ & $BIC$ \\
        \hline \hline
        19 & 15 & 5 & 9 & -518  \\
        20 & 15 & 5 & 9 & -513\\
        21 & 15 & 5 & 9 & -507 \\
        22 & 15 & 5 & 9 & -501\\
        23 & 15 & 4 & 7 & -778\\
        24 & 15 & 2 & 4 & -773\\
        25 & 15 & 2 & 4 & -3325\\
        26 & 15 & 2 & 4 & -3319\\
        27 & 15 & 2 & 4 & -3313
        \end{tabular}
        \caption{Model sizes from 19 to 27}
        \label{Loc1set31}
    \end{subtable}
    \caption{Estimation of $\hat{h}$, $\widehat{ERM}_{1}$, $\widehat{ERM}_{1}$ and $BIC$ in LIN00; note that the model is 32 plus the size of the model due to including $IBLK$}
    \label{Loc11}
\end{table}
Table \ref{Loc11} gives the estimate of $\hat{h}$, the values of $\widehat{ERM}_{1}$, $\widehat{ERM}_{2}$ and $BIC$ in Lincoln 2000. In these cases, we included the design structure represented by the incomplete block.

We observe small variability in the estimate of $\hat{h}$: it takes on values 13, 15, or 16. The smallest discrepancy between the size of the conjectured model and $\hat{h}$ occurs when the size of the conjectured model is 16. In fact, at that point $\hat{h} = 16$. We also see that there is  variability in the values of $\widehat{ERM}_{1}$ and $\widehat{ERM}_{2}$ as the size of the conjectured model increases. The smallest value of $\widehat{ERM}_{1}$ namely two, occurs first when the size of the conjectured model is 24; this value stays constant up to the most complex model. So, so we would pragmatically say that $\widehat{ERM}_{1}$ picks a model of size 24. The smallest value of $\widehat{ERM}_{2}$ first occurs at the conjectured model of size 24 and stays constant up to the model of size 27. Being parsimonious, we would say that $\widehat{ERM}_{2}$ also picks a model of size 24. On the other hand, $BIC$ picks the most complex model i.e.the  model with the most variables.

\begin{figure}
    \centering
    \begin{subfigure}[b]{0.4\textwidth}
    \centering
    \includegraphics[width = \textwidth]{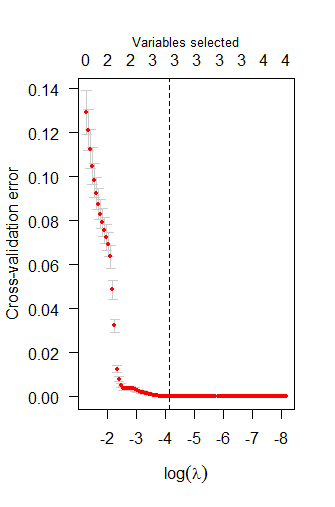}
    \caption{Optimal value of $\lambda$ In Lincoln 2000 with design structure included}
    \label{lambdadesign}
    \end{subfigure}
    \hfill
     \begin{subfigure}[b]{0.4\textwidth}
    \centering
    \includegraphics[width = \textwidth]{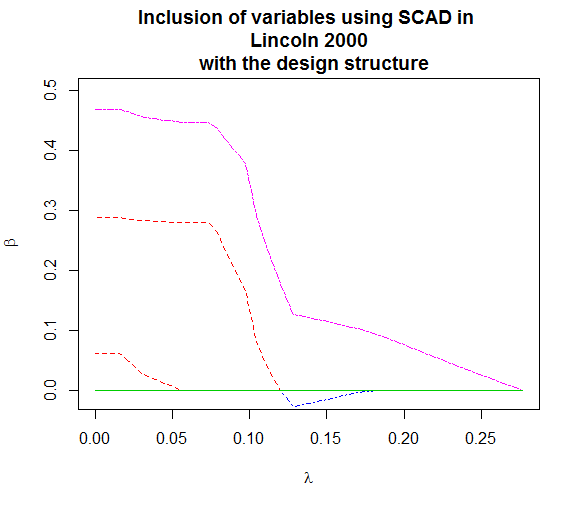}
    \caption{Trace of variables in Lincoln 2000 with design structure included using SCAD \cite{Fan&Li}}
    \label{struct}
    \end{subfigure}
    \caption{Analysis of Wheat data set in Lincoln 2000 with the design structure included in the model using SCAD \cite{Fan&Li}}
    \label{DesignStruct}
\end{figure}
Turning our attention to the sparsity methods, we observe from Fig. \ref{lambdadesign} that the optimal value of $\lambda$ is $\hat{\lambda} = 0.0158356$; this indicates that the optimal number of parameters is 3. From Fig. \ref{struct}, under SCAD, the variables chosen are $TKWT\cdot KPSM$, $KPSM$ and $TKWT$. The model can be written as:
\begin{equation}
\label{DesignSCAD}
    \widehat{YIELD}  = -0.18 + 0.06\cdot KPSM\cdot TKWT + 0.47\cdot KPSM
    + 0.29\cdot TKWT.
\end{equation}
The coefficients of $IBLK$ under SCAD were set to zero, but including $IBLK$ in our SCAD analysis give us two extra terms compare to \eqref{Loc1SCAD}.

We note here that all model selection techniques pick a different model list. $BIC$ picks the most complex model list, $\widehat{ERM}_{1}$ and $\widehat{ERM}_{2}$ pick a model list of size 27, $\hat{h}$ pick a model of size 16 and SCAD pick a model of size 3. We also note that in \eqref{DesignSCAD} we have the interaction between $TKWT$ and $KPSM$ and their main effect. This is the first time that this occurs using shrinkage methods. The model given by SCAD is the best smallest model that we can have, however, the question is: Is this model the overall best model? I would say that this is not the overall best model because this model does not take into account the different type of varieties of wheat. (We comment that, here we have not used ALASSO because the R code we used did not support categorical variables.)

The effect of including the design variables is substantial. In Table \ref{Loc1}, $\hat{h} = 1$; here $\hat{h} = 16$. Likewise, there are large difference in the other methods; for instance, compare \eqref{Loc1SCAD} to \eqref{DesignSCAD}, or $\widehat{ERM}_{1}$ and $\widehat{ERM}_{2}$ in Table \ref{Loc1} to the values in Table \ref{Loc11}.

\subsubsection{Analysis of Lincoln 2001 data with the design structure included in the model}

\begin{table}[h]
    \begin{subtable}[h]{0.45\textwidth}
        \centering
        \begin{tabular}{l | l | l | l | l}
          Size& $\hat{h}$ & $\widehat{ERM}_{1}$ & $\widehat{ERM}_{2}$ & $BIC$ \\
        \hline \hline
        1 & 16 & 31 & 42  & 284\\
        2 & 12 & 4  & 7  & -783\\
        3 & 10 & 3  & 5  & -943\\
        4 & 16 & 10 & 16  & -239\\
        5 & 16 & 10 & 16  & -234\\
        6 & 15 & 10 & 16  & -234\\
        7 & 14 & 7  & 11  & -441\\
        8 & 11 & 3  & 5  & -922\\
        9 & 12 & 3  & 6  & -924
        \end{tabular}
        \caption{Model sizes from 1 to 9}
        \label{Loc2set12}
    \end{subtable}
    \hfill
    \begin{subtable}[h]{0.45\textwidth}
        \centering
        \begin{tabular}{l | l | l | l | l}
          Size& $\hat{h}$ & $\widehat{ERM}_{1}$ & $\widehat{ERM}_{2}$ & $BIC$ \\
        \hline \hline
        10 & 11 & 3 & 6 & -918\\
        11 & 12 & 3 & 6 & -913\\
        12 & 9 & 3 & 4 & -1000\\
        13 & 9 & 3 & 4 & -1000\\
        14 & 9 & 3 & 4 & -992\\
        15 & 9 & 3 & 4 & -986\\
        16 & 9 & 3 & 4 & -980\\
        17 & 8 & 3 & 4 & -981\\
        18 & 42 & 3 & 19 & -1452
        \end{tabular}
        \caption{Model sizes from 10 to 18}
        \label{Loc2set22}
    \end{subtable}
    \begin{subtable}[h]{0.45\textwidth}
        \centering
        \begin{tabular}{l | l | l | l | l}
          Size& $\hat{h}$ & $\widehat{ERM}_{1}$ & $\widehat{ERM}_{2}$ & $BIC$ \\
        \hline \hline
        19 & 32 & 3 & 7 & -1853\\
        20 & 32 & 3 & 7 & -1831\\
        21 & 32 & 2 & 3 & -1826\\
        22 & 28 & 3 & 7 & -1829\\
        23 & 28 & 3 & 7 & -1823\\
        24 & 28 & 3 & 7 & -1817\\
        25 & 27 & 3 & 6 & -1819\\
        26 & 27 & 3 & 6 & -1819\\
        27 & 27 & 3 & 6 & -1819
        \end{tabular}
        \caption{Model sizes from 19 to 27}
        \label{Loc2set32}
    \end{subtable}
    \caption{Estimation of $\hat{h}$, $\widehat{ERM}_{1}$, $\widehat{ERM}_{1}$ and $BIC$ in LINO1 with the design structure included}
    \label{Loc22}
\end{table}

Table \ref{Loc22} gives the estimate of $\hat{h}$, the values of $\widehat{ERM}_{1}$, $\widehat{ERM}_{2}$ and $BIC$ in Lincoln 2001. In these cases, we included the design structure represented by the incomplete block.

We observe some a lot of variability in the estimate of $\hat{h}$: it takes on values from 8 to 42. The smallest discrepancy between the size of the conjectured model and $\hat{h}$ occurs when the size of the conjectured model is 27. In fact, at that point $\hat{h} = 27$. We also see that there is variability in the values of $\widehat{ERM}_{1}$ and $\widehat{ERM}_{2}$ as the size of the conjectured model increases. The smallest value of $\widehat{ERM}_{1}$ namely two, occurs first when the size of the conjectured model is 21. The smallest value of $\widehat{ERM}_{2}$ also occurs at the conjectured model of size 21. The smallest value of $BIC$ occurs with model of size 19.
\begin{figure}
    \centering
    \begin{subfigure}[b]{0.4\textwidth}
    \centering
    \includegraphics[width = \textwidth]{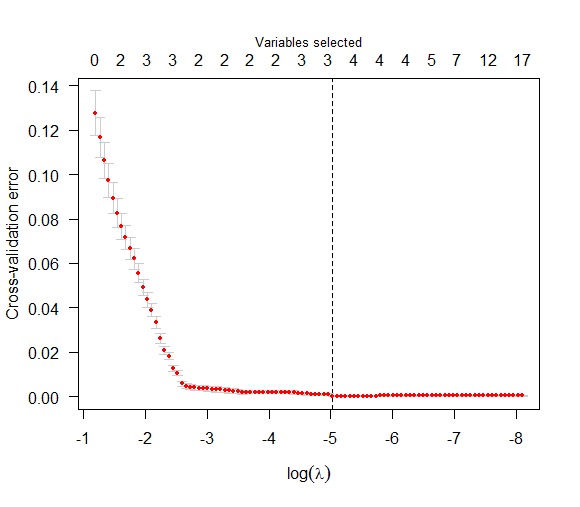}
    \caption{Optimal value of $\lambda$ In Lincoln 2001 with design structure included}
    \label{lambdadesign01}
    \end{subfigure}
    \hfill
     \begin{subfigure}[b]{0.4\textwidth}
    \centering
    \includegraphics[width = \textwidth]{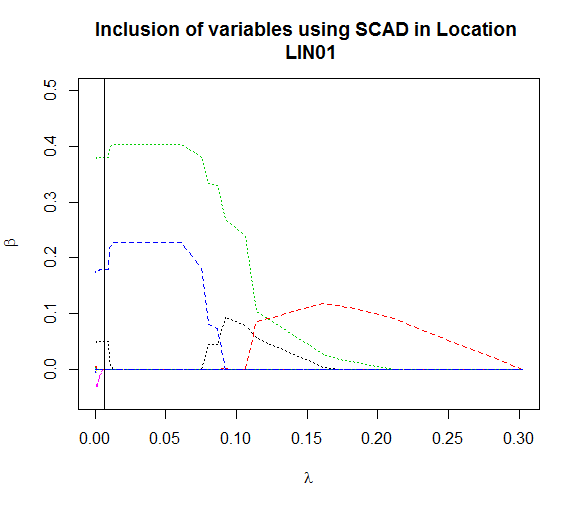}
    \caption{Trace of variables in Lincoln 2001 with design structure included using SCAD \cite{Fan&Li}}
    \label{struct01}
    \end{subfigure}
    \caption{Analysis of Wheat data set in Lincoln 2001 with the design structure included in the model using SCAD \cite{Fan&Li}}
    \label{DesignStruct01}
\end{figure}
Turning our attention to the sparsity methods, we observe from Fig. \ref{lambdadesign01} that the optimal value of $\lambda$ is $\hat{\lambda} = 0.00651666$; this indicates that the optimal number of parameters is 3. From Fig. \ref{struct}, under SCAD, the variable chosen is $TKWT\cdot KPSM$, $KPSM$ and $TKWT$
and the model can be written as:
\begin{equation}
\label{DesignSCAD01}
    \widehat{YIELD}  = 1.10 + 0.05\cdot KPSM\cdot TKWT + 0.38\cdot KPSM
    + 0.18\cdot TKWT.
\end{equation}
This has the same explanatory variables as \eqref{DesignSCAD}.

We note again that all model selection techniques pick a different model. $BIC$ picks a model with 19 variables, $\widehat{ERM}_{1}$ and $\widehat{ERM}_{2}$ pick a model of size 21, $\hat{h}$ pick a model of size 27 and SCAD pick a model of size 3. We also note that in \eqref{DesignSCAD01} we have the interaction between $TKWT$ and $KPSM$ and their main effects.

Overall, the model sizes range from 3 (SCAD), 19 (BIC), 21 ($\widehat{ERM}_{1}$ and $\widehat{ERM}_{2}$) to 27 ($\hat{h}$). Since SCAD and BIC are sparsity or sparsity-like procedures, it is not surprising that they give the smallest models and unless a coarse model is desired, SCAD and possibly BIC can be ignored. In this example, we are left with models of size 21 and 27 and we recall that, compared to Bayes methods such as SCAD and BIC, our method is more sensitive to bias or, more precisely to exactly which variables are being included. We suspect that as in other cases, $\hat{h}$ is even more sensitive than $\widehat{ERM}_{1}$ and $\widehat{ERM}_{2}$ and has detected that some variables are relevant when blocking in included. Indeed, the elevated size of $\hat{h}$ shows how important the blocking is.


\section{Analysis of Wheat data using SNP information}
\label{AnalSNP}
Our goal in this section is to see the effect of how taking the SNP (Single Nucleotide Polymorphism) information into account will affect the estimate of VCD. Since the variables representing the SNP's have so much missing data, we imply dropped all rows with missing SNP values and only use SNP's that are complete (No missing values i.e., we dropped columns as well.) Thus we retained only  6 SNP's and our sample size was reduced to 2631. As before, we ordered the inclusion of covariates in our model using correlation. Under correlation is inclusion of terms in our model is as follows:
$ KPSM, TSTWT, SPSM, $ $KPS\cdot TKWT, TKWT, KPS\cdot TSTWT,$
$ TKWT\cdot KPSM, KPS\cdot HT, barc67, $ $HT, $ $cmwg680bcd366, $ $ bcd141, barc86,$ $ gwm155, barc12, $
$KPS, SPSM\cdot HT, KPSM\cdot HT,$
$ SPSM\cdot TKWT, SPSM^2, KPS^2,$
$HT^2, SPSM\cdot KPS, KPSM\cdot TSTWT,$
$TKWT\cdot TSTWT, KPSM\cdot SPSM, TSTWT^2,$
$KPSM^2,$ $ SPSM\cdot TSTWT, HT\cdot TSTWT,$ $ TKWT^2,$ $ $ $TKWT\cdot HT $. Here the SNP's are symply denoted by their labels in the data.
With this order, we fit 32 different models, each time, we estimate $\hat{h}$, find values of $\widehat{ERM}_{1}$, and $\widehat{ERM}_{2}$ and models for $BIC$, SCAD and ALASSO.
\begin{table}[h]
    \begin{subtable}[h]{0.45\textwidth}
        \centering
        \begin{tabular}{l | l | l | l | l}
          Size& $\hat{h}$ & $\widehat{ERM}_{1}$ & $\widehat{ERM}_{2}$ & $BIC$ \\
        \hline \hline
        1 & 37 & 380 & 439 &  2381\\
        2 & 37 & 193 & 235 & 587\\
        3 & 37 & 185 & 226 & 477\\
        4 & 37 & 181 & 222  & 430\\
        5 & 36 & 37 & 55  & -3930\\
        6 & 37 & 37 & 56  & -3924\\
        7 & 37 & 8 & 17  & -9100\\
        8 & 37 & 8  & 17  & -9093\\
        9 & 36 & 8  & 16  & -9089\\
        10 & 36 & 8 & 16 & -9081
        \end{tabular}
        \caption{Model sizes from 1 to 10}
        \label{Gene1}
    \end{subtable}
    \hfill
    \begin{subtable}[h]{0.45\textwidth}
        \centering
        \begin{tabular}{l | l | l | l | l}
          Size& $\hat{h}$ & $\widehat{ERM}_{1}$ & $\widehat{ERM}_{2}$ & $BIC$ \\
        \hline \hline
        11 & 36 & 8 & 16 & -9074\\
        12 & 36 & 8 & 16 & -9066\\
        13 & 36 & 8 & 16 & -9058\\
        14 & 36 & 8 & 16 & -9052\\
        15 & 36 & 8 & 16 & -9045\\
        16 & 36 & 8 & 16 & -9040\\
        17 & 36 & 8 & 16 & -9033\\
        18 & 36 & 8 & 16 & -9026\\
        19 & 36 & 8 & 16 & -9019\\
        20 & 36 & 8 & 16 & -9014
        \end{tabular}
        \caption{Model sizes from 11 to 20}
        \label{Gene2}
    \end{subtable}
    \begin{subtable}[h]{0.45\textwidth}
        \centering
        \begin{tabular}{l | l | l | l | l}
          Size& $\hat{h}$ & $\widehat{ERM}_{1}$ & $\widehat{ERM}_{2}$ & $BIC$ \\
        \hline \hline
        21 & 36 & 8 & 16 & -9006\\
        22 & 36 & 8 & 16 & -8999\\
        23 & 41 & 8 & 17 & -8993\\
        24 & 41 & 8 & 17 & -8989\\
        25 & 41 & 8 & 17 & -8982\\
        26 & 41 & 8 & 17 & -8978\\
        27 & 41 & 8 & 17 & -8970\\
        28 & 41 & 8 & 17 & -8962 \\
        29 & 41 & 8 & 17 & -8961 \\
        30 & 41 & 8 & 17 & -8956 \\
        31 & 41 & 8 & 17 & -8948 \\
        32 & 41 & 8 & 17 & -8940 \\
        \end{tabular}
        \caption{Model sizes from 19 to 27}
        \label{Gene3}
    \end{subtable}
    \caption{Estimation of $\hat{h}$, $\widehat{ERM}_{1}$, $\widehat{ERM}_{1}$ and $BIC$ with the SNPS included}
    \label{Gene}
\end{table}
Table \ref{Gene} gives the estimate of $\hat{h}$, the values of $\widehat{ERM}_{1}$, $\widehat{ERM}_{2}$ and $BIC$. In these cases, we included the SNP'S variables.

We observe some variability in the estimate of $\hat{h}$: it takes values 36, 37, and 41. The smallest discrepancy between the size of the conjectured model and $\hat{h}$ occurs when the size of the conjectured model is 32. In fact, at that point $\hat{h} = 41$,  suggests that the sample size may be low for our purpose here (We recall that model selection requires more data than parameter estimation). We also observe some variability in the values of $\widehat{ERM}_{1}$ and $\widehat{ERM}_{2}$ as the size of the conjectured model increases. There is a big drop in their value when the size of the conjectured model is 7, and thereafter, they become almost flat. The smallest value of $\widehat{ERM}_{1}$, namely eight, and $\widehat{ERM}_{2}$, namely 16, occur for models of size 7 and 9 respectively. The smallest value of $BIC$ occurs for model of size 8.

\begin{figure}
    \centering
    \begin{subfigure}[b]{0.4\textwidth}
    \centering
    \includegraphics[width = \textwidth]{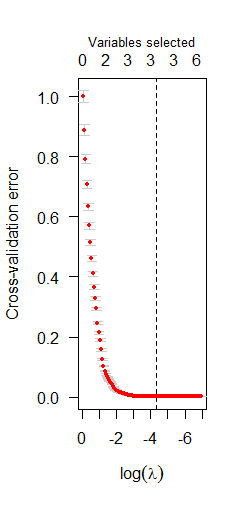}
    \caption{Optimal value of $\lambda$ for wheat data with SNP variables included}
    \label{lambdaSNIP}
    \end{subfigure}
    \hfill
     \begin{subfigure}[b]{0.4\textwidth}
    \centering
    \includegraphics[width = \textwidth]{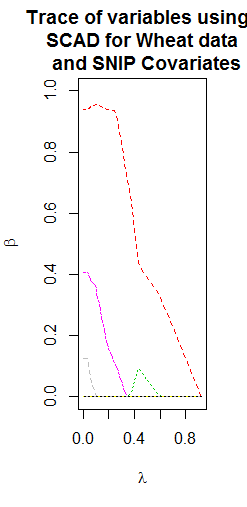}
    \caption{Trace of variables for wheat dataset with SNP variables included using SCAD }
    \label{structSNIP}
    \end{subfigure}
    \caption{Analysis of Wheat data with SNP covariates included in the model using SCAD}
    \label{DesignStructSNIP}
\end{figure}
Turning our attention to the sparsity methods, we observe from Fig. \ref{lambdaSNIP} that the optimal value of $\lambda$ is $\hat{\lambda} = 0.01311611$; this indicates that the optimal number of parameters is 3. From Fig. \ref{struct}, under SCAD, the variable chosen is $TKWT\cdot KPSM$, $KPSM$ and $TKWT$
and the model can be written as:
\begin{equation}
\label{DesignSCADSNIP}
    \widehat{YIELD}  = 0.13\cdot KPSM\cdot TKWT + 0.94\cdot KPSM
    + 0.41\cdot TKWT.
\end{equation}
Similarly, using ALASSO we have,
\begin{equation}
\label{DesignALASSOSNIP}
    \widehat{YIELD}  = 0.13\cdot KPSM\cdot TKWT + 0.94\cdot KPSM
    + 0.41\cdot TKWT.
\end{equation}
So, \eqref{DesignALASSOSNIP} and \eqref{DesignSCADSNIP} are identical.
We also note that $BIC$ and $\widehat{ERM}_{1}$ pick the same variables. $\widehat{ERM}_{2}$ picks a model of size 9 (Large by one term), $\hat{h}$ picks the most complex model, and SCAD and ALASSO pick a the smallest models, of size 3. We also note that in \eqref{DesignSCADSNIP} and \eqref{DesignALASSOSNIP} we have the interaction between $TKWT$ and $KPSM$ and their main effects. As noted before, $\hat{h}$ tends to be most sensitive to bias, so it is not surprising that it is the largest and includes all SNP's. On the other hand, the first SNP to appear on the list is $barc67$, $9^{th}$ on the list. This means SCAD, ALASSO, $\widehat{ERM}_{1}$ and $BIC$ choose no SNP's. However, $\widehat{ERM}_{2}$ includes $barc67$, and $\hat{h}$ includes all the SNP's that had no missing values. Moreover, $\hat{h}$ being larger than 32 -- the largest model size-- suggests that other SNP's that may have been excluded due to missing values, or other missing variables such as variety, may be critical to explain $YIELD$.

1\chapter{General Conclusions and Future Work}
\label{futurework}
\thispagestyle{fancy}

In this Dissertation, we have developed a general method to estimate the VCD $h$ for classes of regression models.

Our method rests on minimizing an upper bound on an error. The upper bound depends on $h$ and generates an estimator $\hat{h}$. We have applied our method to a variety of linear regression problems. Our method is highly data driven and at least for linear regression models with large enough sample sizes,  seems to give better results on simulated and real datasets in contrast with other established methods. Otherwise put,
at a minimum, the models selected by our method are, on balance, no worse and often better than models selected by other methods such as $BIC$, SCAD, ALASSO, and two forms of empirical risk minimization ($\widehat{ERM}$). The comparison between $\hat{h}$ and the $\widehat{ERM}$'s is not a surprise since the $\widehat{ERM}$'s are function of $\hat{h}$.

In Sec. \ref{GenreConcl}, we review what the Dissertation has accomplished, and in Sec. \ref{sectfuturework}, we identify gaps that remain to be filled and future work

\section{General Conclusions}
\label{GenreConcl}
In Chapter \ref{chap:Vapnik:Cherv:Dim:Cov:Num}, we reviewed the Vapnik-Chervonenkis concept of complexity often called a dimension here abbreviated $VCD$. Since it is relatively unfamiliar to statisticians, we recalled the geometric, combinatorial and covering number definitions of $VCD$. Then we summarized the work of Vapnik et al. \cite{Vapnik:etal:1994} and McDonald et al. \cite{McDonald:etal:2011} including the main theoretical results and the computational procedure.

All of this previous work was done for classification problems whereas our work generalizes it to regression problems. This led us to modify substantially not just the theory, but also the computing. Moreover we have had at various junctures to correct simplifying assumptions made by earlier authors so as to obtain good performance of the VCD as a model selection principle.

Chapter \ref{chap:bounds} presents the main mathematical results undergirding our methodology. Specifically, we convert a generic regression problem into $m$ classifications problems and apply corrected form of Vapnik et al to each of the $m$ problems. This gives us a bound on the Expected Maximum Difference between Two Empirical Losses (EMDBTEL)  in term of $h$, the correct VCD.

In Chapter \ref{chap:Numerical:Studies}, computational exploration showed this bound was not tight enough for good model selection. So we optimized the constant factor in the upper bound and changed the error criterion so it ressembled a cross-validation error rather than simply the difference between two models. We verified that the mathematics of Chapter \ref{chap:bounds} continued to hold with this change of the form of error.

The second major contribution of Chapter \ref{chap:Numerical:Studies} was to correct the Vapnik et al \cite{Vapnik:etal:1994} computational procedure so it would accurately encapsulate the mathematical quantities the upper bound used. We denoted this estimate of $h$ by $\hat{h}$. Our improved procedure is more computationally demanding. But not terribly so in term of running time. Thus, in our extensive simulations with linear models, we found the new method gave good results compared to Vapnik's original algorithm as well as compared to $BIC$ and two forms of Empirical Risk Minimization propose by Vapnik \cite{Vapnik:1998}.

In Chapter \ref{chap:Numerical:Studies} we also introduced the concept of `consistency at the true model' to help identify how exactly our method should be used. Indeed, all our method give is an estimate $\hat{h}$ of the $VCD$ of the data generator (DG). To obtain this estimate in practice, we must conjecture the data generator and compare the various $\hat{h}$'s from the DG's to the known true $h$'s from the DG's. Thus using $\hat{h}$ can always cut down a model list to a sub-list although in general $\hat{h}$ does not by itself identify a DG. To get around this problem we used nested model lists where the VCD of the models increased with their size. For these cases, $\hat{h}$ is able to give a unique model, apart from random variability -- which can be considerable.

Our used of $\hat{h}$ required the selection of design points to use in a non-linear regression. Different design points lead to different values of $\hat{h}$ as seen in Chapter \ref{chap:Numerical:Studies}, however, we note a sample size effect: As $n$ increases it seems that
\begin{enumerate}
    \item the design points matter less and less and
    \item the accuracy of $\hat{h}$ for selecting among nested models (with variables ordered for inclusion only by their absolute correlation with the response or by the order of inclusion using SCAD) increases.
\end{enumerate}

In chapter \ref{example}, we estimated $\hat{h}$ using the Tour De France Dataset. We fitted two set of models. The first set is based on the following covariates $Y, D, D^2, Y^2$, and  $Y:D$, the second contains $Y, D, S$, and $A$ (Definitions and notations as in Chapter \ref{example}). Whichever of the two model lists is chosen, $\hat{h} = 4$. That is, the best model to predict the average speed of the winner of the \sf{Tour De France} dataset should have $4$ variables. We conclude that the better model list is the one containing $Y$, $D$, $Y^2$, $D^2$.

Our analysis in this example is exhaustive so we can get an indication of how our method performs relative to $BIC$ and two forms of ERM. We emphasize that although this chapter uses real datasets it is really a sort of `toy' problem since the sample size $n$ is not large compared to the number of inferences we want to make.

In chapter \ref{Wheatdata}, we implemented our method on more complex data sets. We also performed sparsity driven analyses, using SCAD and ALASSO in addition to estimating $\hat{h}$, and finding $\widehat{ERM}_{1}$, $\widehat{ERM}_{2}$ and $BIC$. We examined 2 datasets that are commonly regarded as `difficult' or complex. We started with the {\sf{Abalone}} dataset and found $\hat{h}= 9$ no matter the conjectured linear model. However, there were only 7 variables since we only used first order models. We observed that SCAD and ALASSO picked the same variables, but the estimated coefficient for these variables were extremely different from the two models, this may be due to the high correlation between covariates. These models included $Shucked~weight$ but not $Length$, whereas the models chosen by $BIC$, $\widehat{ERM}_{1}$ and $\widehat{ERM}_{2}$ included $Length$ but not $Shucked ~weight$. That is, the sparsity methods led to the same variables, and $\widehat{ERM}_{1}$, $\widehat{ERM}_{2}$ and $BIC$ used the same variables (albeit a different set) while $\hat{h}$ included all the variables, suggesting that some variables are missing i.e there is unavoidable bias. We regard the model chosen by $\hat{h}$ as the most reasonable because of the non-linear relationship between the response and the explanatory variables.
\newline
The second dataset that we used was the wheat dataset. We performed the estimation of $\hat{h}$ using our method by location first and second we combined all locations to perform a multilocation analysis. Overall, we see that complexity varies from one location to another; this can be observed
by looking at the estimated $\hat{h}$, and the ordering of the variables in the models. The least complex location is Lincoln 2000 since the $\hat{h} = 1$
for all model lists.  We also observe that $TKWT \cdot KPSM$
was  the  most  important  variable  no  matter  which  location  you  choose  and  also had the highest correlation with yield.

Thus, we provide a complete analysis for all 7 locations using the phenotypic variables. Then we turned to two other analyses. The first used the phenotypic variables and the design structure variables, the second used the phenotypic variables and the SNP variables. That is we have shown how our method extends to models permitting two data types. In both cases, we found that the results were more complex that with uni-type datasets and that some SNP's and some design variable were important to include. We could have done multi-type data analysis including phenotypic variables, design variables and SNP variables, but lacked of time to complete this. We note that the {\sf{Wheat}} dataset has other variables in it e.g., variety, that we did not use. However,  using them might have made the effective sample size too small relative to the model classes we wanted to use.

Over all, we found that our estimated VCD $\hat{h}$ had higher values for real datasets than our comparably sized synthetic datasets indicating $\hat{h}$ was in fact reflecting the complexity of the DG. Moreover, in all cases, $\hat{h}$ gave results that were reasonable even when other methods gave models that were obviously too parsimonious or otherwise implausible. We attribute this to the sensitivity of $\hat{h}$ to bias. By contrast, often other methods have some level of built-in-sparsity so they give smaller models than are reasonable.

As a pragmatic recommendation, we therefore suggest that for complex datasets having a large enough $n$ that the conjecture model to be nested and the model selected by $\hat{h}$ be compared with the model selected by other methods. That is, we think that no one method for model selection is \textit{the} method to use. We think that judicious comparison of models selected unless different principles (sparsity, small bias, complexity, etc) is the appropriate way to do model selection. In this context, our method is simply another well performing method whose output should be taken into consideration for optimal model building.

\section{Future Work}
\label{sectfuturework}
No Dissertation can solve all the problems related to a new methodology and this Dissertation is no exception.. There are (at least) five problems that we would like to address in future developing the VCD for model selection.
\begin{enumerate}
    \item As it stands, $\hat{h}$ come from optimizing an upper bound. It is not clear how to re-do the analysis leading to the upper bound so as to derive a variance for $\hat{h}$. At the present, we can only suggest using a bootstrap approach: form $\hat{h}_{1}, \hat{h}_{2}, \cdots, \hat{h}_{B}$ from $B$ bootstrap samples and then form $\hat{\sigma}_{\hat{h}}$ from them.
    \item We have limited our work to linear models although our general theory applies to any class of regression models for which VCD's can be calculated. This includes Trees, Kernel methods, and neural networks, amongst others. These classes are more complicated but being non-linear may give better results in real problems.
    \item A theoretical gap in our derivation in Chapter \ref{chap:bounds} is that we only have an upper bound on the EMDBTEL. We have tried to make it tight computationally, but having either a provably tight upper bound on the EMDBTEL or a provably tight lower bound on the EMDBTEL would make our inferences more convincing.
    \item Another theoretical gap is to prove $\hat{h} \rightarrow h_{true}$ in probability. McDonald et al \cite{McDonald:etal:2011} tried to do this in the classification case but there are some gaps in their proof: For instance, they have a bound for the sum of absolute errors but they want a bound for the absolute error of the sum. We run into the exact same problem when we adapt their proof to regression, see Appendix \ref{ProofOfConsistency}. However we think we have overcome it.
    \item Although, we have shown in an example (see \ref{Design}, \ref{AnalSNP}) how our method can be used for two data types with linear models, the general question of how well our method performs in comparison to other methods for multi-type data remains to be explored. The issue is that as more data types become available the modeling becomes more complex and may outstrip $\hat{h}$'s ability to scale up with the number of data types and overall sample sizes.

    In Sec. \ref{Phenotype}, and \ref{MultiLoc} we compared a location-by-location analysis with pooled data analysis. This is not the same as a multi-type data, but is a facet of the multi-type data problem and shows some of the complexities that may occur.
\end{enumerate}
Obviously, this list is not complete, but it may serve as a warning that however well our method seems to perform in preliminary heats, reasonable questions about it remain to be answered before it is ready for general use.



\printnomenclature
\appendix

\chapter{CHAPTER 2 APPENDIX }
\label{AppenChap2}
\thispagestyle{fancy}

\section{Proof of Theorem \ref{Theo23and4} clause 1}

\begin{proof}
The LHS of the statement of the theorem equals
\begin{eqnarray}
  \int_{0}^{\infty}P\left(A_{\epsilon,m}\right)\emph{d}\epsilon
   \leq \int_{0}^{u}\emph{d}\epsilon + \int_{u}^{\infty}P\left(A_{\epsilon}\right)\emph{d}\epsilon
   \leq u +  2m\left(\frac{2ne}{h}\right)^{h}\int_{u}^{\infty}\exp\left(-\frac{n\epsilon^{2}}{m^{2}}\right)\emph{d}\epsilon .
\label{useid2341}
\end{eqnarray}
Observing that
$$
\epsilon>u \Rightarrow \frac{n\epsilon^{2}}{m^{2}}>\frac{nu\epsilon}{m^{2}}\Rightarrow \int_{u}^{\infty}\exp\left(-\frac{n\epsilon^{2}}{m^{2}}\right)\emph{d}\epsilon$$
$$\leq  \int_{u}^{\infty}\exp\left(-\frac{nu\epsilon}{m^{2}}\right)\emph{d}\epsilon=\frac{m^{2}}{nu}\exp\left(-\frac{nu^{2}}{m^{2}}\right),
$$
we have
\begin{equation}
    \label{expected2341}
    E\left(\sup_{\alpha_{1}, \alpha_{2}\in\Lambda}\left|\nu_{1}\left(Z_{2},\alpha_{1}\right)-\nu_{2}\left(Z_{1}, \alpha_{2}\right)\right|\right) \leq u + 2m^{3}\left(\frac{2ne}{h}\right)^h\frac{1}{nu}\exp\left(-\frac{nu^2}{m^2}\right).
\end{equation}
Let
$$
f(u) = u + 2m^{3}\left(\frac{2ne}{h}\right)^h\frac{1}{nu}\exp\left(-\frac{nu^2}{m^2}\right),
$$

$$
f^{'}(u) = 1 - 2m^{3}\left(\frac{2ne}{h}\right)^{h}\left[\frac{1}{nu^2} + \frac{2}{m^2}\right]\exp\left(-\frac{nu^2}{m^2}\right)
$$
$$
f^{'}(u) = 0 \Rightarrow \left[\frac{1}{nu^2} + \frac{2}{m^2}\right]\exp\left(-\frac{nu^2}{m^2}\right) = \frac{1}{2m^3}\left(\frac{h}{2ne}\right)^h
$$
\begin{eqnarray}
  \exp\left(-\frac{nu^2}{m^2}\right) &=& \frac{1}{2m^3}\left(\frac{h}{2ne}\right)^h\\
  -\frac{nu^2}{m^2} &=& \ln\left(\frac{1}{2m^3}\left(\frac{h}{2ne}\right)^h\right)\\
  u^2 &=& \frac{m^2}{n}\ln\left(2m^3\left(\frac{2ne}{h}\right)^h\right)\\
  u &=& m\sqrt{\frac{1}{n}\ln\left(2m^3\left(\frac{2ne}{h}\right)^h\right)}
  \label{234u1}
\end{eqnarray}
Substituting \ref{234u1} in \ref{expected2341} gives the statement of the proof.
\end{proof}

\section{Proof of Theorem \ref{Theo23and4} clause 2}

We will first state and proof the following Lemma

\begin{lemma}
\label{lema2342}
Let $\epsilon \geq 0$. The probability of the supremal difference between two empirical losses using cross-validation form of the error is bounded above by
\begin{equation}
    \begin{split}
        & P\left(\sup_{\alpha_{1}, \alpha_{2}\in \Lambda}\left|\nu_{1}\left(Z_{2},\alpha_{1}\right)- \nu_{2}\left(Z_{1},\alpha_{2}\right)\right|\geq\epsilon\right) \\
        & \leq  2P\left(\sup_{\alpha_{1}\in \Lambda}\left|\nu_{1}\left(Z_{1},\alpha_{1}\right)-E\left(Z_{2}, \alpha_{1}\right)\right|\geq\frac{\epsilon}{2}\right)
    \end{split}
\end{equation}
\end{lemma}

\begin{proof}
The LHS is bounded by
  \begin{eqnarray*}
&&   P\left(\sup_{\alpha_{1}, \alpha_{2}\in\Lambda}\left|\nu_{1}\left(Z^{2},\alpha_{1}\right)-E\left(Q(Z^{2},\alpha_{1})\right)+E\left(Q(Z^{2},\alpha_{1})-\nu_{2}\left(Z^{1},\alpha_{2}\right)\right)\right|\geq \epsilon\right)
\nonumber\\
   & \leq& P\left(\sup_{\alpha_{1}\in\Lambda}\left|\nu_{1}\left(Z^{2},\alpha_{1}\right)-E\left(Q(Z^{2},\alpha_{1}\right)\right| \geq\frac{\epsilon}{2}\right) +\\ && P\left(\sup_{\alpha_{2}\in\Lambda}\left|E\left(Q(Z^{1},\alpha_{2})\right)-\nu_{2}\left(Z^{1},\alpha_{2}\right)\right|\geq\frac{\epsilon}{2}\right)
\nonumber\\
   &=&2P\left(\sup_{\alpha_{1}\in\Lambda}\left|\nu_{1}\left(Z^{2},\alpha\right)-E\left(Q(Z^{2},\alpha_{1})\right)\right|\geq\frac{\epsilon}{2}\right).
\end{eqnarray*}
if
\begin{equation*}
    \begin{split}
        & P\left(\sup_{\alpha_{1}\in\Lambda}\left|\nu_{1}\left(Z^{2},\alpha\right)-E\left(Q(Z^{2},\alpha_{1})\right)\right|\geq\frac{\epsilon}{2}\right)  \geq \\
         & P\left(\sup_{\alpha_{2}\in\Lambda}\left|E\left(Q(Z^{1},\alpha_{2})\right)-\nu_{2}\left(Z^{1},\alpha_{2}\right)\right|\geq\frac{\epsilon}{2}\right)
    \end{split}
\end{equation*}
\end{proof}

\begin{proposition}
  \label{Prop42}
Let $h = VCDim\left\{Q\left(\cdot,\alpha\right):\alpha\in\Lambda\right\}$, where $Q\left(\cdot,\alpha\right)$ is unbounded, $n\in \cal{N}$ be the sample size. If $h< \infty$, and
$$
D_{p}\left(\alpha_{1}\right) = \int_{0}^{\infty}\sqrt[p]{P\left\{Q(Z_{2},\alpha_{1})\ge c\right\}} dc \leq \infty
$$
where $1 < p \leq 2$ is some fixed parameter,
we have
\begin{align}
  P\left(\sup_{\alpha_{1}, \alpha_{2}\in\Lambda}\left| \nu_{1}\left(Z^{2},\alpha_{1}\right)-\nu_{2}\left(Z^{1},\alpha_{1}\right)\right|\geq\epsilon\right) & \leq & 16\left(\frac{ne}{h}\right)^{h}\exp\left\{-\left(\frac{\epsilon n^{1-\frac{1}{p}}}{D_{p}(\alpha_{1}^{*})2^{2.5+\frac{1}{p}}}\right)^{2}\right\},
\label{probbd2}
\end{align}
where $\alpha_{1}^{*}=\sup_{\alpha_{1}\in\Lambda}D_{p}(\alpha_{1})$.
\end{proposition}

\begin{proof}
By Lemma \ref{lema2342}, we have
\begin{equation}
    \label{LHSLemma11}
    \begin{split}
       &  P\left(\sup_{\alpha_{1}, \alpha_{2}\in\Lambda}\left| \nu_{1}\left(Z^{2},\alpha_{1}\right)-\nu_{2}\left(Z^{1},\alpha_{2}\right)\right|\geq\epsilon\right) \leq \\
        & 2P\left(\sup_{\alpha_{1}\in\Lambda}\left|\nu_{1}\left(z^{2},\alpha_{1}\right)-E\left(Q\left(z^{2},\alpha_{1}\right)\right)\right|\geq\frac{\epsilon}{2}\right)
    \end{split}
\end{equation}

To bound the RHS of \ref{LHSLemma11}, use Lemma \ref{lemma2}  to observe that

\begin{align}
\label{posneg2}
  P\left(\sup_{\alpha\in\Lambda} \left| \nu_{n}\left(z,\alpha\right)-E\left(Q(z,\alpha)\right)\right|
\geq\frac{\epsilon}{2}\right)
& \leq
P\left(\sup_{\alpha\in\Lambda}\left|E\left(Q(z,\alpha)^{+}\right)-\frac{1}{n}\sum_{i=1}^{n}Q(z_{i},\alpha)^{+}\right| \geq\frac{\epsilon}{4}\right)
\nonumber\\
& \quad + P\left(\sup_{\alpha\in\Lambda}\left| E\left(Q(z,\alpha)^{-}\right)-\frac{1}{n}\sum_{i=1}^{n}Q(z,\alpha)^{-}\right| \geq\frac{\epsilon}{4}\right).
\end{align}

Each probability on the right hand side of \eqref{posneg2} can be bounded.  Since
$\alpha_{1}^*$ is a maximum, we have $1/D_{p}(\alpha_{1}^{*}) \leq 1/D_{p}(\alpha_{1})$.  Thus, for either the positive or negative parts in \eqref{posneg2} we have
\begin{equation*}
    \begin{split}
     & \sup_{\alpha_{1}\in\Lambda}\frac{E\left(Q(Z_{2},\alpha_{1})\right)-\nu_{1}(Z_{2},\alpha_{1})}{D_{p}(\alpha_{1}^{*})}\\
        & \leq \sup_{\alpha_{1}\in\Lambda}\frac{E\left(Q(Z_{2},\alpha_{1})\right)-\nu_{1}(Z_{2},\alpha_{1})}{D_{p}(\alpha_{1})} \\
       & P\left(\sup_{\alpha_{1}\in\Lambda}\frac{E\left(Q(Z_{2},\alpha_{1})\right)-\nu_{1}(Z_{2},\alpha_{1})}{D_{p}(\alpha_{1}^{*})}\geq\frac{\epsilon}{4D_{p}(\alpha_{1}^{*})}\right) \\
       & \leq P\left(\sup_{\alpha_{1}\in\Lambda}\frac{E\left(Q(Z_{2},\alpha_{1})\right)-\nu_{1}(Z_{2},\alpha_{1})}{D_{p}(\alpha_{1})}\geq\frac{\epsilon}{4D_{p}(\alpha_{1}^{*\emph{}})}\right) \Rightarrow\\
       & P\left(\sup_{\alpha_{1}\in\Lambda}\left(E\left(Q(Z_{2},\alpha_{1})\right)-\nu_{1}(Z_{2},\alpha_{1})\right)\geq\frac{\epsilon}{4}\right) \\
       & \leq P\left(\sup_{\alpha_{1}\in\Lambda}\frac{E\left(Q(Z_{2},\alpha_{1})\right)-\nu_{1}(Z_{2},\alpha_{1})}{D_{p}(\alpha_{1})}\geq\frac{\epsilon}{4D_{p}(\alpha_{1}^{*})}\right).
    \end{split}
\end{equation*}
Letting $\delta = \frac{\epsilon }{4D_{p}(\alpha^{*})}$ and using Theorem \ref{Theo1}, Clause 3, the last inequality gives
\begin{align}
\label{useclauseIII2}
  P\left(\sup_{\alpha_{1}\in\Lambda} \left(E\left(Q(Z_{2},\alpha_{1})\right)-\nu_{1}(Z_{2},\alpha_{1})\right)\geq\frac{\epsilon}{4} \right) & \leq 4\exp\left\{\left(\frac{H_{ann}^{\Lambda,\mathbb{\beta}}(n)}{n^{2-\frac{2}{p}}}-\frac{\delta^{2}}{2^{1+\frac{2}{p}}}\right)n^{2-\frac{2}{p}}\right\}.
\end{align}
Using \eqref{useclauseIII2}, both terms on the right in \eqref{posneg2} can be bounded.  This gives
\begin{align}
\label{8factor2}
  P\left(\sup_{\alpha_{1}\in\Lambda}\left| \nu_{1}(Z_{2},\alpha_{1})-E\left(Q(Z_{2},\alpha_{1})\right)\right| \geq\frac{\epsilon}{2}\right) & \leq 8\exp\left\{\left(\frac{H_{ann}^{\Lambda,\mathbb{\beta}}(n)}{n^{2-\frac{2}{p}}}-\frac{\delta^{2}}{2^{1+\frac{2}{p}}}\right)n^{2-\frac{2}{p}}\right\}.
\end{align}
\noindent
The presence of $\mathcal{\beta}$ in the the exponent of the hannealed entropy will not change the validity of Clause I of  Theorem \ref{Theo1}, so we have,
$
H_{ann}^{\Lambda,\mathbb{\beta}}(n) \leq G^{\Lambda}(n)\leq \ln\left(\frac{en}{h}\right)^{h}
$
therefore,
$
\exp\left(H_{ann}^{\Lambda, \mathbb{\delta}}(n)\right)\leq \exp\left(G(n)\right)\leq \left(\frac{en}{h}\right)^{h}.
$
Using this in \eqref{8factor2} gives
\begin{align*}
  P\left(\sup_{\alpha_{1}\in\Lambda}\left( \nu_{1}(Z_{2},\alpha_{1})-E\left(Q(Z_{2},\alpha_{1})\right)\right) \geq\frac{\epsilon}{2}\right) & \leq 8\left(\frac{ne}{h}\right)^{h}\exp\left\{-\left(\frac{\delta n^{1-\frac{1}{p}}}{2^{0.5+\frac{1}{p}}}\right)^{2}\right\}
\nonumber\\
  & = 8\left(\frac{ne}{h}\right)^{h}\exp\left\{-\left(\frac{\epsilon n^{1-\frac{1}{p}}}{D_{p}(\alpha_{1}^{*})2^{2.5+\frac{1}{p}}}\right)^{2}\right\}.
\end{align*}
Recalling the extra factor of 2 in  inequality \eqref{LHSLemma11} gives the statement of the Proposition.
\end{proof}

Now let prove Theorem \ref{Theo23and4} clause 2.
\begin{proof}
Using the integral of probabilities identity and Proposition \ref{Prop42}, the left-hand side of \eqref{Hannbd} equal
\begin{eqnarray*}
&&  \int_{0}^{\infty}P\left(\sup_{\alpha_{1}, \alpha_{2}\in\Lambda}\left|\nu_{1}(Z_{2},\alpha_{1})-\nu_{2}(Z_{1},\alpha_{2})\right|\geq\epsilon\right) \emph{d}\epsilon\nonumber\\
   &\leq& \int_{0}^{\infty}16\left(\frac{ne}{h}\right)^{h}\exp\left\{-\left(\frac{\epsilon n^{1-\frac{1}{p}}}{D_{p}(\alpha_{1}^{*})2^{2.5+\frac{1}{p}}}\right)^{2}\right\}\emph{d}\epsilon \nonumber\\
   &\equiv& \int_{0}^{u}\emph{d}\epsilon +16\left(\frac{ne}{h}\right)^{h}\int_{u}^{\infty}\exp\left\{-\left(\frac{\epsilon n^{1-\frac{1}{p}}}{D_{p}(\alpha_{1}^{*})2^{2.5+\frac{1}{p}}}\right)^{2}\epsilon^{2}\right\}\emph{d}\epsilon \nonumber\\
   &\leq& u + 16\left(\frac{ne}{h}\right)^{h}\int_{u}^{\infty}\exp\left\{-\left(\frac{n^{1-\frac{1}{p}}}{D_{p}(\alpha_{1}^{*})2^{2.5+\frac{1}{p}}}\right)^{2}u\epsilon\right\}\emph{d}\epsilon \nonumber\\
   &=&  u  + \left(\frac{D_{p}(\alpha_{1}^{*})2^{2.5+\frac{1}{p}}}{n^{1-\frac{1}{p}}}\right)^{2}\frac{16\left(\frac{ne}{h}\right)^{h}}{u}
\exp\left\{-\left(\frac{un^{1-\frac{1}{p}}}{D_{p}(\alpha_{1}^{*})2^{2.5+\frac{1}{p}}}\right)^{2}\right\}.\\
\end{eqnarray*}
Choosing
$$
u = D_{p}(\alpha_{1}^{*})2^{2.5+\frac{1}{p}}\frac{\sqrt{h\ln\left(\frac{ne}{h}\right)}}{n^{1-\frac{1}{p}}}
$$
gives the statement of the Theorem.
\end{proof}

\section{Proof of Theorem \ref{Theo23and4} clause 3}
\label{Clause3Appen}

\begin{proof}
 Theorem \ref{Theo23and4} clause 1 gives
\begin{eqnarray}\label{E162}
  E\left(\sup_{\alpha_{1}, \alpha_{2}\in\Lambda}\left|\nu_{1}^{*}\left(Z_{2},\alpha_{1},m\right)-\nu_{2}^{*}\left(Z_{1},\alpha_{2},m\right)\right|\right) &\leq& \sqrt{\frac{\ln(2m)}{n} + \frac{h}{n}\ln\left(\frac{2ne}{h}\right)} \nonumber\\
   &\overset{\infty}{=}& \sqrt{\frac{h}{n}\ln\left(\frac{2ne}{h}\right)},
\end{eqnarray}
where $\overset{\infty}{=}$ indicates a limit as $n \rightarrow \infty$ has been taken.
Similarly, Theorem \ref{Theo23and4} clause 2 gives
\begin{equation}\label{E172}
  E\left(\sup_{\alpha_{1}, \alpha_{2}\in\Lambda}\left|\nu_{1}\left(Z_{2},\alpha_{1}\right)-\nu_{2}\left(Z_{1},\alpha_{2}\right)\right)|\right)\leq 8D_{p}\left(\alpha_{1}^{*}\right)\sqrt{\frac{h}{n}+\frac{h}{n}\ln\left(\frac{2ne}{h}\right)}+E(n,h)
\end{equation}
where $E(n,h)$ is of smaller order than the first term on the right in \ref{E17}. Thus,
\begin{equation}\label{OVERALLBOUND2}
 E\left(\sup_{\alpha_{1},\alpha_{2}\in\Lambda}\left|\nu_{1}\left(Z_{2},\alpha_{1}\right)-\nu_{2}\left(Z_{1},\alpha_{2}\right)\right)|\right)\leq 8D_{p}\left(\alpha_{1}^{*}\right)\sqrt{\frac{h}{n}\ln\left(\frac{2ne}{h}\right)}.
\end{equation}
Taking the minimum over the RHS of equations \eqref{E162} and \eqref{OVERALLBOUND2} gives the Theorem.
\end{proof}

\section{Proof of Consistency}
\label{ProofOfConsistency}
Here, we offer a proof of consistency of $\hat{h}$ to $h$. In almost  all respects this tentative of proof should be credited to \cite{McDonald:etal:2011}. Our contribution is the examination of the gaps in their proof. Our two ugly hypotheses and the discussion about how to satisfy them fill the gaps in their proof.  We have put it in as appendix rather than the main part of the thesis because we have not finished verifying all the details satisfactorily. Instead, we have not convince ourselves it is error free. Nevertheless, we have the following. $\epsilon_{i}$'s represent the $\xi_{i}$'s in Algorithm \ref{Algo1}

\begin{lemma}
\label{Lemma1consis}
Let $t\geq 0$ and let $\epsilon\left(n_{l}\right) = \frac{1}{W}\sum_{i=1}^{W}\epsilon_{i}\left(n_{l}\right)$, where $\epsilon_{i}\left(n_{l}\right)$'s are independent for any given $n_{l}$'s. Then at any design point $n_{l}$, we have
\begin{equation}
 \label{chernoffLemma}
 E\left(\exp\left(t\epsilon_{i}\left(n_{l}\right)\right)\right) \leq \exp\left(\frac{t^{2}B^{2}}{8Wm^{2}}\right)
\end{equation}
\end{lemma}

\begin{proof}
 Fix $j = 0, 1, 2, \cdots, m-1$. Now, for any $t\geq 0$ and $i \geq 0$, and $\epsilon_{i}\left(n_{l}\right) \in \left(\frac{jB}{m}, \frac{(j+1)B}{m}\right]$, using Hoeffding's lemma we have
 \begin{equation*}
     E\left(\exp\left(t\epsilon_{i}\left(n_{l}\right)\right)\right)\leq \exp\left(\frac{t^{2}B^{2}}{8m^{2}}\right).
 \end{equation*}
 Thus, we have
 \begin{eqnarray} \label{consisLemma1}
    E\left(\exp\left(t\epsilon_{i}\left(n_{l}\right)\right)\right) &=& E\left(\exp\left(\frac{t}{W}\sum_{i=1}^{W}\epsilon_{i}\left(n_{l}\right)\right)\right) \nonumber \\
    &=& E\left(\Pi_{i=1}^{W}\exp\left(\frac{t}{W}\epsilon_{i}\left(n_{l}\right)\right)\right) \nonumber \\
    &=& \Pi_{i=1}^{W}E\left(\exp\left(\frac{t}{W}\epsilon_{i}\left(n_{l}\right)\right)\right) \nonumber \\
    &\leq & \Pi_{i=1}^{W}\exp\left(\frac{t^{2}b^{2}}{8Wm^{2}}\right) \nonumber \\
    &=& \exp\left(\frac{t^{2}B^{2}}{8Wm^{2}}\right)
 \end{eqnarray}
\end{proof}

\begin{proposition}
\label{Prop1Consistency}
  Suppose that $\left\{\epsilon_{l} \equiv \epsilon\left(n_{l}\right), l = 1,2 \cdots k\right\}$ is a set of random variables satisfying Lemma \ref{Lemma1consis}. Then for any $\gamma \in \mathbb{R}^{k}$ and $\rho \geq 0$, we have
  \begin{equation}
   \label{Prop1ConEq1}
   P\left(\mid \sum_{l=1}^{k}\epsilon_{l}\gamma_{l} \mid\geq \rho \right) \leq 2\exp\left(-\frac{2Wm^{2}\rho^{2}}{B^{2}\sum_{l=1}^{k}\gamma_{l}^{2}}\right)
  \end{equation}
\end{proposition}

\begin{proof}

 \begin{eqnarray}
    P\left(\mid \sum_{l=1}^{k}\epsilon_{l}\gamma_{l} \mid\geq \rho \right) &\leq& 2  P\left(\sum_{l=1}^{k}\epsilon_{l}\gamma_{l} \geq \rho \right)  \nonumber \\
    &=& 2P\left(\exp\left(t\sum_{l=1}^{k}\epsilon_{l}\rho_{l}\right) \geq \exp(t\rho)\right)  \nonumber \\
    &\leq& \frac{2E\left(\exp\left(t\sum_{l=1}^{k}\epsilon_{l}\rho_{l}\right)\right)}{\exp(t\rho)}  \nonumber \\
    &=& \frac{2E\left(\Pi_{l=1}^{k}\exp\left(t\epsilon_{l}\gamma_{l}\right)\right)}{\exp(t\rho)} \nonumber \\
    &=& \frac{2\Pi_{l=1}^{k}E\left(\exp\left(t\epsilon_{l}\gamma_{l}\right)\right)}{\exp(t\rho)}  \nonumber \\
    &\leq& \frac{2\Pi_{l=1}^{k}\exp\left(\frac{t^{2}B^{2}\gamma_{l}^{2}}{8Wm^{2}}\right)}{\exp(t\rho)} \nonumber \\
    &=& 2\exp\left(\frac{t^{2}B^{2}}{8Wm^{2}}\sum_{l=1}^{k}\gamma_{l}^{2} - t\rho\right)
    \label{eqProp1conc}
 \end{eqnarray}
 Since \ref{eqProp1conc} is true for all $t$, the RHS of \ref{eqProp1conc} attains its minimum for
 \begin{equation}
  \label{consiT}
  t = \frac{8Wm^{2}\rho}{2B^{2}\sum_{l=1}^{k}\gamma_{l}^{2}}
 \end{equation}
 Replacing $t$ by its value in \ref{eqProp1conc} gives the statement of the proposition.
\end{proof}

\begin{definition}
\label{DefEntropy}
$\Phi = \left\{\phi_{h} : h\in \left(1, M\right]\right\}$. For $s = 0, 1, 2, \cdots,$ let $\{\phi_{j}^{s}\}_{j=1}^{N_{s}}$ be the minimal $2^{-s}\tau$-covering set of $\left(\Phi, \parallel \cdot \parallel_{Q}\right).$ So $N_{s} = N\left(2^{-s}\tau, \Phi\right)$ and for each $h$, there exists a $\phi_{h}^{s} \in \left\{\phi_{1}^{s}, \phi_{2}^{s}, \cdots, \phi_{N_{s}}^{s}\right\}$ such that $\parallel\phi_{h} - \phi_{h}^{s} \parallel_{Q} ~ \leq ~2^{-s}\tau$. We choose $\phi_{h}^{0}\equiv 0$, since $\parallel\phi_{h}\parallel_{Q}\leq \tau$.
The key argument is based on the entropy of the restricted class $\Phi(\tau) = \left\{\phi_{h}\in \Phi : \parallel \phi_{h}-\phi_{h_{T}}\parallel_{Q} \leq \tau\right\}$. $\phi_{h}\left(n\right)$ is defined as follows:
\begin{equation}
 \label{phibound}
 \phi_{h}\left(n\right) = \hat{c}\sqrt{\frac{h}{n}\log\left(\frac{2ne}{h}\right)}
\end{equation}
\end{definition}

We assume without any proof that the following theorem is true.
\begin{thm}
\label{Theo23456}
Let assume that we have a collection of functions inside $Im(\Phi)$ so that
$\parallel\phi_{h} - \phi_{h}^{s} \parallel_{2} ~ \leq ~2^{-s}\tau$ $\Longrightarrow $ $\parallel\phi_{h} - \phi_{h}^{s} \parallel_{1} ~ \leq ~2^{-s}c^{'}f(\tau)$, where $f(\tau)$ is an increasing function of $\tau$.
\end{thm}

\begin{corollary}
\label{CoroConsis}
If $\sup_{\phi_{h}\in \Phi(\tau)}
\parallel \phi_{h} \parallel_{Q} ~\leq ~ \tau$ and the conclusion of Lemma \ref{Lemma1consis} holds for all design points $n_{l}$, then for
\begin{equation}
 \label{deltaeq}
 \delta \geq \frac{2B\tau}{m\sqrt{kW}}\left(288\log\left(2\right)\right)^{0.5},
\end{equation}
we have

\begin{equation}
 \label{CorProp1}
 P\left(\sup_{\phi_{h}\in \Phi}\left|\frac{1}{k}\sum_{l=1}^{k}\epsilon_{l}\phi_{h}\left(n_{l}\right)\right|\geq \delta\right) \leq  4\exp\left(-\frac{kWm^{2}\delta^{2}}{1152B^{2}\tau^{2}}\right).
\end{equation}
\end{corollary}

\begin{proof}

\begin{equation}
\label{eqSplt}
\begin{split}
\left|\frac{1}{k}\sum_{l=1}^{k}\epsilon_{l}\left(\phi_{h}^{s}-\phi_{s}^{s-1}\right)\right| & = \frac{1}{k}\sum_{l=1}^{k}sign(\epsilon_{l})|\epsilon_{l}|sign\left(\phi_{h}^{s}-\phi_{h}^{s-1}\right)\left|\left(\phi_{h}^{s}-\phi_{h}^{s-1}\right)\right|
\end{split}
\end{equation}

\begin{equation}
\label{eq1norm}
 \left|\frac{1}{k}\sum_{l=1}^{k}\epsilon_{l}\left(\phi_{h}^{s}-\phi_{s}^{s-1}\right)\right| \leq \frac{1}{k}\sum_{l=1}^{k}\epsilon_{l}sign\left(\phi_{h}^{s}-\phi_{h}^{s-1}\right)\left(\phi_{h}^{s}-\phi_{h}^{s-1}\right)
\end{equation}
So since the $\phi_{h}$ and $\phi_{h}^{s}$ are continuous, and if $N_{s}$ is large enough, using Def. \ref{DefEntropy} and Theorem \ref{Theo23456} we have
\begin{equation}
 \label{eq2norm}
 \left|\phi_{h}^{s}-\phi_{h}^{s-1}\right| \leq 3\tau 2^{-s}.
\end{equation}
Using \ref{eq2norm}, we have
$$
-3\tau2^{-s} \leq \phi_{h}^{s} - \phi_{h}^{s-1} \leq 3\tau 2^{-s} \Longrightarrow
$$

$$
-3\tau2^{-s}\epsilon_{l} \leq \left( \phi_{h}^{s} - \phi_{h}^{s-1}\right)\epsilon_{l} \leq 3\tau 2^{-s} \epsilon_{l} \Longrightarrow
$$

\begin{eqnarray}
  \label{bbq1}
  \left|\frac{1}{k}\sum_{l=1}^{k}\epsilon_{l}\left(\phi_{h}^{s}-\phi_{h}^{s-1}\right)\right| &\leq& 
  \frac{3\cdot 2^{-s}\tau}{k}\left|\sum_{l=1}^{k}\epsilon_{l}\right|
\end{eqnarray}
Let $\eta_{s}$ be a positive number satisfying $\sum_{s=1}^{S}\eta_{s}\leq 1$, and using \ref{bbq1}, we have

\begin{equation}
    \begin{split}
    \label{bbq2}
        & P\left(\sup_{h\in [1,M]}\left|\frac{1}{k}\sum_{s=1}^{S}\sum_{l=1}^{k}\epsilon_{l}\left(\phi_{h}\left(n_{l}\right) - \phi_{h}\left(n_{l}\right) \right)\right| \geq \delta\right) \nonumber \\
        & \leq \sum_{s=1}^{S}P\left(\sup_{h\in [1,M]}\frac{3\cdot 2^{-s}\tau}{k}\left|\sum_{l=1}^{k}\epsilon_{l}\right|\geq \frac{\delta \eta_{s}}{2}\right) \nonumber \\
        & = \sum_{s=1}^{S}P\left(\left|\sum_{l=1}^{k}\epsilon_{l}\right|\geq\frac{k\delta \eta_{s}}{3\cdot 2^{-s+1}\tau}\right) \nonumber \\
        & \leq 2\sum_{s=1}^{S}\exp\left(-\frac{2Wm^{2}k^{2}\delta^{2}\eta{s}^{2}}{9\cdot B^{2}k\cdot 2^{-2s+2}\tau^{2}}\right)\nonumber \\
        & = 2 \sum_{s=1}^{S}\exp\left(-\frac{2m^{2}Wk\delta^{2}\eta^{2}}{9B^{2}\cdot 2^{-2s+2}\tau^{2}}\right)
    \end{split}
\end{equation}

Next, $\eta_{s}\geq \frac{2^{-s}\sqrt{s}}{8}
\Longrightarrow $
\begin{equation}\label{bbq3}
-\frac{2m^{2}Wk\delta^{2}\eta_{s}^{2}}{9B^{2}2^{-2s+2}\tau^{2}}\leq -\frac{Wk\delta^{2}m^{2}s}{1152B^{2}\tau^{2}}
\end{equation}
Using \ref{bbq2} in the previous inequality, we have

\begin{equation}
\begin{split}
    & 2 \sum_{s=1}^{S}\exp\left(-\frac{2m^{2}Wk\delta^{2}\eta^{2}}{9B^{2}\cdot 2^{-2s+2}\tau^{2}}\right)  \\
    & \leq 2 \sum_{s=1}^{S}\exp\left(-\frac{Wk\delta^{2}m^{2}s}{1152B^{2}\tau^{2}}\right) \\
    & =  \frac{2\exp\left(-\frac{Wk\delta^{2}m^{2}}{1152B^{2}\tau^{2}}\right)}{1 - \exp\left(-\frac{Wk\delta^{2}m^{2}}{1152B^{2}\tau^{2}}\right)}
\end{split}
\label{bbq4}
\end{equation}
Using \ref{deltaeq}, we have
$\frac{m^{2}\delta^{2}kW}{1152B^{2}\tau^{2}} \geq \log(2) \Longrightarrow \exp\left(-\frac{m^{2}\delta^{2}kW}{1152B^{2}\tau^{2}}\right) \leq \frac{1}{2} \Longrightarrow$

\begin{equation}
 \label{CorProp11}
 P\left(\sup_{\phi_{h}\in \Phi}\left|\frac{1}{k}\sum_{l=1}^{k}\epsilon_{l}\phi_{h}\left(n_{l}\right)\right|\geq \delta\right) \leq  4\exp\left(-\frac{kWm^{2}\delta^{2}}{1152B^{2}\tau^{2}}\right).
\end{equation}
\end{proof}

\begin{thm}
\label{TheoConsis}
Let
\begin{equation}
 \label{deltaeq2}
 \delta \geq \frac{2B}{m\sqrt{kW}} \left(288\log(2)\right)^{0.5}
\end{equation}
and suppose that $h_{T}\in \left(0,M\right]$. Then,
\begin{eqnarray}
  \label{concistheo}
  P\left(\parallel\phi_{\hat{h}} - \phi_{h_{T}}\parallel_{Q} \geq \delta\right) \leq \frac{28}{3} \exp\left(-\frac{kWm^{2}\delta^{2}}{4608B^{2}}\right).
\end{eqnarray}
\end{thm}

\begin{proof}
Note that $\parallel\phi_{\hat{h}} - \phi_{h_{T}}\parallel_{Q} \leq \frac{2}{k}\sum_{l=1}^{k}\epsilon_{l}\left(\phi_{\hat{h}}\left(n_{l}\right) - \phi_{h_{T}}\left(n_{l}\right)\right)$.
using this last inequality, we have
\begin{equation}
 \begin{split}
     & P\left(\parallel\phi_{\hat{h}} - \phi_{h_{T}}\parallel_{Q} \geq \delta\right) \\
     & \leq P\left(2 \left|\frac{1}{k}\sum_{l=1}^{k}\epsilon_{l}\left(\phi_{\hat{h}} - \phi_{h_{T}}\right)\right| \geq \delta \right)  \\
     &\leq P\left(\sup_{\phi_{h}\in \Phi\left(2^{s+1}\delta\right)}\left|\frac{1}{k}\sum_{s=1}^{S}\sum_{l=1}^{k}\epsilon_{l}\left(\phi_{\hat{h}}^{s}\left(n_{l}\right)-\phi_{h_{T}}^{s-1}\left(n_{l}\right)\right)\right|\geq 2^{2s-1}\delta^{2}\right)  \\
     & \leq \sum_{s=1}^{S}P\left(\sup_{\phi_{h}\in \Phi\left(2^{s+1}\delta\right)}\left|\frac{1}{k}\sum_{l=1}^{k}\epsilon_{l}\left(\phi_{\hat{h}}\left(n_{l}\right) - \phi_{h_{T}}\left(n_{l}\right)\right)\right| \geq   2^{2s-1}\delta^{2} \right)
 \end{split}
\end{equation}
using corollary \ref{CoroConsis}, we have
\begin{equation}
 \begin{split}
     &\sum_{s=1}^{S}P\left(\sup_{\phi_{h}\in \Phi\left(2^{s+1}\delta\right)}\left|\frac{1}{k}\sum_{l=1}^{k}\epsilon_{l}\left(\phi_{\hat{h}}\left(n_{l}\right) - \phi_{h_{T}}\left(n_{l}\right)\right)\right| \geq   2^{2s-1}\delta^{2} \right) \\
     & \leq \sum_{s=1}^{S} 4\exp\left(-\frac{kWm^{2}2^{4s-2}\delta^{4}}{1152B^{2}2^{2s+2}\delta^{2}}\right) \\
     & = 4 \sum_{s=1}^{S}\exp\left(-\frac{kWm^{2}\delta^{2}2^{2s-4}}{1152B^{2}}\right) \\
     & = 4\exp\left(-\frac{kWm^{2}\delta^{2}}{4\cdot 1152 B^{2}}\right) + 4\left(1 - \exp\left(-\frac{2kwm^{2}\delta^{2}}{1152B^{2}}\right)\right)^{-1}\exp\left(-\frac{2kWm^{2}\delta^{2}}{1152B^{2}}\right).
 \end{split}
 \label{eqqq2}
\end{equation}
Using \ref{deltaeq2}, we have

\begin{equation}
 \label{eqqq}
 \exp\left(-\frac{2kWm^{2}\delta^{2}}{1152}\right) \leq \frac{1}{4}.
\end{equation}
Substituting \ref{eqqq} in \ref{eqqq2} gives
\begin{equation}
 \begin{split}
     & 4\exp\left(-\frac{kWm^{2}\delta^{2}}{4\cdot 1152 B^{2}}\right) + 4\left(1 - \exp\left(-\frac{2kWm^{2}\delta^{2}}{1152B^{2}}\right)\right)^{-1}\exp\left(-\frac{2kWm^{2}\delta^{2}}{1152B^{2}}\right) \\
     & \leq 4\exp\left(-\frac{kWm^{2}\delta^{2}}{4\cdot 1152 B^{2}}\right) + \frac{16}{3}\exp\left(-\frac{2\cdot kWm^{2}\delta^{2}}{ 1152 B^{2}}\right).
 \end{split}
\end{equation}
Since the first exponential is the largest, this completes the proof.
\end{proof}

\addcontentsline{toc}{chapter}{Bibliography}
\bibliography{references}

\end{document}